%% file: thesis.tex
\documentclass[12pt]{rutgersthesis}

\usepackage{amssymb}        % provides AMS fonts and symbols ex. \mathbb{...}

\usepackage{amsthm}
\usepackage{graphicx}
\usepackage{overpic}
\usepackage[linesnumbered,algoruled,boxed,lined]{algorithm2e}
\usepackage{array}
\usepackage{multirow}
\usepackage{hyperref}

\newtheorem{problem}{Problem}
\newtheorem{proposition}{Proposition}[section]
\newtheorem{lemma}{Lemma}[section]

\newtheorem{theorem}{Theorem}[section]
\theoremstyle{definition}

\theoremstyle{remark}
\newtheorem*{remark}{Remark}

\newcommand{\D}{\mathcal{D}}

\def\bl#1{\textcolor{blue}{#1}}

\def\ldg{G_{\mathcal A_s,\mathcal A_g}^{\,\ell}}
\def\ldgg{G^{\,\ell}}
\def\udg{G_{\mathcal A_s,\mathcal A_g}^{\,u}}
\def\udgg{G^{\,u}}
\def\toro{\texttt{TORO} }

\def\lrbm{\texttt{LRBM}\xspace} 
\def\urbm{\texttt{URBM}\xspace} 
\def\mrb{\textsc{MRB}\xspace}
\def\rb{\textsc{RB}\xspace}
\def\fvs{\textsc{MFVS}\xspace}

\def\DFSDP{\textsc{DFDP}\xspace}
\def\PQS{\textsc{PQS}\xspace}
\def\spp{\textsc{SepPlan}\xspace}
\def\vsp{\texttt{VSP}\xspace}
\def\minvs{\textsc{MinVS}\xspace}
\def\vertsep{\textsc{VS}\xspace}
\def\ilpmrb{\textsc{TB}_{\mrb}\xspace}
\def\ilptb{\textsc{TB}_{\mrb}\xspace}
\def\ilpfvs{\textsc{TB}_{\textsc{FVS}}\xspace}

\def\trlb{\textsc{TRLB}\xspace}
\def\toroe{\texttt{TORE}\xspace}
\def\toroi{\texttt{TORI}\xspace}
\def\modap{\texttt{MODAP}\xspace} 
\def\orla{ORLA*\xspace}
\def\motar{MoTaR\xspace}
\def\model{StabilNet\xspace}
\def\astar{A*\xspace}
\def\hete{\texttt{HeTORI}\xspace}
\def\heteCP{\texttt{HeCP}\xspace}
\def\heteTI{\texttt{HeTI}\xspace}

\newcommand{\thesisType}{Dissertation} % Dissertation for PhD and Thesis for Master
\newcommand{\thesisTitle}{Tabletop Object Rearrangement: Structure, Complexity, and Efficient Combinatorial Search-Based Solutions} % your thesis/dissertation title, e.g., How to Write A Thesis
\newcommand{\yourName}{Kai Gao} % your name, e.g., Tim Cook
\newcommand{\yourDegree}{Doctor of Philosophy} % your degree to be obtained, e.g., Master of Science, Doctor of Philosophy
\newcommand{\yourProgram}{Computer Science} %insert your graduate program’s official name
\newcommand{\yourDirector}{Jingjin Yu} % your thesis director name here, his/her title not included
\newcommand{\yourCommitteeNumber}{4} % IMPORTANT! the number of the committee members for your defense
\newcommand{\yourMonth}{January} % your graduation month, not your defense month
\newcommand{\yourYear}{2025} % your graduation year

\bibliography{references}

\begin{document}
% \doublespacing  %set line spacing

\input{copyright}

\makeTitlePage{Month}{Year}

\begin{frontmatter}
    \input{abstract}
    \input{acknowledgments}
    \makeTOC
    \makeListOfTables
    \makeListOfFigures

\input{abbrevs}
\end{frontmatter}

\begin{thesisbody}
    \input{chapters/intro}

\input{chapters/RBM}

\input{chapters/TRLB}

    \input{chapters/dual_arm}
    \input{chapters/distance_optimal}

    \input{chapters/conclusion}

\input{chapters/appendix.tex}
    \input{publications.tex}
    \makeBibliography
\end{thesisbody}

\end{document}

%% file: copyright.tex
\clearpage
\begin{center}

\vspace*{\fill}

\copyright { }{\yourYear}\\
{\yourName}\\
ALL RIGHTS RESERVED\\

\vspace*{\fill}

\end{center}

\pagenumbering{gobble}
\clearpage

%% file: abstract.tex
\begin{my_abstract}

This thesis aims to provide a complete structural analysis and efficient algorithmic solutions to tabletop object rearrangement with overhand grasps (\toro). 
This problem captures a common task that we solve on a daily basis and is essential in enabling truly intelligent robotic manipulation.
When rearranging many objects in a confined workspace, on the one hand, action sequencing with the least pick-n-places in \toro is NP-hard\cite{han2018complexity}; on the other hand, temporarily relocating objects to some free space (``buffer poses'') may be necessary but highly challenging in a cluttered environment.

Focusing on these two challenges, the thesis covers \toro in four different setups, including varied workspace assumptions (with/without external buffers) and manipulator settings (single/dual-arms or a mobile manipulator).

The thesis first explores \toro with external buffers (\toroe), addressing the size of needed space for temporary object relocation (“running buffers”). 
This study shows that finding the maximum running buffers (\mrb) is NP-hard and that \mrb can grow unbounded with an increasing number of objects, even with uniform shapes. 
Exact algorithms developed for both labeled and unlabeled settings can scale to over 100 objects. 

The thesis further extends the \toroe algorithms to tabletop rearrangement with internal buffers (\toroi), where all temporary object placements need to be inside the workspace. A two-step baseline planner is developed, generating a primitive plan based on object-object collisions and selecting buffer locations. 
Using a “lazy” planner within a bi-directional tree search framework, the method efficiently produces robust, high-quality solutions in simulations, outperforming existing approaches for large-scale problems.

The thesis later explores lazy buffer verification for dual-arm task planning in non-monotone rearrangement tasks. 
Handling complex dependencies often requires moving objects multiple times and coordinating hand-offs between two arms in a shared workspace. Our task planning algorithms effectively sequence and distribute pick-and-place tasks, yielding significant time savings over greedy or single-robot approaches. For motion planning, we introduce a tightly integrated pipeline that combines novel sampling methods with advanced trajectory optimization. The proposed planner achieves superior execution times and trajectory compliance with acceleration and jerk constraints, advancing coordination in dual-arm systems.

Finally, the thesis incorporates lazy buffer verification into an A* framework (\orla) for time-optimal multi-object rearrangement in mobile robot tabletop setups. 
The proposed method applies delayed evaluation to optimize object pick-and-place sequences, factoring in both end-effector and robot base movements. 
With learning-based stability predictions, \orla can handle complex, multi-layered rearrangement tasks. Extensive simulations and ablation studies validate \orla’s effectiveness, providing high-quality solutions for challenging rearrangement scenarios.
\end{my_abstract}

%% file: acknowledgments.tex
\begin{acknowledgments}

First and foremost, I would like to express my sincere gratitude to my PhD advisor, Prof. Jingjin Yu. JJ's insightful guidance and dedication have been invaluable throughout my PhD journey. He has set an inspiring example of how to approach research and excel as a scientist in robotics, and I am deeply grateful for the trust and freedom he provided to pursue my ideas.

I am also thankful to my co-authors and collaborators, especially Kostas, Baichuan, and Rui, whose insights and guidance have been true milestones in my research journey. Working together has been a deeply rewarding experience, and I greatly appreciate their expertise, constructive feedback, and the depth they brought to our work.

Additionally, I would like to thank my labmates, Siwei, Zihe, Duo, Haonan, Tzvika, and Teng, for the discussions and exchange of ideas. The collaborative environment and support within the lab have enriched my research experience and have been an essential part of my PhD.

To my parents and my girlfriend, who have been a constant source of support and encouragement, thank you for being there through the highs and lows of this journey. Your belief in me has kept me moving forward.

Thank you all for your support. This accomplishment would not have been possible without each of you.

\end{acknowledgments}

%% file: abbrevs.tex
\makeListOfAcronyms

%% file: chapters/intro.tex
\chapter{Introduction and Background}
\thispagestyle{myheadings}

This thesis provides a brief overview of my works on tabletop rearrangement planning. 
Rearrangement planning is required in nearly all aspects of our daily lives. Be it work-related, at home, or for play, objects are to be grasped and rearranged, e.g., tidying up a messy desk, cleaning the table after dinner, or solving a jigsaw puzzle. 
Similarly, many industrial and logistics applications require repetitive rearrangements of many objects, e.g., the sorting and packaging of products on conveyors with robots, and doing so efficiently is of critical importance to boost the competitiveness of the stakeholders. 

\begin{figure}
    \centering
\includegraphics[width=0.9\textwidth]{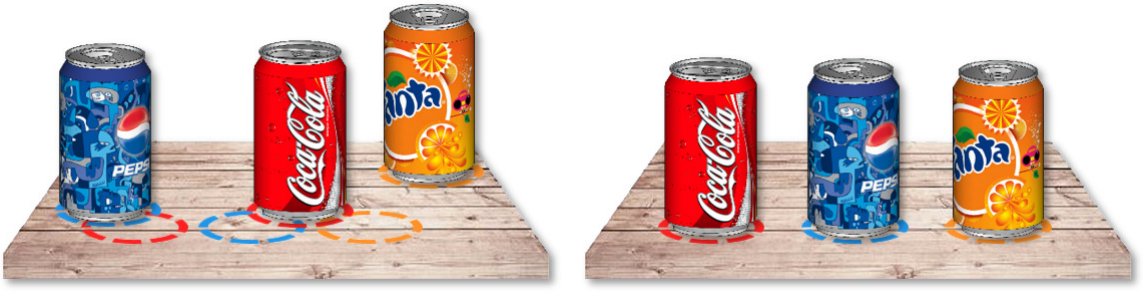}
\caption{A \toro instance where the three soda cans are to be rearranged 
from the left configuration to the right configuration.}
\label{fig:toro}
\end{figure}

However, even without the challenge of grasping, deciding the object manipulation order for optimizing a 
rearrangement task is non-trivial. To that end,  \cite{han2018complexity} examined the problem of \emph{tabletop object rearrangement with overhand grasps} (\toro), where objects may be picked up, moved around, and then placed at poses that are not in collision with other objects. 
An object that is picked up but cannot be directly placed at its goal is temporarily stored at a \emph{buffer}.
For example, for the setup given in \ref{fig:toro}, using a single manipulator, either the Coke or the Pepsi must be moved to a buffer before the task can be completed.

The thesis introduces efforts to solve the two main challenges in \toro: (1) the inherent combinatorial challenge in action sequencing, and (2) buffer pose allocation when objects need temporary placements.
\ref{chap:rbm} and \ref{chap:trlb} investigate single-arm \toro with external and internal buffers.
\ref{chap:cdr} extends the single-arm planners to dual-arm setups, where two robot arms cooperate to rearrange objects in a confined shared workspace.
\ref{chap:orla} proposes time-optimal rearrangement planners and examines the idea in a mobile robot setup.

\ref{chap:rbm} investigates \toro with external buffers (\toroe), where external free space is assumed for buffer allocation.
Details can be found in publications \labelcref{pubs:P5} and \labelcref{pubs:P12}.
This setup raises the natural question of how many simultaneous storage spaces, or ``running buffers'', are required to make certain tabletop rearrangement problems feasible.
This work examines the problem for both labeled and unlabeled settings. 
On the structural side, we observe that finding the minimum number of running buffers (\mrb) can be carried out on a dependency graph abstracted from a problem instance and show that computing \mrb is NP-hard. 
We then prove that under both labeled and unlabeled settings, even for 
uniform cylindrical objects, the number of required running buffers may 
grow unbounded as the number of objects to be rearranged increases. 
We further show that the bound for the unlabeled case is tight. 
On the algorithmic side, we develop effective exact algorithms 
for finding \mrb for both labeled and unlabeled tabletop rearrangement 
problems, scalable to over a hundred objects under very high object 
density. More importantly, our algorithms also compute a sequence 
witnessing the computed \mrb that can be used for solving object 
rearrangement tasks. 
Employing these algorithms, empirical evaluations reveal that random 
labeled and unlabeled instances, which more closely mimic real-world 
setups generally have fairly small \mrb{s}. 
Using real robot experiments, we demonstrate that the running buffer
abstraction leads to state-of-the-art solutions for the in-place rearrangement of 
many objects in a tight, bounded workspace.

In contrast to \toroe, there are many practical rearrangement scenarios where objects have to be displaced inside the workspace.
\ref{chap:trlb} applies the proposed \toroe algorithms to \toro with internal buffers (\toroi).
Details of this study can be found in publications \labelcref{pubs:P3} and \labelcref{pubs:P10}.
In a given \toroi instance, objects may need to be placed at temporary positions (``buffers'') to complete the rearrangement. However, allocating these buffer locations can be highly challenging in a cluttered environment. To tackle the challenge, a two-step baseline planner is first developed, which generates a primitive plan based on inherent combinatorial constraints induced by start and goal poses of the objects and then selects buffer locations assisted by the primitive plan. We then employ the ``lazy'' planner in a tree search framework which is further sped up by adapting a novel preprocessing routine. Simulation experiments show our methods can quickly generate high-quality solutions and are more robust in solving large-scale instances than existing state-of-the-art approaches.

\ref{chap:cdr} employs the idea of lazy buffer verification for dual-arm task planning and coordination.
Details of this work can be found in publications \labelcref{pubs:P0} and \labelcref{pubs:P7}.
In a non-monotone rearrangement task, complex object-object dependencies require moving some objects multiple times to solve an instance. 
In working with two arms in a large workspace, some objects must be 
handed off between the robots, which further complicates the planning process.

For task planning in the challenging dual-arm tabletop rearrangement problem, we develop 
effective algorithms for scheduling the pick-n-place sequence 
that can be properly distributed between the two arms. 
We show that, even without using a sophisticated motion planner, our 
method achieves significant time savings in comparison to greedy approaches 
and naive parallelization of single-robot plans.

For motion planning, we aim at optimizing motion plans for a real dual-arm system in which the two arms operate in close vicinity to solve highly constrained tabletop multi-object rearrangement problems. 
Toward that, we construct a tightly integrated planning and control optimization pipeline, \bl{M}akespan-\bl{O}ptimized \bl{D}ual-\bl{A}rm \bl{P}lanner (\modap) that combines novel sampling techniques for task planning with state-of-the-art trajectory optimization techniques. 
Compared to previous state-of-the-art, \modap produces task and motion plans that better coordinate a dual-arm system, delivering significantly improved execution time improvements while simultaneously ensuring that the resulting time-parameterized trajectory conforms to specified acceleration and jerk limits.

\begin{figure}
% \vspace{2mm}
    \centering
    \includegraphics[width=0.9\textwidth]{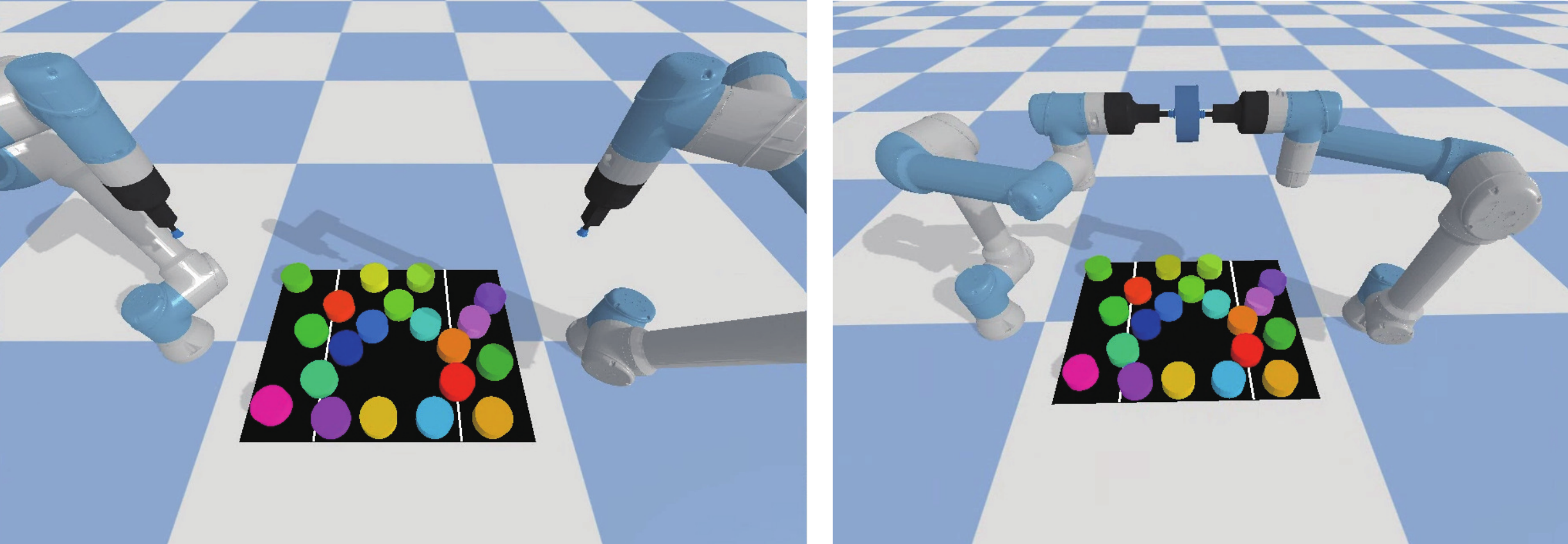}
    \caption{[Left] PyBullet setup for the Cooperative Multi-Robot Rearrangement problem, where only a portion of the environment (the region between two white lines) is reachable by both arms. [Right] Handoff operation at a pre-computed pose above the environment.}
    \label{fig:problem}
\end{figure}

\ref{chap:orla} further integrates the idea of lazy buffer verification into the A* framework, enabling time-optimal multi-object rearrangement planning.
The idea is examined in the scenario of mobile robot tabletop rearrangement (\ref{fig:intro}).
Details of this work can be found in publications \labelcref{pubs:P00}.

In this work, we propose \orla, which leverages delayed/lazy evaluation in searching for a high-quality object pick-n-place sequence that considers both end-effector and mobile robot base travel. \orla readily handles multi-layered rearrangement tasks powered by learning-based stability predictions. Employing an optimal solver for finding temporary locations for displacing objects, \orla can achieve global optimality. Through extensive simulation and ablation study, we confirm the effectiveness of \orla in delivering high quality solutions for challenging rearrangement instances.

\begin{figure}
\vspace{2mm}
    \centering
    \includegraphics[width=0.9\columnwidth]{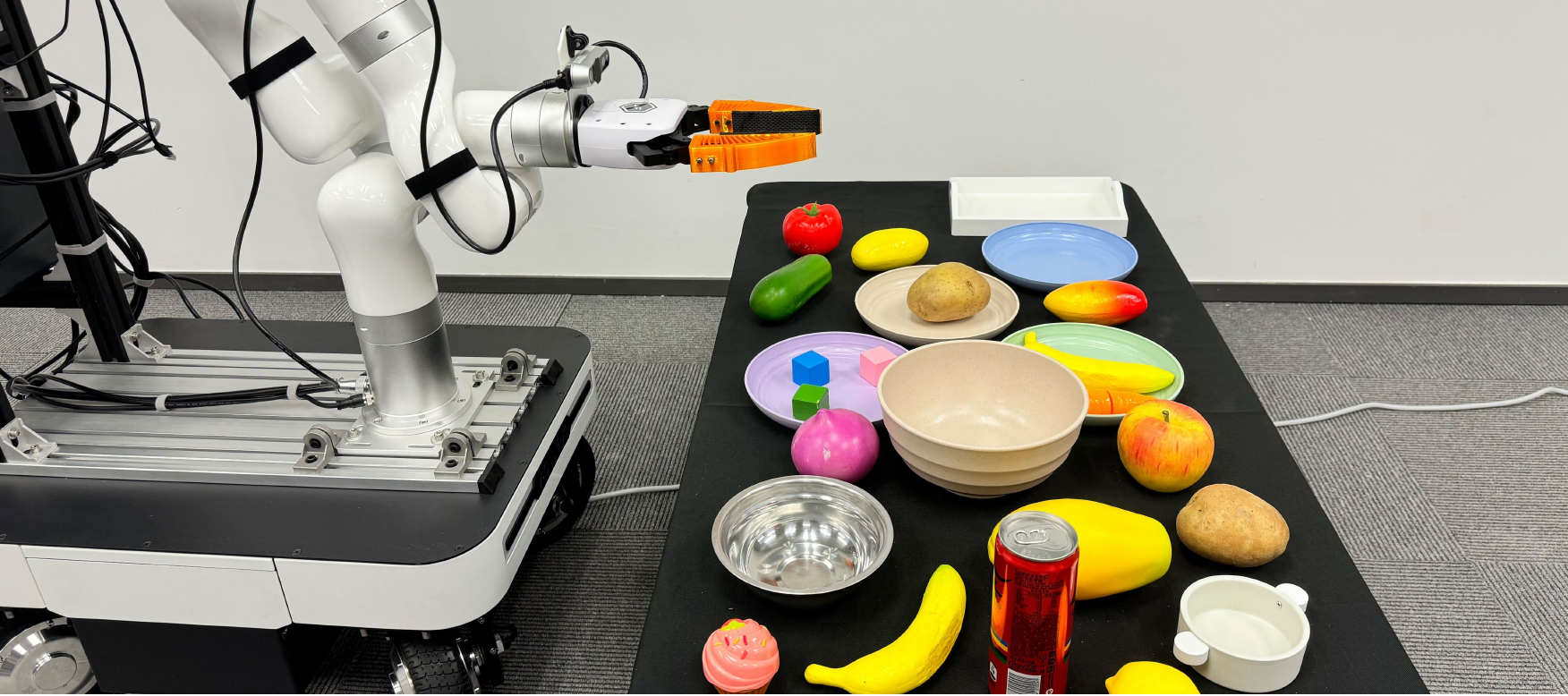}
    \caption{An example of the Mobile Robot Tabletop Rearrangement (\motar) setup.}
    \label{fig:intro}
\end{figure}

%% file: chapters/RBM.tex
\chapter{Running Buffer Minimization (RBM) for \toro with External Buffers (\toroe)} \label{chap:rbm}
\thispagestyle{myheadings}

\section{Motivation}
This chapter investigates \toro with external buffers (\toroe), where we assume an external free space to temporarily place objects during the rearrangement process.
In \toroe, the study focuses on a practical rearrangement objective of minimizing the number of \emph{running buffers} (\rb) in solving a \toroe instance, which is the number of objects stored at buffers simultaneously. 
The maximum running buffer during rearrangement is denoted as \mrb. Whereas the total number of buffers used has more bearing on 
global solution optimality, \mrb sheds more light on \emph{feasibility}. 
Knowing the \mrb tells us whether a certain number of external 
buffers will be sufficient for solving a class of rearrangement problems. 
This is critical for practical applications where the number of 
external buffers is generally limited to a small constant. 
For example, in solving everyday rearrangement tasks, e.g., sorting or retrieving items in the fridge, a human may attempt to temporarily hold multiple items, sometimes awkwardly. 
There is clearly a limit to the number of items that can be held simultaneously this way. 

This chapter brings forth several novel technical contributions. First, we show that computing \mrb on arbitrary \emph{dependency graphs}, which encode the combinatorial 
information of \toroe instances, is NP-hard. Second, this study establishes that for an $n$-object 
\toroe instance, \mrb can be lower bounded by $\Omega(\sqrt{n})$ for uniform cylinders, 
even when all objects are \emph{unlabeled}. This implies that the same is true for 
the \emph{labeled} setting. Third, we provide a matching algorithmic upper bound 
$O(\sqrt{n})$ for the unlabeled setting. 
%\sw{for all instances}
%\jy{I think this is not necessary to state, since that is what upper bound means}
Fourth, we develop multiple highly effective and optimal algorithms for 
computing rearrangement plans with \mrb for \toro. 
In particular, we present a dynamic programming method for the labeled setting, 
a priority queue-based algorithm for the unlabeled setting, and a much more 
efficient \emph{depth-first-search dynamic programming} routine that readily 
scales to instances with over a hundred objects for both settings. Furthermore, 
methods for computing plans with the minimum number of \emph{total 
buffers} subject to the \mrb constraints are provided. 
These algorithms not only provide the optimal number of buffers but also provide a 
rearrangement plan that witnesses the optimal solution.

\section{Related Work}\label{sec:related-works}
As a high utility capability, manipulation of objects in 
a bounded workspace has been extensively studied, with works devoted to 
perception/scene understanding \cite{saxena2008robotic,gualtieri2016high,
mitash2017self,xiang2017posecnn}, task/rearrangement planning 
\cite{stilman2005navigation,
treleaven2013asymptotically,havur2014geometric,
krontiris2015dealing,king2016rearrangement,
han2018complexity,lee2019efficient}, 
manipulation \cite{taylor1987sensor,goldberg1993orienting,
lynch1999dynamic,dogar2011framework,bohg2013data,dafle2014extrinsic,
boularias2015learning,chavan2015prehensile}, as well as integrated, 
holistic approaches \cite{kaelbling2011hierarchical,levine2016end,mahler2017dex,
zeng2018robotic,wells2019learning}.
As object rearrangement problems often embed within them multi-robot motion planning 
problems, rearrangement inherits the PSPACE-hard complexity \cite{hopcroft1984complexity}. 
These problems remain  NP-hard even without complex geometric constraints 
\cite{wilfong1991motion}. Considering rearrangement plan quality, e.g., minimizing the 
number of pick-n-places or the end-effector travel, is also computationally 
intractable \cite{han2018complexity}. 

For rearrangement tasks using mainly prehensile actions, the algorithmic 
studies of Navigation Among Movable Obstacles \cite{stilman2005navigation,
stilman2007manipulation} result in backtracking search methods that can 
effectively deal with monotone instances which restrict the robot to move each obstacle at most once.
Via carefully calling monotone solvers, difficult non-monotone cases can 
be solved as well \cite{krontiris2015dealing,wang2022lazy}.
\cite{han2018complexity} relates tabletop rearrangement problems 
to the Traveling Salesperson Problem \cite{papadimitriou1977euclidean}  
and the Feedback Vertex Set problem \cite{karp1972reducibility},
both of which are NP-hard. Nevertheless, integer programming 
models could quickly compute high-quality solutions for 
practical sized ($\sim20$ objects) problems. 
Focusing mainly on the unlabeled setting, bounds on the number of 
pick-n-places are provided for disk 
objects in \cite{bereg2006lifting}.
In \cite{lee2019efficient}, a complete algorithm is developed that reasons 
about object retrieval, rearranging other objects as needed, with later 
work \cite{nam2019planning} considering plan optimality and sensor 
occlusion. 
While objectives in most problems focus on the number 
of motions, \cite{halperin2020space} seeks to minimize 
the space needed to carry out a rearrangement task for discs moving inside the workspace.

Non-prehensile rearrangement has also been extensively studied\cite{ben1998practical,huang2019large}, with object
singulation as an early focus \cite{chang2012interactive,
laskey2016robot,eitel2020learning}. In this problem, a robot is tasked to separate a target object from surrounding obstacles with non-prehensile actions, e.g., pushing and poking, in order to provide room for performing grasping actions. An iterative search was employed in 
\cite{huang2019large} for accomplishing a multitude of rearrangement tasks spanning singulating, separation, and sorting of identically shaped cubes.
\cite{song2019multi} combines Monte Carlo Tree Search
with a deep policy network for separating many objects into coherent clusters
within a bounded workspace, supporting non-convex objects. 
More recently, a bi-level planner is proposed \cite{pan2020decision}, 
engaging both (non-prehensile) pushing and (prehensile) overhand grasping 
for sorting a large number of objects. 
Synergies between non-prehensile and prehensile actions have been explored 
for solving clutter removal tasks \cite{zeng2018learning,huang2020dipn} 
and more challenging object retrieval tasks \cite{huang2021visual,vieira2022persistent}
using a minimum number of pushing and grasping actions.

On the structural side, a central object that we study is the 
\emph{dependency graph} structure.  Similar dependency structures were first introduced to multi-robot path planning problems to deal with path conflicts between agents \cite{buckley1988fast,van2009centralized}.  
Subsequently, the structure was employed for reasoning about and solving challenging 
rearrangement problems \cite{krontiris2015dealing,krontiris2016efficiently,wang2020robot,gao2022toward}. 
The full labeled dependency graph, as induced by a rearrangement 
instance, is first introduced and studied in \cite{han2018complexity}. This 
current work introduces the unlabeled dependency graph. 
We observe that, in the labeled setting, through the dependency graph, the 
running buffer problems naturally connect to \emph{graph layout} problems 
\cite{diaz2002survey,garey1979computers,papadimitriou1976np,garey1974some,gavril2011some, bodlaender1995approximating},
where an optimal linear ordering of graph vertices is sought. Graph layout 
problems find a vast number of important applications, including VLSI design, scheduling 
\cite{shin2011minimizing}, and so on. For the unlabeled 
setting, the dependency graph becomes a planar one for uniform objects with a round base. 
Rearrangement can be tackled through partitioning of the dependency 
graph using a \emph{vertex separator} 
\cite{lipton1979separator,gilbert1984separator,alon1990separator,elsner1997graph}. For a survey on 
these topics, see~\cite{diaz2002survey}.

\section{\toro with External Buffer Space (\toroe)}
For \toroe tasks where objects assume similar geometry, there are two natural practical settings depending on whether the objects are distinguishable, i.e., whether they are \emph{labeled} or \emph{unlabeled}. 
We describe external buffer formulations under these two distinct settings, and 
discuss the important \emph{dependency graph} structure for both. 

\subsection{Labeled \toro with External Buffers}
Consider a bounded workspace $\mathcal W \subset \mathbb R^2$ with a set of 
$n$ objects $\mathcal O = \{o_1, \ldots, o_n\}$ placed inside it. All objects 
are assumed to be \emph{generalized cylinders} with the same height. A 
\emph{feasible arrangement} of these objects is a set of poses $\mathcal A 
=\{x_1,\ldots, x_n\}, x_i \in SE(2)$ in which no two objects collide. 
Let $\mathcal A_s = \{x_1^s, \ldots, x_n^s\}$ and $\mathcal A_g = 
\{x_1^g, \ldots, x_n^g\}$ be two feasible arrangements, a tabletop object 
rearrangement problem 
\cite{han2018complexity}
seeks a plan using 
\emph{pick-n-place} operations that move the objects from $\mathcal A_s$ to 
$\mathcal A_g$ (see \ref{fig:rbm-ex-prob}(a) for an 
example with 7 uniform cylinders). In each pick-n-place operation, an object 
is grasped by a robot arm, lifted above all other objects, transferred to and 
lowered at a new pose $p \in SE(2)$ where the object will not be in collision 
with other objects, and then released. A pick-n-place operation can be 
formally represented as a 3-tuple $a = (i, x', x'')$, denoting that object 
$o_i$ is moved from pose $x'$ to pose $x''$. A full rearrangement plan $P = (a_1, a_2, 
\ldots)$ is then an ordered sequence of pick-n-place operations. 

\begin{figure}[h]
    \centering
\begin{overpic}
[width=\textwidth]{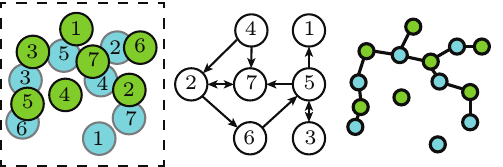}
\put(13, -4){{\small (a)}}
\put(48, -4){{\small (b)}}
\put(80, -4){{\small (c)}}
\end{overpic}
\vspace{2mm}
    \caption{A 7-object labeled instance with uniform cylinders; we 
    will use this instance as a running example. (a) The 
    unshaded discs (as projections of cylinders) represent the start arrangement
    $\mathcal{A}_s$ and the shaded discs represent the goal arrangement $\mathcal{A}_g$. 
    (b) The corresponding labeled dependency graph. (c) The corresponding 
    unlabeled dependency graph, which is bipartite and planar.}
    \label{fig:rbm-ex-prob}
\end{figure}

Depending on $\mathcal A_s$ and $\mathcal A_g$, it may not always be 
possible to directly transfer an object $o_i$ from $x_i^s$ to $x_i^g$
in a single pick-n-place operation, because $x_i^g$ may be occupied 
by other objects. This creates \emph{dependencies} between objects. 
If object $o_i$ at pose $x^g_i$ intersects objects $o_j$ at 
pose $x^s_j$, we say $o_i$ \emph{depends} on $o_j$. This suggests that 
object $o_j$ must be moved first before $o_i$ can be placed at its goal
pose $x^g_i$. 

It is possible to have circular dependencies. As an example, for the instance given in \ref{fig:rbm-ex-prob}(a), objects $3$ and $5$ have dependencies on each other. 
In such cases, some object(s) must be temporarily moved to an intermediate 
pose to solve the rearrangement problem. Similar to \cite{han2018complexity}, for this chapter,
we assume that \emph{external buffers} outside the workspace are used for intermediate poses, which avoids time-consuming geometric computations 
if the intermediate poses are to be placed within $\mathcal W$. 
During the execution of a rearrangement plan, multiple objects can be stored at buffer locations. We call the buffers that are currently in use \emph{running buffers} (\rb). %
With the introduction of buffers, there are three types of pick-n-place 
operations: 1) pick an object at its start pose and place it at a buffer,
2) pick an object at its start pose and place it at its goal pose, and 3) 
pick an object from a buffer and place it at its goal pose. 
% Notice that buffer 
% poses are not important. 
%
Naturally, it is desirable to be able to solve a rearrangement problem
with the least number of running buffers, giving rise to the \emph{labeled 
running buffer minimization} problem. 

\begin{problem}[Labeled Running Buffer Minimization (\lrbm)]\label{p:1} Given 
feasible arrangements $\mathcal A_s$ and $\mathcal A_g$, find a rearrangement 
plan $P$ that minimizes the maximum number of running buffers in use at any 
given time as the plan is executed. 
\end{problem}

In an \lrbm instance, the set of all dependencies induced by $\mathcal 
A_s$ and $\mathcal A_g$ can be represented using a directed graph $\ldgg = 
(V, A)$, where each $v_i \in V$ corresponds to object $o_i$ and there is an 
arc $v_i \to v_j$ for $1 \le i, j \le n, i \ne j$ if object $o_i$ 
depends on object $o_j$. We call $\ldgg$ a \emph{labeled dependency graph}. 
The labeled dependency graph for \ref{fig:rbm-ex-prob}(a) is given in 
\ref{fig:rbm-ex-prob}(b).
Based on the dependency graph $\ldgg$, we can immediately identify multiple 
circular dependencies in the graph, e.g., between objects $3$ and $5$, or among 
objects $7, 2, 6$ and $5$. The cycles form strongly connected components
of $\ldgg$, which can be effectively computed \cite{tarjan1972depth}.  
Since moving objects to external buffers does not create additional dependencies, we have
\begin{proposition}
$\ldgg$ fully captures the information needed to solve the tabletop rearrangement problem with external buffers moving objects from $\mathcal A_s$ to $\mathcal A_g$.
\end{proposition}

We use two examples to illustrate the relationships between \toroe, its dependency graph, and external buffer-based solutions. First, for the example given in \ref{fig:toro}, the corresponding start/goal configurations are given in \ref{fig:rbm-coke-pepsi}[top left]. The labeled dependency graph is given in  \ref{fig:rbm-coke-pepsi}[top right]. Because of the existence of cyclic dependencies between Coke and Pepsi, one of these two must be temporarily moved to a buffer. Suppose that Pepsi is moved to buffer. This changes the configuration and dependency graph as shown in the second row of \ref{fig:rbm-coke-pepsi}. The \toroe instance can then be readily solved. The complete solution sequence is $\langle Pepsi \to b, Coke \to g, Pepsi \to g, Fanta \to g \rangle$, where $b$ refers to a buffer and $g$ refers to the goal of the corresponding object.

\begin{figure}[ht!]
    \centering
    \includegraphics[width=0.8\textwidth]{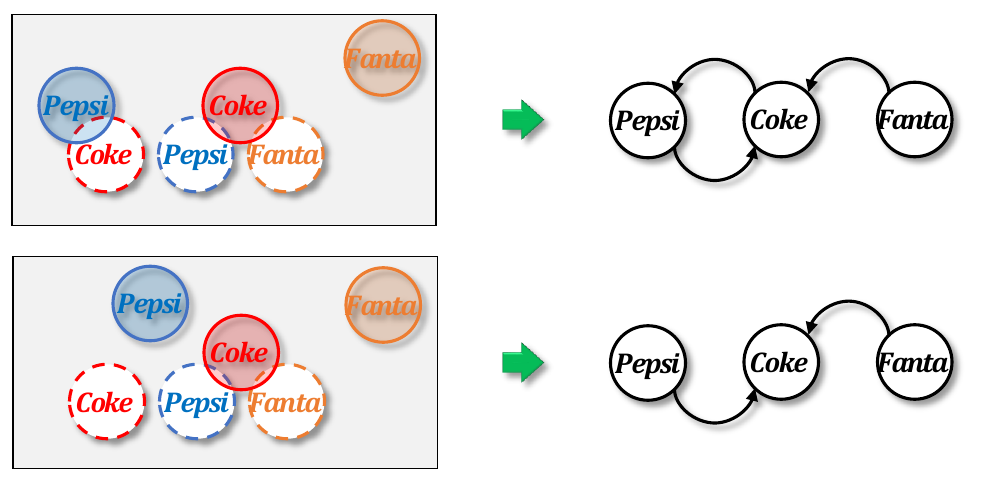}
    \caption{Two configurations of the setup given in \ref{fig:toro} and the corresponding dependency graphs.}
    \label{fig:rbm-coke-pepsi}
\end{figure}

For the example given in \ref{fig:rbm-ex-prob}[a], its labeled dependency graph \ref{fig:rbm-ex-prob}[b] shows two cycles ($2\leftrightarrow7$ and $3\leftrightarrow 5$). One planning sequence for solving this can be derived that moves $2$ and $3$ to buffer: $\langle 1\to g, 2\to b, 3 \to b, 7 \to g, 4 \to g, 5 \to g, 6 \to g, 2 \to g, 3 \to g\rangle$. 

\subsection{Unlabeled \toro with External Buffers}
In the unlabeled setting, objects are interchangeable. That is, it does not 
matter which object goes to which goal. For example, in \ref{fig:rbm-ex-prob}, 
object $5$ can move to the goal for object $6$. We call this version the 
\emph{unlabeled running buffer minimization} (\urbm) problem, which is intuitively 
easier. The plan for the unlabeled problem can be represented similarly as the 
labeled setting; we continue to use labels so that plans can be clearly represented
but do not require matching labels for start and goal poses.

For the unlabeled setting, the dependency structure remains but appears in a different form. It is now an \emph{undirected bipartite} graph. That is, $\udgg = (V_s\cup V_g, E)$ 
where each $v \in V_s$ (resp., $v \in V_g$) corresponds to a start (resp., goal) 
pose $p \in \mathcal A_s$ (resp., $p \in \mathcal A_g$). We denote the vertices representing the start and goal poses as \emph{start vertices} and \emph{goal 
vertices}, respectively. There is an edge between 
$v_s \in V_s$ and $v_g \in V_g$ if the objects at the corresponding poses 
overlap. The unlabeled dependency graph for \ref{fig:rbm-ex-prob}(a) is illustrated in \ref{fig:rbm-ex-prob}(c). 

In practice, many \toroe instances have objects with footprints that are uniform regular polygons (e.g., squares) or discs. In such settings, we can say something additional about the resulting unlabeled dependency graphs (the proofs of the two propositions below can be found in the appendix (\ref{sec:app-toroe})).

\begin{proposition}\label{p:udg-polygon}
For unlabeled \toroe where footprints of objects are uniform regular polygons, the maximum degree of the unlabeled dependency graph is upper bounded by 19. 
\end{proposition}

\begin{proposition}\label{p:udg}
For unlabeled \toroe where footprints of objects are uniform discs, the dependency graph is a planar bipartite graph with a maximum degree of 5. 
\end{proposition}

For either \lrbm or \urbm, the minimum required number of running buffers is denoted as \mrb, which can be computed based on the corresponding dependency graph.

\section{Structural Analysis and NP-Hardness}\label{sec:struture}
The introduction of running buffers to tabletop rearrangement problems 
induces unique and interesting structures. 
We highlight some important structural properties of \lrbm, including 
the comparison to minimizing the total number of buffers \cite{han2018complexity}, 
the solutions of \lrbm and \emph{linear arrangement} \cite{shiloach1979minimum} or 
\emph{linear ordering} \cite{adolphson1973optimal} of its dependency graph, and the hardness of computing \mrb for \toroe. 

As will be established in this section, computing \mrb solutions for \toroe is intimately connected to finding a certain optimal linear ordering of vertices in the associated dependency graph. Since finding an optimal linear ordering is hard, it renders computing \mrb solutions hard as well. This in turn limits the algorithmic solutions that one can secure for computing \mrb solutions for \toroe tasks.

\subsection{Running Buffer versus Total Buffer}
In solving \toroe, running buffers are related to but different from the total number of buffers, as studied in \cite{han2018complexity}. 
It is shown that the minimum number of total buffers for solving \toroe is the same as the size of the minimum \emph{feedback vertex set} (FVS) of the underlying dependency graph. 
An FVS is a set of vertices the removal of which leaves a graph acyclic. A \toroe with an acyclic dependency graph can be solved 
without using any buffer because there are no cyclic dependencies between any pairs of objects. We denote the size of the minimum FVS as 
\fvs. 

As a labeled \toroe problem, the example from \ref{fig:toro} has $\fvs = \mrb = 1$. For the example from \ref{fig:rbm-ex-prob}(a)(b), $\fvs = 2$ (an FVS set is $\{2, 3\}$) and $\mrb = 2$.

As an extreme example illustrating the difference between \mrb and \fvs, consider a labeled dependency graph that is formed by $n$ copies of $2$-cycles, e.g., \ref{fig:MFVS_vs_MRB}. 
On one hand, the \fvs of the instance is $n$ because one vertex must be removed from each of the $n$ cycles to make the graph acyclic. On the other hand, 
the \mrb is just $1$ because each cyclic dependency can be resolved independently from other cycles using a single external buffer. 
Therefore, whereas the total number of buffers used has more bearing on 
global solution optimality, \mrb sheds more light on \emph{feasibility}. 
Knowing the \mrb tells us whether a certain number of external 
buffers will be sufficient for solving a class of rearrangement problems. 
This is critical for practical applications where the number of 
external buffers is generally limited to be a small constant. 
For example, in solving everyday rearrangement tasks, e.g., sorting things or retrieving items in the fridge, a human may attempt to temporarily hold multiple items, sometimes awkwardly. There is clearly a limit to the number of items that can be held simultaneously this way. 

\begin{figure}[h]
    \centering
    \includegraphics[width=0.9\textwidth]{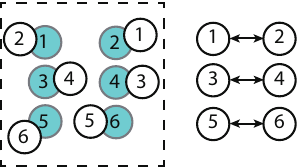}
    \caption{An instance with the labeled dependency graph formed 
by 3 copies of $2$-cycles. The \fvs is 3. On the other hand, 
the \mrb is just $1$ for the problem. The example scales to have arbitrarily large \fvs with \mrb remaining at $1$.}
    \label{fig:MFVS_vs_MRB}
\end{figure}

We give an example where the \mrb and \fvs cannot always be optimized 
simultaneously, i.e., they form a Pareto front. 
For the setup in \ref{fig:ex-mrb-fvs}, the \fvs $\{7, 9, 10\}$ has size $3$ - it can be readily checked that deleting  $\{7, 9, 10\}$ leaves the dependency graph acyclic. Using our algorithms, the \mrb can be computed to be $2$ with the witness sequence being $1, 10, 8, 4, 5, 3, 6, 7, 2, 9$, corresponding to the plan of $\langle 1 \to g, 10\to b, 8 \to g, 4 \to b, 5\to g, 3\to g, 10 \to g, 6 \to b, 7 \to g, 6 \to g, 2\to b,  9\to g, 4 \to g, 2 \to g \rangle$, where $b$ refers to a buffer and $g$ refers to the goal of the corresponding object. The RB reaches the maximum $2$ when $4$ and $6$ are moved to the buffer.
However, in this case, the total number of buffers used is $4$: $2, 4, 6$, and $10$ are moved to buffers. This turns out to be the best we can do for this example, that is,  if we are constrained by solutions with $\mrb=2$, the total number of buffers that must be used is at least $4$ instead of the \fvs size of $3$.

We note that it is rarely the case that \fvs will be larger when the solution space is constrained to \mrb solutions; for uniform cylinders, for example, the total number of buffers needed after first minimizing the running buffer is almost always the same as the \fvs size. 
\begin{figure}[h!]
    \centering
    \begin{overpic}[width=\textwidth]{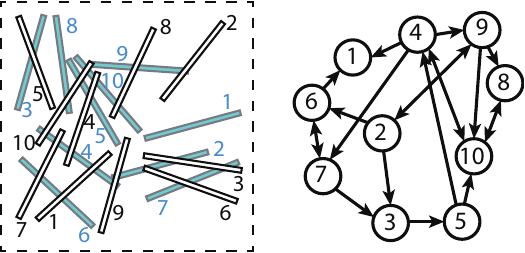}
    \end{overpic}
    \caption{An \lrbm instance with uniform thin cuboids (left) and its 
    labeled dependency graph, where the total number of buffers needed is 
    more than the size of the \fvs when the number of running buffers is 
    minimized.}
    \label{fig:ex-mrb-fvs}
\end{figure}

\subsection{Running Buffers and Linear Vertex Ordering}\label{subsec:logvrb}
Given a graph with vertex set $V$, a \emph{linear ordering} of $V$ is a bijective 
function $$\varphi: \{1, \ldots, |V|\} \to V$$ 
Given a labeled dependency graph $ \ldgg(V, A)$ for an \lrbm problem and a linear ordering $\varphi$ for $V$, we may turn $\varphi$ into a plan $P$ for the \lrbm problem by sequentially picking up objects corresponding to vertices $\varphi(1), \varphi(2), \ldots$ For each object that is picked up, it is moved to its goal pose if it has no further dependencies; otherwise, it is stored in the buffer. Objects already in the buffer will be moved to their goal pose at the earliest possible opportunity. 

For example, given the linear ordering $1, 5, 6, 3, 4, 2, 7$ for the dependency graph from \ref{fig:rbm-ex-prob}(b), first, $1$ can be directly moved to its goal. Then, $5$ is moved to the buffer because it has dependencies on $3$ and $7$ (but no longer on $1$). Then, $6$ can be directly moved to its goal because $5$ is now at a buffer location. Similarly, $3$ can be moved to its goal next. Then, $4$ and $2$ must be moved to the buffer, after which $7$ can be moved to its goal directly. Finally, $2, 4$, and $5$ can be moved to their respective goals from the buffer. This leads to a maximum running buffer size of $3$. This is not optimal; an optimal sequence is $5, 6, 2, 7, 4, 3, 1$, with $\mrb = 2$. Both sequences are illustrated in \ref{fig:lr}.
\begin{figure}[h!]
    \centering
    \begin{overpic}
    [width=\textwidth]{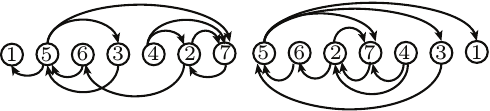}
    \put(22, -3.5){{\small (a)}}
    \put(74, -3.5){{\small (b)}}
    \end{overpic}
    \vspace{2mm}
    \caption{Two linear orderings of vertices of the labeled 
    dependency graph from \ref{fig:rbm-ex-prob}(b) (i.e., these are different representations of the same graph). The right one minimizes \mrb.}
    \label{fig:lr}
\end{figure}

The discussion suggests that we may effectively view the number of running buffers as a function of a dependency graph $ \ldgg$ and a linear ordering $\varphi$. We thus define $\rb( \ldgg, \varphi)$ as the number of running buffers needed for rearranging $ \ldgg$ following the order given by $\varphi$. It is straightforward to see that $\mrb( \ldgg) = \min_{\varphi}\rb( \ldgg, \varphi)$.

\subsection{Intractability of Computing \mrb}
Since computing \fvs is NP-hard \cite{han2018complexity}, one would expect that computing \mrb for a labeled dependency graph, which can be any directed graph, is also hard. We show that this is indeed the case, by examining the interesting relationship between \mrb and the \emph{vertex separation problem} (\vsp), which is equivalent to path width, gate matrix layout, and search number problems as described in Theorem 3.1 in \cite{diaz2002survey}, resulting from a series of studies \cite{kirousis1986searching, kinnersley1992vertex, fellows1989search}. Unless $P=NP$, there cannot be an absolute approximation algorithm for any of these problems \cite{bodlaender1995approximating}. First, we describe the vertex separation problem. 
Intuitively, given an undirected graph $G = (V, E)$, \vsp seeks a linear ordering $\varphi$ 
of $V$ such that, for a vertex with order $i$, the number of vertices that come no later than 
$i$ in the ordering, with edges to vertices that come after $i$, is minimized. 
%

% \begin{problem}[VertSep]
% \end{problem}

\vspace{1mm}
\noindent\fbox{\begin{minipage}{0.94\textwidth}
\vspace{1mm}
\noindent
\textbf{Vertex Separation (\vsp)}\\
\noindent
\textbf{Instance}: Graph $G(V,E)$ and an integer $K$.\\
\noindent
\textbf{Question}: Is there a bijective function $\varphi: \{1, \dots, n\} \to 
V$, such that for any integer $1 \le i \le n$, 
$|\{u\in V \mid \exists (u, v) \in E \ and \ \varphi(u)\leq i < \varphi(v) \}| \leq K$?
\vspace{1mm}
\end{minipage}}
\vspace{1mm}

As an example, in \ref{fig:vsp}(a), with the given linear ordering, at the second vertex, both the first and the second vertices have edges crossing the vertical separator, yielding a crossing number of $2$. Given a graph $G$ and a linear ordering $\varphi$, we define $\vertsep(G, \varphi) := \max_i | \{u\in V \mid \exists (u, v) \in E  \ and \  \varphi(u)\leq i < \varphi(v) \} |$, \vsp seeks $\varphi$ that minimizes $\vertsep(G, \varphi)$. Let $\minvs(G)$, the vertex separation number of graph $G$, be the minimum $K$ for which a \vsp instance has a yes answer, then $\minvs(G) = \min_\varphi \vertsep(G, \varphi)$.
Now, given an undirected graph $G$ and a labeled dependency graph $\ldgg$ obtained from $G$ by replacing each edge of $G$ with two directed edges in opposite directions, we observe that there are clear similarities between $\vertsep(G, \varphi)$ and $\rb(\ldgg, \varphi)$, which is characterized by the following lemma. 

\begin{lemma}\label{l:vs-rb}
$\vertsep(G, \varphi) \leq \rb(\ldgg, \varphi) \leq \vertsep(G, \varphi) + 1$.
\end{lemma}
\begin{proof}
Fixing a linear ordering  $\varphi$, it is clear that $\vertsep(G, \varphi) \le \rb(\ldgg, \varphi)$, since the vertices on the left side of a separator with edges crossing the separator for $G$ correspond to the objects that must be stored at buffer locations. For example, in \ref{fig:vsp}(a), past the second vertex from the left, both the first and the second vertices have edges crossing the  vertical ``separator''. In the corresponding dependency graph shown in \ref{fig:vsp}(b), objects corresponding to both vertices must be moved to the external buffer. 
%
%Intuitively, this is the case because, for any vertex separator
%
On the other hand, we have $\rb(\ldgg, \varphi) \leq \vertsep(G, \varphi) + 1$ because as we move across a vertex in the linear ordering, the corresponding object may need to be moved to a buffer location temporarily. 
For example, as the third vertex from the left in \ref{fig:vsp}(a) is 
passed, the vertex separator drops from $2$ to $1$, but for dealing with 
the corresponding dependency graph in \ref{fig:vsp}(b), the object 
corresponding to the third vertex from the left must be moved to the 
buffer before the first and the second objects stored in the buffer can be 
placed at their goals. 
\end{proof}

\begin{figure}[h!]
    \centering
\begin{overpic}
[width=\textwidth]{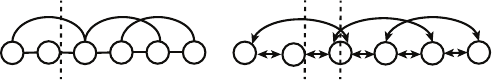}
\put(19, -3){{\small (a)}}
\put(72, -3){{\small (b)}}
\end{overpic}
\vspace{3mm}
\caption{(a) An undirected graph and a linear ordering of its vertices. (b)
A corresponding labeled dependency graph with the same vertex ordering.}
\label{fig:vsp}
\end{figure}

\begin{theorem}\label{t:mrb-hard}
Computing \mrb, even with an absolute approximation, for a labeled dependency graph is NP-hard. 
\end{theorem}
\begin{proof}
Given an undirected graph $G$, we reduce from approximating \vsp within a constant to approximating \mrb within a constant for a dependency graph  $\ldgg$ from $G$ constructed as stated before, replacing each edge in $G$ as a bidirectional dependency. 

Unless $P=NP$, \vsp does not have absolute approximation in polynomial time. 
Henceforth, if $\mrb(\ldgg, \varphi)$  can be approximated within $\alpha$ in polynomial time, which means for graph $G$, we can find a $\varphi^*$ in polynomial time such that $\rb(\ldgg, \varphi^*) \leq  \mrb(\ldgg) + \alpha$, we then have $\vertsep(G, \varphi^*)\leq \rb(\ldgg, \varphi^*)  \leq \alpha + \mrb(\ldgg) \leq \minvs(G)+\alpha+1$, which shows vertex separation can have 
an absolute approximation, implying $P=NP$.
\end{proof}

With some additional effort, we can show that computing \mrb solutions for certain \toroe instances is NP-hard.

\begin{theorem}\label{t:mrb-toro-hard}
Computing \mrb for \toroe is NP-hard. 
\end{theorem}
\begin{proof}
By \cite{monien1988min}, \vsp remains NP-complete for planar graphs with a maximum degree of three. 
From \ref{l:vs-rb} and \ref{t:mrb-hard} and the corresponding proofs, if we can show that we can convert an arbitrary planar graph with maximum degree three into a corresponding \toroe instance, then we are done. That is, all we need to show is that, if \ref{fig:vsp}(a) is a planar graph with maximum degree three, we can construct a \toroe instance for which the dependency graph is \ref{fig:vsp}(b). \\
We proceed to show something stronger: instead of showing the above for planar graphs with a maximum degree of three, we will do so for all planar graphs. By \cite{istvan1948straight}, all planar graphs can be drawn in the plane without edge crossings using only straight-line edges. Given an arbitrary planar graph and one of its straight-line-edge non-crossing embedding in the plane, we show how we can convert the embedding to a corresponding \toroe instance.
\begin{figure}[h]
    \centering
\begin{overpic}
[width=\textwidth]{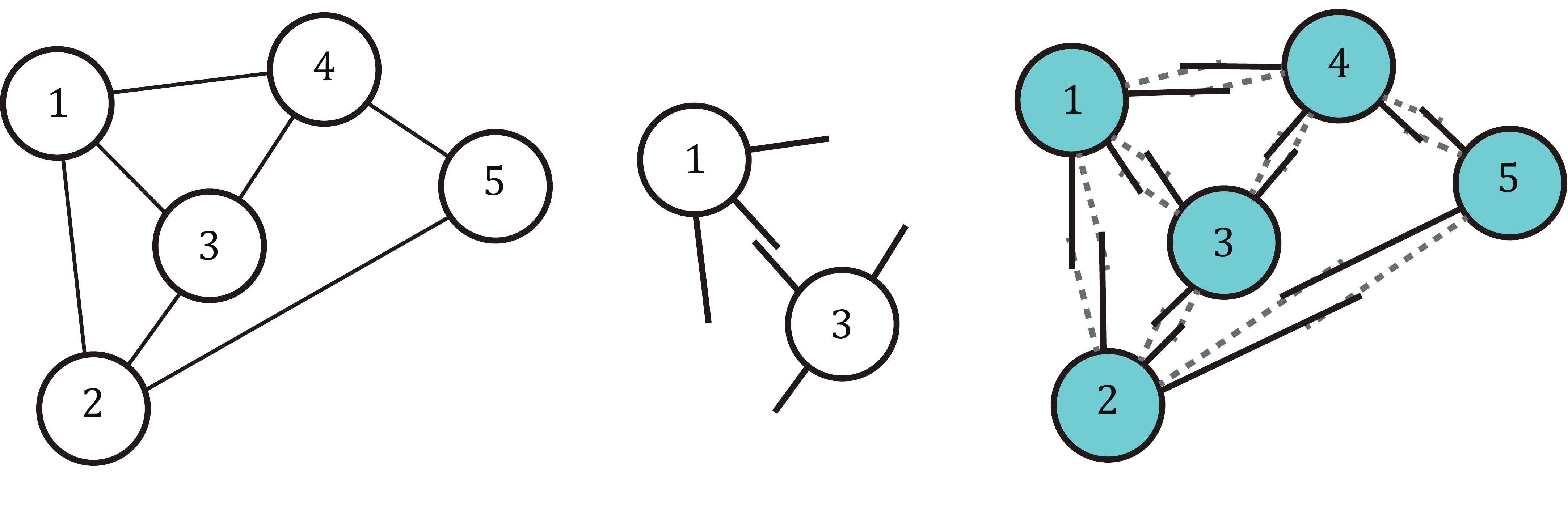}
\put(14, -3){{\small (a)}}
\put(45, -3){{\small (b)}}
\put(82, -3){{\small (c)}}
\end{overpic}
\vspace{3mm}
\caption{(a) An undirected planar graph with five vertices (b)
``Object gadgets'' for two of the vertices. (c) The corresponding \toroe instance.}
\label{fig:mrb-toro-hard}
\end{figure} 
\linebreak\indent The construction process is explained using \ref{fig:mrb-toro-hard} as an illustration, where \ref{fig:mrb-toro-hard}(a) shows a straight-line-edge non-crossing embedding of planar graph. We construct \emph{object gadgets} based on the embedding's geometric structure as follows. For each node of the graph, we take the node and a bit more than half of each of the edges coming out of it. For example, for nodes $1$ and $3$, the corresponding objects are given in \ref{fig:mrb-toro-hard}(b). We now construct a \toroe instance with a dependency graph corresponding to replacing edges of \ref{fig:mrb-toro-hard}(a) with bidirectional edges. To do so, we place each object gadget as they appear in \ref{fig:mrb-toro-hard}(a). For the goal configuration, we rotate each object \emph{counterclockwise} by some small angle $\varepsilon$ around the center of the corresponding node. For the start configuration, we perform a similar rotation but in the \emph{clockwise} direction. The process yields \ref{fig:mrb-toro-hard}(c) for  \ref{fig:mrb-toro-hard}(a). 
It is straightforward to check that the construction converts each edge $(i, j)$ in the planar graph into a mutual dependency between two objects $i$ and $j$. For example, there is a cyclic dependency between object gadgets $1$ and $3$ in \ref{fig:mrb-toro-hard}(c).
\end{proof}

\section{Lower and Upper Bounds on \mrb}\label{sec:bounds}
After connecting computing \mrb to vertex linear ordering and proving the computational intractability, we proceed to establish quantitative bounds on \mrb, i.e., what is the lowest possible \mrb for \lrbm and \urbm, and what is the best that we can do to lower \mrb? 

\subsection{Intrinsic \mrb Lower Bounds} 
When there is no restriction on object geometry, \mrb can easily reach the maximum possible $n -1$ for an $n$ object instance, even in the \urbm case. An example of when this happens is given in \ref{fig:stick}, where $n = 6$ thin cuboids are aligned horizontally in $\mathcal A_s$, one above the other. The cuboids are vertically aligned in $\mathcal A_g$. Every pair of start pose and goal pose then induces a (unique) dependency.
Clearly, this yields a bidirectional $K_6$ labeled dependency graph in the \lrbm case and a $K_{6, 6}$ unlabeled dependency graph in the \urbm case. For both, $n - 1 = 5$ objects must be moved to buffers before the problem can be resolved. The example readily generalizes to an arbitrary number of objects. 

\begin{proposition}
\mrb lower bound is $n - 1$ for $n$ objects for both \lrbm and \urbm, which is the maximum possible, even for uniform convex objects. 
\end{proposition}

\begin{figure}[h]
    \centering    \includegraphics[width=0.4\textwidth]{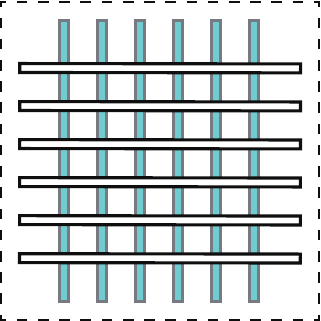}
    \caption{An instance with 6 cuboids where horizontal and vertical sets represent 
    start and goal poses, respectively. An arbitrary labeling of the objects may be 
    given in the labeled setting.}
    \label{fig:stick}
\end{figure}
%\jy{The figure can be updated to be thin rectangles/cuboids.}

The lower bound on \mrb being $\Omega(n)$ is undesirable, but it seems to depend on having objects that are thin; everyday objects are not often like that. An ensuing question of high practical value is then: what happens when the footprint of the objects is ``fat''? Next, we establish that, the lower bound drops to $\Omega(\sqrt{n})$ for uniform cylinders, which approximate many real-world objects/products. Furthermore, we show that this lower bound is tight for \urbm (in Section~\ref{subsec:upper}).

For convenience, assume $n$ is a perfect square, i.e., $n = m^2$ for some integer $m$. 
To establish the $\Omega(\sqrt{n})$ lower bound, a grid-like unlabeled dependency graph is used, which we call a \emph{dependency grid}, where $\mathcal A_s$ and $\mathcal A_g$ have similar patterns with $\mathcal A_g$ offset from $\mathcal A_s$ to the left (or right/above/below) by the length of one grid edge. An illustration of a portion of such a setup is given in \ref{fig:DependencyGrid}. We use $\D(w, h)$ to denote a dependency grid with $w$ columns and $h$ rows. 

\begin{lemma}\label{l:urbm-lower}
Given a \urbm instance with $n = m^2$ objects and whose dependency graph is $\D(m, 2m)$, its \mrb is lower bounded by $\Omega(m) = \Omega(\sqrt{n})$
\end{lemma}

Because the proof has limited relevance to the delivery of the main contributions of the paper, we highlight the main idea and refer the readers to the appendix for the complete proof. What we show is that, for certain unlabeled \toroe instance with $n$ objects with such a dependency graph, at least $\Omega(\sqrt{n})$ objects must be moved to buffers simultaneously to solve the \toroe instance.

\begin{figure}[h!]
    \centering
    \includegraphics[width=\textwidth]{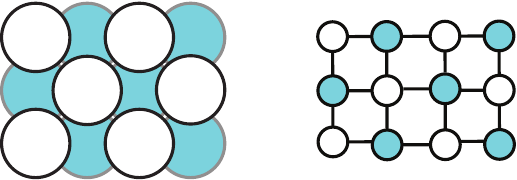}
    \caption{A \urbm instance (left) and its unlabeled dependency graph (right), a 
    $4\times 3$ dependency grid. Unshaded (resp., shaded) discs/vertices indicate start poses (resp., goal poses).}
    \label{fig:DependencyGrid}
\end{figure}

Because \urbm  \  instances always have lower or equal \mrb than the \lrbm  \  instances with the same objects and goal placements, the conclusion of \ref{l:urbm-lower} directly applies to \lrbm. Therefore, we have 

\begin{theorem}
For both \urbm and \lrbm with $n$ uniform cylinders, \mrb is lower bounded by $\Omega(\sqrt{n})$.
\end{theorem}

For uniform cylinders, while the lower bound on \urbm is tight (as shown in Section~\ref{subsec:upper}), we do not know whether the lower bound on \lrbm is tight; our conjecture is that $\Omega(\sqrt{n})$ is not a tight lower bound for \lrbm. Indeed, the $\Omega(\sqrt{n})$ lower bound can be realized when uniform cylinders are simply arranged on a cycle, an illustration of which is given in \ref{fig:lrbm-cycle}. 
For a general construction, for each object $o_i$, let $o_i$ depend on $o_{(i-1\mod n)}$ and $o_{(i+\sqrt{n}\mod n)}$, where $n$ is the number of objects in the instance. From the labeled dependency graph, we can construct the actual \lrbm instance where start and goal arrangements both form a cycle. 
It can be shown that when $n/2$ objects are at the goal poses, $\Omega(\sqrt{n})$ objects are 
at the buffer. We omit the proof, which is similar in spirit to that for \ref{l:urbm-lower}. 
\begin{figure}[h!]
    \centering
    \includegraphics[width=\textwidth]{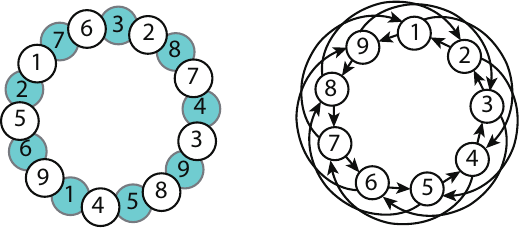}
    \caption{An example of a 9-object \lrbm yielding $\Omega(\sqrt{n})$ \mrb (left)
    and the corresponding dependency graph (right).}
    \label{fig:lrbm-cycle}
\end{figure}

\subsection{Upper Bounds on \mrb  \  for  \urbm}\label{subsec:upper}
We now establish, regardless of how $n$ uniform cylinders are to be 
rearranged, the corresponding \urbm instance admits a solution using 
$\mrb = O(\sqrt{n})$. Lower and upper bounds on \urbm agree;
therefore, the $O(\sqrt{n})$ bound is tight. 

To prove the upper bound, We propose an $O(n\log(n))$-time algorithm \spp for the 
setting based on a vertex separator of $\udg$.
\spp computes a sequence of goal vertices to be removed from the dependency graph.
Given a sequence of goal vertices to be removed, 
the running buffer size at each moment equals
$
\max(0,\|N(g,\udg)\|-\|g\|)
$,
where $g$ is the set of removed goal vertices at this moment,
and $N(g,\udg)$ is the set of neighbors of $g$ in $\udg$.
We prove that \spp can find a rearrangement plan 
with $ O(\sqrt{n})$ running buffers.

\begin{algorithm}
\begin{small}
    \SetKwInOut{Input}{Input}
    \SetKwInOut{Output}{Output}
    \SetKwComment{Comment}{\% }{}
    \caption{ \spp}
		\label{alg:spp}
    \SetAlgoLined
		\vspace{0.5mm}
    \Input{$\udg(V,E)$: unlabeled dependency graph}
    \Output{$\pi$: goal sequence}
		\vspace{0.5mm}
		$\pi,V,E\leftarrow$ RemovalTrivialGoals($\udg(V,E)$)\\
        \lIf{$V$ is $\emptyset$}{
        \Return $\pi$
        }
        $A, B, C \leftarrow$ Separator($\udg(V,E)$)\\
        $\pi \leftarrow \pi + g(C)$\\
        $A' \leftarrow A-N(g(C),\udg(V,E))$\\
        $B' \leftarrow $\ $B$$-N(g(C),\udg(V,E))$\\
		$\pi_{A'}\leftarrow $ \spp($\udg(A',E(A'))$)\\
		$\pi_{B'}\leftarrow $ \spp($\udg(B',E(B'))$)\\
		\lIf{ $\delta(A')\geq \delta(B')$}{
		    $\pi\leftarrow \pi + \pi_{A'} + \pi_{B'}$
		}
		\lElse{
		$\pi\leftarrow \pi + \pi_{B'} + \pi_{A'}$
		}
		\Return $\pi$\\
\end{small}
\end{algorithm}

The algorithm is presented in \ref{alg:spp}. 
\spp consumes a graph $\udg(V,E)$, which is a subgraph of $\udg$ induced by vertex set $V$.
To start with, the isolated goal vertices or those with only one dependency in $\udg(V,E)$ can be removed without using buffers (Line 1).
After that, $V$ can be partitioned into three disjointed subsets $A$, $B$ and $C$ \cite{lipton1979separator} (Line 3), 
such that there is no edge connecting vertices in $A$ and $B$, $|A|,|B|\leq 2|V|/3$, and $|C|\leq 2\sqrt{2|V|}$ (\ref{fig:separator}(a)). 
For the start vertices in $C$ and the neighbors of the 
goal vertices in $C$, we remove them from $\udg$. 
Since there are at most $5$ neighbors for each goal vertex, there are at most $10\sqrt{2|V|}$ objects moved to the buffer in this operation. 
After that, we remove the goal vertices 
in $C$,
which should be isolated now (Line 4). 
Function $g(\cdot)$ obtains the goal vertices in a given vertex set.
Let $A'$, $B'$ be the remaining vertices in $A$ and $B$ (Line 5-6). 
 And correspondingly, let $C'$ be the removed vertices, i.e., $C':=(A\bigcup B \bigcup C)\backslash (A'\bigcup B')$.
Function $N(\cdot,\cdot)$ obtains the neighbors of a vertex set in a given dependency subgraph.
With the removal of  $C'$ from $\udg$, $A'$ and $B'$ form two independent subgraphs (\ref{fig:separator}(b)). 
We can deal with the subgraphs one after the other by 
recursively calling \spp(\ref{fig:separator}(c)) (Line 7-8). 
Let $\delta(V'):= |g(V')|-|s(V')|$ where $g(V')$ and $s(V')$ are the goal and start vertices in a vertex set $V'$ respectively. 
Between vertex subsets $A'$ and $B'$, we prioritize the one with larger $\delta(\cdot)$ value(Line 9-10).  That is because, after solving a rearrangement subproblem induced by a vertex set $V'$, there is $\delta(V')$ fewer objects in buffers or $\delta(V')$ more available goal poses.

\begin{figure}[h!]
    \centering
\begin{overpic}
[width=\textwidth]{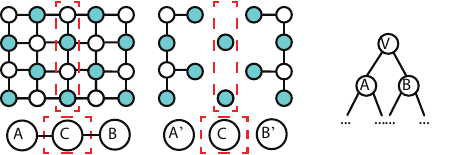}
\put(12,-5){(a)}
\put(46,-5){(b)}
\put(82,-5){(c)}
\end{overpic}
\vspace{4mm}
    \caption{The recursive solver \spp for \urbm. (a) A $O(\sqrt{|V|})$ vertex separator for the planar dependency graph. (b) By removing the start vertices in $C$ and the neighbors of the goal vertices in $C$, the remaining graph consists of two independent subgraphs and isolated goal poses in $C$. (c) The problem can be solved by recursively calling \spp.}
    \label{fig:separator}
\end{figure}

As shown in detail in the appendix, \spp guarantees an \mrb upper bound of  $\dfrac{20}{1-\sqrt{2/3}}\sqrt{n}$.

\begin{theorem}\label{t:urbm-upper}
For \urbm with $n$ uniform cylinders, a polynomial time algorithm can compute a
plan with $O(\sqrt{n})$ \rb, which implies that \mrb is  bounded by $O(\sqrt{n})$.
\end{theorem}

%\begin{proof}See \ref{sec:proofs}.\end{proof}

\section{Algorithms for \toroe}\label{sec:algorithms}
In this section, several algorithms for computing solutions for \toroe tasks under various \mrb objectives are described. The hardness result (\ref{t:mrb-toro-hard}) suggests that the challenge in computing \mrb solutions for \toroe resides with filtering through an exponential number of vertex linear orderings of the dependency graph. This leads to our algorithms being largely search-based, which are enhanced with many effective heuristics.

We first describe a dynamic programming-based method for 
\lrbm (\ref{subsec:dp}).
Then, we propose a priority queue-based method in \ref{subsec:pqs} for \urbm. 
Furthermore, a significantly faster 
depth-first modification of DP for computing \mrb is provided in 
\ref{subsec:dfsdp}.
Finally, we also developed an integer
programming model, denoted $\ilpmrb$, for computing the minimum total number 
of buffers needed subject to the \mrb constraint, which is compared 
with the algorithm that computes the minimum total buffer without 
the \mrb constraint, denoted as $\ilpfvs$, from \cite{han2018complexity}.
A brief description of $\ilpmrb$ is given in \ref{subsec:ilp}.

\subsection{Dynamic Programming (DP) for \lrbm}\label{subsec:dp}
As established in \ref{subsec:logvrb}, a rearrangement plan in \lrbm can be represented as a linear ordering of object labels.
That is, given an ordering of objects, $\pi$, we start with $o_{\pi(1)}$. 
If $x^g_{\pi(1)}$ is not occupied, then $o_{\pi(1)}$ is directly moved there. Otherwise, it is moved to a buffer location. We then continue with the second object in the order, and so on. After we work with each object in the given order, we always check whether objects in buffers can be moved to their goals, and do so if an opportunity is present. 
We now describe a \emph{dynamic programming} (DP) algorithm for computing such a linear ordering that yields the \mrb. 

%The strategies denoted by \StratStart and \StratGoal, are shown below. Given an object ordering $\pi$, \StratStart and \StratGoal interpret $\pi$ as an ordering of objects sorted by the increasing time of leaving the start poses and arriving at the goal poses respectively. Among the class of rearrangement plans represented by the same ordering, the one found by \StratStart follows the rule that objects in the buffer will go to the goal poses as soon as possible. Similarly, the plan found by \StratGoal follows the rule that objects will not go to the buffer until they have to.
%\begin{definition}[\StratStart]
%For each object $o_i$ in the order of $S$
%\begin{enumerate}
%    \item If $c^g_i$ is accessible, move $o_i$ to $c^g_i$. Otherwise, move $o_i$ to the buffer.
%    \item For each object $o_j$ in the buffer, move $o_j$ to $c^g_j$ if $c^g_j$ is accessible after the movement of $o_i$.
%\end{enumerate}
%\end{definition}

%\begin{definition}[\StratGoal]
%For each object $o_i$ in the order of $S$
%\begin{enumerate}
%    \item Move all the objects occupying $c^g_i$ to the buffer.
%    \item Move $o_i$ to $c^g_i$
%\end{enumerate}
%\end{definition}

%\begin{corollary}\label{cor:strategies}
%By enumerating all the object orderings, both \StratStart and \StratGoal can find the optimal %rearrangement plan.
%\end{corollary}
%The proof of Corollary~\ref{cor:strategies} is available in the extended version of this paper on arXiv. 

%\subsubsection{Dynamic Programming Approach}
The pseudo-code of the algorithm is given in \ref{alg:dp}. The algorithm maintains 
a search tree $T$, each node of which represents an arrangement where a set of objects
$S$ have left the corresponding start poses. We record the objects currently at the 
buffer ($T[S]$.b) and the minimum running buffer from the start arrangement 
$\mathcal{A}_s$ to the current arrangement ($T[S].\mrb$). The DP starts with an empty 
$T$. We let the root node represent $\mathcal A_s$ (Line 1). At this moment, there is
no object in the buffer and the \mrb is 0(Line 2-3). And then we enumerate all the arrangements with $|S|=$ 1, 2, $\cdots$ and finally $n$(Line 4-5). For arbitrary $S$, 
the objects at the buffer are the objects in $S$ whose goal poses are still occupied by 
other objects (Line 6), i.e., $\{o\in S | \exists o' \in \mathcal{O} \backslash S, 
(o, o')\in A\}$, where $A$ is the set of arcs in $\ldg$. $T[S].\mrb$, the minimum running
buffer from the root node $T[\emptyset]$ to $T[S]$, depends on the last object $o_i$ 
added into $S$ and can be computed by enumerating $o_i$ (Line 7-20):
$$
\begin{array}{l}
    T[S].\mrb = \displaystyle\min_{o_i\in S}\max( T[S\backslash \{o_i\}].\mrb, \\ 
    \qquad\qquad\qquad\qquad\qquad |T[S\backslash \{o_i\}].b|\ +\text{TC}(S\backslash \{o_i\},S)),
\end{array}
$$
where the \emph{transition cost} TC is given as 
$$
\text{TC}(S\backslash \{o_i\},S)=\begin{cases}1,\  o_i \in T[S].b,\\ 
0,\  otherwise,\end{cases}
$$
with $x^g_i$ currently occupied (Line 10), the transition cost is due to objects 
dependent on $o_i$ cannot be moved out of the buffer before moving $o_i$ to the 
buffer (Line 11). 
 Specifically, $T[S\backslash \{o_i\}].\mrb$ is the previous $\mrb$ and $|T[S\backslash \{o_i\}].b|\ +\text{TC}(S\backslash \{o_i\},S)$ is the running buffer size in the new transition. 
If $T[S].\mrb$ is minimized with $o_i$ being the last object in 
$S$ from the starts, then $T[S\backslash \{o_i\}]$ is the parent node of $T[S]$ 
in $T$ (Line 16-19). Once $T[\mathcal O]$ is added into $T$, $T[\mathcal{O}].\mrb$ 
is the \mrb of the instance (Line 23) and the path in $T$ from $T[\emptyset]$ to
$T[\mathcal{O}]$ is the corresponding solution to the instance.

\begin{algorithm}[h!]
\begin{small}
    \SetKwInOut{Input}{Input}
    \SetKwInOut{Output}{Output}
    \SetKwComment{Comment}{\% }{}
    \caption{Dynamic Programming for \lrbm}
		\label{alg:dp}
    \SetAlgoLined
		\vspace{0.5mm}
    \Input{$ \ldgg(\mathcal{O},A)$: labeled dependency graph}
    \Output{$\mrb$: the minimum number of running buffers}
		\vspace{0.5mm}
		$T.root \leftarrow \emptyset$ \\
		\vspace{0.5mm}
		$T[\emptyset].b \leftarrow\emptyset$ \Comment{{\footnotesize objects currently at the buffer}}
		\vspace{0.5mm}
		$T[\emptyset].\mrb \leftarrow 0$ \Comment{{\footnotesize current minimum running buffer}}
		\vspace{0.5mm}
		\For{$1\leq k\leq |\mathcal{O}|$}{
		\Comment{\footnotesize{enumerate cases where $k$ objects have left the start poses}}
		\For{$S\in \text{k-combinations of } \mathcal{O}$}{
		$T[S].b\leftarrow \{o\in S \mid \exists o' \in \mathcal{O} \backslash S, (o,o')\in A\}$\\
		\Comment{{\footnotesize Find the \mrb from $T[\emptyset]$ to $T[S]$}}
		$T[S].\mrb  \leftarrow \infty$\\
		\For{$o_i \in S$}{
		$parent = S\backslash \{o_i\}$\\
		\If{$o_i\in T[S].b$}{
		$RB \leftarrow \max(T[parent].\mrb$, $|T[parent].b|+1$)
		}
		\Else{
		$RB \leftarrow  T[parent].\mrb$ 
  %$\hcancel[black]{\max(T[parent].\mrb, |T[S].b|)}$
		}
		\If{RB$<T[S].\mrb$}{
		$T[S].\mrb \leftarrow RB$\\
		$T[S].parent \leftarrow$  $parent$
		}
		}
		}
		}
		\vspace{0.5mm}
		\Return $T[\mathcal{O}]$.\mrb\\
\end{small}
\end{algorithm}

For the \lrbm instance in \ref{fig:rbm-ex-prob}, \ref{Tab:DP} shows $T[S].\mrb$  with different last-object options when $S=\{o_2, o_5, o_6\}$. If the last object $o_i$ is $o_5$, then we need to move $o_5$ into the buffer before moving $o_6$ out of the buffer. Therefore, even though the buffer size of the parent node and the current node are both 2, there is a moment when all three objects are at the buffer. However, when we choose $o_2$ or $o_6$ as the last object to add, the $T[S].\mrb$ becomes 2.
% \begin{center}
\begin{table}[h!]
    \caption{\label{Tab:DP}$T[S].\mrb$ for different last objects ($[p]= [parent]$)}
    \centering
    \resizebox{0.5\textwidth}{!}{
    \begin{tabular}{p{4.0em}|p{4.0em}|p{3.4em}|p{3.4em}|p{3.8em}}
        \hline
        \small{Last object} & \small{$T[p].\mrb$} & \small{$T[p].b$} & \small{$T[S].b$} & \small{$T[S].\mrb$}\\
        \hline
        $o_2$ & 1 & \{$o_5$\} & \{$o_2$, $o_5$\} & 2\\
        $o_5$ & 2 & \{$o_2$, $o_6$\} & \{$o_2$, $o_5$\} & 3\\
        $o_6$ & 2 & \{$o_2$, $o_5$\} & \{$o_2$, $o_5$\} & 2\\
    \end{tabular}
    }
\end{table}
% \end{center}

\subsection{A Priority Queue-Based Method for \urbm}\label{subsec:pqs}
Similar to \lrbm, rearrangement plans for \urbm can be represented by a linear ordering of goal vertices in $\udg$. 
We can compute the ordering that yields \mrb by maintaining a search tree as we have done in \ref{alg:dp}. 
Each node $T[g]$ in the tree represents an arrangement where a set of goal vertices $g$ have been removed from $\udg$. 
The remaining dependencies of $T[g]$ is an induced graph of $\udg$, 
formed from $V(\udg)\backslash (g\cup N(g,\udg))$ where $V(\udg)$ is the vertex set of $\udg$, $N(g,\udg)$ is the neighbors of $g$ in $\udg$. 
 When the goal vertices $g$ are removed from $\udgg$, the running buffer size is $max(0,|N(g)|-|g|)$,  i.e., the number of objects cleared away minus the number of available goal poses. Specifically, when $|N(g)|>|g|$, some objects are in buffers; when $|N(g)|<|g|$, some goal poses are available; when $|N(g)|=|g|$, the cleared objects happen to fill all the available goal poses. 
Given an induced graph $I(g)$, denote the goal vertices with no more than one neighbor in $I(g)$ as \emph{free goals}. 
We make two observations. 
First, given an induced graph $I(g)$, we can always prioritize the removal of free goals without optimality loss.
Second, multiple free goals may appear after a goal vertex is removed. For example, in the instance shown in \ref{fig:rbm-ex-prob}(a), when the vertex representing $x^5_g$ is removed, $x^2_g$, $x^3_g$, and $x^4_g$ become free goals and can be added to the linear ordering in an arbitrary order.
 Denote nodes without free goals in the induced graph as \emph{key nodes}. 
In conclusion, the  key nodes
in the search tree may be sparse and enumerating  all possible nodes  with DP carries much overhead. 

As such, instead of exploring the search tree layer by layer like DP, we maintain 
a sparse tree with a priority queue $Q$. While each node still represents an 
arrangement, each edge in the tree represents  actions to clear out the next goal pose $x_i^g$. The corresponding child node of the edge represents the arrangement where $x_i^g$ and free goals are removed from the induced graph. 
%  We only add key nodes to $Q$, rather than all new nodes as A*. }
% We always pop out and develop the node with the smallest \mrb in $Q$. If a child node of the one that we develop already exists in the tree but is with a smaller \mrb than previously claimed, then we update the parent of the child node into the node we 
% are developing. 
% The \mrb of the node representing $\mathcal A_2$ sets an upper bound 
% of the solution and nodes in $Q$ with larger \mrb will be pruned away. 
% The algorithm terminates when $Q$ is empty.
%
 Similar to A*, we always pop out the key node with the smallest \mrb in $Q$. 
We denote this priority queue-based search method \PQS
% , which can also be viewed as a variation of A*
.

\subsection{Depth-First Dynamic Programming}\label{subsec:dfsdp}
Both \lrbm and \urbm can be viewed as solving a series of decision problems, i.e., 
asking whether we can find a rearrangement plan with $k= 1, 2, \ldots$ running buffers. 
As dynamic programming is applied to solve such decision problems, instead of 
performing the more standard breadth-first exploration of the search tree, we 
identified that a depth-first exploration is much more effective. 
We call this variation of dynamic programming \DFSDP, which is a fairly 
straightforward alteration of a standard DP procedure. 
%\DFSDP \cite{wang2021uniform}.
%
The high-level structure of \DFSDP is described in \ref{alg:dfdp}. 
It consumes a dependency graph and returns the minimum running buffer size for the corresponding instance. 
Essentially, \DFSDP fixes the running buffer size RB and checks whether there is a plan requiring 
no more than RB running buffers. As the search tree (see \ref{subsec:dp}) 
is explored, depth-first exploration is used instead of breadth-first (Line 4). 

\begin{algorithm}[ht!]
\begin{small}
    \SetKwInOut{Input}{Input}
    \SetKwInOut{Output}{Output}
    \SetKwComment{Comment}{\% }{}
    \caption{ Dynamic Programming with Depth-First Exploration}
		\label{alg:dfdp}
    \SetAlgoLined
		\vspace{0.5mm}
    \Input{$ \ldgg$: labeled dependency graph}
    \Output{RB: the minimum number of running buffers}
		\vspace{0.5mm}
		RB$\leftarrow 0$\\
		\While{not time exceeded}{
		 $T$.root $ \leftarrow \emptyset$ \\
		$T\leftarrow$ Depth-First-Search($T$,$\emptyset$, RB, $ \ldgg$)\\
		\lIf{ $\mathcal O \in T$}{\Return RB}
		\lElse{RB+=1}
		}
\end{small}
\end{algorithm}

The details of the depth-first exploration are shown in \ref{alg:dfs}.
Similar to DP, each node in the tree represents an object state indicating whether an object is at the start pose, the goal pose, or external buffers.
And as described in \ref{subsec:logvrb}, essential object states can be represented by the set of objects $S$ that have picked up from start poses.
If $S$ is $\mathcal O$, then all the objects are at the goal poses and we find a path on the search tree from the start state to the goal state (Line 1).
Given the set of objects away from the start poses $S$, we can get the sets of objects at start poses $S_S$, goal poses $S_G$, and external buffers $S_B$ (Line 2). Specifically, $S_S=\mathcal O\backslash S$, $S_G$ are the objects in $S$ that have no dependency in $S_S$, and $S_B$ are the objects in $S$ that have dependencies in $S_S$.
We further explore child nodes of the current state by moving one object away from the start pose (Line 3).
By checking the dependencies in $ \ldgg(\mathcal{O},A)$, we determine whether this object should be moved to buffers or its goal pose (Line 4).
The transition from $S$ to $S\bigsqcup \{o\}$ fails if the external buffers are overloaded or the child node has been explored (Line 5).
Otherwise, we can add $S\bigsqcup \{o\}$ into the tree (Line 7) and explore the new node (Line 8).
The algorithm returns a tree.

\begin{algorithm}[ht!]
\begin{small}
    \SetKwInOut{Input}{Input}
    \SetKwInOut{Output}{Output}
    \SetKwComment{Comment}{\% }{}
    \caption{Depth-First-Search}
		\label{alg:dfs}
    \SetAlgoLined
		\vspace{0.5mm}
    \Input{ $T$: Search Tree; $S$: The set of objects away from start poses; $RB$: Running buffer size; $ \ldgg(\mathcal{O},A)$: labeled dependency graph}
    \Output{$T$}
		\vspace{0.5mm}
		\lIf{$S$ is $\mathcal{O}$}{\Return $T$}
		$S_S$, $S_G$, $S_B$ = GetState($S$)\\
		\For{$o\in S_S$}{
		$ToBuffer$=CollisionCheck($S_S\backslash \{o\}$,$o$, $ \ldgg$)\\
		\lIf{(ToBuffer and $\|S_B\|+1>$RB) or ($S\bigsqcup \{o\}\in T$)}{\Return $T$}
		\Else{
		AddNode($T$, $S$, $S\bigsqcup \{o\}$)\\
		$T$ $\leftarrow$ Depth-First-Search($T$, $S\bigsqcup \{o\}$, RB, $ \ldgg$)\\
		\lIf{
		$\mathcal O\in T$}{\Return $T$}
		}
		}
		\Return $T$
\end{small}
\end{algorithm}

The intuition is that, when there are many rearrangement plans on the search tree that do not use more than $k$ running buffers, a depth-first search will quickly 
find such a solution, whereas a standard DP must grow the full search tree 
before returning a feasible solution. 
A similar depth-first exploration heuristic is used in \cite{wang2021uniform}. 
%With \DFSDP, we still develop a search tree as described in \ref{subsec:dp}, 
%each node of which represents an arrangement with a certain set of objects moved 
%away from the start poses. When a search node is added to the tree, we only need 
%to check in linear time whether the transition from the parent arrangement to the 
%child arrangement can be done with $k$ running buffers. We develop the tree in a 
%depth first manner until either the node representing $\mathcal A_2$ is added into 
%the tree or all the leaf nodes in the tree are proven to be dead ends. T

\subsection{Minimizing Total Buffers Subject to \mrb Constraints}\label{subsec:ilp}

 We also construct a Mixed Integer Programming(MIP) model minimizing \mrb or total buffers.
Let binary variables $c_{i,j}$ represent  the dependency graph $\ldg$: $c_{i,j}=1$ if and only if $(i,j)$ is in the arc set of $\ldg$. Let $y_{i,j}(1\leq i<j\leq n)$ be the binary 
sequential variables: $y_{i,j}=1$ if and only if $o_i$ moves out of the start pose before $o_j$.
 $y_{i,j}$ are used to represent the ordering of actions.
% \mrb can be expressed based on three constraints: 
% First, \mrb is at least the size of the running buffer at any moment; 
% Second, an object $o_j\in \mathcal{O}$ is at the goal pose if and only 
% if all the objects $o_k\in \mathcal{O}$ with $c_{j,k}=1$ have left the start 
% poses. Third, an object $o_j\in \mathcal{O}$ is at the buffer if and only 
% if $o_j$ is neither at the start pose nor at the goal pose.
%
We further introduce two sets of binary variables $g_{i,j}$ and $b_{i,j}(1\leq i,j\leq n)$ to indicate object positions at each moment. 
$g_{i,j}=1$ indicates that $o_j$ has no dependency on other objects when moving $o_i$ from the start pose. In other words, the goal pose of $o_j$ is available at the moment. $b_{i,j}=1$ indicates that $o_j$ stays at the buffer after moving $o_i$ away from the start pose.
Finally, binary variables $B_i=1$ if and only if $o_i$ is moved to a buffer at some point.
The objective function consists of two terms: 
the total buffer term and running buffer term.
The total buffer term, scaled by $\alpha$, counts the number of objects that need buffer locations.
\mrb is represented with an integer variable $K$ and scaled by $\beta$.
To minimize total buffers subject to \mrb constraints,
we set $\alpha=1,\beta=n$.
The objective function is adaptable to different demands on rearrangement plans. 
Specifically, when $\alpha=0$,$\beta > 0$,
the MIP model minimizes \mrb.
When $\alpha > 0$, $\beta =0$, the MIP model minimizes total buffers, 
i.e. total actions in the rearrangement plan.
When $\alpha/\beta >n-1$, the MIP model first minimizes total buffers, 
and then minimizes running buffers.
When $\beta/\alpha >n-1$, the MIP model first minimizes running buffers, 
and then minimizes total buffers.

In the MIP model, Constraints~\ref{eq:c1} imply the rules for sequential variables  to make sure a valid ordering of indices $1,\dots,n$ can be encoded from $y_{i,j}$.
Constraints~\ref{eq:c2} imply that $B_j=1$ if $o_j$ has been to buffers in the plan.
Constraints~\ref{eq:c3} imply that \mrb $K$ is lower bounded by the maximum number of objects concurrently placed in buffers.
With Constraints~\ref{eq:c4} and \ref{eq:c5}, $g_{i,j}=0$ if and only if $o_j$ depends on an object $o_k$ which is still at the start pose when $o_i$ is moved.
With Constraints~\ref{eq:c6}-\ref{eq:c8}, 
$b_{i,j}=1$ if and only if $o_j$ is moved before $o_i$ and the goal pose is still unavailable when $o_i$ is moved from the start pose.
%\jy{``constraint x'' $\to$ ``Constraint x''}

\begin{equation}
    \arg \min \alpha [\sum_{i=1}^{n}B_i]+\beta K
\end{equation}
\begin{equation}\label{eq:c1}
    0\leq y_{i,j}+y_{j,k}-y_{i,k}\leq 1 \ \ \  \forall 1\leq i < j < k \leq n
\end{equation}
% \begin{equation}\label{eq:c2}
%     y_{i,i}=1 \ \ \  \forall 1 \leq i \leq n
% \end{equation}
\begin{equation}\label{eq:c2}
    B_j \geq \sum_{1\leq i\leq n} \dfrac{b_{i,j}}{n}\ \ \  \forall 1\leq j\leq n
\end{equation}
\begin{equation}\label{eq:c3}
    K \geq \sum_{1\leq j\leq n} b_{i,j}\ \ \  \forall 1\leq i\leq n
\end{equation}
\begin{equation}\label{eq:c4}
    \begin{split}
    \sum_{1\leq k < i} \dfrac{c_{j,k}(1-y_{k,i})}{n} + \sum_{i < k \leq n} \dfrac{c_{j,k}y_{i,k}}{n} \leq 1-g_{i,j} \\
    \forall 1\leq i,j \leq n
    \end{split}
\end{equation}
\begin{equation}\label{eq:c5}
    \begin{split}
    1-g_{i,j}\leq \sum_{1\leq k < i} c_{j,k}(1-y_{k,i}) + \sum_{i < k \leq n} c_{j,k}y_{i,k}\\ 
    \forall 1\leq i,j \leq n
    \end{split}
\end{equation}
\begin{equation}\label{eq:c6}
\begin{split}
    \dfrac{g_{i,j}+y_{i,j}}{2}\leq 1-b_{i,j} \leq g_{i,j}+y_{i,j}\\
    \forall 1\leq i<j \leq n
\end{split}
\end{equation}
\begin{equation}\label{eq:c7}
\begin{split}
    \dfrac{g_{j,i}+(1-y_{i,j})}{2}\leq 1-b_{j,i} \leq g_{j,i}+(1-y_{i,j})\\
    \forall 1\leq i<j \leq n
\end{split}
\end{equation}
\begin{equation}\label{eq:c8}
b_{i,i}=1-g_{i,i}
\end{equation}

% \vspace{2mm}
% The constraints can be added into the \fvs ILP formulation \cite{han2018complexity} to 
% find the minimum total buffer size with at most $k$ running buffers. We denote this method
% as $\ilptb$ when $k = \mrb$, and the \fvs method from \cite{han2018complexity} as $\ilpfvs$.

%Therefore, we have 
%two ILP solutions, $\ilpmrb$, which computes solutions optimizing \mrb, and $\ilpfvs$
%which computes the minimum number of total buffers subject to having minimum \mrb.
%Note that even though each vertex $o_i$ are split into $o^{in}_i$ and $o^{out}_i$ in that formulation, the object ordering is indicated by the ordering of $\{o^{in}_i\}_{i=1}^n$.
%\kg{Here I use the ordering of the in-vertices to represent the original object ordering}

\section{Experimental Results}
\label{sec:rbm-experiments}
Our evaluation focuses on uniform cylinders, given their prevalence in 
practical applications. 
For simulation studies, instances with different object densities are created, as measured 
by \emph{density level} $\rho := n\pi r^2/(h*w)$, where $n$ is the number of objects and 
$r$ is the base radius. $h$ and $w$ are the height and width of the workspace.
%\jy{Kai: please update $D$ to $\rho$, which is a more standard symbol for density.}
In other words, $\rho$ is the proportion of the tabletop surface occupied by objects. 
%We notice that instances with a fixed $D$ have roughly the same number of dependencies on average for each object, regardless of the number of objects in the environment. 
%For example, when $D=0.3$, each object has averaging 1.25 dependencies. 

The evaluation is conducted on both random object placements and manually 
constructed difficult setups (e.g., dependency grids with $MRB = \Omega(\sqrt{n}$)).
For generating test cases with high $\rho$ value, we invented a physic engine 
(we used Gazebo \cite{koenig2004design}) based approach for doing so. Within a rectangular box, we sample 
placements of cylinders at lower density and then also sample locations for some 
smaller ``filler'' objects (see \ref{fig:compression}, left). From here, 
one side of the box is pushed to reach a high-density setting 
(\ref{fig:compression}, right), which is very difficult to generate via random 
sampling. By controlling the ratio of the two types of objects, different density 
levels can be readily obtained. \ref{fig:density} shows three random object placements 
for $\rho = 0.2, 0.4$ and $0.6$.

\begin{figure}[h!]
    \centering
    \includegraphics[width=\textwidth]{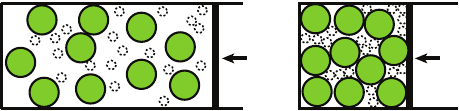}
    \vspace{1mm}
    \caption{Generating dense instances using a physics-engine-based simulator
    through compression of the left scene to the right scene.}
    \label{fig:compression}
\end{figure}

\begin{figure}[h!]
    \centering
    \includegraphics[width=\textwidth]{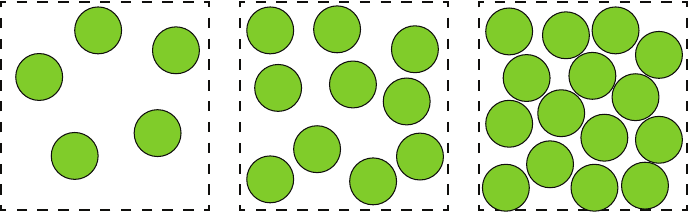}
    \caption{Unlabeled arrangements with $\rho=0.2, 0.4, 0.6$ respectively.}
    \label{fig:density}
\end{figure}

From two randomly generated object placements with the same $\rho$ and $n$ values, a 
\urbm instance can be readily created by superimposing one over the other. \lrbm
instances can be generated from \urbm instances by assigning each object a random 
label in $[n]$ for both start and goal configurations.

%In random scenario, we consider cylindrical objects, whose bottoms are unit discs, placed in a square environment without collisions. 
%The size of the square environment are determined by the density level $D$ and the number of objects $n$. 
%Unlabeled arrangements are generated under uniform distribution of objects: Object centroids are assigned one by one following the uniform distribution in the workspace.
%An object is placed if the assigned location is out of collision with previous-placed objects. Otherwise, other randomly generated locations will be assigned to it until a collision-free location is found. 
%For each instance, $\mathcal A_s$ and $\mathcal A_g$ are randomly chosen from the generated unlabeled arrangements and in \lrbm, the labels of the each object are also randomly assigned.
%
%When the desired density level is too high, the generation process becomes slow: collision-free positions are hard to find for objects at the bottom of the list. In this case, we generate the arrangements by compressing sparse instances(\ref{fig:conpression}) in Gazebo, a simulation software\cite{koenig2004design}. To enrich the randomness of the object positions, smaller cylindrical obstacles are temporarily added into the environment so that the density is balanced all over the environment.
%
%In the special scenario, we evaluate the special cases discussed in the previous sections whose $\mrb=\Theta(\sqrt{n})$.

\begin{table}[h]
\caption{Evaluated methods for \toroe. }
\label{tab:toroe}
% \resizebox{\columnwidth}{!}{
\begin{tabular}{ |p{0.25\columnwidth}|p{0.65\columnwidth}| } 
 \hline
 Problems & Methods \\ 
 \hline
 \lrbm & \DFSDP, DP, $\ilpfvs$, $\ilptb$ \\ 
 \hline
 \urbm & \DFSDP, \PQS \\ 
 \hline
\end{tabular}
% }
\end{table}

The proposed algorithms are implemented in Python and all experiments are executed 
on an Intel$^\circledR$ Xeon$^\circledR$ CPU at 3.00GHz. For solving  MIP, Gurobi 9.16.0 \cite{gurobi} is used.

\subsection{\lrbm over Random Instances }
In \ref{fig:LabeledAlgorithms}, we compare the effectiveness of DP and \DFSDP, in terms of computation time and success rate, for different densities. 
Each data point is the average of $30$ test cases minus the unfinished ones, if any, 
subject to a time limit of $300$ seconds per test case. For \lrbm, we are able to 
push to $\rho = 0.4$, which is fairly dense. 
The results clearly demonstrate that \DFSDP significantly outperforms the baseline
DP.
Based on the evaluation, both methods can be used to tackle practical-sized problems 
(e.g., tens of objects), with \DFSDP demonstrating superior efficiency and robustness. 
\begin{figure}[h!]
    \vspace{4mm}
\centering
    \begin{overpic}[width=\textwidth]{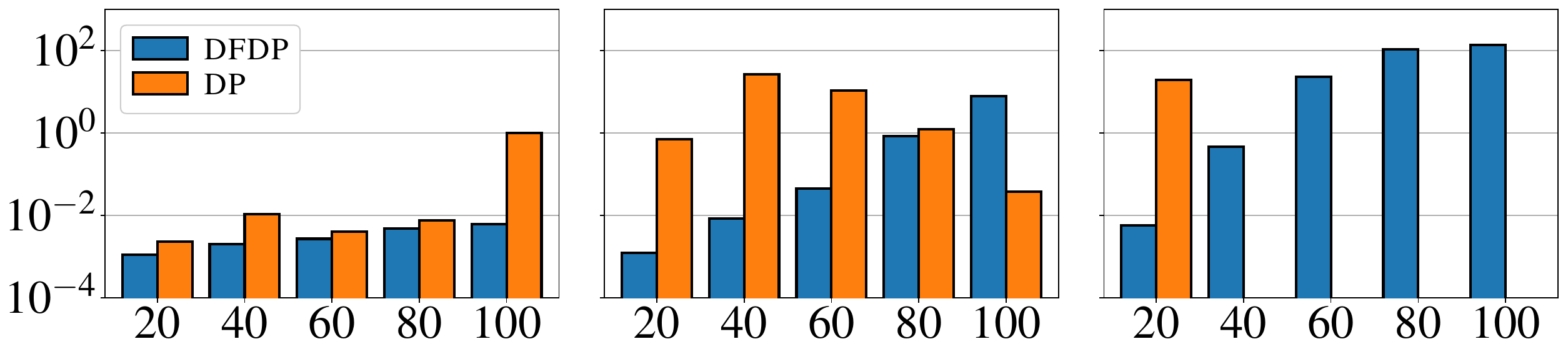}
    \put(0.5,-22.5){ \includegraphics[width=\textwidth]{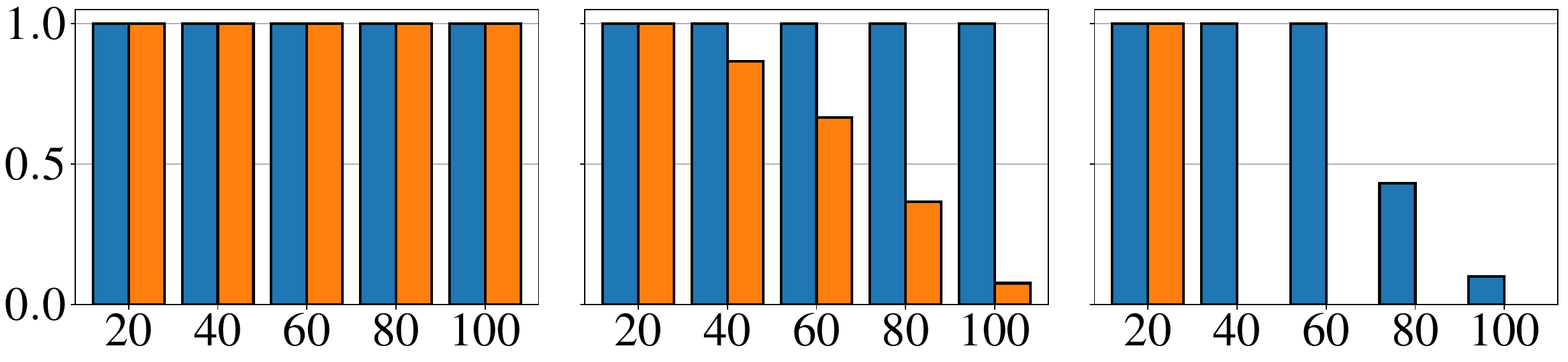}}
    \end{overpic}
    \vspace{28mm}
    \caption{Performance of \DFSDP and DP over \lrbm. The top row shows the average
    computation time (s) and the bottom row the success rate, for density levels
    $\rho=0.2, 0.3$, $0.4$, from left to right. The $x$-axis denotes the number of 
    objects involved in a test case.}
    \label{fig:LabeledAlgorithms}
\end{figure}

The actual \mrb sizes for the same test cases from \ref{fig:LabeledAlgorithms} 
are shown in \ref{fig:LabeledResults}
on the left. We observe that \mrb is rarely very large even for fairly large \lrbm
instances. The size of \mrb appears correlated to the size of the largest connected
component of the underlying dependency graph, shown in \ref{fig:LabeledResults}
on the right.
\begin{figure}[h!]
    \centering
\begin{overpic}
[width=0.48\textwidth]{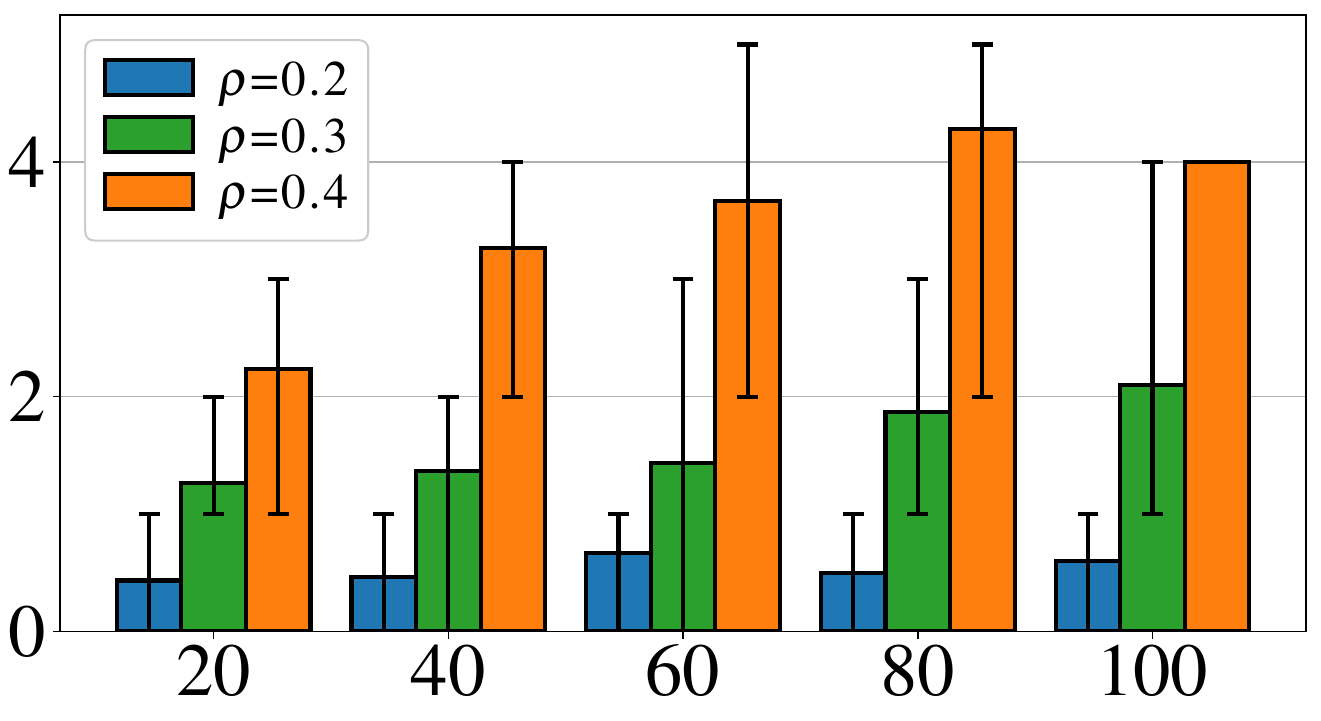}
\end{overpic}
\begin{overpic}
[width=0.48\textwidth]{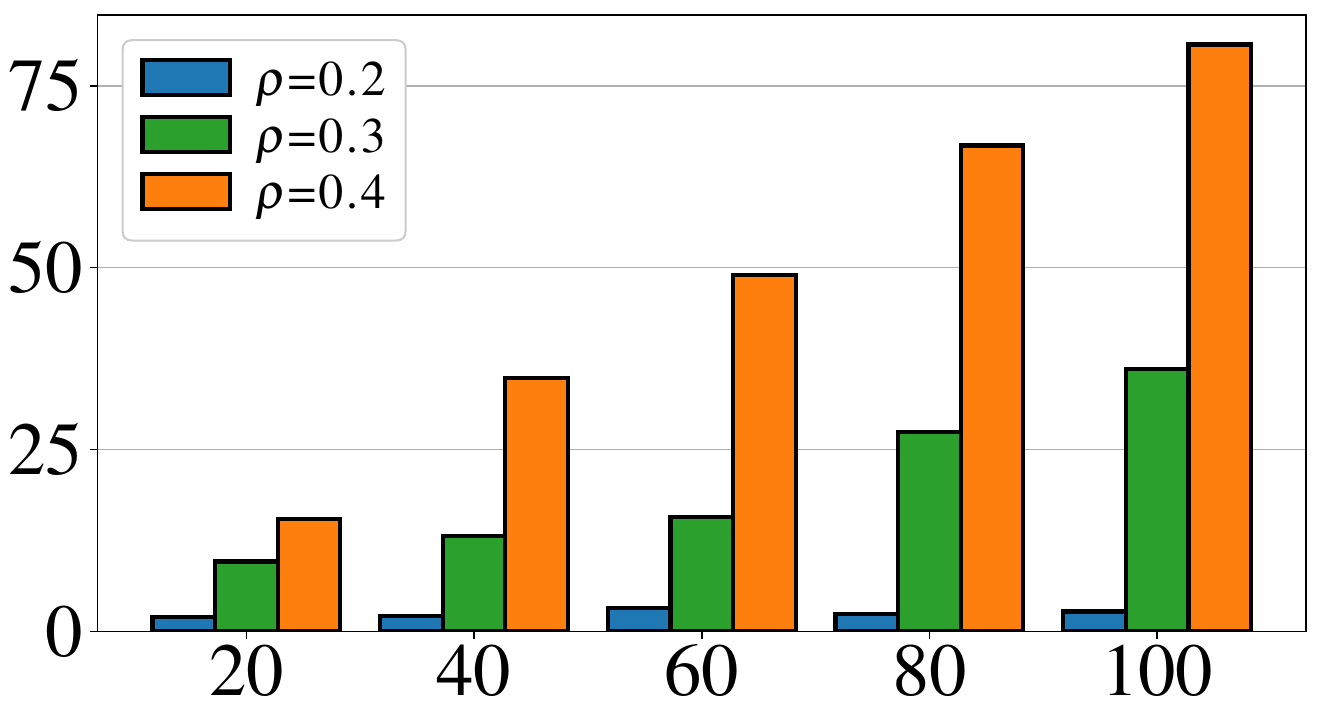}
\end{overpic}
    \caption{For \lrbm instances with $\rho = 0.2$-$0.4$ and $n=20$-$100$, the 
    left figure shows average \mrb size and range. The right figure shows
    the size of the largest connected component of the dependency graph.}
    \label{fig:LabeledResults}
\end{figure}
% \jy{We should probably get rid of the range for the second 
% figure in \ref{fig:LabeledResults}; otherwise we probably 
% need to add to all other figures...}

For $\lrbm$ with $\rho = 0.3$ and $n$ up to $50$, we computed the \fvs sizes 
using $\ilpfvs$ (which does not scale to higher $\rho$ and $n$) and compared that 
with the \mrb sizes, as shown in \ref{fig:MRBFVS} (a). We observe that the
\fvs is about twice as large as \mrb, suggesting that \mrb provides more reliable 
information for estimating the design parameters of pick-n-place systems. For 
these instances, we also computed the total number of buffers needed subject to 
the \mrb constraint using $\ilptb$. Out of about $150$ instances, only $1$ showed 
a difference as compared with \fvs (therefore, this information is not shown in 
the figure). In \ref{fig:MRBFVS} (b), we provided a computation time 
comparison between $\ilpfvs$ and $\ilptb$, showing that $\ilptb$ is practical,
if it is desirable to minimize the total buffers after guaranteeing the 
minimum number of running buffers. 

\begin{figure}[h!]
    \centering
    \begin{overpic}[width=\textwidth]{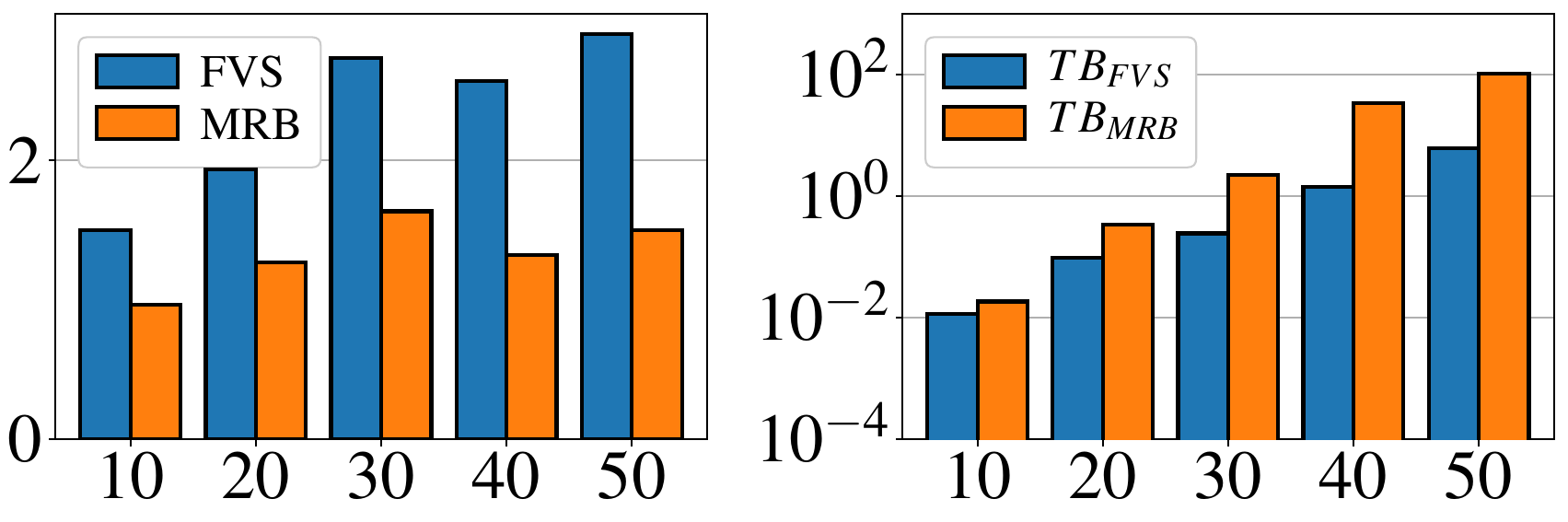}
    \put(22,-3.5){{\small (a)}}
    \put(76,-3.5){{\small (b)}}
    \end{overpic}
    \vspace{2mm}
    \caption{(a) Comparison between size of \mrb and \fvs. (b)  Computation time 
    comparison between $\ilpfvs$ and $\ilptb$. }
    \label{fig:MRBFVS}
    \vspace{-1mm}
\end{figure}

Considering our theoretical findings and the evaluation results, an important 
conclusion can be drawn here is that \mrb is effectively a small constant for
random instances, even when the instances are very large. Also, 
minimizing the total number of buffers used subject to \mrb constraint can 
be done quickly for practical-sized problems. 

\subsection{\urbm over Random Instances}
For \urbm, we carry out a similar performance evaluation as we have done for \lrbm. 
Here, \PQS and \DFSDP are compared. For each combination of $\rho$ and $n$, $100$ 
random test cases are evaluated. Notably, we can reach $\rho = 0.6$ with relative
ease. From \ref{fig:UnlabeledAlgorithms}, we observe that \DFSDP is more 
efficient than \PQS, especially for large-scale dense settings. In terms of the \mrb 
size, all instances tested have an average \mrb size between $0$ and $0.7$, which is 
fairly small (\ref{fig:UnlabeledResults}). Interestingly, we witness a decrease of 
\mrb as the number of objects increases, which could be due to the lessening 
``border effect'' of the larger instances. That is, for instances with fewer 
objects, the bounding square puts more restriction on the placement of the 
objects inside. For larger instances, such restricting effects become smaller.
We mention that the total number of buffers for random \urbm cases subject to 
\mrb constraints are generally very small. 
%The right figure in \ref{fig:UnlabeledResults} shows the 
%number of running buffers we can get "for free" before we do the first branching.

\begin{figure}[h!]
    \centering
    \begin{overpic}[width=\textwidth]{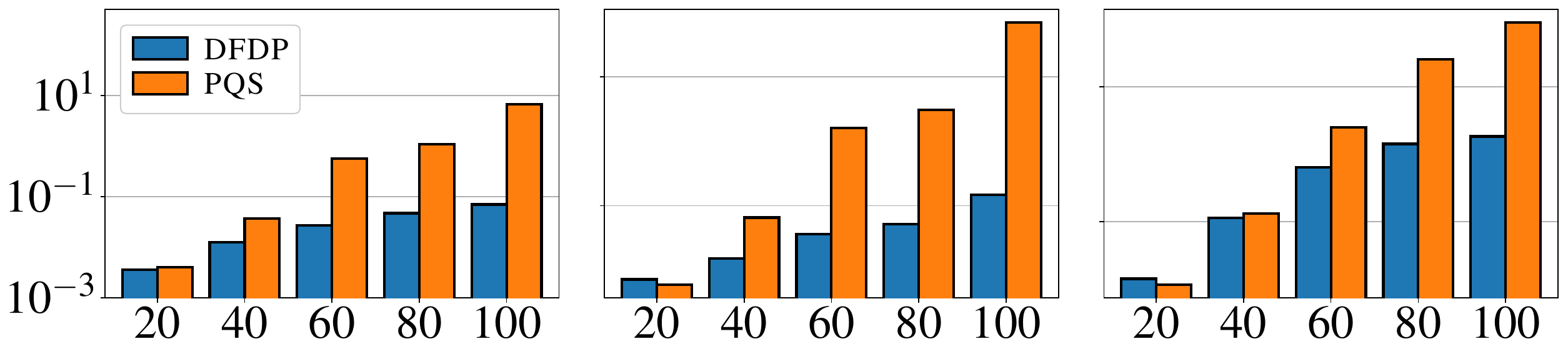}
    \put(0.5,-22.5){ \includegraphics[width=\textwidth]{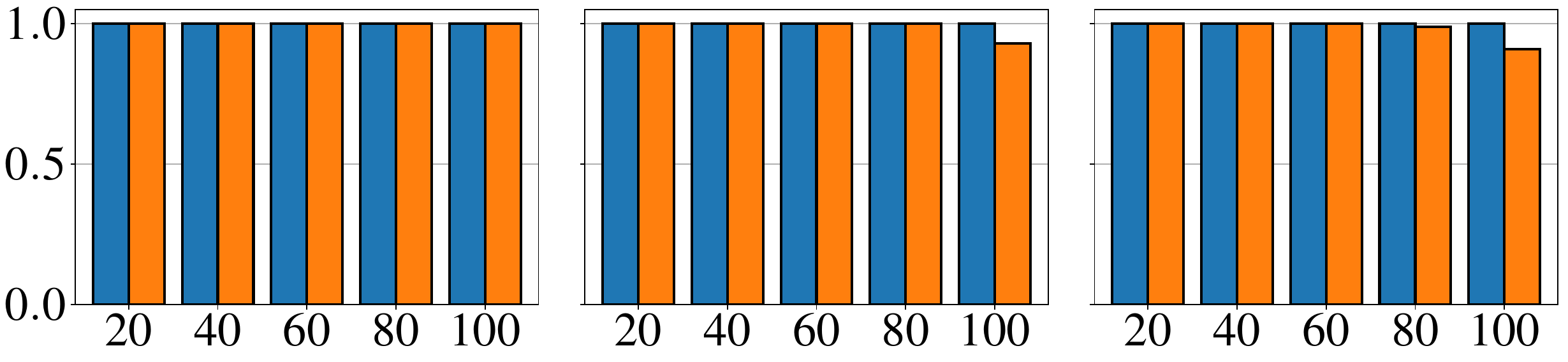}}
    \end{overpic}
        \vspace{32mm}
\caption{Performance of \DFSDP and \PQS over \urbm. The top row shows the average
    computation time and the bottom row shows the success rate, for density levels
    $\rho=0.4, 0.5, 0.6$, from left to right.}
    \label{fig:UnlabeledAlgorithms}
\end{figure}

%\begin{comment}
\begin{figure}[h!]
    \centering
\includegraphics[width=0.7\textwidth]{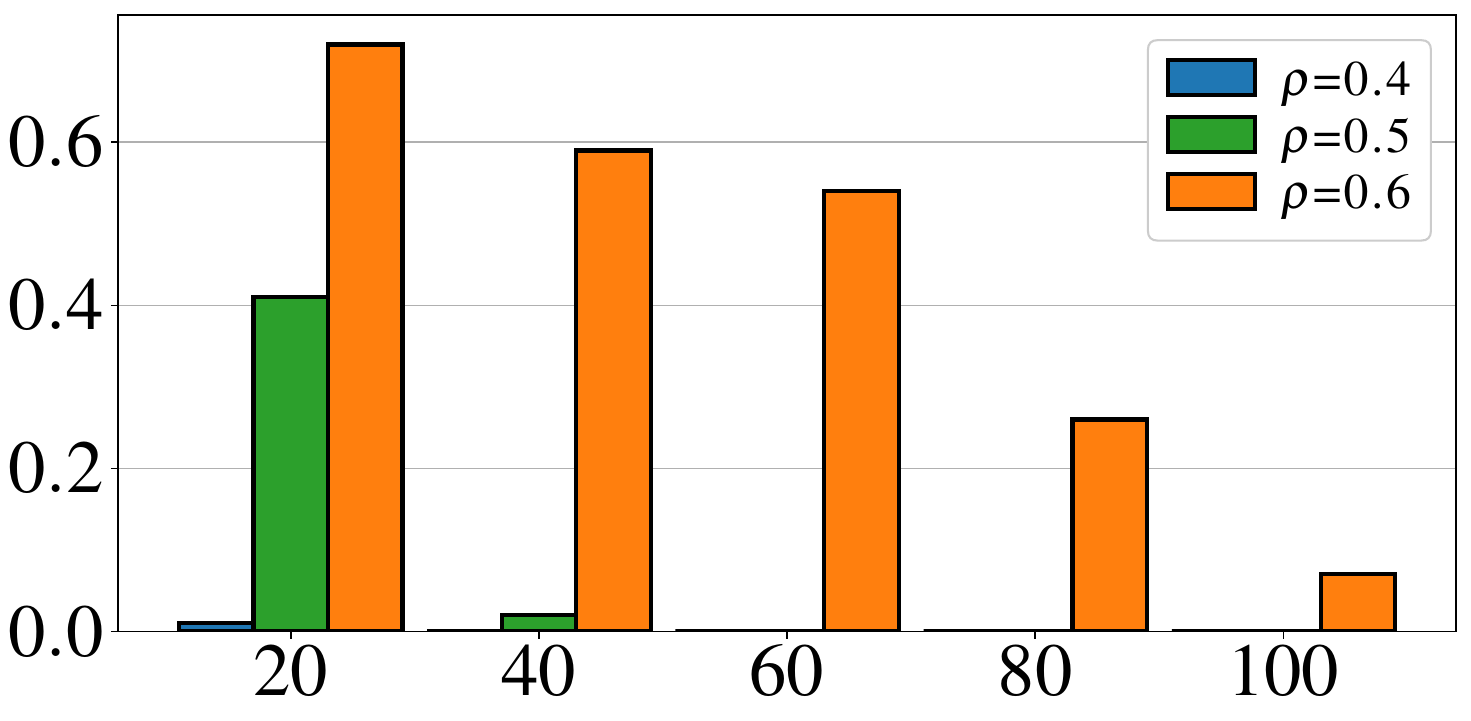}
    \caption{Average \mrb size for \urbm instances with $\rho=0.4-0.6$ and $n=20-100$.
    For $\rho = 0.4$ and $0.5$, the $\mrb$ sizes are near zero as the number of 
    objects goes beyond $20$.}
    \label{fig:UnlabeledResults}
    \vspace{-1mm}
\end{figure}
%\jy{Make \ref{fig:UnlabeledResults} shorter, similar to Fig. 14.}
%\end{comment}

\subsection{Manually Constructed Difficult Cases of \toroe}
In the random scenario, the running buffer size is limited. In particular, for 
\lrbm, the dependency graph tends to consist of multiple strongly connected 
components that can be dealt with independently. We further show the performance 
of \DFSDP on the instances with $\mrb=\Theta(\sqrt{n})$. We evaluate three kinds 
of instances: (1) \textsc{UG}: $m^2$-object \urbm instances whose $\udg$ are 
dependency grid $\mathcal D(m,2m)$ (e.g., \ref{fig:DependencyGrid}); (2) 
\textsc{LG}: $m^2$-object \lrbm instances whose start and goal arrangements are 
the same as the instances in (1). (3) \textsc{LC}: $m^2$-object \lrbm instances 
with objects placed on a cycle (\ref{fig:lrbm-cycle}). The computation time and 
the corresponding \mrb are shown in \ref{fig:SpecialResults}. For 
\textsc{LG} instances, the labels are randomly assigned. We try 30 test cases and 
then plot out the average. We observe that the \mrb is much larger for these 
handcrafted instances as compared with random instances with similar density and
number of objects.

\begin{figure}[h!]
    \centering
    \begin{overpic}[width=\textwidth]{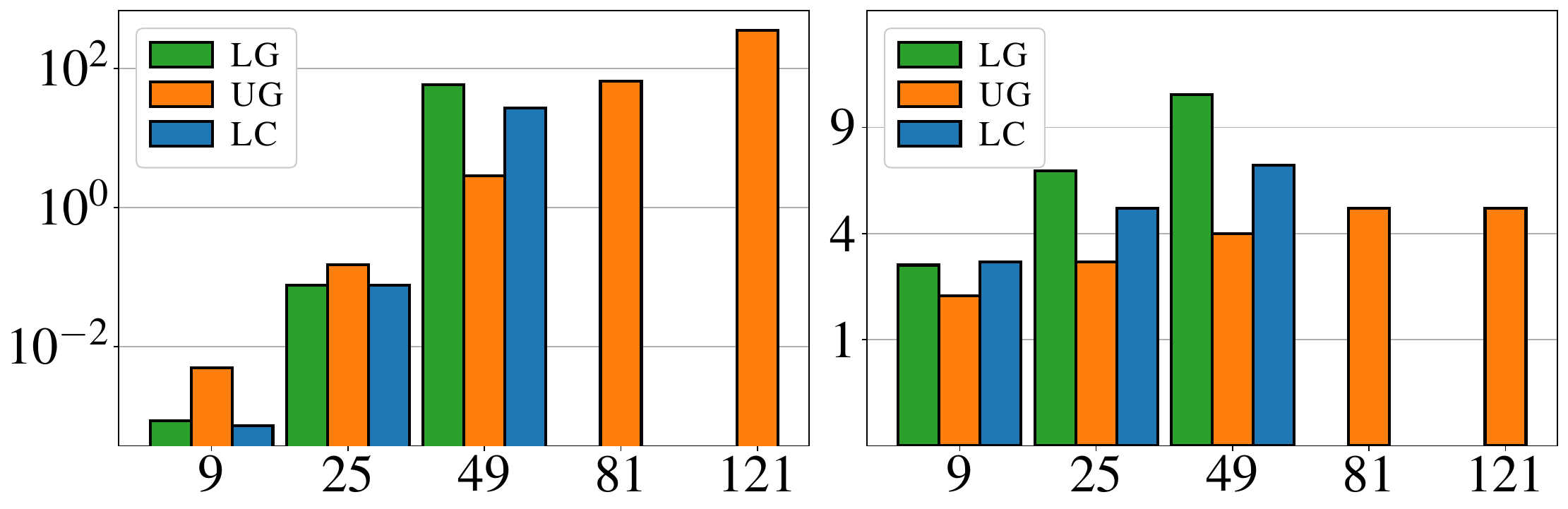}
    \end{overpic}
    %\vspace{-2mm}
    \caption{ For handcrafted cases and different numbers of objects, the left 
    figure shows the computation time by \DFSDP and the right figure the resulting
    \mrb size.}
    \label{fig:SpecialResults}
\end{figure}
%\jy{Why do we have a case with one object? remove it. Also, the figure can be shorter.}

%%%%%%%%%%%%%%%%%%%%%%%%%%%%%%%%%%%%%%%%%%%%%%%%%%%%%%%%%%%%%%%%%%%%%%%%%%%%%%%%%%%
%%%%%%%%%%%%%%%%%%%%%%%%%%%%%%%%%%%%%%%%%%%%%%%%%%%%%%%%%%%%%%%%%%%%%%%%%%%%%%%%%%%
%%%%%%%%%%%%%%%%%%%%%%%%%%%%%%%%%%%%%%%%%%%%%%%%%%%%%%%%%%%%%%%%%%%%%%%%%%%%%%%%%%%
%%%%%%%%%%%%%%%%%%%%%%%%%%%%%%%%%%%%%%%%%%%%%%%%%%%%%%%%%%%%%%%%%%%%%%%%%%%%%%%%%%%
%%%%%%%%%%%%%%%%%%%%%%%%%%%%%%%%%%%%%%%%%%%%%%%%%%%%%%%%%%%%%%%%%%%%%%%%%%%%%%%%%%%
%%%%%%%%%%%%%%%%%%%%%%%%%%%%%%%%%%%%%%%%%%%%%%%%%%%%%%%%%%%%%%%%%%%%%%%%%%%%%%%%%%%
%%%%%%%%%%%%%%%%%%%%%%%%%%%%%%%%%%%%%%%%%%%%%%%%%%%%%%%%%%%%%%%%%%%%%%%%%%%%%%%%%%%
%%%%%%%%%%%%%%%%%%%%%%%%%%%%%%%%%%%%%%%%%%%%%%%%%%%%%%%%%%%%%%%%%%%%%%%%%%%%%%%%%%%
%%%%%%%%%%%%%%%%%%%%%%%%%%%%%%%%%%%%%%%%%%%%%%%%%%%%%%%%%%%%%%%%%%%%%%%%%%%%%%%%%%%

\section{Physical Experiments}\label{sec:hardware}
In this section, we demonstrate that the plans computed by proposed algorithms can be readily executed on real robots in a complete vision-planning-control pipeline.
We first introduce the hardware setup and the pipeline during the execution of the rearrangement plans.
After that, we present experimental results based on the execution of computed rearrangement plans.

\subsection{Hardware Setup}\label{subsec:setup}
In our hardware setup (\ref{fig:workspace}), we use a UR-5e robot arm with an OnRobot VGC 10 vacuum gripper to execute pick-n-places. 
An Intel RealSense D435 RGB-D camera is set up above the environment to provide an overview of the entire workspace.
The camera calibration is done with four 2D fiducial markers at the corners of the environment using a C++ cross-platform software library Chilitags \cite{chilitags}.
An example instance of our experiment is shown in \ref{fig:tore-hardware-exp}\footnote{More demonstrations are available in a \href{https://github.com/gaokai15/gaokai15.github.io/assets/53358252/315865f0-5fc2-4ee8-8805-90a62e5effe6}{\bl{video}} online.}.
The right side of the pad works as an external space for buffer placements.
For the cylindrical objects, 
pose estimation is conducted with the aid of Chilitags. Demonstrations are available online.

\begin{figure}[ht!]
    \centering
    \includegraphics[width=0.5\textwidth]{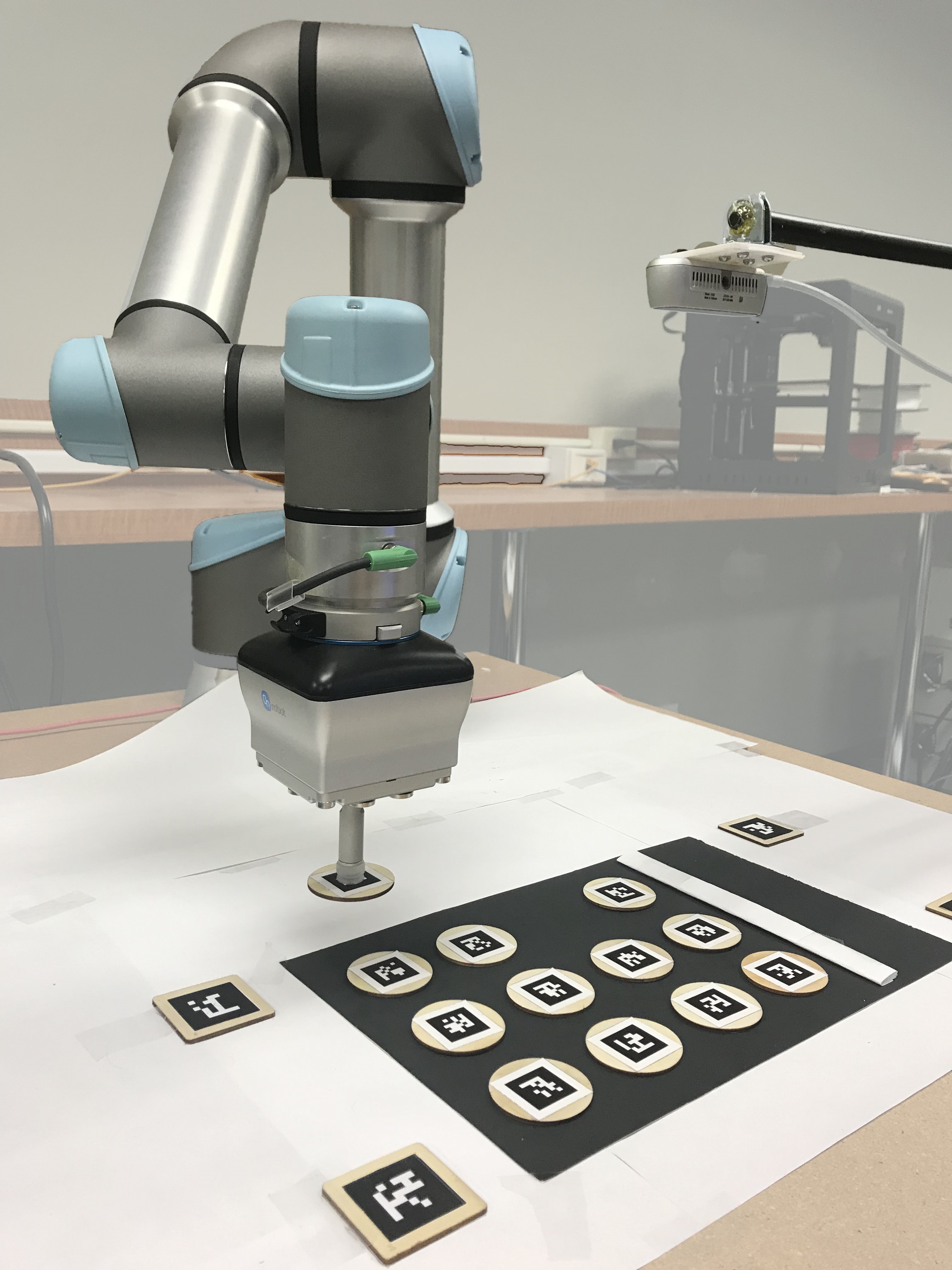}
    \caption{Our hardware setup for executing rearrangement plans computed by proposed algorithms.}
    \label{fig:workspace}
\end{figure}

\begin{algorithm}[h]
\begin{small}
    \SetKwInOut{Input}{Input}
    \SetKwInOut{Output}{Output}
    \SetKwComment{Comment}{\% }{}
    \caption{Rearrangement Pipeline}
    \label{alg:ur5e}
    \SetAlgoLined
		\vspace{0.5mm}
    \Input{$\mathcal A_g$: goal arrangement}
		\vspace{0.5mm}
		$\mathcal A_c \leftarrow PoseEstimation()$\\
		$\pi \leftarrow$ RearrangementSolver($\mathcal A_c$, $\mathcal A_g$)\\
		\For{($o$, $p_c$, $p_t$) in $\pi$}
		{
		$p_c \leftarrow $ UpdatePose($o$)\\
		ExecuteAction(($o$, $p_c$, $p_t$))\\
		}
		
\end{small}
\end{algorithm}

The rearrangement pipeline is shown in \ref{alg:ur5e}.
The system first estimates the current poses of workspace objects $\mathcal A_c$ (Line 1).
Given the current arrangement $
\mathcal A_c$ and goal arrangement $\mathcal A_g$, 
the rearrangement solver computes a rearrangement plan $\pi$ (Line 2).
Each action in $\pi$ consists of three components: manipulating object $o$, current pose $p_c$, and target pose $p_t$ (Line 3).
To improve the accuracy of pick-n-places, 
we update the grasping pose before each grasp (Lines 4-5).

\subsection{Experimental Validation}\label{subsec:hardware_exp}
We conduct hardware experiments on \toroe, 
comparing the execution time of RBM plans and TBM plans. With the same notations as those in \ref{sec:rbm-experiments}, RBM plans minimize running buffer size and TBM plans minimize total buffer size.
In \toroe, minimizing the total buffer size is equivalent to minimizing the number of total actions.
We tried 10 instances with 12 cylindrical objects.
The results are shown in \ref{Tab:TB_RB}.
While minimizing running buffer size,
RBM plans are only 5$\%$ longer than TBM plans on average.

\begin{figure}
    \centering
    \includegraphics[width=0.8\columnwidth]{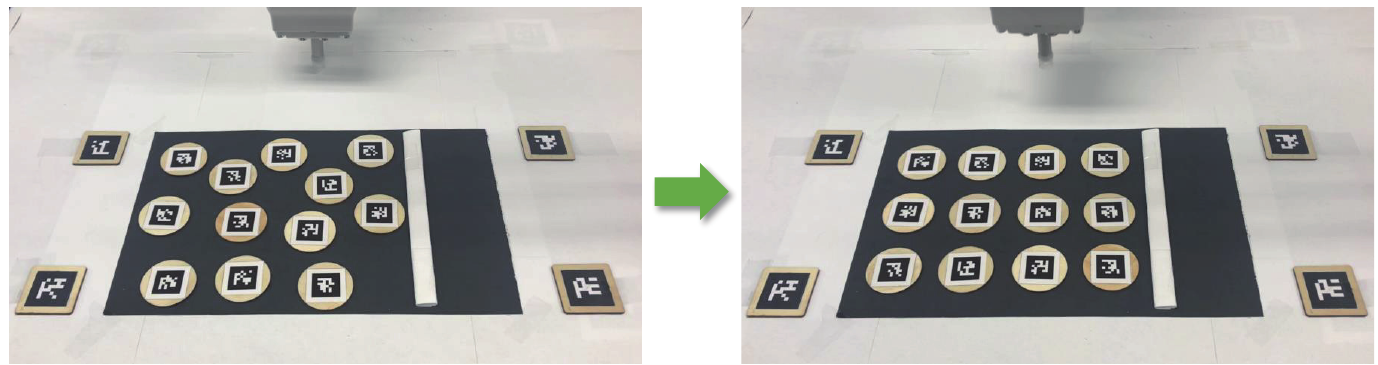}
    \caption{An example instance of our experiment.
The right side of the pad works as an external space for buffer placements.}
    \label{fig:tore-hardware-exp}
\end{figure}

\begin{center}
\begin{table}[h!]
\caption{\label{Tab:TB_RB} Comparison between RBM plans and TBM plans in execution time and the number of actions.}
\centering
\resizebox{0.5\textwidth}{!}{
\begin{tabular}{|l|ll|ll|}
\hline
\multirow{2}{*}{instance} & \multicolumn{2}{l|}{RBM}                                                                    & \multicolumn{2}{l|}{TBM}                                                                   \\ \cline{2-5} 
                          & \multicolumn{1}{l|}{\begin{tabular}[c]{@{}l@{}}execution \\ time (secs)\end{tabular}} & \# actions & \multicolumn{1}{l|}{\begin{tabular}[c]{@{}l@{}}execution\\ time (secs)\end{tabular}} & \# actions \\ \hline
1 & \multicolumn{1}{l|}{211.35}&15& \multicolumn{1}{l|}{194.52}&14\\
2 & \multicolumn{1}{l|}{197.76}&14& \multicolumn{1}{l|}{197.66}&14\\
3 & \multicolumn{1}{l|}{184.49}&13& \multicolumn{1}{l|}{185.82}&13\\
4 & \multicolumn{1}{l|}{201.32}&14& \multicolumn{1}{l|}{205.96}&14\\
5 & \multicolumn{1}{l|}{227.03}&16& \multicolumn{1}{l|}{195.44}&14\\
6 & \multicolumn{1}{l|}{193.84}&14& \multicolumn{1}{l|}{192.61}&14\\
7 & \multicolumn{1}{l|}{216.49}&15& \multicolumn{1}{l|}{198.90}&14\\
8 & \multicolumn{1}{l|}{250.13}&17& \multicolumn{1}{l|}{211.15}&15\\
9 & \multicolumn{1}{l|}{176.75}&13& \multicolumn{1}{l|}{179.83}&13\\
10 & \multicolumn{1}{l|}{179.07}&12& \multicolumn{1}{l|}{160.96}&12\\ \hline
Average & \multicolumn{1}{l|}{203.82}&14.3& \multicolumn{1}{l|}{192.28}&13.7\\ \hline

\end{tabular}
}
\end{table}
\vspace{-3mm}
\end{center}

\section{Summary}
In this work, we investigate the problem of minimizing the number of running 
buffers (\mrb) for solving labeled and unlabeled tabletop rearrangement problems 
with overhand grasps (\toro), which translates to finding a best linear ordering 
of vertices of the associated underlying dependency graph. 
For \toro, \mrb is an important quantity to understand as it determines the 
problem's feasibility if only external buffers are to be used, which is 
the case in some real-world applications \cite{han2018complexity}.
Despite the provably high computational complexity that is involved, we provide 
effective dynamic programming-based algorithms capable of quickly computing \mrb 
for large and dense labeled/unlabeled \toro instances.
In addition, we also provide methods for minimizing the total number of buffers 
subject to \mrb constraints.
Whereas we prove that \mrb can grow unbounded for both labeled and unlabeled 
settings for special cases for uniform cylinders, empirical evaluations 
suggest that real-world random \toro instances are likely to have much smaller
\mrb values. 

We conclude by leaving the readers with some interesting open problems. 
As for bounds, the lower and upper bounds of \mrb for \lrbm for uniform cylinders
do not yet agree; can the bound gap be narrowed further? 
Objects may have different sizes. An interesting question here is if we can move a large object to buffer or a small object to buffer, which is more beneficial? Moving larger objects are more challenging but we may need to do this fewer times.

%% file: chapters/TRLB.tex
\chapter{TRLB: Tabletop Rearrangement with Lazy Buffer Pose Verification}\label{chap:trlb}
\thispagestyle{myheadings}

\section{Motivations}
This chapter investigates \toro with internal buffers (\toroi).
While solving \toroe only deals with \emph{inherent} constraints defined by the start and goal poses,
some objects in \toroi may be temporarily displaced inside the workspace and induce further \emph{acquired} constraints.
For instance, to solve the problem in \ref{fig:toro} with external buffers, 
we can move the Pepsi can to an external buffer to break the cycle, 
move the Coke first and then Fanta to their goal locations, and finally, bring back the Pepsi can into the workspace. 
In \toroi, we must find a temporary location for the Pepsi can in the workspace. 
If the buffer location overlaps with the goal of the Coke can (or the Fanta can),
then the Coke can (or the Fanta can) depends on the Pepsi can again.
To avoid these acquired constraints,
the buffer location needs to avoid the goals of Coke and Fanta.

Due to acquired constraints arising from internal buffer selection, \toroi is more challenging than \toroe. 
Intuitively, selecting buffers inside the workspace (\toroi) is much more difficult and constrained than using buffers outside the workspace (\toroe) to store displaced objects. 
Since \toroe has been shown to be computationally intractable \cite{han2018complexity}, and is  equivalent to a special case of \toroi where the workspace is large enough that collision-free buffer locations are guaranteed,
\toroi is also NP-hard.
Additionally, in \toroi, it may be challenging to allocate valid buffer locations so it is necessary to limit the running buffer size. With this observation, \toroi solutions are computed based on methods for \toroe. 

For \trlb, we propose Tabletop Rearrangement with Lazy Buffers (\trlb), an effective framework for solving \toroi based on \toroe algorithms. 
We first describe a rearrangement solver with lazy buffer allocation, where buffer allocation is delayed after DFDP computes a ``rough'' schedule of object movements. 
Second, we design a preprocessing routine based on the \toroe planner for the unlabeled variant in order for further speedups.
Third, the \trlb framework recovers from buffer allocation failures to enhance scalability to larger and more cluttered instances.
Finally, \trlb is extended to consider object heterogeneity(e.g., object size and shape).

\section{Related Work}
{\bf Rearrangement Approaches:} Object rearrangement is a topic of interest within Task and Motion Planning (TAMP). Typical definitions in this domain \cite{cosgun2011push, wang2021uniform, gaorunning, wang2021efficient} involve arranging multiple objects to specific goals. Certain variations such as NAMO (Navigation among Movable Obstacles) \cite{stilman2005navigation, stilman2007planning, stilman2008planning}, and retrieval problems \cite{ZhangLu-RSS-21, nam2019planning}, 
focus on clearing out a path for a target object while identifying objects that need to be relocated. Rearrangement may be approached either via simple but inaccurate non-prehensile actions, e.g., pushes \cite{cosgun2011push, king2017unobservable, huang2021dipn}, or more purposeful prehensile actions, such as grasps \cite{krontiris2015dealing, krontiris2016efficiently, wang2021efficient, labbe2020monte,morgan2021vision,wen2021catgrasp}.
Focusing on the combinatorial challenges, some planners use external space for temporary object storage \cite{han2018complexity,gaorunning,bereg2006lifting, nam2019planning}, while others exploit problem linearity to simplify search  \cite{okada2004environment, stilman2005navigation, stilman2008planning, levihn2013hierarchical}. By linking rearrangement to established graph-based problems, efficient algorithms have been obtained for various tasks and objectives  \cite{han2018complexity,gaorunning,bereg2006lifting}.
This chapter uses a plan generated given access to external buffer locations as a ``primitive plan'', which then guides buffer allocation inside the workspace.

{\bf Dependency Graph:} We represent the combinatorial constraints of such problems via a dependency graph, which was first applied to general multi-body planning problems \cite{van2009centralized} and then rearrangement \cite{krontiris2015dealing, krontiris2016efficiently}. Choosing different manipulation sequences gives rise to multiple dependency graphs for the same problem instance, which limits scalability. Prior work \cite{han2018complexity} has applied full dependency graphs to address \toro, showing that the challenge embeds the NP-hard Feedback Vertex Set (FVS) and the Traveling Salesperson (TSP) problems. More recently, some of the authors \cite{gaorunning} examined an optimization objective, \emph{running buffers}, which is the size of the external space needed for rearrangement, and also examined an unlabeled setting. Similar structures are also used in other problems, such as packing \cite{wang2020robot}. Deep neural networks have been also applied to detect the embedded dependency graph of objects in a cluttered space to determine the ordering of object retrieval \cite{ZhangLu-RSS-21}.

%In this specific setup, the dependency graph is unique for each instance and completely represents the inherent constraints.

{\bf Buffer Identification:} In rearrangement, collision-free locations are needed for obstacle displacements. Some previous efforts predetermine buffer candidates and place obstacles to accessible candidate locations when necessary \cite{wang2021uniform, cheong2020relocate}. Others decouple the problem into successive subproblems \cite{krontiris2016efficiently, wang2021uniform}.
Intuitively speaking, a valid buffer location needs to avoid other objects at their current poses. Backtracking search further constrains object displacements given paths of future manipulation actions \cite{stilman2007planning, stilman2008planning}. Nevertheless, these methods are computationally expensive since they deal with inherent and acquired constraints at the same time. To reduce the associated  cost, a lazy strategy can be applied \cite{bohlin2000path, denny2013lazy, hauser2015lazy}, which delays path/configuration collision checking. A similar idea is proposed here. Feasible object locations are obtained with the aid of a ``rough schedule'' of object manipulation actions,  which is computed given only the inherent constraints. By decoupling inherent and acquired constraints, the proposed method computes high-quality solutions efficiently.

\section{Tabletop Rearrangement with Internal Buffer Space (\toroi)}
\subsection{Problem Statement}
\begin{figure}[t]
    \centering
    \includegraphics[width=\columnwidth]{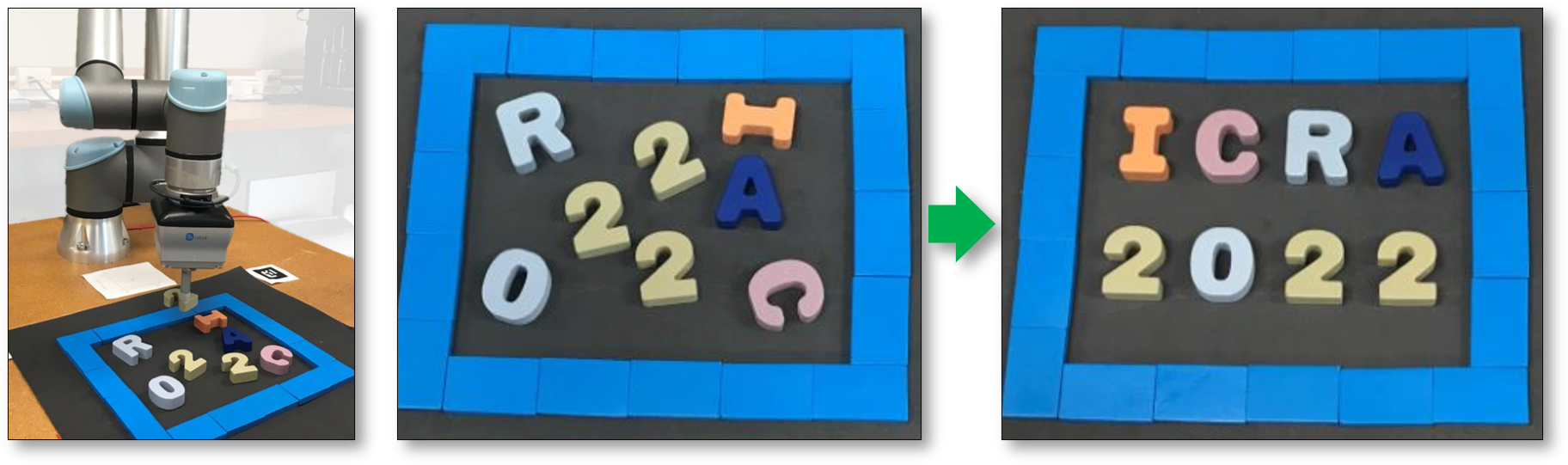}
        \vspace{-.25in}
    \caption{A robot arm rearranges word patterns with overhand grasps, minimizing the number of pick-n-place actions. The letters cannot go out of the bounding box and cannot overlap on the tabletop.}
    \label{fig:TRLBApplicationExample}
\end{figure}

Consider a 2D bounded workspace $\mathcal W\subset \mathbb R^2$ containing a set of $n$ objects $\mathcal O=\{o_1, ..., o_n\}$. 
Each object is assumed to be an upright \emph{generalized cylinder}. A \emph{feasible arrangement} $\mathcal A =\{p_1, ..., p_n\}$, $p_i = \{x_i, y_i, \theta_i\} \in SE(2)$ is a set of poses for objects in $\mathcal O$, such that 
(1) each object's footprint is contained in $\mathcal W$, and
(2) no two objects collide. 

Consider (overhand) pick-n-place actions to move objects one by one.  A pick-n-place action, represented as an ordered pair $(o, p)$, grasps an object $o$ at its current pose and lifts it above all other objects. It then moves it horizontally, and places it at the target pose $p$ within $\mathcal W$. An action is collision-free iff both the current and resulting arrangements are feasible. A plan $P$ from a feasible $\mathcal A_s$ to a feasible $\mathcal A_g$ is a sequence of collision-free pick-n-place actions transforming $\mathcal A_s$ into $\mathcal A_g$. We want to compute feasible plans that minimize the number of pick-n-place actions, which leads to increased system throughput, i.e., \vspace{-.05in}

\begin{problem}[\toro w/ Internal Buffers (\toroi)]
Given feasible arrangements $\mathcal A_s=\{p^s_1, ..., p^s_n\}$ and $\mathcal A_g=\{p^g_1, ..., p^g_n\}$, find a feasible plan $P$ sequentially moving objects from $\mathcal A_s$ to $\mathcal A_g$, which minimizes the number of actions.
\end{problem}
\vspace{-.1in}

\section{Algorithmic Solutions}
We first describes a rearrangement solver with lazy buffer allocation (\ref{sec:LazyBufferGeneration}), where buffer allocation is delayed after getting a ``rough'' schedule of object movements. To enhance scalability to larger and more cluttered instances, the \trlb framework (\ref{sec:PartialPlan}) recovers from buffer allocation failures. Finally, a preprocessing routine helps with further speedups (\ref{sec:Preprocess}).

\subsection{Lazy Buffer Pose Verification}\label{sec:LazyBufferGeneration}
When an object stays at a buffer, it should avoid blocking the upcoming manipulation actions of other objects. Otherwise, either the object in the buffer or the manipulating object has to yield, which increases the number of necessary actions. In other words, we need to carefully choose acquired constraints. If we know the schedule of other objects in advance, a buffer can be selected to minimize unnecessary obstructions. 
This observation motivates
us to ``lazily'' allocate buffers by ignoring initially acquired constraints and solve
% \todo{mention laziness here again.}
the rearrangement problem in two steps: First, compute a \emph{primitive plan}, which is an incomplete schedule ignoring acquired constraints; second, given the incomplete schedule as a reference, generate buffers to optimize the selection of acquired constraints.

%Before moving an object to a goal pose, 
%other objects occupying the pose have to temporarily move to buffer locations if their own goal positions are inaccessible.
% Similar ideas are used in NAMO(Navigation Among Movable Obstacles) problems.
% Stilman et al. used reverse search methods to constrain prior object displacements with future manipulations\cite{stilman2007planning, stilman2008planning}.
% However, 
% these methods, solving combinatorial and geometric challenges at the same time,
% are computationally expensive.

\subsubsection{Primitive Plan}
To compute a \emph{primitive plan}, we assume enough free space is available so that no acquired constraints will be created. This transforms the problem into a \toroe problem, where each object is displaced at most once before it moves to the goal pose. Then, an object $o_i \in \mathcal O$ can have three \emph{primitive} actions:   
%large enough to contain all the objects, 
%for temporary object placements.
%With the large empty area,
%there is no reason to move an object from a buffer location to another.
%Therefore, an object $o_i \in \mathcal O$ can have three primitive actions: 
\begin{enumerate}
    \item $(o_i, s\rightarrow g)$: moving from $p^s_i$ to $p^g_i$;
    \item $(o_i, s\rightarrow b)$: moving from $p^s_i$ to a buffer;
    \item $(o_i, b\rightarrow g)$: moving from a buffer to $p^g_i$.
\end{enumerate}
A primitive plan is a sequence of primitive actions;
%, 
%moving objects from $\mathcal A_s$ to $\mathcal A_g$ without collision.
%Similar to \toroe, 
computing such a plan is similar to finding a linear vertex ordering \cite{adolphson1973optimal, shiloach1979minimum} of the dependency graph. We use dynamic programming based methods in \ref{chap:rbm} to achieve this, which minimizes the number of total buffers or running buffers. 

\subsubsection{Buffer Allocation}
%In many practical cases, 
Free space inside the workspace $\mathcal{W}$ is scarce in cluttered spaces (e.g.,  \ref{fig:trlb-density}) and acquired constraints must be dealt with through  the careful allocation of buffers inside $\mathcal{W}$. We apply a greedy strategy to find feasible buffers based on a primitive plan (\ref{alg:buffer}). The general idea is to incrementally add constraints on the buffers until we find feasible buffers for the whole primitive plan or terminate at a step where there are no feasible buffers for the primitive plan. In \ref{alg:buffer}, $\mathcal O_s, \mathcal O_g, \mathcal O_b$ are the sets of objects currently at start poses, goal poses and buffers respectively.

%
%We do so based on a primitive plan.
%
%Since primitive plans are by nature under a \emph{two-step assumption}: 
%each object in $\mathcal O$ moves at most twice,
%we follow the assumption in this subsection,
%and more general cases will be discussed in Sec. \ref{sec:PartialPlan}.
\begin{algorithm}
\begin{small}
    \SetKwInOut{Input}{Input}
    \SetKwInOut{Output}{Output}
    \SetKwComment{Comment}{\% }{}
    \caption{ Buffer Allocation}
		\label{alg:buffer}
    \SetAlgoLined
		\vspace{0.5mm}
    \Input{$\pi$: a primitive plan; $\mathcal A_s=\{p^s_1,...,p^s_n\}$: start arrangement; $\mathcal A_g=\{p^g_1,...,p^g_n\}$: goal arrangement}
    \Output{$B$: buffers; TerminatingStep: the action step where buffer generation fails, $\infty$ if Success.}
		\vspace{0.5mm}
		$\mathcal O_s$ $\leftarrow$ $\mathcal O$; 
		$\mathcal O_g$, $\mathcal O_b$ $\leftarrow$ $\emptyset$; 
		$B \leftarrow$ RandomPoses($\mathcal O$)\\
		\For{   $(o_i, m)\in \pi$}{
    		\If{$m$ is s $\rightarrow$ b}{
        		$\mathcal O_b$.add($o_i$)\\
        		Constraints[$o_i$]$\leftarrow$GetPoses($\mathcal O_s \bigcup \mathcal O_g - \{o_i\}$)\\
    		}
    		\ElseIf{$m$ is b $\rightarrow$ g}{
        		\lFor{$o\in \mathcal O_b\backslash\{o_i\}$}{
        		Constraints[$o$].add($p^g_i$)
    		}
    		}
    		\Else{
        		%\Comment{$m$ is s $\rightarrow$ g}
        		\lFor{$o\in \mathcal O_b$}{
        		    Constraints[$o$].add($p^g_i$)
        		}
    		}
    		Success, $B'$ $\leftarrow$ BufferGeneration($\mathcal O_b$, Constraints, $B$)\\
    		\If{Success}{
    		$B$ $\leftarrow$ $B'$\\
    		$\mathcal O_s, \mathcal O_g, \mathcal O_b \leftarrow$ UpdateState($\mathcal O_s, \mathcal O_g, \mathcal O_b$)\\
    		}
    		\lElse{
    		\Return $B$, $\pi$.index(action)
    		}
		}
		\vspace{0.5mm}
		\Return $B$, $\infty$\\
\end{small}
\end{algorithm}

We start with $\mathcal A_s$ where all the objects are at start poses and the buffers are initialized at random poses (Line 1). Each action in $\pi$ indicates an object $o_i$ that is manipulated and the action $m$ performed (Line 2). If $o_i$ is moved to a buffer (Line 3),  then we add it into $\mathcal O_b$ (Line 4). The current poses of other objects in $O_s\bigcup O_g$ are seen as fixed obstacles for $o_i$ (Line 5). If $o_i$ is leaving the buffer (Line 6), then other objects in $\mathcal O_b$ should avoid the goal pose $p^g_i$ of $o_i$ (Line 7). If $o_i$ is moving directly from $p^s_i$ to $p^g_i$ (Line 8, the ``else'' corresponds to $m$ being $s\to g$, e.g., directly go from start to goal), then all buffers for objects in the current $\mathcal O_b$ need to avoid $p^g_i$ (Line 9). After setting up acquired constraints, we generate new buffers for objects in $O_b$ to satisfy these constraints by either sampling or solving an optimization problem (Line 10). Old buffers in $B$ satisfying new constraints will be directly adopted. If feasible buffers are found (Line 11), then buffers and object states will be updated (Line 12-13). Otherwise, we return the feasible buffers computed and record the terminating step  of the algorithm (Line 14). 
In the case of a failure, we provide a \trlb framework to recover, which is presented in \ref{sec:PartialPlan}.
% \todo{Previously, we mention ``partial plans'' without saying the usage, which may confuse readers. Instead, I directly mention the TRLB.}
% In the case of a failure, the returned buffers provide a \emph{partial plan}.
%, so the 
%break the two-step assumption and develop a robust planner which does not rely too much on the quality of primitive plans. 
%The details of this part is disclosed in Sec. \ref{sec:PartialPlan}.

\ref{fig:AlgoExample} illustrates the buffer allocation process via an example. The green, cyan, and transparent discs represent the current poses, goal poses and allocated buffers respectively. When we move $o_1$ to a buffer $B_1$ (\ref{fig:AlgoExample}(b)),  it only needs to avoid collision with $p^s_2$ and $p^s_3$.  But as we move $o_3$ to a buffer,  $B_1$ needs to avoid $o_3$'s buffer $B_3$ as well. To satisfy the added constraint, $B_1$ will be reallocated. Since the new buffers $B_1$ and $B_3$ (\ref{fig:AlgoExample}(c)) satisfy the constraints added in the following steps, they need not to be relocated. Note that the buffer originally selected for $o_1$ but then replaced will not appear in the resulting plan, i.e., $o_1$ will move directly to the new buffer (\ref{fig:AlgoExample}(c)-(f)). \ref{alg:buffer} works with one strongly connected component of the dependency graph at a time, treating objects in other components as fixed obstacles.
%whose start poses are the same as goal poses.

%Based on the topological order of the strongly connected components of the dependency graph, 
%the rearrangement task can be decoupled into independent subproblems;
%The input primitive plan for each call of \ref{alg:buffer} only move objects in one strongly connected component in $G$. 
%The decomposition speeds up the method, especially in large scale problems.

Once the feasible buffers are found, all the primitive actions can be transformed into feasible pick-n-place actions inside the workspace. And therefore, the primitive plan can be transformed into a rearrangement plan moving objects from $\mathcal A_s$ to $\mathcal A_g$. The function BufferGeneration is implemented by either sampling or solving an optimization problem, both of which are discussed below.

\begin{figure}
    \vspace{2mm}
    \centering
    \includegraphics[width=0.9\columnwidth]{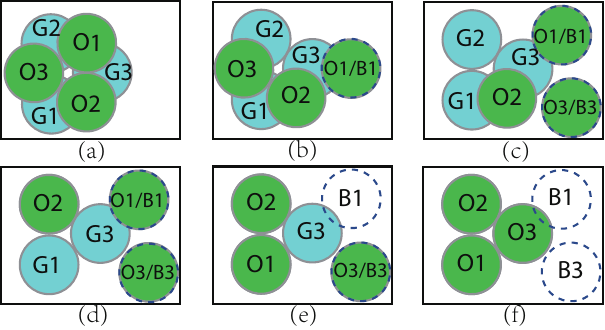}
    \vspace{-.1in}
    \caption{A working example with three objects defined in (a).
    The primitive plan is [($o_1$, s $\rightarrow$ b), ($o_3$, s $\rightarrow$ b), ($o_2$, s $\rightarrow$ g), ($o_1$, b $\rightarrow$ g), ($o_3$, b $\rightarrow$ g)].  Figures (b)-(f) show the steps of Alg. 1 after each action. The transparent discs with dashed line boundaries ($B_i$) represent the buffers satisfying constraints up to each step. For each object $o_i$, the green ($O_i$) and cyan ($G_i$) discs represent the current and goal poses respectively.}
    \label{fig:AlgoExample}
\end{figure}

\paragraph{Sampling}\label{sec:Sampling}
Given the object poses that buffers need to avoid so far, feasible buffers can be generated by sampling poses inside the free space. When objects stay in buffers at the same time,we sample buffers one by one and previously sampled buffers will be seen as obstacles for latter ones.

\paragraph{Optimization}\label{sec:Optimization}
For cylindrical objects $o_i$ at $(x_i,y_i)$ with radius $r_i$, and $o_j$ at $(x_j,y_j)$ with radius $r_j$, they are collision-free when
$(x_i-x_j)^2+(y_i-y_j)^2 \geq (r_i+r_j)^2$ holds.
Therefore, the collision constraints obtained in \ref{alg:buffer} can be transformed into quadratic inequalities.
% \todo{ mention that the collision constraints are transformed into quadratic inequalities.}
By further restricting the range of buffer centroids to assure they are in the workspace, the buffer allocation problem can be transformed into a quadratic programming problem with no objective function. For objects with general shapes, collision avoidance cannot be presented by inequalities of object centroids. We can construct the optimization problem with $\phi$ functions of the objects \cite{chernov2010mathematical} and solve the problem with gradients.

\subsection{Failure Recovery with Bi-Directional Search}\label{sec:PartialPlan}
The \trlb framework builds on the insight that a new \toroi instance is generated when lazy buffer allocation fails. The new instance has the same goal $\mathcal A_g$ as the original one but some progress has been made in solving the \toroi task. There are two straightforward implementations of \trlb: forward search and bidirectional search. In the first case, by accepting partial solutions, a rearrangement plan can be computed by developing a search tree $T$ rooted at $\mathcal A_s$.  In the search tree $T$, nodes are feasible arrangements and edges are partial plans containing a sequence of collision-free actions. When buffer allocation fails, we add the resulting arrangement into the tree and resume the rearrangement task from a random node in $T$.  This randomness and the randomness in primitive plan computation and buffer allocation allows \trlb to recover from failures. 
\vspace{0.05in}

%In the case where the lazy buffer allocation fails, 
%planning can be resumed from resulting arrangements of the partial plans.
%With the search tree, the algorithm can find more general plans breaking the two-step assumption and relies less on the quality of primitive plans.

\begin{algorithm}[h]\label{alg:BST}
\begin{small}
    \SetKwInOut{Input}{Input}
    \SetKwInOut{Output}{Output}
    \SetKwComment{Comment}{\% }{}
    \caption{\trlb with Bidirectional Search}
		\label{alg:BS}
    \SetAlgoLined
		\vspace{0.5mm}
    \Input{$\mathcal A_s$, $\mathcal A_g$, $max\_time$}
    \Output{ Search trees: $T_1$, $T_2$}
		\vspace{0.5mm}
		$T_1$.root, $T_2$.root$\leftarrow \mathcal A_s, \mathcal A_g$\\
		\While{not exceeding $max\_time$}{
		$\mathcal A_{rand}\leftarrow$ RandomNode($T_1$)\\
		$\mathcal A_{new1} \leftarrow$ LazyBufferAllocation($\mathcal A_{rand}$, $T_2$.root)\\
		$T_1$.add($\mathcal A_{new1}$)\\
		\lIf{$\mathcal A_{new1}$ is $T_2$.root}{\Return $T_1$, $T_2$}
		$\mathcal A_{near}\leftarrow$ NearestNode($\mathcal A_{new1}$, $T_2$)\\
		$\mathcal A_{new2} \leftarrow$ LazyBufferAllocation($\mathcal A_{near}$, $\mathcal A_{new1}$)\\
		$T_2$.add($\mathcal A_{new2}$)\\
		\lIf{$\mathcal A_{new2}$ is $\mathcal A_{new1}$}{\Return $T_1$, $T_2$}
		$T_1, T_2 \leftarrow T_2, T_1$\\
		}
\end{small}
\end{algorithm}

%Besides conducting an one-directional search from $\mathcal A_s$ to $\mathcal A_g$, 
In bidirectional search, two search trees rooted at $\mathcal A_s$ and $\mathcal A_g$ are developed. This more involved procedure is shown in  \ref{alg:BS}, which computes two search trees that connect $\mathcal A_s$ and $\mathcal A_g$.  In Line 1, the trees are initialized. For each iteration, we first rearrange between a random node $\mathcal A_{rand}$ on $T_1$ to the root node of $T_2$ (Line 3-5). The function LazyBufferAllocation refers to the overall algorithm developed in \ref{sec:LazyBufferGeneration}. A found path yields a feasible plan for \toroi (Line 6). Otherwise, we rearrange between the new arrangement $\mathcal A_{new1}$ and its nearest neighbor in $T_2$ (Line 7-9). 
This paper defines the distance between arrangement nodes as the number of objects at different poses.
% \todo{Explain the definition of nearest neighbors}
If a path is found, then we find a feasible rearrangement plan for \toroi (Line 10). Otherwise, we switch the trees and attempt rearrangement from the opposite side (Line 11).

\subsection{Preprocessing for Rearrangement in Confined Workspace}\label{sec:Preprocess}
In dense environments, allocating buffers is hard, motivating minimizing the number of running buffers (mentioned in \ref{chap:rbm}), which is generally low even in high density settings if we treat objects as \emph{unlabeled}. Based on this, for each component of the dependency graph that is not a single vertex/cycle, we reduce the running buffer size to at most 1 by first solving an \emph{unlabeled} instance (also mentioned in \ref{chap:rbm}). After preprocessing, we obtain a \toroi requiring at most one running buffer. \ref{fig:preprocessing} shows an example of preprocessing.  $o_1$, $o_2$ and $o_3$ form a complete graph, where at least two objects need to be placed at buffers simultaneously. We conduct preprocessing of the three-vertex component by moving $o_2$ to a buffer position, $o_1$ to $p^g_3$ and $o_3$ to $p^g_2$. $o_2$ will not move to $p^g_1$ since it does not occupy other goal poses. The preprocessing step needs one buffer and the resulting rearrangement problem is monotone.

\begin{figure}[h]
    \centering
    \includegraphics[width=\columnwidth]{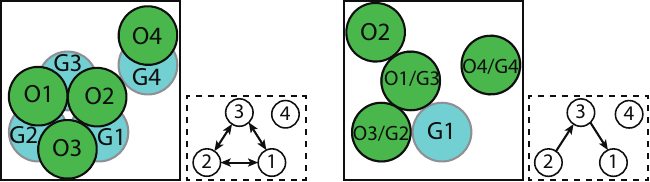}
    \vspace{-.2in}
    \caption{A four-object example of preprocessing. The green and cyan discs represent current and goal arrangements respectively. Before preprocessing (left), two buffers need to be allocated synchronously. After  preprocessing (right), the problem becomes monotone.}
    \label{fig:preprocessing}
    \vspace{-.1in}
\end{figure}

\section{Heterogeneous \toroi}
\subsection{Motivation}
Latest algorithmic research on \toroi is mainly limited to solving instances with mostly homogeneous objects \cite{han2018complexity,labbe2020monte,shome2020fast, gao2022fast}. 
%\jy{Kai: please provide relevant references, including our work and some of Kostas', and maybe more.}
%
The assumption of homogeneous objects leads to a relatively simple model where each object manipulation action has the same or very similar cost, which in turn limits the applicability of the resulting methods. 
In many real-world tabletop rearrangement problems, objects are often heterogeneous. For example, in \ref{fig:intro_example}, some objects are small and thin, e.g., pens, while others can be bulky, e.g., the white bag. Intuitively, objects' sizes and shapes impact the rearrangement planning process - it is more difficult to relocate larger objects temporarily. 
Furthermore, other properties, such as weight, affect the effort needed to move them around. This further complicates the computation of optimal rearrangement plans.  
% Intuitively speaking, moving large objects or stick-like objects are more likely to make the goal poses of other objects available due to the overlap between their footprints and those of other objects in goal poses.
% However, it is also more difficult to allocate buffer locations for these objects and their goal poses are more likely to be occupied by more other objects.
We thus ask a natural question here: 
% How much impact do object size and shape have on the difficulty of buffer allocation? And 
Can we leverage objects' properties to guide multi-object rearrangement planning?

\begin{figure}[t]
\vspace{2mm}
    \centering
    \includegraphics[width=\columnwidth]{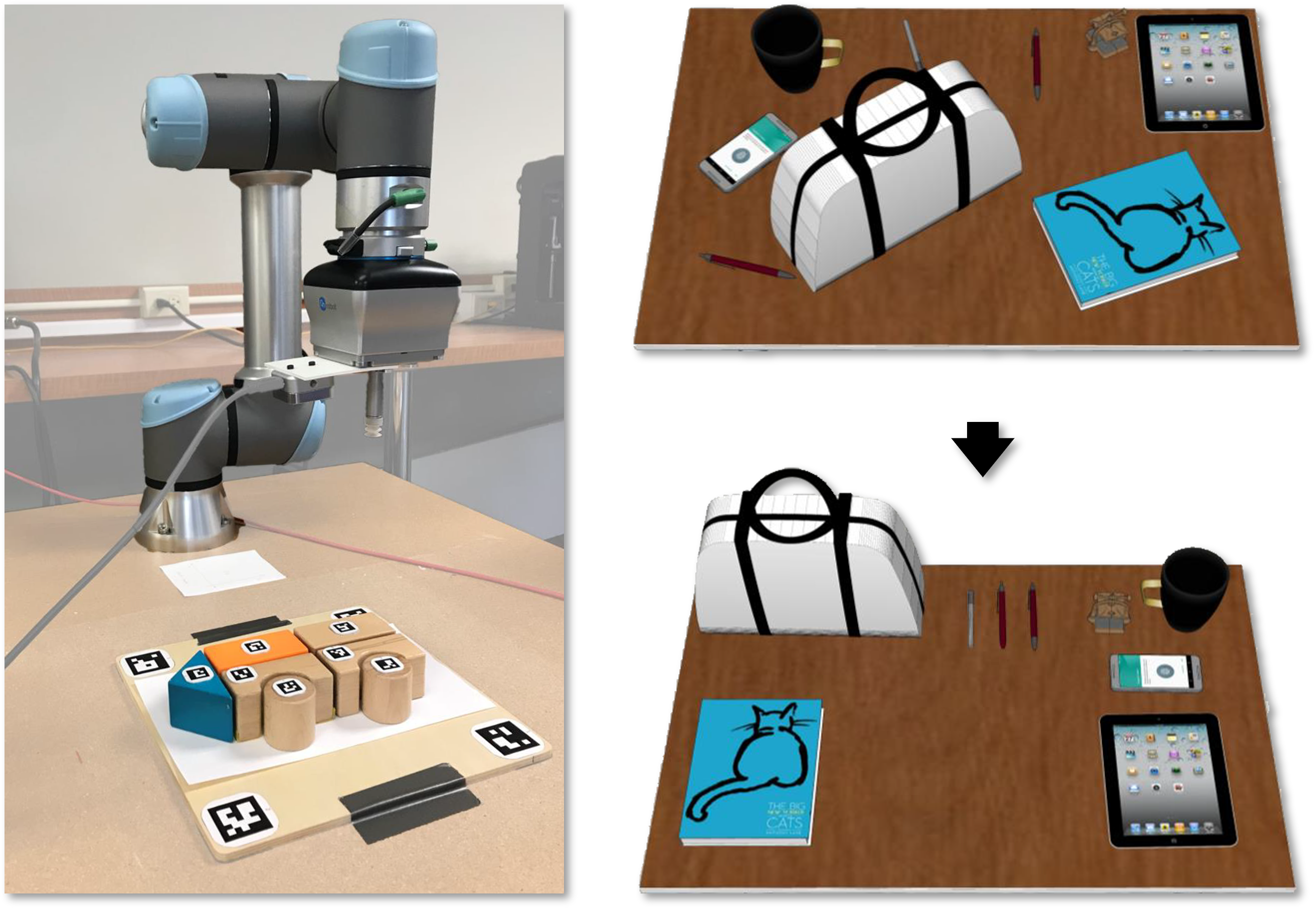}
    \caption{[Left] Our hardware setup for evaluating tabletop rearrangement of multiple heterogeneous objects. One of the test cases is also shown. [Right] The illustration of an example instance where many objects of drastically different sizes and weights are to be rearrangements.}
    \label{fig:intro_example}
\end{figure}

With the question in mind, we investigate heterogeneous \toroi (\hete), in which the characteristics of workspace objects, such as size, shape, weight, and other properties, can vary significantly. 
Taking object characteristics into consideration immediately leads to a much more general problem model, allowing each manipulation operation to have a cost that is affected by the properties of the objects. 
For example, manipulating an object that is large, heavy, and/or fragile is naturally more challenging. As a result, such an operation could take more time to complete. Therefore, a higher cost may be given for such operations than manipulating an object that is small, light, and/or tough. 
This section focuses on developing methods that take such differences into account. 
More specifically, we develop methods for addressing two problem variants. In one of the variants, object shape and size are considered in generating the manipulation plan. Object characteristics are directly reflected in the optimization objective in the other variant. 

\subsection{Problem Formulation}
\hete assumes full knowledge of object characteristics $\mathcal I=\{I_1, I_2,\dots,I_n\}$.
$\mathcal I$ may include properties of objects, including geometry, mass, material properties, and so on. 
In practice, $\mathcal I$ can be estimated based on sensing data.

For \hete, if the robot is sufficiently powerful, it may still be the case that grasping/releasing objects takes the same amount of time for each object, regardless of their sizes and mass (with some range); but it could also be that different objects requires different (time) cost to grasp/release.
Based on the observation, for \hete, two cost objectives are examined, which are: 
\begin{enumerate}
    \item \textbf{(\texttt{PP})} Total number of relocations, i.e., $J(\Pi)=|\Pi|$.
    \item \textbf{(\texttt{TI})} Total task \emph{impedance}, i.e., \vspace{2mm}\begin{align}J(\Pi)=\sum_{(o_i,p_i^1,p_i^2)\in \Pi} f(I_i).\label{eq:ti}\end{align}
\end{enumerate}

Specifically, the PP objective is the number of grasps or places a robot needs to execute in a rearrangement plan.
The TI objective enables the prioritization of manipulation actions based on objects' intrinsic properties, where $f(I_i)$ represents the actual cost of manipulating object $o_i$, as determined by the robot and $I_i$. 
Compared to PP, the TI objective represents the manipulation cost with higher levels of fidelity. 
For example, if a robot is tasked to rearrange gold and iron bars of the same size and shape, it is certainly better to move iron bars more due to iron's significantly lower density.  
%This feature is useful in \hete since one may want to frequently pick and place lightweight or handy objects but avoid manipulations of heavy or fragile objects.
%$f(I_i)$ measures the cost of moving $o_i$ based on information $I_i$. 
%It can be the mass of object $o_i$ or the estimated difficulty of grasping $o_i$.

Based on the descriptions so far, \hete can be formally summarized as follows:

\begin{problem}[\hete]
Given two feasible arrangements $\mathcal A_s$ and $\mathcal A_g$ of objects $\mathcal O$, compute a rearrangement plan $\Pi$ moving $\mathcal O$ from $\mathcal A_s$ to $\mathcal A_g$ with the minimum cost $J(\Pi)$.
\end{problem}

\subsection{Weighting Heuristics for \hete}\label{sec:weighting}
Our algorithmic design methodology is general: different (weighting) heuristics can be employed by different seach strategies (e.g., TRLB and MCTS) for solving \hete. 
Here, we first describe our carefully constructed heuristics that are applicable to both PP and TI objectives.

\subsubsection{Collision Probability  (\heteCP) for PP}\label{sec:collision}
In a cluttered workspace, it is difficult to allocate collision-free poses for use as buffers.
Even when a pose is collision-free in the current arrangement, it may overlap/collide with some goal poses, blocking the movements of other objects.
It is generally more challenging to allocate high-quality buffers for objects with large sizes or aspect ratios (e.g., stick-like objects).
A natural question would be:
How much do the object's size and shape impact its collision probability with other objects? 
We compute an estimated collision probability for each object to answer the question. Based on the collision probability, we propose a heuristic \heteCP (Heterogeneous object Collision Probability) to address the buffer allocation challenge in \hete and guide the rearrangement search. 
% \heteCP suggests the difficulty of allocating high-quality buffers.

Given two objects $A$ and $B$ with fixed orientations, positions of $B$ colliding with $A$ can be represented by the \emph{Minkowski difference} of the two objects. That is,
$$
A\ominus B(0,0) := \{a-b| a \in P_A, b\in P_B(0,0)\},
$$
where $P_A$ is the point set of $A$ at the current pose and $P_B$ is the point set of $B$ at position $(0,0)$ with the current orientation. 
Note that when $B$ is rotational symmetric, the Minkowski difference is the same as the Minkowski sum, which replaces the minus with a plus in the formula above.
%\jy{Please update to using Minkowski sum.}

\begin{figure}
    \centering
    \vspace{-2mm}
    \includegraphics[width=0.5\textwidth]{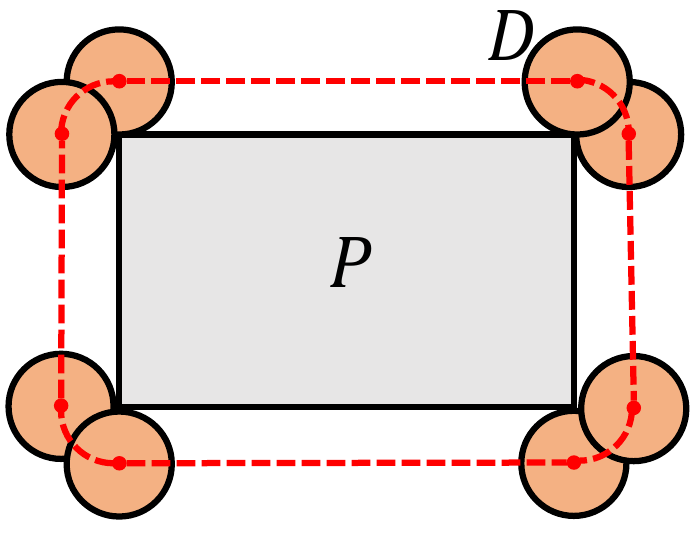}
    % \vspace{-6.5mm}
    \caption{An illustration of the Minkowski difference $P\ominus D(0,0)$ (red dashed region) of a rectangle $P$ and a disc $D$.}
    \label{fig:no_fit}
\end{figure}

Specifically, for a convex polygon $P$ and a disc $D$, the area of $P\ominus D(0,0)$ can be simply represented as
$
S_{N}=S_{P}+S_{D}+r_DC_{P},
$
where $S$ and $C$ are the area and circumference of corresponding objects, $r_D$ is the radius of $D$. 
In the example of \ref{fig:no_fit}, $S_P$ comes from the polygon $P$ in the center of $N$, $S_D$ comes from the four rounded corners, and $r_DC_P$ comes from the area along the edges.
Since one of the objects is a disc, the formula is rotation invariant. 
Furthermore, if $P$ is at least $2r_D$ away from the boundary of $\mathcal W$, the probability of $D$ colliding with $P$ is
\begin{equation}\label{eq:pro}
\Pr=\dfrac{S_{N}}{(H-2r_D)(W-2r_D)}=\dfrac{S_{P}+S_{D}+r_D C_{P}}{(H-2r_D)(W-2r_D)}    
\end{equation}

Assuming both $P$ and $D$ are small compared to $\mathcal W$, $\Pr$ is a fairly tight upper bound of the collision probability of $P$ and $D$ in the workspace.
% Eq.~\ref{eq:pro} suggests that the probability is affected by object area, shape, and relative size to the other objects in the workspace.

% If the size of $P$ changes by a scalar $\alpha$ to $P' (\alpha > 1)$, 
% $$S_{P'}=\alpha S_{P}$$.
% , we have 
% $$\sqrt{\alpha} C_{P} \leq C_{P'} < \alpha C_{P}$$.
% Since $\Pr$ is affected by three terms, depending on the degree of asymmetry of $P$ and the size difference between $P$ and $D$, 
% Therefore, we have:
% \begin{enumerate}
%     \item when $P$ is much larger than $D$, $\Pr'\approx \alpha \Pr$;
%     \item when $P$ is much smaller than $D$, $\Pr'\approx \Pr$; 
%     \item when $P$ and $D$ has similar size, and $P$ has high degree of asymmetry, $\Pr'\approx \Pr$
% \end{enumerate}

Based on \ref{eq:pro}, the \heteCP heuristics can be computed as in \ref{alg:heteCP}. 
% \heteCP is used to measure the difficulty of buffer allocation.
The algorithm consumes the object characteristics $\mathcal I$. We first compute the area and circumference of each object based on $\mathcal I$ (Lines 1-4).
In Lines 5-6, we create a disc whose size is the average size of workspace objects. 
The area and radius of the average-sized disc are denoted as $\overline{s}$ and $\overline{r}$, respectively.
The weight of each object is then computed as the estimated collision probability between the object and the average-sized disc (Lines 7-9).

In general, \heteCP measures the difficulty of allocating a collision-free pose for an object and the probability of the sampled pose blocking other objects' movements. 

\begin{algorithm}[ht]
\begin{small}
    \SetKwInOut{Input}{Input}
    \SetKwInOut{Output}{Output}
    \SetKwComment{Comment}{\% }{}
    \caption{\heteCP}
		\label{alg:heteCP}
    \SetAlgoLined
		\vspace{0.5mm}
		\Input{$\mathcal I$: object characteristics.
		}
        \Output{\heteCP: weights on $\mathcal O$}
		\vspace{0.5mm}
    \vspace{1mm}
    \For{$o_i\in \mathcal O$}
    {
    $s_i \leftarrow $ getSize($I_i$)\\
    $c_i \leftarrow $ getCircumference($I_i$)\\
    }	
    \vspace{1mm}
    $\overline{s} = (\sum_{1 \leq i \leq n} s_i)/n$\\
    \vspace{1mm}
    $\overline{r} = \sqrt{\overline{s}/\pi}$\\
    \vspace{1mm}
    \For{$o_i\in \mathcal O$}
    {
    $w_i \leftarrow \dfrac{s_{i}+\overline{s}+\overline{r} c_{i}}{(H-2\overline{r})(W-2\overline{r})}$ \\
    }	
    \vspace{1mm}
    \Return $\{w_1, w_2,\dots, w_n\}$
\end{small}
\end{algorithm}

\subsubsection{Task Impedance (\heteTI) for TI}
When task impedance is evaluated, directly applying the existing methods, which seek rearrangement plans with the minimum number of actions, may lead to undesirable sub-optimality.
For the TI objective, the weighting heuristic is denoted as \heteTI. 
\heteTI gives each object $o_i$ a weight $f(I_i)$ as defined in \ref{eq:ti}.

\subsection{Extended TRLB (ETRLB) for \hete}\label{sec:hete-trlb}
For TRLB, we construct a weighted dependency graph with the heuristics and formulate the corresponding feedback vertex set problem and running buffer minimization problem.

Recall that TRLB first computes a primitive plan and then allocates buffers to check plan feasibility.
For \hete, we leverage weighting heuristics to guide the primitive plan computation in TRLB, which yields to \emph{Extended TRLB} (ETRLB).
To do so, we first construct a weighted dependency graph $\mathcal G_w$, where each vertex $v_i$ has a weight $w_i$.
Correspondingly, in the primitive plan computation process, we solve weighted TBM and weighted RBM of $\mathcal G_w$.
For convenience, we use the notations ETBM and ERBM to represent ETRLB with weighted TBM and RBM, respectively.
ETBM computes a primitive plan minimizing the total weights of objects moving to buffers.
We transform the problem into computing the feedback vertex set of $\mathcal G_w$ with the minimum sum of weights, which can be solved by dynamic programming or integer programming similar to TBM in TRLB.
The same as the TRLB with RBM, ERBM computes the minimum running buffer size for the rearrangement problem but it assumes each object $o_i$ to take $w_i$ amount of space in buffers.

In the original RBM computation, TRLB fixes the running buffer size and checks whether there is a feasible rearrangement plan given the running buffer size in order to make use of the efficiency of depth-first search. 
If not, the algorithm increases the running buffer size by one and checks the feasibility again. 
This process continues until it finds the lowest running buffer size with a feasible primitive plan.
Since this framework needs an integer running buffer size, ERBM rounds the weights to the nearest integers so that the sum of weights is still an integer. 

For \hete with the PP objective, applying \heteCP in primitive plan computation discourages relocation of objects for which the allocation of high-quality buffers is difficult, which reduces the challenge in the buffer allocation process.
When weighting heuristic \heteCP is used, ETBM and ERBM seek primitive plans minimizing the total amount of ``difficulty on buffer allocation'' and the maximum ``concurrent difficulty on buffer allocation''.
For the TI objective, we apply \heteTI to primitive plan computation instead.

\subsection{Extended MCTS (EMCTS) for \hete}\label{sec:hete-mcts}
For MCTS, we use the heuristics with additional improvements, which yields Extended Monte Carlo Tree Search (EMCTS) for heterogeneous instances.
Specifically, for \hete with the PP objective, EMCTS applies \heteCP to the exploration term in the UCB function, replacing constant $C$ with 
$$C(1+\dfrac{w_i}{\sum^n_{i=1} w_i }).$$
EMCTS prioritizes actions making progress in sending high-weight objects to the goals (moving them to goals or clearing out their goal poses). 
We do this for two reasons. 
First, as mentioned in \ref{sec:weighting}, it is difficult to move objects with high \heteCP weights to buffers. 
And their allocated buffers are more likely to block other objects' goal poses.
Second, these high-weighted objects at goal poses are guaranteed to be collision-free from other goal poses.
For \hete with the TI objective, the reward in EMCTS is the total \heteTI weights of objects in the goal poses.

Besides the application of weighting heuristics, to tackle the challenge in \hete, we introduce two additional improvements to both MCTS and EMCTS.
\subsubsection{Reducing action space}
For each state, we only consider actions $a_i$ whose corresponding object $o_i$ are away from goal poses. 
This modification significantly reduces the action space when the state is close to the goal state and will not lead to optimality loss since the neglected actions will not make any difference to the workspace.
\subsubsection{Efficient collision checker for \hete}
In \toroi, MCTS spends most of the computation time on collision checks for buffer allocation and goal pose feasibility checks.
The original paper\cite{labbe2020monte} uses the bounding discs for collision checks.
For \hete, our collision checker has two steps: in the broad phase, we check collisions between bounding discs of object poses, which aims at pruning out most objects in the workspace that are far from the target object.
In the narrow phase, for objects whose bounding discs overlap with that of the target object, we check the collision of their polygonal approximations.
We also tried other broad-phase collision checkers, such as quadtrees and uniform space partitions. 
While these methods have good performance in large-scale collision checking, we observe limited speed-up but a large amount of overhead in maintaining the structures in \hete instance.
That is because few robotic manipulation instances need to deal with more than 100 objects simultaneously.

\section{Experimental Results for \toroi}
We implemented the algorithms of the \trlb framework in Python. Simulated experiments use environments with different density levels $\rho$, defined as the proportion of the tabletop surface occupied by objects, 
i.e., $\rho:=(\Sigma_{o_i\in \mathcal O} S_{o_i})/S_{\mathcal W}$, 
where $S_{o_i}$ is the base area of $o_i$ and $S_{\mathcal W}$ is the area of $\mathcal W$. 

\begin{figure}[h!]
\centering
\includegraphics[width=0.35\columnwidth]{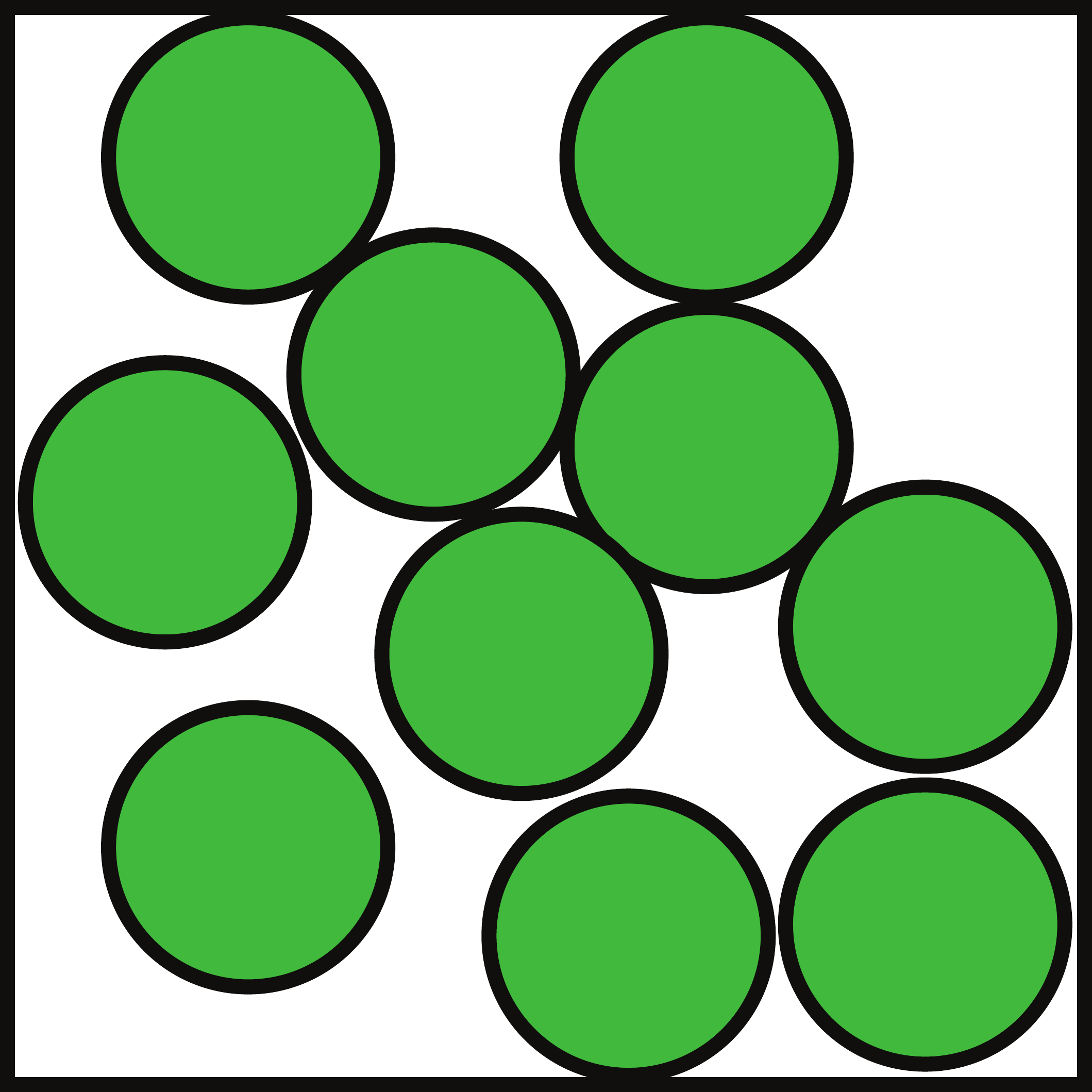}\hspace{10mm}
\includegraphics[width=0.35\columnwidth]{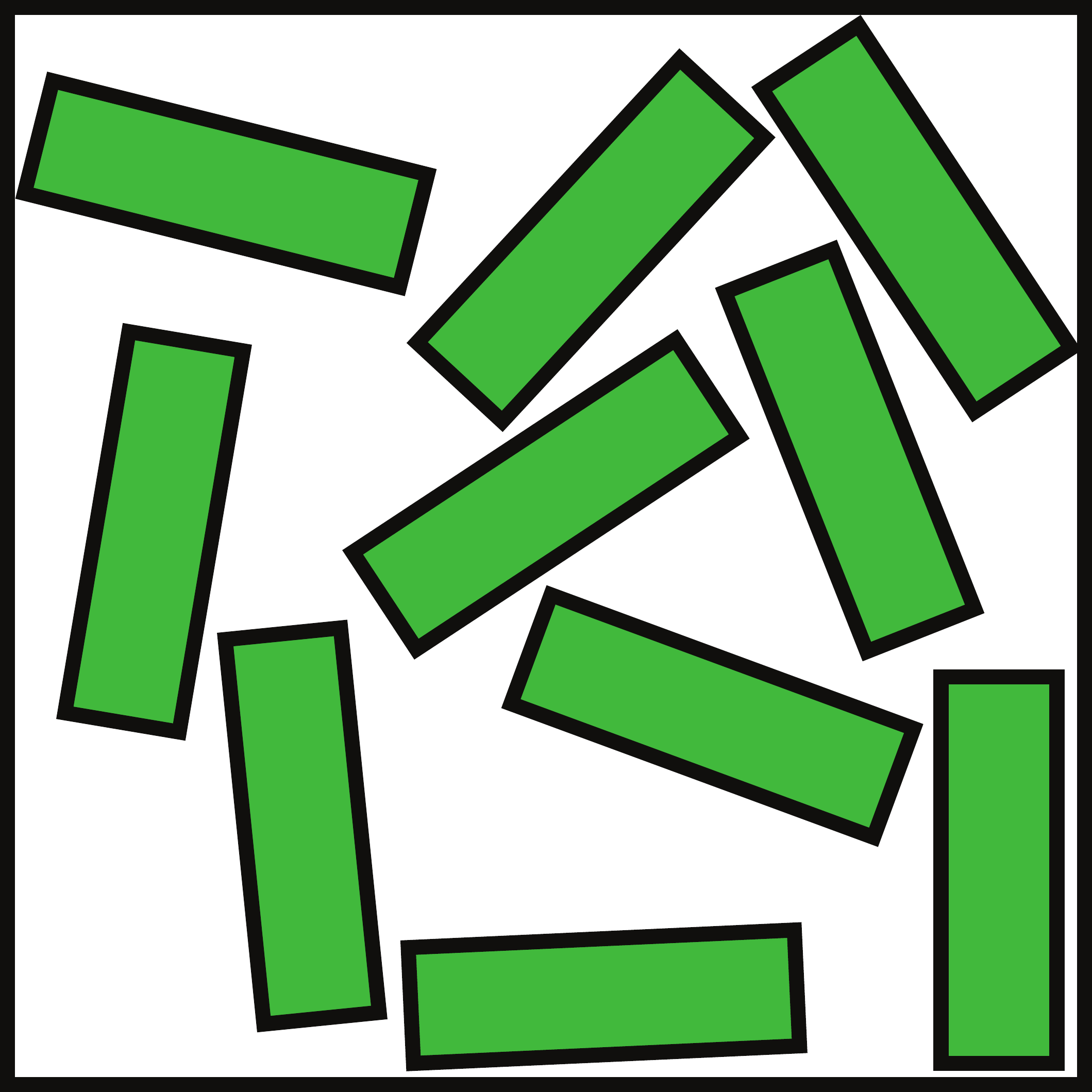}
\vspace{-.05in}
    \caption{[left] 10 cylinders with $\rho=0.5$, [right] 10 cuboids with $\rho=0.4$.}
    \label{fig:trlb-density}
\end{figure}

\ref{fig:trlb-density} shows a dense 10-cylinder arrangement with $\rho=0.5$, and a dense 10-cube arrangement with $\rho=0.4$. The experiments are executed on an Intel$^\circledR$ Xeon$^\circledR$ CPU at 3.00GHz. Each data point is the average of $30$ cases except for unfinished trials, given a time limit of $300$ sec. per case.

\subsection{Ablation Study for Cylindrical Objects}
We first present experiments with cylindrical objects to compare lazy buffer generation algorithms given different options, including:
(1) Primitive plan computation: running buffer minimization (RBM), total buffer minimization (TBM), random order (RO); (2) Buffer allocation methods: optimization (OPT), sampling (SP); (3) High level planners: one-shot (OS), forward search tree (ST), bidirectional search tree (BST); and (4) With or without preprocessing (PP).
%
%For primitive plans, we plans minimizing the running action size (RBM), total buffer size (TAM) and with random object orderings(RO); for the buffer allocation process, we compare the optimization method (OPT) and sampling (SP); 
%for high-level planners, we compare one method that does not accept partial plans (OS) and ones that maintains a one-directional search tree (ST) or bi-directional search tree (BST).
Here, the one-shot (OS) planner is using primitive plans and buffer allocation (\ref{sec:LazyBufferGeneration}) without tree search (\ref{sec:PartialPlan}). In OS, we attempt to compute a feasible rearrangement plan up to $30|\mathcal O|$ times before announcing a failure.
Notice that at least $|\mathcal O|$ actions are required for solving any instance. 
%Besides that, we also evaluate the effectiveness of the preprocessing (PP).

A full \trlb algorithm is a combination of components, e.g.,  RBM-SP-BST stands for using the primitive plans that minimize running buffer size,  performing buffer allocation by sampling, maintaining a bidirectional search tree, and doing so without preprocessing.
%Following the short-terms, methods can be represented with the option combinations. 

For evaluation, we first compare the primitive plan computation options, using sampling-based buffer allocation, bidirectional tree search and no preprocessing.  TBM and RBM plans are computed using dynamic programming solvers from \ref{chap:rbm}. The results are shown in \ref{fig:Primitive}. 
Even though plans generated by TBM-SP-BST are slightly shorter than RBM-SP-BST, TBM-SP-BST is less scalable as either the density level or the number of objects in the workspace increases. Compared to RBM plans, individual RO plans can be generated almost instantaneously but there is little benefit in computation for the overall algorithm. %Noting that DFDP deals with strongly connected components in $G$ one after another,
%the RBM plan based buffer allocation tends to deal with problems with a smaller scale,
%while objects in RO plans stay in the buffers for longer time. 
The results indicate that RBM should be used for primitive plan computation as it results in efficient and high-quality solutions.

\begin{figure}[h!]
    \vspace{2mm}
\centering
    \begin{overpic}[width=1\columnwidth]{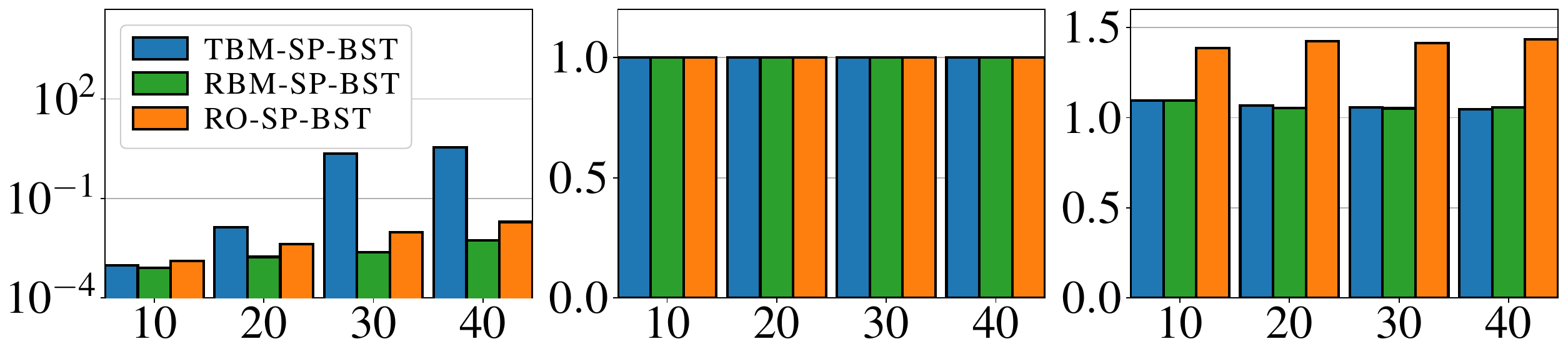}
    \put(-1.5,-22.5){ \includegraphics[width=1\columnwidth]{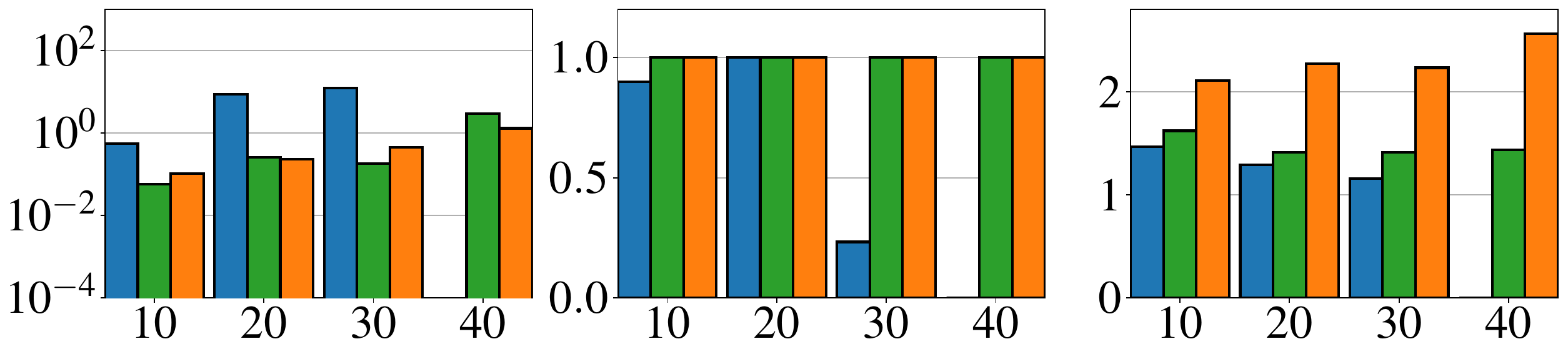}}
    \end{overpic}
    \vspace{25mm}
    \caption{Comparison of primitive planners with $10$-$40$ cylinders and density levels $\rho=0.3$ (top), 0.5 (bottom) (left: computation time in seconds; middle: success rate; right: number of actions as multiples of $|\mathcal O|$).}
    \label{fig:Primitive}
\end{figure}

In \ref{fig:BufferAllocation}, buffer allocation methods are compared using the RBM primitive planner and the OS high-level planner. Optimization-based allocation guarantees completeness and generates high-quality plans but it is computationally expensive. When $\rho=0.5$, the success rate tends to be low in instances with a small number of objects.  That is because for the given density level, the smaller the 
number of objects, the larger the object size relative to the environment, and the smaller the configuration space size relative to the environment. Thus, precisely allocating buffer locations with OPT is helpful in these cases.

\begin{figure}[h!]
    \vspace{2mm}
\centering
    \begin{overpic}[width=1\columnwidth]{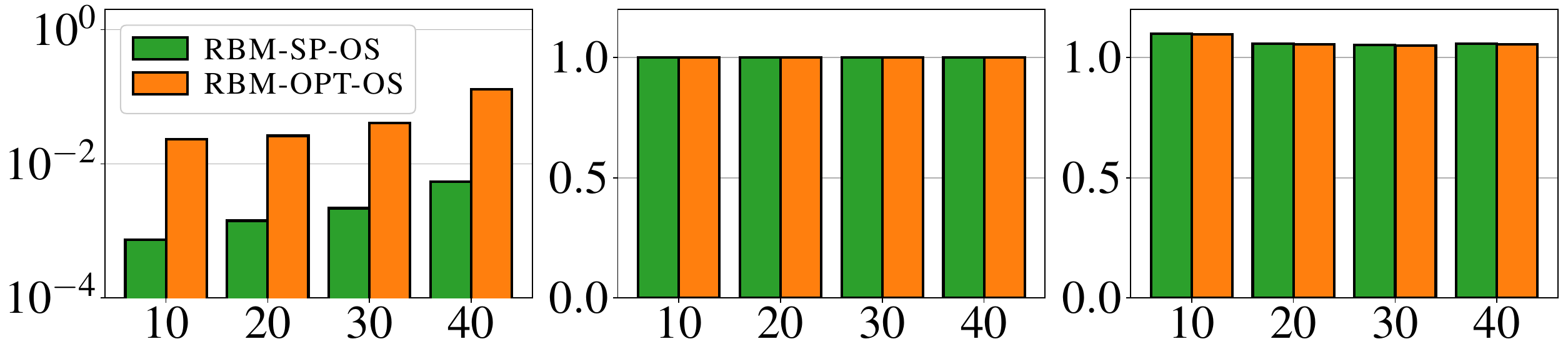}
    \put(-1.5,-22.5){ \includegraphics[width=1\columnwidth]{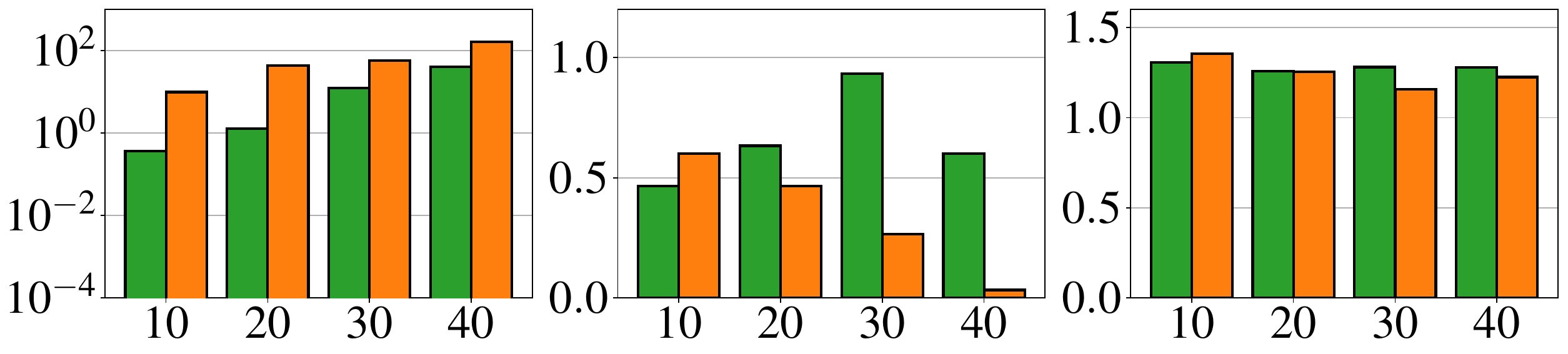}}
    \end{overpic}
    \vspace{25mm}
    \caption{Comparison of buffer allocation methods with $10$-$40$ cylinders at density levels $\rho=0.3$ (top), 0.5 (bottom) (left: computation time in seconds; middle: success rate; right: number of actions as multiples of $|\mathcal O|$).}
    \label{fig:BufferAllocation}
\end{figure}

The effectiveness of the high-level planners and preprocessing are shown in \ref{fig:HighLevel}, which suggests that ST, BST and preprocessing are all effective in increasing success rate in dense environments. In addition, preprocessing significantly speeds up computation in large scale dense cases at the price of extra actions to execute preprocessing. By simplifying the dependency graph with preprocessing, less time is needed to compute a primitive plan.
\begin{figure}[h!]
    \vspace{1mm}
\centering
    \begin{overpic}[width=1\columnwidth]{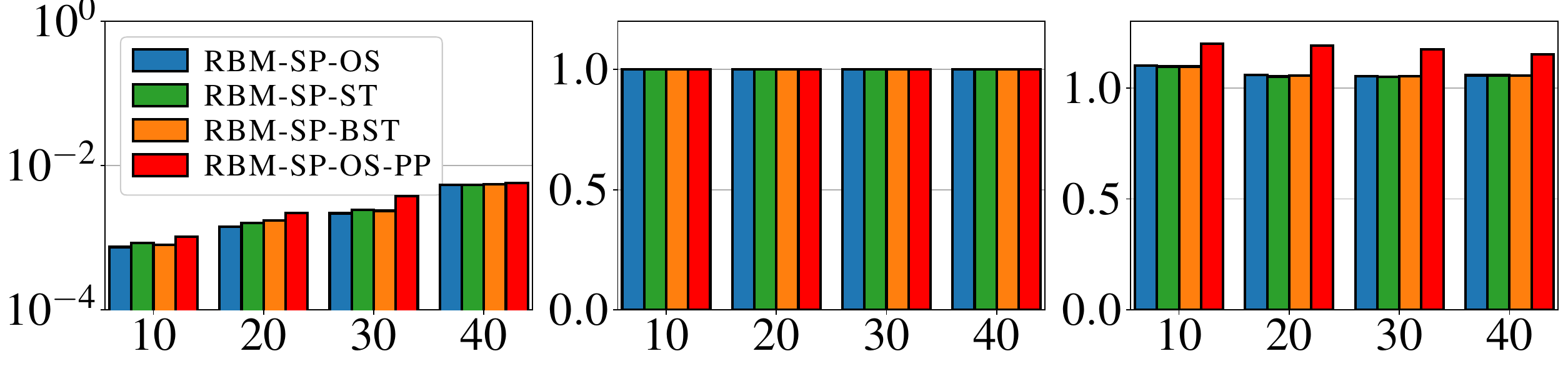}
    \put(-1.5,-22.5){ \includegraphics[width=1\columnwidth]{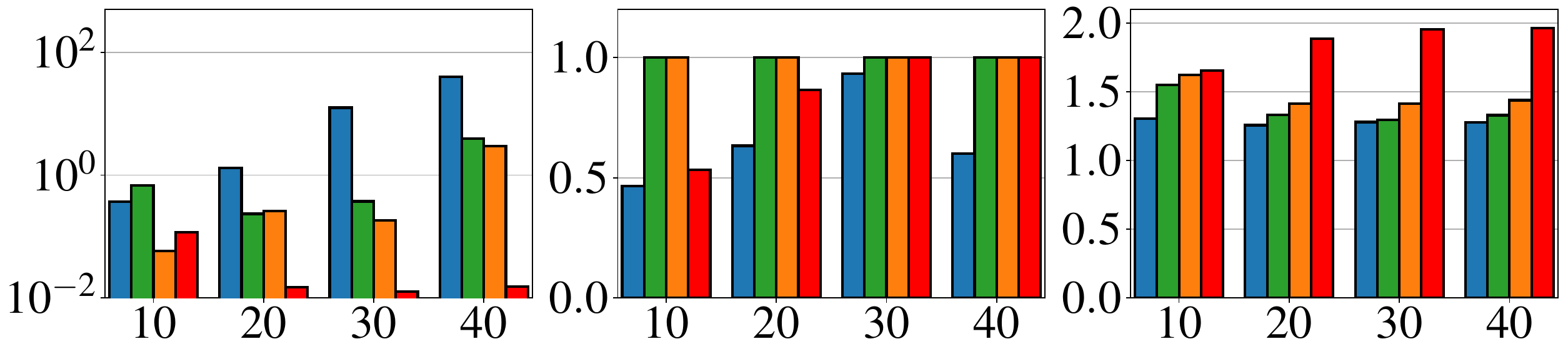}}
    \end{overpic}
    \vspace{25mm}
    \caption{Comparison among high-level frameworks and OS with preprocessing. There are 10-40 cylinders in the workspace at density levels $\rho=0.3$ (top), $0.5$ (bottom) (left: computation time in seconds; middle: success rate; right: number of actions as multiples of $|\mathcal O|$).}
    \label{fig:HighLevel}
\end{figure}

The robustness of ST and BST are further evaluated with ``dense-small'' instances where a few objects are packed densely (\ref{fig:HighLevelDenseSmall}). The bidirectional search tree has a higher success rate in these cases, especially in 5-object instances.

In addition to the above evaluations, we also tried integrating the preprocessing into the BST framework (RBM-SP-BST-PP), which speeds up the computation:  for $60$-object instances with $\rho=0.5$,  only $63\%$ of them can be solved by RBM-SP-BST in 300 seconds. All of them, however, can be solved together with the preprocessing and the solution time averaged 0.29 seconds. Similarly to the results in \ref{fig:HighLevel}, preprocessing makes the solution plan much longer than necessary (needs around $30\%$ more actions than RBM-SP-BST). Based on the analysis on computation time, success rate, and solution quality, RBM-SP-BST is the best overall combination, and preprocessing significantly speeds up the solver with a reduction of solution quality.

\begin{figure}[h!]
    \vspace{2mm}
\centering
\includegraphics[width=0.6\columnwidth]{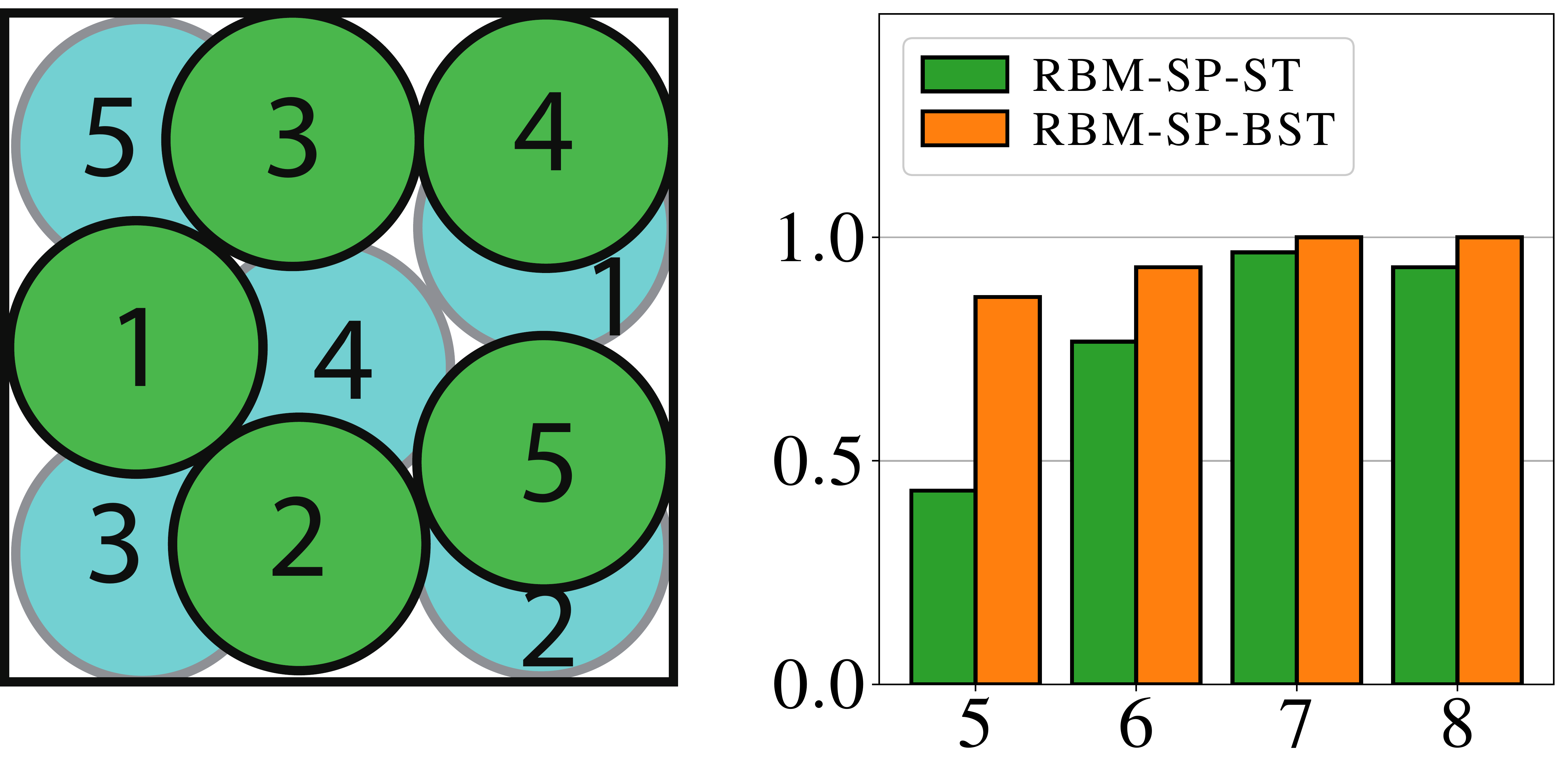}
\vspace{-0.1in}
    \caption{Comparison between ST and BST frameworks with ``dense-small'' instances where 5-8 cylinders packed in the environment and density level $\rho=0.5$. [left] An example with 5 cylinders, [right] success rate of methods.}
    \label{fig:HighLevelDenseSmall}
\end{figure}

\subsection{Comparison with Alternatives for Cylindrical Objects}
We compare the proposed method RBM-SP-BST with BiRRT(fmRS) \cite{krontiris2016efficiently} and a MCTS planner \cite{labbe2020monte}, which, to the best of our knowledge, are state-of-the-art planners for \toroi. The MCTS planner is a C++ solver, while the other two methods are implemented in Python. Besides success rate, solution quality, and computation time, 
we also compare the number of collision checks which are time-consuming in most planning tasks.  In \ref{fig:LargeScale}, we compare the methods in large scale problems with $\rho=0.3$. The success rate is $100\%$ for all.
Our method, RBM-SP-BST, avoids repeated collision checks due to the use of the dependency graph. BiRRT(fmRS), which only uses dependency graphs locally, spends a lot of time and conducts a lot of collision checks to generate random arrangements. MCTS generates solutions with similar optimality but does so also with a lot of collision checking, which slows down computation.  We note that a value of $1$ in the right figure (number of actions) is the minimum possible, so both RBM-SP-BST and MCTS compute high-quality solutions, which RBM-SP-BST does slightly better. To sum up, in sparse large scale instances, RBM-SP-BST is two magnitudes faster and conducts much fewer collision checks than the alternatives.

\begin{figure}[h!]
    \vspace{1mm}
\centering
    \begin{overpic}[width=1\columnwidth]{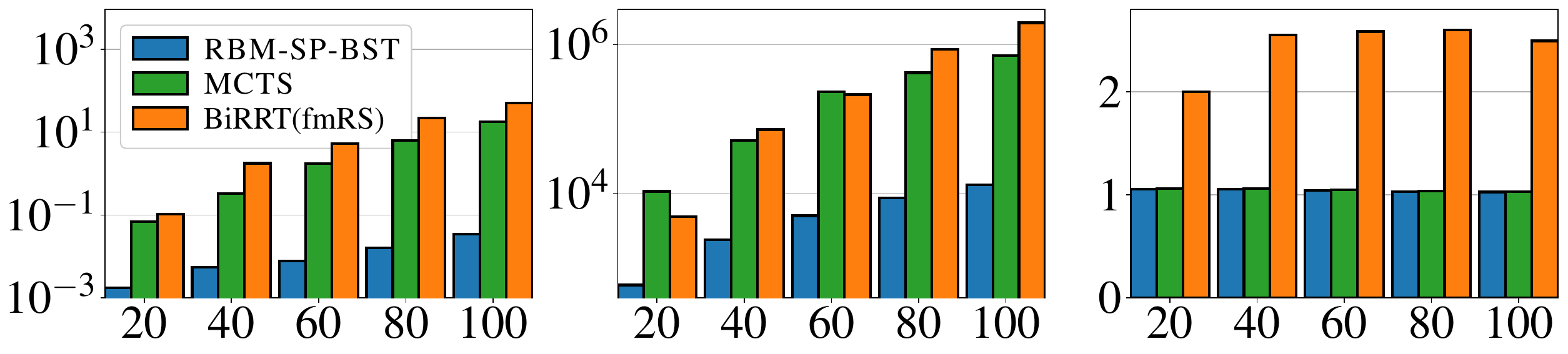}
    \end{overpic}
            \vspace{-.25in}
    \caption{Comparison of algorithms with 20-100 cylinders at density level $\rho=0.3$ (left: computation time in seconds; middle: number of collision checks; right: number of actions as multiples of $|\mathcal O|$).}
    \label{fig:LargeScale}
\end{figure}

Next, we compare the methods in ``dense-small'' instances (\ref{fig:DenseSmallComaprison}). Here, RBM-SP-BST is the only method that maintains high success rate in these difficult cases.

\begin{figure}[h!]
    % \vspace{2mm}
\centering
    \includegraphics[width=1\columnwidth]{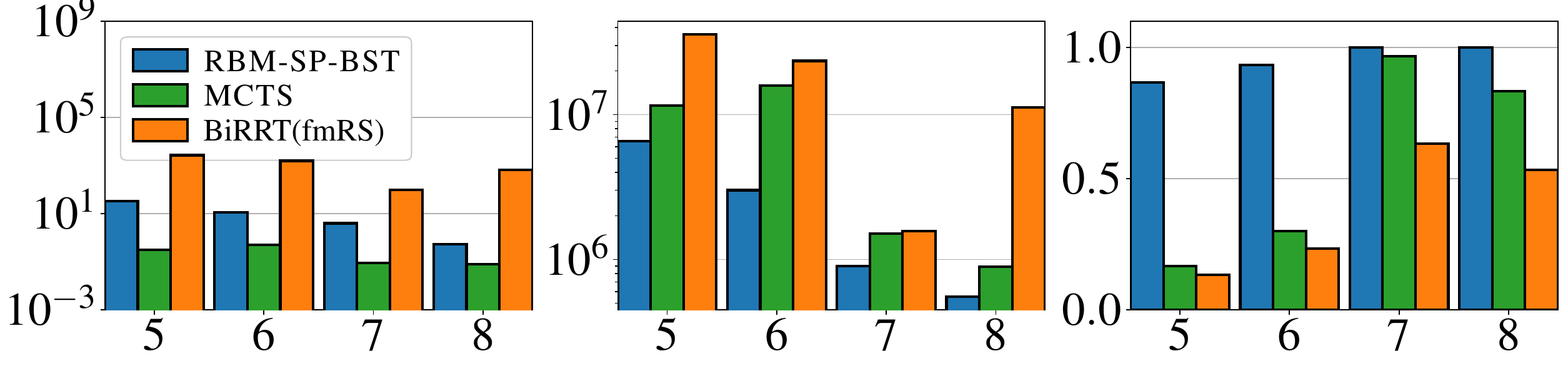}
            \vspace{-.25in}
    \caption{Comparison of methods on ``dense-small'' instances where 5-8 objects are packed in an environment with $\rho=0.5$ (left: computation time in seconds; middle: number of collisions; right: success rate).}
    \label{fig:DenseSmallComaprison}
\end{figure}

We further compare the performance of RBM-SP-BST and MCTS in lattice rearrangement problems recently studied in the literature \cite{yurearrangement}. 
An example with 15 objects is shown in \ref{fig:lattice}[left]. 
In the start and goal arrangements, gaps between adjacent objects are set to be 0.01 object radius, and thus buffer allocation is challenging for sampling-based methods. While MCTS tries all actions on each node, RBM-SP-BST can detect the embedded combinatorial object relationship via the dependency graph and therefore needs less buffer allocations. As shown in \ref{fig:lattice}[right], RBM-SP-BST has much higher success rate in lattice rearrangement tasks.

\begin{figure}[h!]
    % \vspace{2mm}
\centering
\includegraphics[width=0.6\textwidth]{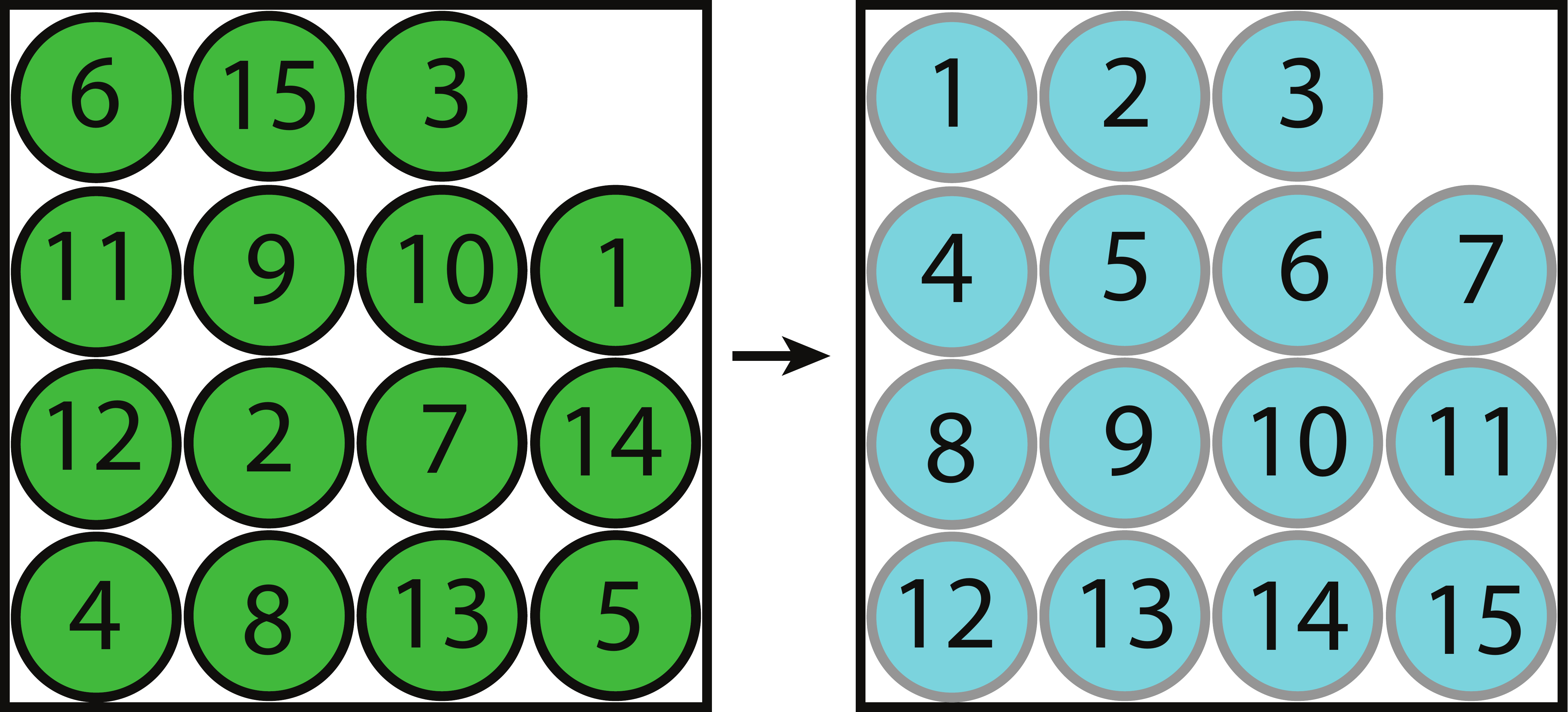}\hspace{2mm}
    \includegraphics[width=0.3\columnwidth]{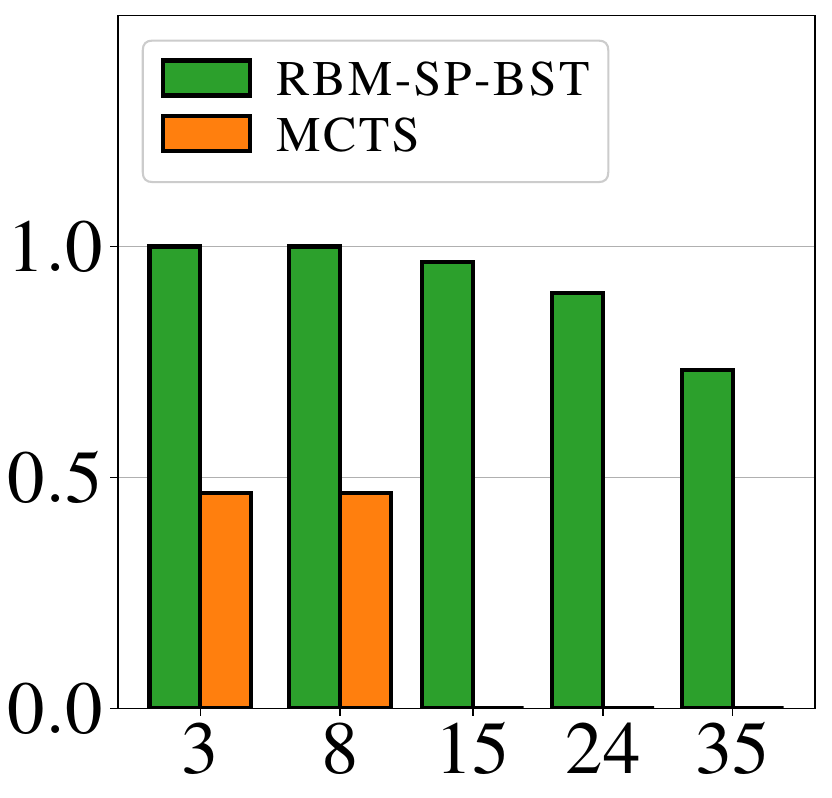}
        \vspace{-.05in}
    \caption{Comparison among methods in lattice instances with 3-35 objects. [left] lattice example; [right] success rate}
    \label{fig:lattice}
\end{figure}
\vspace{-1mm}
\subsection{Cuboid Objects}
Because the MCTS solver only supports cylindrical objects, we only compare RBM-SP-BST and BiRRT(fmRS) in the cuboids setup (\ref{fig:trlb-density}[right]).
When $\rho=0.3$, RBM-SP-BST computes high quality solutions efficiently, 
while BiRRT(fmRS) can only solve instances with up to 20 cuboids.
We mention that, when $\rho=0.4$, BiRRT(fmRS) cannot solve any instance, 
but RBM-SP-BST can solve 50-object rearrangement problems in 28.6 secs on average.

\begin{figure}[h!]
    % \vspace{2mm}
\centering
    \includegraphics[width=1\columnwidth]{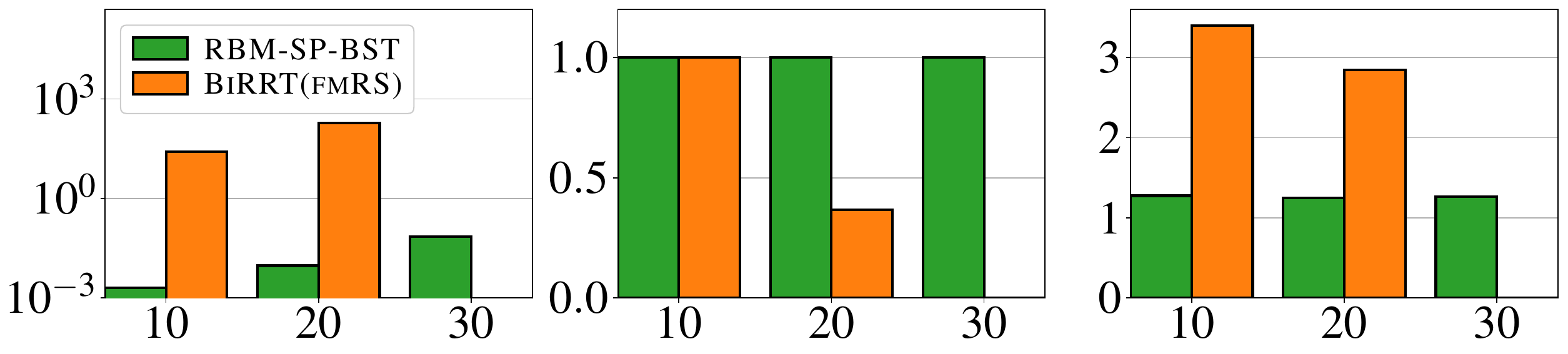}
        \vspace{-.25in}
    \caption{Comparison between methods in cuboid instances with $10$-$30$ cuboids and $\rho=0.3$ (left: computation time in seconds; middle: success rate; right: number of actions as multiples of $|\mathcal O|$).}
    \label{fig:StickComaprison}
\vspace{-1mm}
\end{figure}

\subsection{Hardware Demonstration}
\begin{figure}[h!]
    \vspace{-1mm}
    \centering
    \includegraphics[width=\columnwidth]{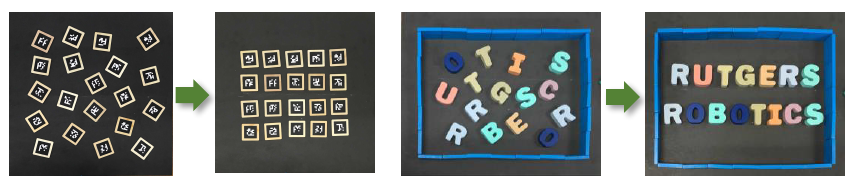}
    \vspace{-.25in}
    \caption{Experimental settings in addition to that of  \ref{fig:TRLBApplicationExample}.}
    \label{fig:HardwareSetups}
    \vspace{-1mm}
\end{figure}

We further demonstrate that the plans computed by \trlb can be readily executed on real robots in a complete vision-planning-control pipeline. Our hardware setup consists of a UR-5e robot arm, an OnRobot VGC 10 vacuum gripper, and an Intel RealSense D435 RGB-D camera. As shown in the accompanying video\footnote{The video is available \href{https://www.youtube.com/watch?v=hegO3JenKjo&t=4s}{\bl{online}}. More demonstrations are available \href{https://github.com/gaokai15/gaokai15.github.io/assets/53358252/315865f0-5fc2-4ee8-8805-90a62e5effe6}{\bl{here}}.}, \trlb solves all attempted instances (\ref{fig:TRLBApplicationExample} and \ref{fig:HardwareSetups}), which involves concave objects, in an apparently natural and efficient manner. 
%
%examine \trlb with an experimental hardware platform, whose overview is presented in \ref{fig:TRLBApplicationExample}.
%The manipulator we use is an UR-5e robot arm with an OnRobot VGC 10 vacuum gripper.
%To determine the object positions in the workspace, an Intel RealSense D435 is mounted above a tabletop, 
%providing a top view of the entire workspace.
%And we place physical markers at the four corners of the workspace to compute the object poses.

\section{Experimental Results for \hete}\label{sec:conclusion}
The proposed algorithms' performance is evaluated using simulation and real robot experiments.
%\jy{Need to explain the implementation (e.g., the language used) and the machine used to evaluate.}
All methods are implemented in Python, and the experiments are executed on an Intel$^\circledR$ Xeon$^\circledR$ CPU at 3.00GHz. Each data point is the average of $50$ cases except for unfinished trials, given a time limit of $600$ seconds per case. 

\subsection{Numerical Results in Simulation}
Simulation experiments are conducted in two scenarios:
\begin{enumerate}[leftmargin=4.5mm]
    \item (RAND) In the RAND (random) scenario, for each object, the base shape is randomly chosen between an ellipse and a rectangle. The aspect ratios and the areas are uniformly random values in $[1.0,3.0]$ and $[0.05,1]$, respectively.
    The areas of the objects are then normalized to a certain sum $S$, whose value depends on specific problem setups. Examples of the RAND scenario are shown in the first row of \ref{fig:exp_example}.
    This setup is used to examine algorithm performance in general \hete instances.
    \item (SQ) In the SQ (squares) scenario, there are $2$ large squares and $|\mathcal O|-2$ small squares. The area ratio between the large and small squares is $9:1$. This setup simulates the rearrangement scenario where object size varies significantly.
    Examples of arrangements in the SQ scenario are shown in the second row of \ref{fig:exp_example}.
\end{enumerate}

For each test case, object start and goal poses are uniformly sampled in the workspace for both scenarios.
The sampling process for a single test case is repeated in the event of pose collision until a valid configuration is obtained.  
%If the sampled object pose overlaps with existing objects or workspace boundaries, we sample another pose for the object until a collision-free pose is found.
We use the \emph{density level $\rho$} to measure how cluttered a workspace is. 
It is defined as $(\sum_{o_i\in \mathcal O} S_i)/S_W$, where $S_i$ and $S_W$ represents the size of $o_i$ and the workspace respectively. 
In other words, $\rho$ represents the proportion of the workspace occupied by objects. 
\ref{fig:exp_example} shows examples of RAND and SQ scenarios for $\rho=0.2$-$0.4$, where $\rho = 0.4$ is already fairly dense (e.g., 3rd column \ref{fig:exp_example}).
We limit the maximum number of objects in a test case to $40$ as typical \hete cases are unlikely to have more objects. 

\begin{figure}[h]
    \centering
    \vspace{1mm}
    \includegraphics[width=0.9\textwidth]{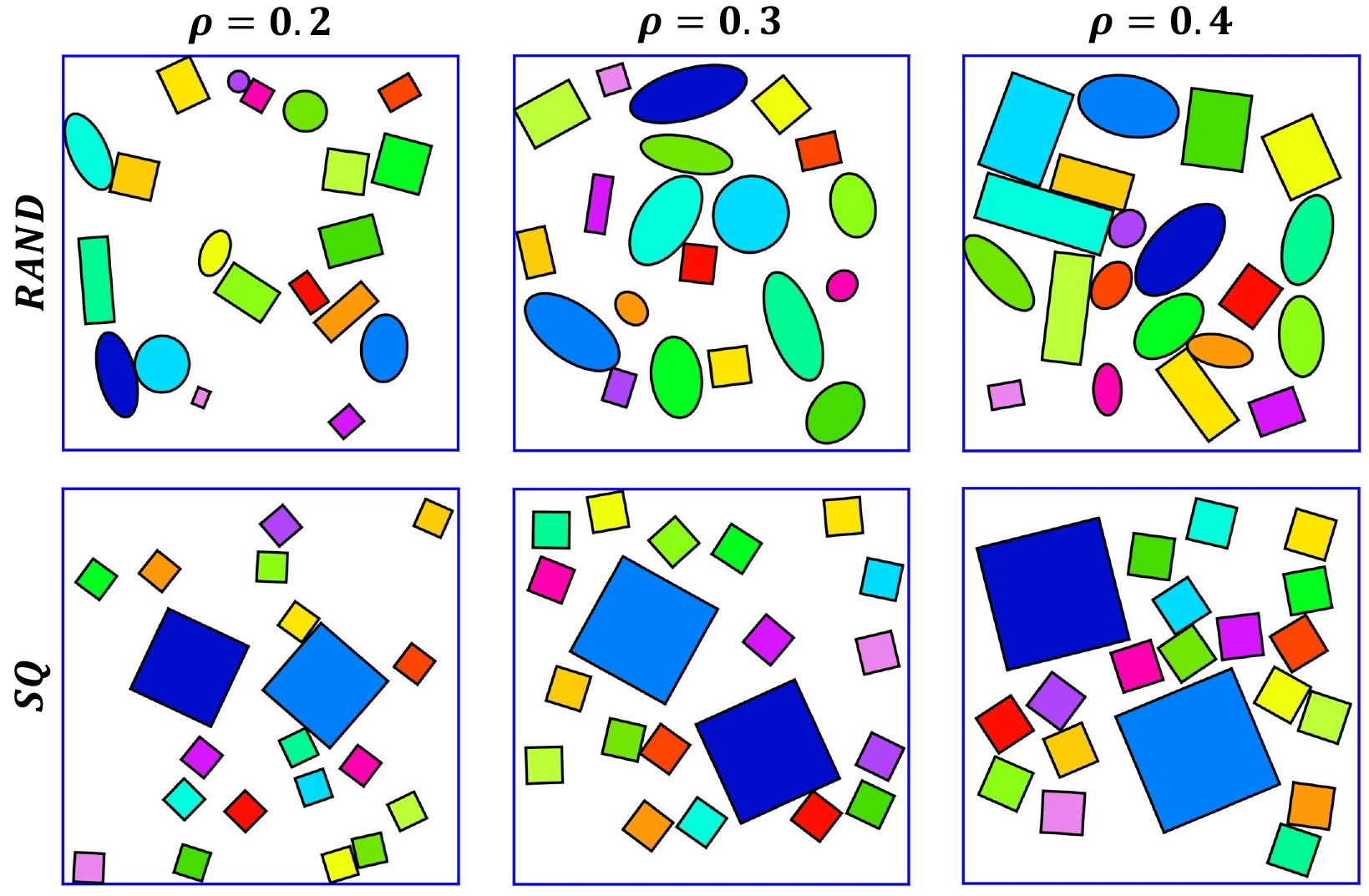}
    \caption{Illustrations of some simulation test cases at different densities. [top] RAND instances under different density levels. [bottom] SQ instances under different density levels.}
    \label{fig:exp_example}
\end{figure}
%\jy{For all figures in this section, the fonts are somewhat small. They hurt my eyes a bit. If there is time, please make them one size larger.}

\begin{figure}[h]
    \centering
    \includegraphics[width=0.96\textwidth]{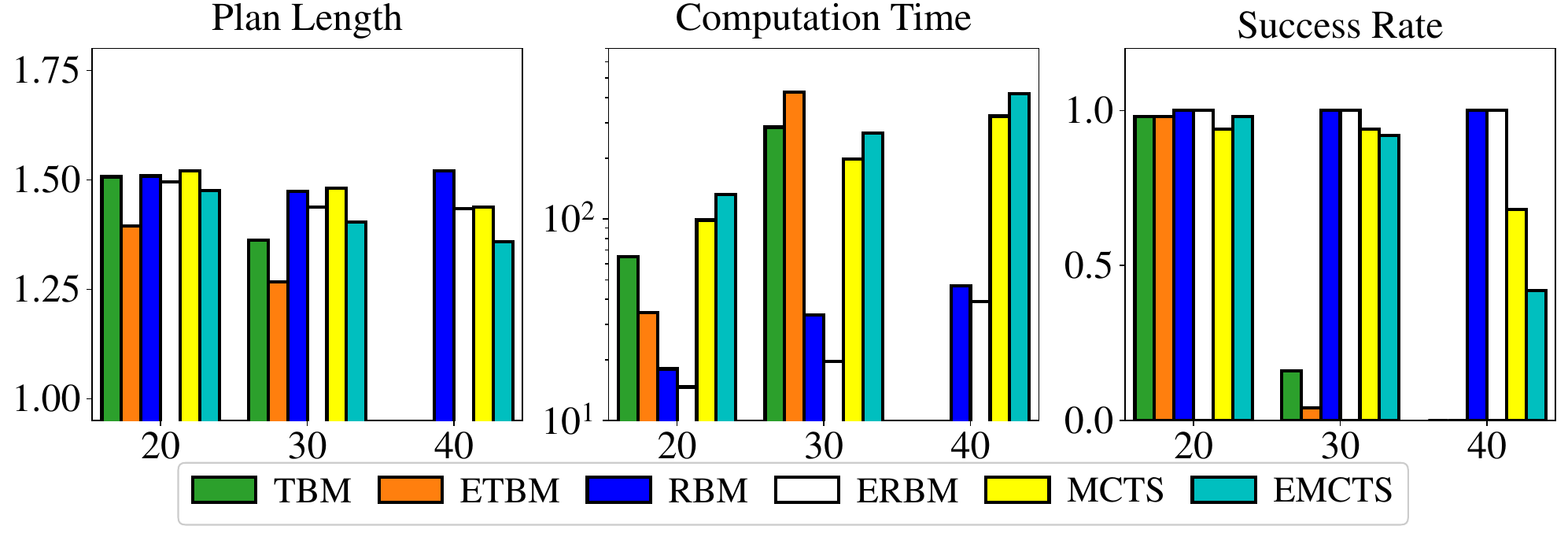}
    \caption{Algorithm performance in RAND instances with $\rho=0.4$ and 20-40 objects. [Left] Plan length as multiplies of $|\mathcal O|$. [Middle] Computation time (secs). [Right] Success rate.}
    \label{fig:rd_pp}
\end{figure}
In \ref{fig:rd_pp}, we first evaluate algorithm performance in dense RAND instances ($\rho=0.4$) under the PP objective. 
% In these instances, none of the TBM methods ($TBM$ and $ETBM$) or MCTS methods ($MCTS$ and $EMCTS$) is scalable as the number of objects increases.
Compared with the original TBM, ETBM shows a $7.8\%$ improvement in solution quality with only $52\%$ computation time, when both TBM and ETBM have $98\%$ success rates in 20-object instances.
Suffering from the expansion of solution space, neither TBM or ETBM is scalable as the number of objects increases. 
With \heteCP, both ERBM and EMCTS show clear advantages compared to the original methods (RBM and MCTS) in solution quality as the number of objects increases. 
In 40-object instances, plans computed by ERBM and EMCTS save $3.6$ and $3.2$ actions over those by RBM and MCTS, respectively.
While ERBM spends less time than RBM to solve RAND instances, EMCTS spends slightly more time on computation since the weights are added to the exploration term.
In general, ETBM computes solutions with the best quality in small-scale instances, ERBM is most scalable as $|\mathcal O|$ increases, and EMCTS balances between optimality and scalability,  computing high-quality solutions in 30-object instances with a high success rate.

\begin{figure}[h]
    \centering
    %\vspace{2mm}
    \includegraphics[width=\columnwidth]{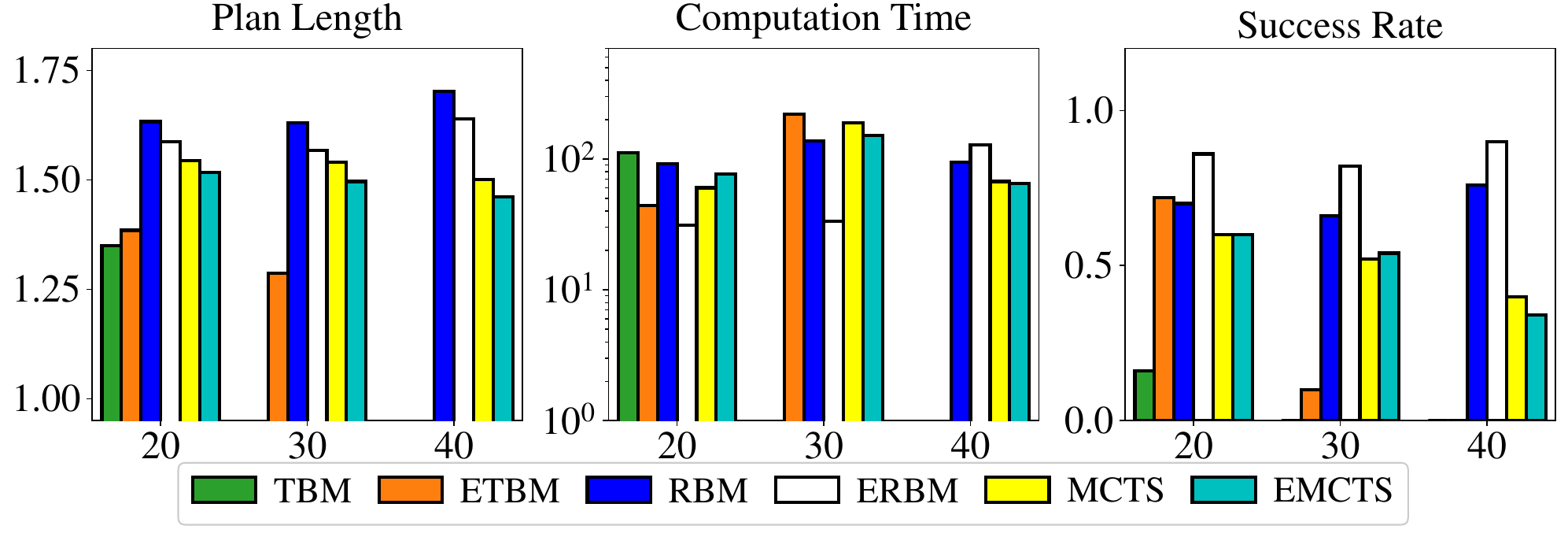}
    \caption{Algorithm performance in SQ instances with $\rho=0.4$ and 20-40 objects. [Left] Plan length as multiplies of $|\mathcal O|$. [Middle] Computation time (secs). [Right] Success rate.}
    \label{fig:sq_pp}
\end{figure}

In SQ instances, it is difficult to displace the large squares in the cluttered workspace. 
\ref{fig:sq_pp} shows the algorithm performance in SQ instances with $\rho=0.4$ under the PP objective.
The results suggest that ETBM with \heteCP effectively increases the success rate from $16\%$ to $72\%$ when $|\mathcal O|=20$.
Similarly, the success rate of ERBM is around $15\%$ higher than RBM when $|\mathcal O|=20,30,$ and $40$.
While ERBM solves more test cases, the computed solutions are even $5\%-7\%$ shorter than RBM.
As for MCTS methods, EMCTS computes $3\%-5\%$ shorter paths than EMCTS with a similar success rate and computation time.

Besides dense instances, we also evaluate the efficiency gain of \heteCP at different density levels when $|\mathcal O|=20$.
The results of RAND and SQ scenarios are shown in \ref{fig:rd_density} and 
\ref{fig:sq_density} respectively.
In sparse RAND instances ($\rho=0.2$ or $0.3$), ETBM and ERBM return plans with similar optimality without additional computational overheads.
When $\rho=0.3$, EMCTS spends additional $5$ seconds for exploration and receives around $3\%$ improvements on path length.

\begin{figure}[h]
    \centering
    \includegraphics[width=0.96\textwidth]{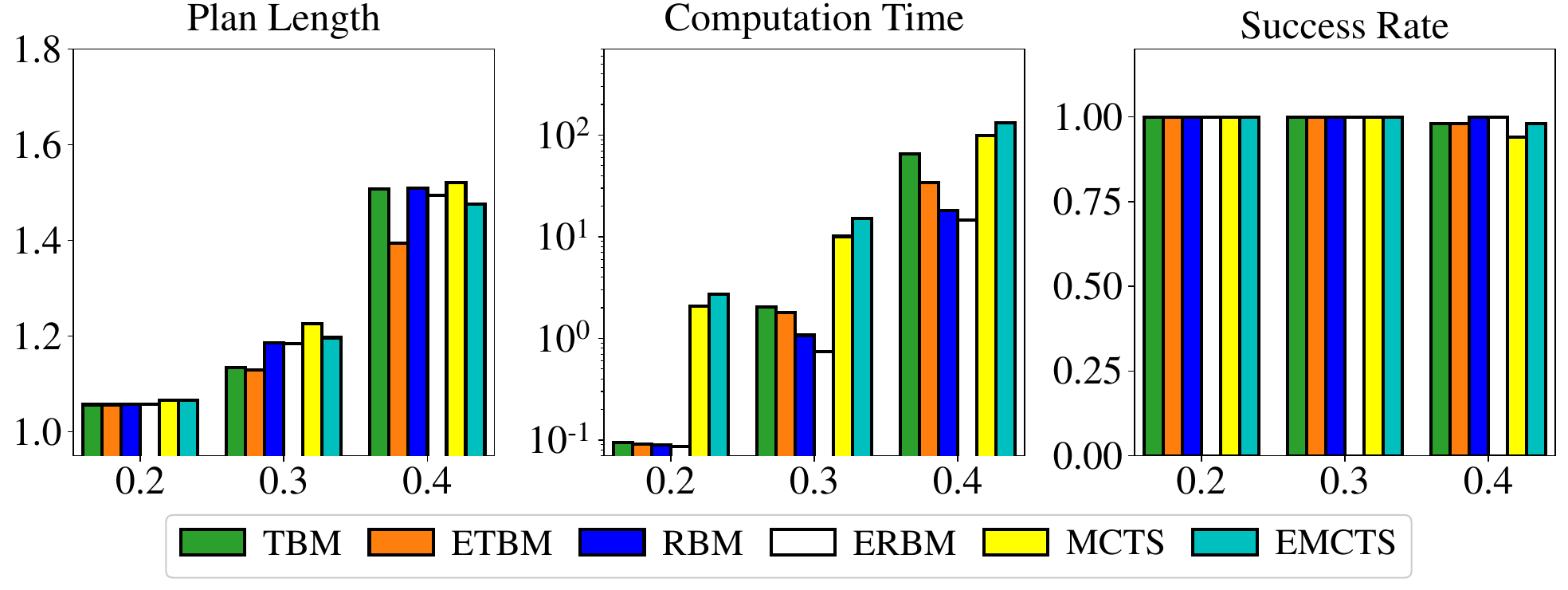}
    \caption{Algorithm performance in RAND instances under different density levels and $|\mathcal O|=20$. [Left] Plan length as multiplies of $|\mathcal O|$. [Middle] Computation time (secs). [Right] Success rate.}
    \label{fig:rd_density}
\end{figure}

In sparse SQ instances, \heteCP helps improve computation time, which contributes to an increase in success rate.
Compared with TBM, ETBM saves $70\%$ computation time when $\rho=0.2$ and $0.3$ and increases the success rate from $82\%$ to $98\%$ when $\rho=0.2$.
Similarly, compared with RBM, ERBM saves $56\%$ and $82\%$ computation time when $\rho=0.2$ and $0.3$.
In ETBM and ERBM, \heteCP speeds up computation without reducing solution quality.
Compared with MCTS, EMCTS computes $5\%$ shorter plans, saves $33\%$ computation time, and increases the success rate by $16\%$ when $\rho=0.3$.

\begin{figure}[h]
    \centering
    \vspace{1mm}
    \includegraphics[width=0.96\textwidth]{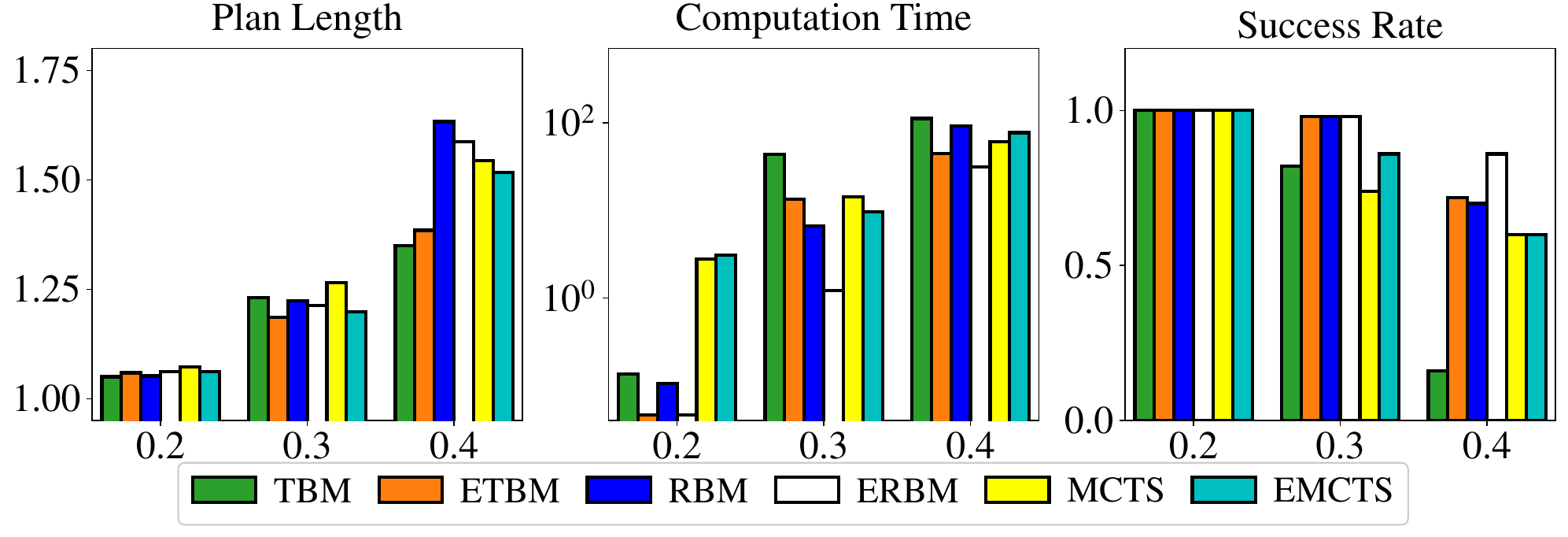}
    \caption{Algorithm performance in SQ instances under different density levels and $|\mathcal O|=20$. [Left] Plan length as multiplies of $|\mathcal O|$. [Middle] Computation time (secs). [Right] Success rate.}
    \label{fig:sq_density}
\end{figure}

In \ref{fig:rand_effort} and \ref{fig:squares_effort}, we show the algorithm performance of \hete with TI objective in RAND and SQ instances respectively.
$\rho=0.3$ in the test cases.
In RAND instances, where only around $10\%$ objects need buffer locations, the total task impedance of computed plans is similar.
Compared with original methods (TBM, RBM, and MCTS), ETBM, ERBM, and EMCTS reduce task impedance by $1\%-7\%$.
In SQ instances, object impedance varies greatly. In this case, methods with \heteTI compute much better solutions.
Compared with TBM, ETBM reduce task impedance by over $30\%$.
Similarly, ERBM reduces task impedance by around $27\%$.
Compared with MCTS, EMCTS reduces task impedance by $6\%-15\%$ but with additional overhead.

\begin{figure}[h]
    \centering
    \vspace{2mm}
    \includegraphics[width=0.96\textwidth]{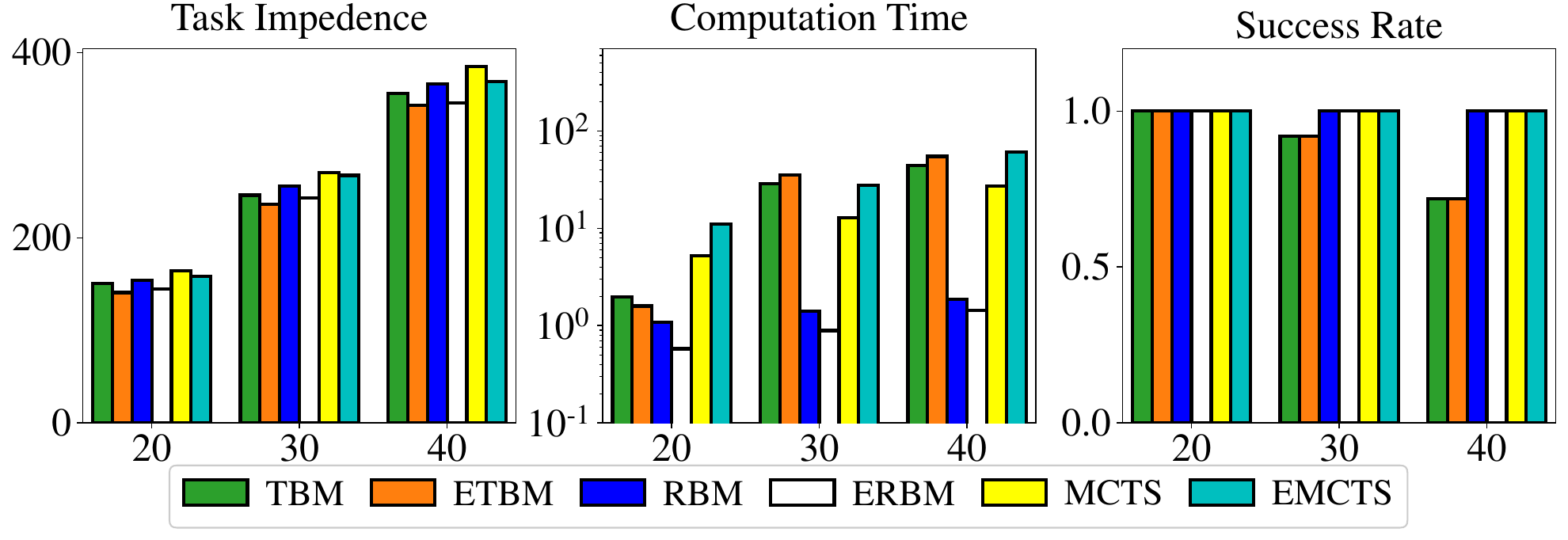}
    \caption{
    Algorithm performance in RAND instances 
    with TI objective. $\rho=0.3$, $20 \leq |\mathcal O|\leq 40$.
    [Left] Total task impedance. [Middle] Computation time (secs). [Right] Success rate.}
    \label{fig:rand_effort}
\vspace{-2mm}
\end{figure}

\begin{figure}[h]
    \centering
    \includegraphics[width=0.96\textwidth]{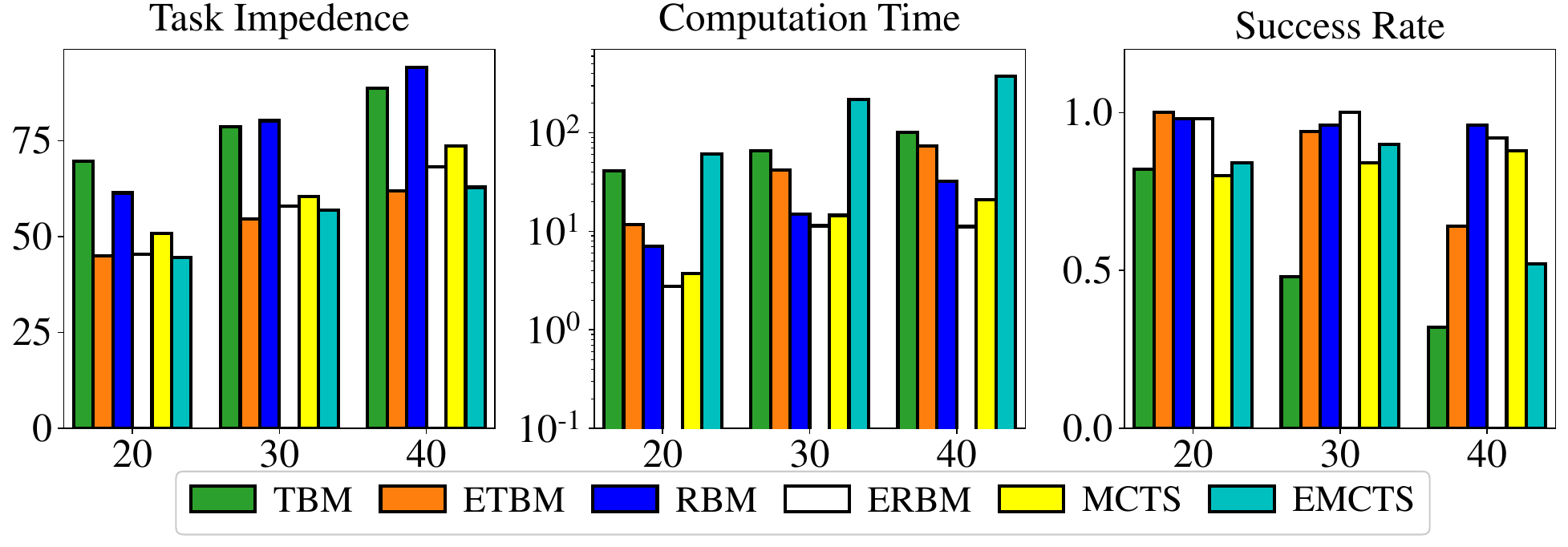}
    \caption{Algorithm performance in SQ instances 
    with TI objective. $\rho=0.3$, $20 \leq |\mathcal O|\leq 40$.
    [Left] Total task impedance. [Middle] Computation time (secs). [Right] Success rate.}
    \label{fig:squares_effort}
\end{figure}

\subsection{Demonstrations in Simulation and Real-World}
In the accompanying video 
\footnote{The video is available \href{https://github.com/gaokai15/gaokai15.github.io/assets/53358252/87e438da-7831-4042-a8e5-deb56c7aacd3}{\bl{online}}.}
, we further demonstrate that our algorithms can be applied to practical instances with general-shaped objects in both simulation and the real world. 
For simulation instances, the rearrangement plans are executed in Pybullet. 
For real-world instances, the objects are detected with an Intel RealSense D405 RGB-D camera and manipulated with an OnRobot VGC 10 vacuum gripper on a UR-5e robot arm. Since object recognition and localization are not the focus of this study, we use fiducial markers for realizing these. 
The workspaces for rearrangement are rather confined (in \ref{fig:demo}, the white areas for the first two and the wooden board for the third), translating to very high object densities, between $35$-$37\%$.
In the SQ instance with 10 wooden planks, our proposed methods ETBM and ERBM offer advantages over TBM and RBM in both solution quality and computation time, saving $1$ and $4$ actions, respectively. 
Furthermore, EMCTS is able to find a $15$-action plan in just $13.70$ secs while MCTS fails to find a solution even after $2$ minutes of computation.
%\jy{The $30$-$50\%$ is just my guess. We should replace them with actual numbers.}

\begin{figure}[h]
    \centering
    % \vspace{1mm}
    \includegraphics[width=\columnwidth]{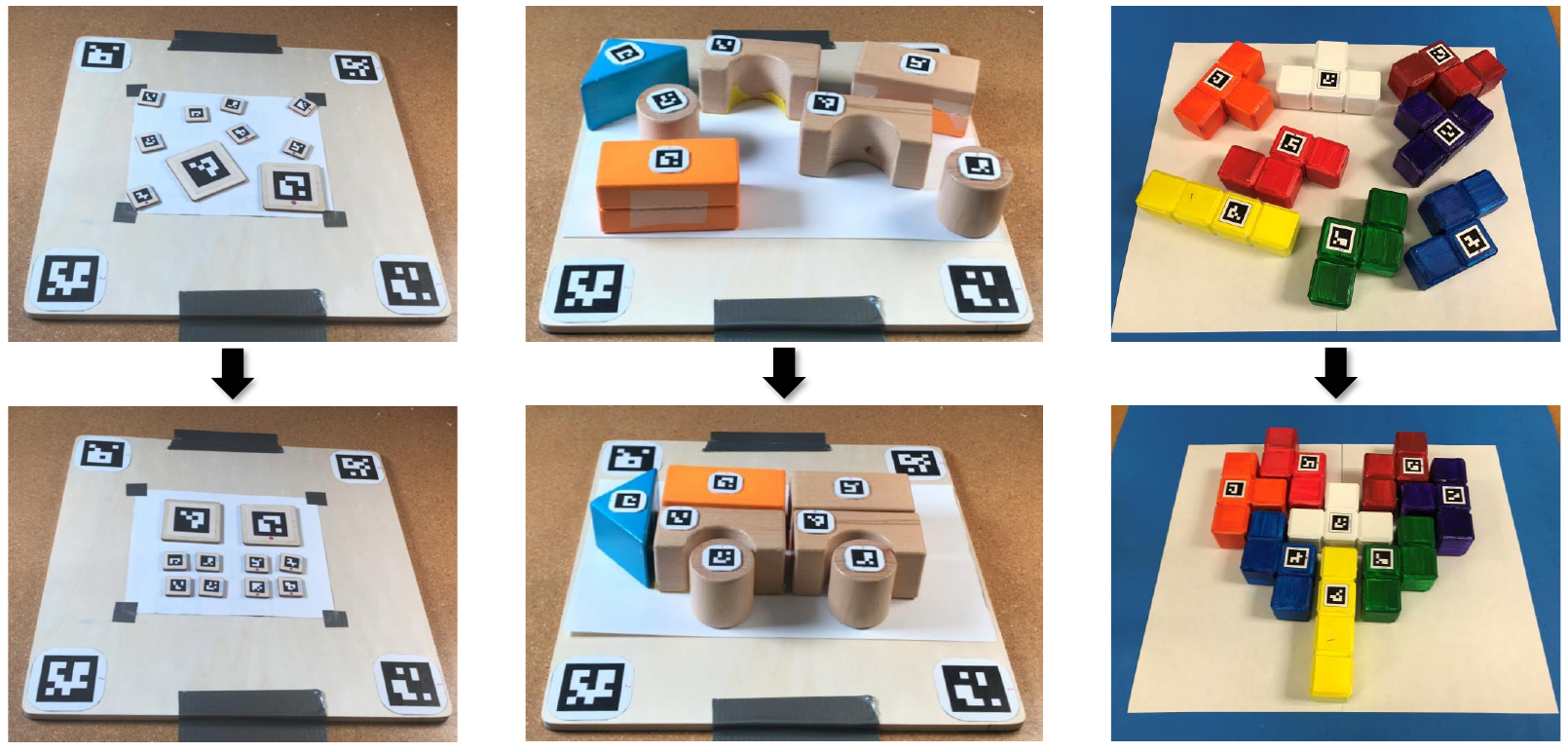}
    \caption{Physical experiment/demo setups. [left] Dense SQ instance with $10$ square wooden planks. [middle] Rearranging densely-packed thick wooden blocks of different shapes and sizes. [right] Rearranging 3d-printed Tetris pieces into a heart shape.}
    \label{fig:demo}

\end{figure}

\section{Summary}
The \trlb framework proposed in this work employs the dependency graph representation and a lazy buffer allocation approach for efficiently solving the problem of rearranging many tightly packed objects on a tabletop using internal buffers (\toroi). Extensive simulation studies show that \trlb computes rearrangement plans of comparable or better quality as the state-of-the-art methods, and does so up to $2$ magnitudes faster. 

We further examine rearranging heterogeneous objects on a tabletop (\hete) using two different objectives: PP and TI. 
The PP objective assumes each object requires an equal amount of manipulation cost for a pick-n-place, while utilizing the object's shape and size can enhance rearrangement efficiency. 
On the other hand, the TI objective evaluates the total task impedance, where the manipulation cost of each object varies based on its specific characteristics, such as size or weight. 
To optimize these objectives, we propose weighting heuristics and build SOTA algorithms with them. Extensive simulation experiments confirm the efficiency of our proposed methods. 

Demonstration with real robot experiments shows that the solutions computed by the proposed \toroi and \hete solvers are efficient and appear natural. As such, the solutions can potentially be deployed as part of home automation and industrial automation solutions in next-generation robots.

%% file: chapters/dual_arm.tex
\chapter{TRLB for Dual Arm Coordination}\label{chap:cdr}
\thispagestyle{myheadings}

\section{Motivation}
In solving multi-object rearrangement problems, employing multiple arms is 
a straightforward way to expand the workspace \cite{shome2020synchronized} 
and increase overall system throughput\cite{shome2021fast}.
Toward enabling effective dual-arm (and multi-arm) coordination, a difficult challenge at present, this work investigates the \emph{cooperative dual-arm rearrangement} 
(CDR) problem, where two robot arms' workspaces partially overlap, and each robot is responsible for a portion of the workspace (\ref{fig:problem}[Left]). 
In CDR, robot-robot collaboration must be considered to realize higher throughput. 
Certain settings benefit from having more robots, e.g., one arm may hold an object while other arms make the goal pose of the object available. 
This will lead to efficiency gain as compared with using a single robot.
There are also added challenges, such as some objects being handed off from one robot to another, which requires careful synchronization. 

For task planning, the proposed method employs lazy or delayed verification and efficiently computes high-quality solutions for non-monotone instances in a cluttered environment. 
As a key component of the algorithmic solution, this work proposes a dependency graph-guided heuristic search procedure for coordinating robot-robot object handoffs 
and individual pick-n-places that support multiple makespan evaluation metrics.

For motion planning, we propose \emph{Makespan-Optimized Dual-Arm (Task and Motion) Planner} (MODAP). 
MODAP has the following main contributions: 
First, MODAP intelligently samples inverse kinematics (IK) solutions during the task planning (i.e., when objects are picked or placed) phase, generating multiple feasible motion trajectories for downstream planning and improving long-horizon plan optimality. 
Second, MODAP leverages a recent GPU-accelerated planning framework,  cuRobo \cite{sundaralingam2023curobo}, to quickly generate high-quality motion plans for our dual-arm system, resulting in magnitudes faster computation of dynamically smooth motions.\footnote{We note that, while cuRobo can be applied to dual-arm motion planning, it cannot be directly applied to solve CDR because the two robots' plans cannot be readily synchronized due to asynchronous pick-n-place actions.} MODAP performs further trajectory optimization over the trajectories from cuRobo to accelerate the full plan.

\section{Related Work}
%{\bf Multi-Robot/Gripper Manipulation:} 
%Multi-robot coordination increases efficiency and enhances flexibility in automation 
%systems.
%
%Recently, multi-robot coordination has been applied to coordinated control\cite{mirrazavi2018unified,gan2019multi}, resolving resource conflicts\cite{guo2021spatial,sharon2015conflict}, 
%localization\cite{nerurkar2011hybrid,prorok2012low}, environment exploration\cite{faigl2012goal,pei2012steiner}, search and rescue\cite{balakirsky2007towards, baxter2009shared}, and transportation\cite{kube2000cooperative,yan2012multi}, 
%to list a few. 
%In our paper, we discuss robot cooperation in 
%rearrangement tasks. 

{\bf Multi-Object Rearrangement:} Single arm object rearrangement lies within the broader area of Task and Motion Planning (TAMP).
Typical problems in this domain \cite{cosgun2011push, wang2021uniform, gaorunning, wang2021efficient, gao2021fast} involve arranging multiple objects to assume specific goal poses. Certain variations, however, such as NAMO (Navigation among Movable Obstacles) \cite{stilman2005navigation,  stilman2008planning}, and retrieval problems \cite{ nam2019planning, ZhangLu-RSS-21, vieira2022persistent, huang2022self}, 
focus on clearing out a path for a target object or robot. During the process, obstacles that need to be displaced are identified.
Rearrangement may be approached either via simple but inaccurate non-prehensile actions, e.g., pushes \cite{cosgun2011push, king2017unobservable, huang2021dipn, vieira2022persistent}, or more purposeful prehensile actions, such as grasps \cite{krontiris2015dealing, krontiris2016efficiently, wang2021efficient}.

Focusing on inherent combinatorial challenges associated with rearrangement tasks, some planners assume external space for temporary object storage \cite{bereg2006lifting,han2018complexity, nam2019planning,gaorunning}, while others exploit problem linearity to simplify the search  \cite{okada2004environment, stilman2005navigation, stilman2008planning, levihn2013hierarchical}.
%Due to the embedded combinatorial challenges,
By linking multi-object rearrangement to established graph-based problems, efficient algorithms have been obtained for various tasks and objectives  \cite{han2018complexity,gaorunning,bereg2006lifting}.
When external free space is not available, the robot arm needs to allocate collision-free locations for temporary object displacements\cite{krontiris2016efficiently,cheong2020relocate,gao2021fast}.
In our work, we adopt the idea of ``lazy buffer allocation''\cite{gao2021fast} to the dual-arm scenario.

Multi-robot rearrangement requires additional computation on task assignment and coordination.
Ahn et al.\cite{ahn2021coordination} coordinate robots by maximizing the number of ``turn-taking'' moments when one arm is picking an obstacle from the workspace while the other arm is placing the previous obstacle at the external space.
Shome et al.\cite{shome2021fast} assume robot arms pick-n-places simultaneously and only solve monotone rearrangement problems, where each object can move directly to the goal pose.
For objects that need to be moved to a pose outside the reachable region of a robot arm,
the task can be accomplished by coordinating multiple arms to handoff the objects around\cite{shome2020synchronized}.
Cooperative pick-n-place, a problem related to multi-arm rearrangement, is well-studied in the area of printed circuit board assembly tasks\cite{li2008enhancing,moghaddam2016parallelism}. 
These problems do not have ordering constraints on pick-n-place tasks, i.e. all the items can move directly to goal poses in any ordering.
Research in these topics tends to consider multiple grippers equipped on a single arm rather than multiple arms, which reduces system flexibility.
In our work, we compute dual-arm rearrangement plans for dense non-monotone instances without external storage space.

{\bf Dependency Graph:} Dependencies between objects in rearrangement tasks can be naturally represented as a dependency graph, which was first applied to general multi-body planning problems \cite{van2009centralized} and then rearrangement \cite{krontiris2015dealing, krontiris2016efficiently}. 
In rearrangement, different choices of grasp poses and paths give rise to multiple dependency graphs for the same problem instance, which limits the scalability in computing a solution via such representations. 
Prior work \cite{han2018complexity} has applied full dependency graphs to address \toro (Tabletop Object Rearrangement with Overhand Grasps), showing that the problem embeds a NP-hard Feedback Vertex Set problem \cite{karp1972reducibility} and a Traveling Salesperson Problem \cite{papadimitriou1977euclidean}.
%In this specific setup, the dependency graph is unique for each instance and completely represents the inherent constraints.

More recently, some of the authors \cite{gaorunning} examined another optimization objective, running buffers, which is the size of the external space needed for temporary object displacements in the rearrangement task, and also examined an unlabeled setting,
where goal poses of objects are interchangeable. 
Similar graph structures are also used in other robotics problems, such as packing problems \cite{wang2020robot}. 
Deep neural networks have been also applied to detect the embedded dependency graph of objects in a cluttered environment to determine the ordering of object retrieval \cite{ZhangLu-RSS-21}.

\section{Cooperative Dual-Arm Rearrangement (CDR)}
\subsection{Problem Statement}
Consider a 2D bounded workspace $\mathcal W \subset \mathbb R^2$ containing a set of $n$ cylindrical objects $\mathcal O=\{o_1, ..., o_n\}$. 
An arrangement $\mathcal A$ of these objects is a set of poses $\{p_1, ..., p_n\}$ with each pose  $p_i=(x_i,y_i)\in \mathcal W$.
$\mathcal A$ is \emph{feasible} if the footprints of objects are inside $\mathcal W$ and pair-wise disjoint.
Outside the workspace, two robot arms $\mathcal R=\{r_1, r_2\}$ are tasked to manipulate objects from a feasible start arrangement $\mathcal A_s$ to another feasible goal arrangement $\mathcal A_g$ (e.g., \ref{fig:problem}[Left]). 
Each robot arm $r_i$ has a connected reachable region $\mathcal S(r_i)\subseteq \mathcal W$. 
%A pose  $p\in \mathcal W$ is \emph{reachable} by an arm $r_i$ if $p\in \mathcal S(r_i)$. 
Robot $r_i$ can only manipulate within $\mathcal S(r_i)$.

It is assumed that $\mathcal S(r_1)\cup \mathcal S(r_2)=\mathcal W$.
The overlap ratio $\rho$ is defined as $\rho = \dfrac{|\mathcal S(r_1)\cap \mathcal S(r_2)|}{|\mathcal{W}|}$, 
which is the proportion of the environment that can be reached by both robots. 
In \ref{fig:CDR-Exp}[Left], we show a workspace with $\rho=0.3$. 
Robot arm $r_1$ can reach objects whose centers are on the left of the red dashed line, which include all the objects at start poses($O_1, O_2, O_3$) and $o_1$, $o_2$ at goal poses($G_1, G_2$). Robot arm $r_2$ can reach objects whose centers are on the right of the blue dashed line, which include all the objects at goal poses($G_1, G_2, G_3$) and $o_1$, $o_2$ at start poses($O_1, O_2$). 
Objects with centers between the dashed lines can be reached by both.

\begin{figure}[ht]
    \centering
    \includegraphics[width=0.36\textwidth]{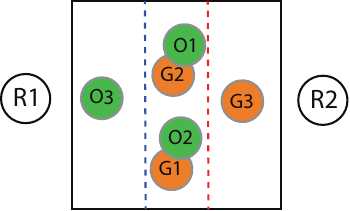}
    \hspace{5mm}
    \includegraphics[width=0.46\textwidth]{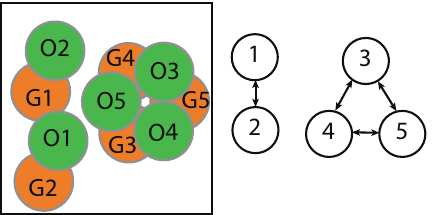}
    \caption{ [Left] A working example of CDR instance with $\rho=0.3$. For each object $o_i$, the start and goal poses are represented by green disc $O_i$ and orange disc $G_i$ respectively. [Right] A CDRF instance and its corresponding dependency graph.}
    \label{fig:CDR-Exp}
\end{figure}

We consider two manipulation primitives: (overhand) \emph{pick-n-place} and \emph{handoff}.
In each pick-n-place, an arm moves above an object $o_j$, grasps it from $\mathcal W$, transfers it atop the workspace, and places it at a collision-free pose  inside $\mathcal W$.
We allow a \emph{handoff} operation when an object needs to cross between  $\mathcal S(r_1)\backslash \mathcal S(r_2)$ and $\mathcal S(r_2)\backslash \mathcal S(r_1)$.
In each handoff, an arm, say $r_1$, grasps an object $o_j$ from $\mathcal W$, 
passes the object to the other arm $r_2$ at the predefined handoff pose above the environment.
$r_2$ receives $o_j$ and places it at a collision-free pose  in $\mathcal W$.
A handoff is shown in \ref{fig:problem}[Right].

% To further increase efficiency in the multi-arm system, we allow handoff operations when an object needs to travel between the two disjointed regions $S_1\backslash S_2$ and $S_2\backslash S_1$. In an handoff operation, a robot arm, say $r_1$ grasp an object from its current pose , transfers the object to a predefined handoff pose  above the environment, passes it to the other arm as shown in figure. After the handoff, the other arm places it at the destination.

As for the motion of each arm $r_i$, $\mathcal C^i_{free}$ is a set of collision-free configurations for $r_i$ without considering collisions with workspace objects or the other robot arm.
Throughout the rearrangement, each arm has $n+1$ motion modes: transit mode $TS$ (moving on its own) and transfer modes $TF_j$ (transferring the object $o_j$), $\forall j=1,...,n$.
Therefore, the state space $\Gamma_i$ of an arm $r_i$ is the Cartesian product $\mathcal C^i_{free} \times M$, where $M$ is the motion mode set. 
A path of a robot arm $r_i$ is defined as $\pi_i:[0,T]\rightarrow \Gamma_i$ with $\pi_i[0]=\pi_i[T]=(q^*_i, TS)$, where $T$ is the makespan of the rearrangement plan and $q^*_i$ is the rest pose of arm $r_i$.

The rearrangement problem we study seeks a path set $\Pi=\{\pi_1, \pi_2\}$, such that robot paths are collision-free: for all $t\in [0,T]$, arm $r_1$ at state $\pi_1[t]$ does not collide with arm $r_2$ at state $\pi_2[t]$.
Additionally, the solution quality is evaluated by the makespan under two levels of fidelity:
\begin{enumerate}
    \item (MC) manipulation cost based makespan;
    \item (FC) estimated execution time based makespan.
\end{enumerate}
In MC (manipulation cost) makespan, we assume the manipulation time dominates 
the execution time, which is often the case in today's systems. MC makespan is defined as the number of \emph{steps} for the rearrangement task. In each step, 
each robot arm can complete an individual pick-n-place, execute a coordinated handoff, or idle.

In FC (full cost) makespan, solution quality is measured by the estimated 
complete time based on four parameters: 
maximum horizontal speed $s$ of the end-effectors, the execution time of each pick ($t_g$), place ($t_r$), and handoff ($t_h$). The details of FC metric and its corresponding heuristic search method is presented in the appendix (\ref{sec:CDR-KC})

Given the setup, the problem studied in this paper can be summarized as:

\begin{problem}[CDR: Collaborative Dual-Arm Rearrangement]
Given feasible arrangements $\mathcal A_s$, $\mathcal A_g$, and makespan metric (MC or FC), determine a collision-free path set $\Pi$ for $\mathcal R$ moving objects from $\mathcal A_s$ to $\mathcal A_g$ minimizing the makespan $T$.
\end{problem}

We also investigate two special cases of CDR.
On one hand, when $\rho=1$, $\mathcal S(r_1)=\mathcal S(r_2)=\mathcal W$, 
which yields to Collaborative Multi-Arm Rearrangement with Full Overlap (CDRF). 
In this case, handoff is not needed since each robot can execute pick-n-place operations individually inside the workspace. 
On the other hand, when $\rho=0$, $\mathcal S(r_1)\bigsqcup \mathcal S(r_2)=\mathcal W$, 
which yields to Collaborative Multi-Arm Rearrangement with no Overlap (CDRN).

\section{Task Scheduling for CDR}
In this section, we study the scheduling problem of primitive actions in CDR. 
A primitive action can be represented by a tuple $(r_i, o_j,v )$ indicating that $r_i$ is tasked to transfer $o_j$ to new status $v$, where $v$ is one kind of the object status in $\{S,G,H,\mathcal B(r_1),\mathcal B(r_2)\}$.
$S$, $G$, $H$, $\mathcal B(r_1)$, and $\mathcal B(r_2)$ represent that the object is at the start pose, goal pose, handoff pose, a buffer in $\mathcal S(r_1)$, and a buffer in $\mathcal S(r_2)$ respectively.
For each primitive action $a$, the schedule provides the starting time and estimated ending time of $a$.
Noting that buffer locations are not determined in the task scheduling process, 
they are categorized into $\mathcal B(r_1)$ and $\mathcal B(r_2)$ to indicate the arm that they are reachable by.
The buffer allocation is discussed in \ref{sec:buffer}.
Since the scheduling problem only depends on the availability of goal poses and reachability of robot arms to the start and goal poses, 
the constraints can be fully expressed by the dependency graph.
We describe a dependency graph guided heuristic search method for CDR under MC makespan metrics (\ref{sec:CDR-MC}). Its variant for FC makespan is shown in the appendix (\ref{sec:CDR-KC}).

\subsection{Arrangement Space Heuristic Search for CDR in MC Makespan (MCHS)}\label{sec:CDR-MC}
% We treat the rearrangement problem as a path finding problem from $\mathcal A_s$ to $\mathcal A_g$.
% In single arm rearrangement problems, objects are manipulated one at a time so the rearrangement problem is equivalent to a sequence of pick-n-place actions.
% Previous works maintain search trees rooted at the initial state, where all the objects are at $\mathcal A_s$, to search for the goal state, where all the objects are at $\mathcal A_g$.
% In[], each node on the tree represents a sequence of manipulations, so that the search tree has $O(n!)$ nodes even in monotone cases.
% Recent works[] construct state-space search trees, where each node on the tree represents an arrangement.
% This formulation reduces the size of the monotone search space to $O(2^{n})$.
When the makespan is counted in MC metric, 
we can assume primitive actions are executed simultaneously in the task scheduling process without loss of generality and the makespan is counted as the number of \emph{action steps} in the schedule.
Each action at action step $t$ is scheduled to start at time $t$ and end at time $t+1$.
Similar to the single arm rearrangement problem, we search for a plan by maintaining a search tree in the arrangement state space.
Each arrangement state in the tree represents an object arrangement, which can be expressed as a mapping $\mathcal L: \mathcal O\rightarrow \{S,G,\mathcal B(r_1),\mathcal B(r_2)\}$. The arrangement state represents a moment when all objects are stably located in the workspace and both arms are empty.

Arrangements are connected by primitive actions of two arms, 
which consist of one out of the three possible options: 
(1) individual pick-n-places: one or both arms move objects to poses that are currently collision-free; 
(2) coordinated pick-n-places: move an object to the goal pose while the other arm tries to clear away the last obstacle at its goal pose, e.g. swapping poses of $o_1$ and $o_2$ in \ref{fig:CDR-Exp}[Left]; 
(3) a coordinated handoff for an object when its start and goal poses are not both reachable by either arm.
We ignore primitive actions that hold an object for more than one action step due to the unnecessary loss of efficiency.

\ref{fig:flow-charts} shows the flow charts indicating the decision process of generating possible primitive actions in each arrangement state.
\ref{fig:flow-charts}[Upper] shows the decision process for $r_1$ to pick-n-place $o\in \mathcal O$ without cooperation with the other arm. 
That for $r_2$ is equivalent.
\ref{fig:flow-charts}[Middle] shows the decision process for $r_1$ to manipulate an object $o\in \mathcal O$ when only one object $o'$ is blocking the goal pose of $o$. In this case, two arms cooperate to send $o$ to the goal.
That for $r_2$ is equivalent.
\ref{fig:flow-charts}[Below] shows the decision process when an object $o\in \mathcal O$ needs handoff.
In the process, the availability of goal poses can be checked in the precomputed dependency graph and the reachability of arms is indicated by the object status in the arrangements. Therefore, with the dependency graph, no additional collision-checking is needed in the task scheduling process of CDR.

\begin{figure}
    \centering
    \includegraphics[width=\textwidth]{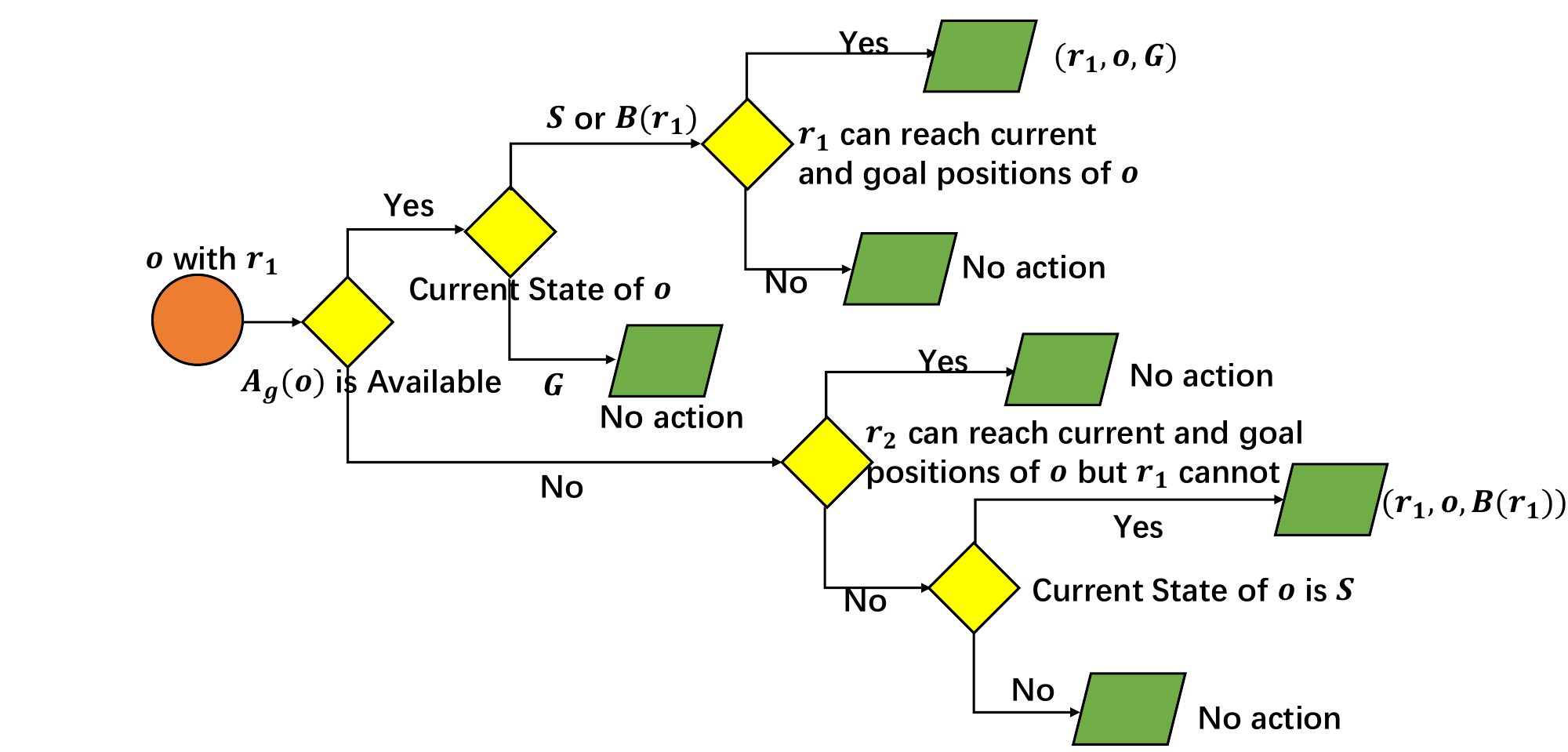}
    \includegraphics[width=\textwidth]{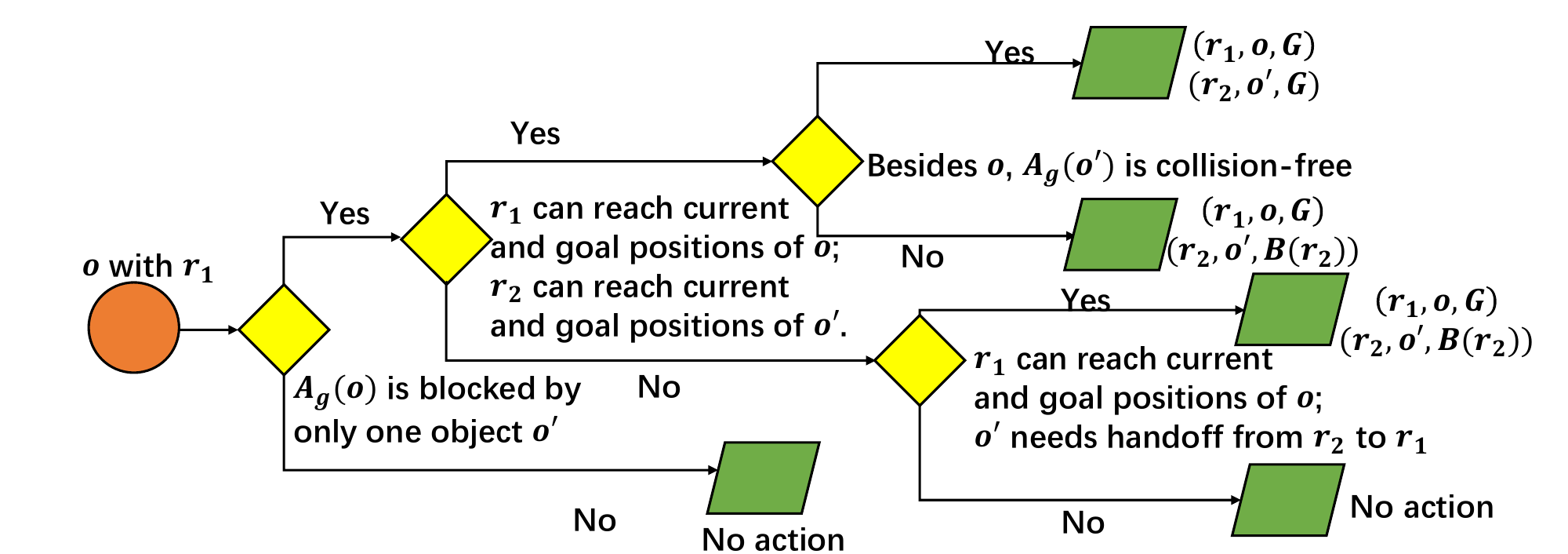}
    \includegraphics[width=\textwidth]{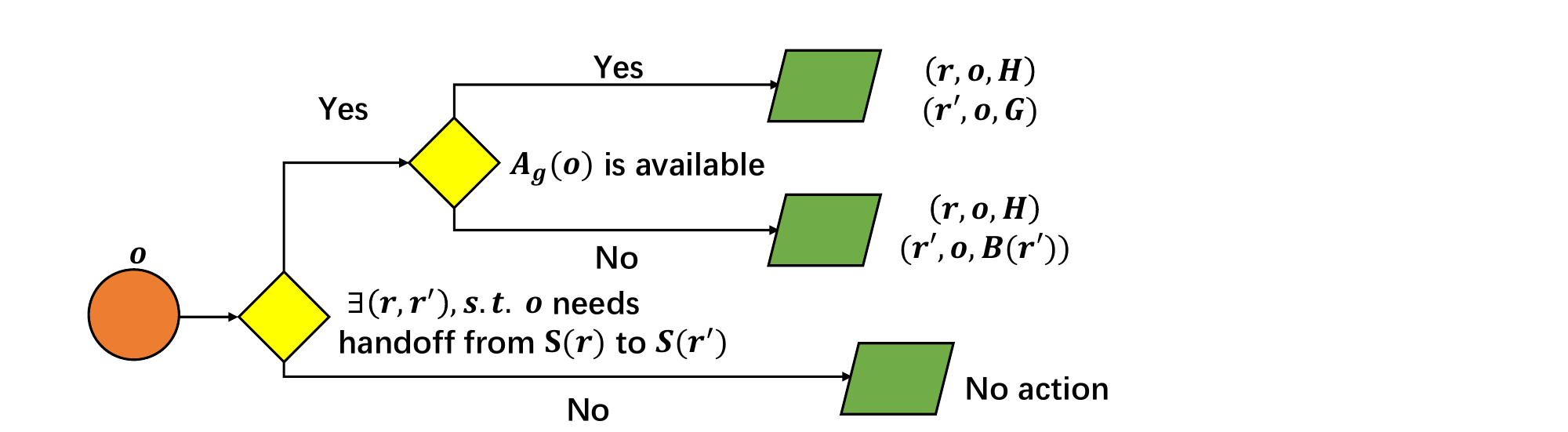}
    \caption{
    Decision process of generating possible primitive actions.
    [Upper] The decision process for $r_1$ to manipulate $o\in \mathcal O$ without cooperation with the other arm. That for $r_2$ is equivalent. 
    [Middle] The decision process for $r_1$ to manipulate an object $o\in \mathcal O$ when only one object $o'$ is blocking the goal pose of $o$. 
    That for $r_2$ is equivalent.
    [Bottom] The decision process when an object $o\in \mathcal O$ needs handoff.
    }
    \label{fig:flow-charts}
\end{figure}

For the working example in \ref{fig:CDR-Exp}[Left], a corresponding rearrangement plan, in the form of a path in the arrangement state space is presented in \ref{fig:MC-Plan}[Left]. 
At the start arrangement state $s_1$, $r_1$ and $r_2$ move $o_2$ and $o_1$ respectively to the goal poses, 
which yields the next arrangement state $s_2$.
At $s_2$, two arms coordinate to execute a handoff delivering $o_3$ to the goal and reach the goal arrangement state $s_3$. 
By concatenating primitive actions from $s_1$ to $s_3$, we obtain a corresponding task schedule for the instance (\ref{fig:MC-Plan}).

\begin{figure}
    \centering
    \includegraphics[width=0.5\textwidth]{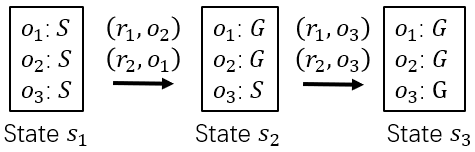}
    \hspace{3mm}
    \includegraphics[width=0.24\textwidth]{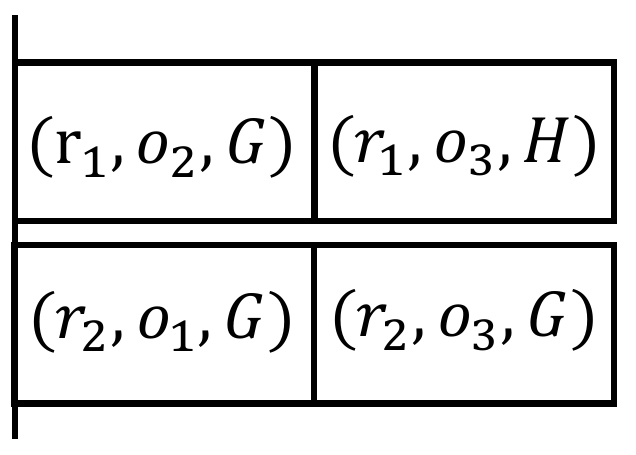}
    \caption{[Left] A path on the arrangement state space for the instance in \ref{fig:CDR-Exp}[Left] under MC metric. [Right] The corresponding schedule of the example instance in \ref{fig:CDR-Exp}[Left].}
    \label{fig:MC-Plan}
\end{figure}

The arrangement state heuristic search based on MC makespan (MCHS) explores the arrangement state space in a best-first manner, always developing the arrangement state $s$ with the lowest $f(s)=g(s)+h(s)$.
The $g$ function is number of action steps from the start arrangement state to the current arrangement state. For example, 
in \ref{fig:MC-Plan}, $g(s_3)=2$.
For each state $s$, we define the heuristic function $h(s)$ based on the needed cost to move all objects directly to goals. 

The details are presented in \ref{alg:MC-heuristic}.
In CDR, costs can be either exclusive cost of $r_1$, or exclusive cost of $r_2$, or shared cost of $r_1$ and $r_2$. 
For example, on one hand, a pick-n-place in region $\mathcal S(r_1)\backslash \mathcal S(r_2)$ can only be executed by $r_1$, 
so it leads to an exclusive cost of $r_1$. 
On the other hand, a pick-n-place in the overlapping region $\mathcal S(r_1)\bigcap \mathcal S(r_2)$ can be executed by either arm so the cost can be shared by both arms.
Based on this observation, we count the costs cost[$r_1$], cost[$r_2$], and sharedCost in the heuristic and initialize them as 0 (Line 1).
We ignore the objects at the goal poses (Line 3).
For the remaining objects, the cost category is determined by $r_1$, $r_2$'s reachability to the current pose (buffer or start pose) and goal pose of the object.
ArmSet is the set of arms able to reach both the current pose and goal pose (Line 4).
% We don't allow a robot to manipulate objects in the buffer in the reachable region of the other arm, e.g. $r_1$ is not allowed to move objects in $\mathcal B(r_2)$.
If both arms are in armSet (Line 5), then the cost of moving $o_i$ to the goal will be added to sharedCost, which means that it is a pick-n-place in the overlapping region.
If only one arm is in armSet (Line 6-7), then the cost of moving $o_i$ to the goal will be added to the cost of the specific arm, 
which means that this action happens out of reach of the other arm.
If neither arm is in armSet (Line 8-10), then the object needs a handoff to the destination.
In this case, both arms need to take one action for the handoff operation.
Finally, on one hand, if the difference of the exclusive costs is smaller than sharedCost (Line 11), then $h(s)$ is half of the total cost (Line 12). 
In other words, the cost of two arms may be balanced by splitting the sharedCost.
On the other hand, if the difference of the exclusive costs is too large (Line 13), then the unoccupied arm bears all the shared cost, and $h(s)$ equals the cost of the busy arm.

\begin{algorithm}
\begin{small}
    \SetKwInOut{Input}{Input}
    \SetKwInOut{Output}{Output}
    \SetKwComment{Comment}{\% }{}
    \caption{Manipulation Cost Search Heuristic}
		\label{alg:MC-heuristic}
    \SetAlgoLined
		\vspace{0.5mm}
    \Input{$s$: current state}
    \Output{h: heuristic value of $s$}
		\vspace{0.5mm}
		cost[$r_1$], cost[$r_2$], sharedCost $\leftarrow 0,0,0$\\
		\For{$o_i\in \mathcal O$}{
		\lIf{$\mathcal L(o_i)$ is $G$}{continue}
		armSet $\leftarrow$ ReachableArms($s$, $o_i$, $\{\mathcal L(o_i), G\}$)\\
		\lIf{armSet is $\{r_1, r_2\}$}{sharedCost$++$}
		\lElseIf{armSet is $\{r_1\}$}{cost[$r_1$]$++$}
		\lElseIf{armSet is $\{r_2\}$}{cost[$r_2$]$++$}
		\Else{
		cost[$r_1$] $++$;\\
		cost[$r_2$] $++$;
		}
		}
		\vspace{1mm}
		\lIf{$\|$cost[$r_1$]-cost[$r_2$] $\|\leq$ sharedCost}{
		\\ \Return (cost[$r_1$]+cost[$r_2$]+sharedCost)/2
		}
		\lElse{
		\Return max(cost[$r_1$], cost[$r_2$])
		}
\end{small}
\end{algorithm}

Since $cost[r_1]$ and $cost[r_2]$ provide lower bounds of necessary manipulations for each arm, $h(s)$ is a lower bound of the manipulation cost from the state $s$ to the goal arrangement. Therefore, we have:

\begin{proposition}
The manipulation cost based heuristic $h(s)$ is consistent and admissible.
\end{proposition}

\section{Buffer Allocation}\label{sec:buffer}
% Primitive plans 
In this section, we allocate feasible buffer locations based on the primitive plan computed in \ref{sec:CDR-MC} following the idea of lazy buffer allocation in \ref{chap:trlb}. 
Given the primitive plan, we know not only the specific timespans that objects stay in buffers, 
but also the staying poses of the other objects in the timespans, which are treated as fixed obstacles that the buffer location needs to avoid.
Besides the obstacles, 
the buffer location of an object moved to $\mathcal B(r)$ in the primitive plan is restricted to be inside $\mathcal S(r)$, the reachable region of $r$.

\begin{figure}
    \centering
    \includegraphics[width=0.4\textwidth]{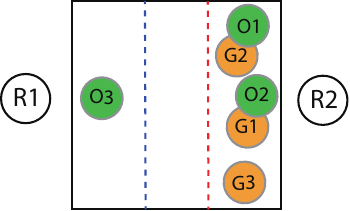}
    \vspace{5mm}\\
    \includegraphics[width=0.98\textwidth]{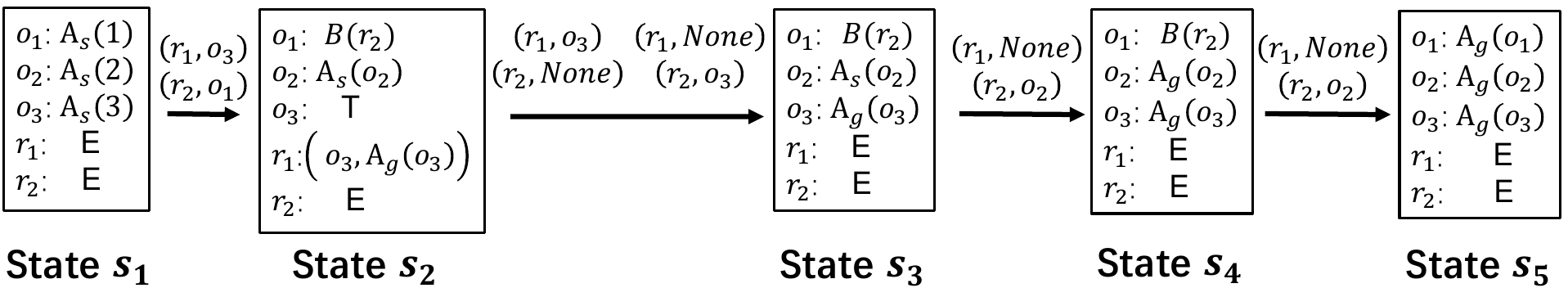}
    \vspace{1mm}
    \caption{[Top] A CDR instance that needs a buffer. [Bottom] An rearrangement plan for the example.}
    \label{fig:CDR-Buffer_Exp}
\end{figure}

We illustrate the buffer allocation process with a CDR instance in \ref{fig:CDR-Buffer_Exp}[Top], which is a variant of the example in \ref{fig:CDR-Exp}[Left]. 
When $o_1$ and $o_2$ are in $\mathcal S_2 \backslash \mathcal S_1$,
at least one of them needs to be moved to buffer locations to break the cycle in the dependency graph.
A rearrangement plan is shown in \ref{fig:CDR-Buffer_Exp}[Bottom]. 
In this plan, $o_1$ is moved to a buffer when $o_2$ is at $\mathcal A_s(o_2)$, and $o_3$ waits for a handoff. 
And $o_1$ leaves the buffer when $o_2$ and $o_3$ are both at goal positions.
Therefore, the buffer location of $o_1$ needs to be collision-free with $o_2$ at $\mathcal A_s(o_2)$, $\mathcal A_g(o_2)$ and $o_3$ at $\mathcal A_g(o_3)$. 
Additionally, it is restricted in $\mathcal S(r_2)$.
When multiple objects stay in buffers at the same time, the buffer locations need to be allocated disjointedly.
In this paper, we search for feasible buffer locations by sampling candidates in the reachable region and check collisions with the fixed obstacles indicated by the primitive plan.
The same as TRLB in \ref{chap:trlb}, 
when buffer allocation fails, 
we accept partial plans and conduct a bi-directional search in the arrangement state to connect the start and goal arrangement.

\section{GPU based Task and Motion Planning in High Dimensional C-Space}
\subsection{MCHS Task Planning}
MCHS does not consider the motion planning for the dual-arm system. In reality, robot arms take different times to rearrange a single object and may need to coordinate to avoid collisions in the shared workspace. 
Using the task plan produced by MCHS, we decouple primitive actions in the plan into lower-level \emph{sub-tasks}.
For example, for the primitive action $(r_1, o_3)$ in \ref{fig:CDR-Buffer_Exp}, it can be decoupled into four sub-tasks of $R_1$: move to pre-handoff pose, move to handoff pose, open gripper, and move back to pre-handoff pose.
Even though the original task plan is synchronized (i.e., two robots start and finish individual actions simultaneously in each action step), the generated sub-task sequences allow us to enable asynchronous execution and generate smooth and executable dual-arm motions.
The sub-tasks are categorized as follows.
\begin{enumerate}[leftmargin=5mm]
    \item A \textbf{Traveling Task (TT)} includes motions between home, pre-picking, pre-placing, and pre-handoff poses. 
    A TT sub-task can be represented as $(p_1, p_2)$, where $p_1,p_2\in SE(3)$, and the robot needs to move from the current pose $p_1$ to the target pose $p_2$.
    In our implementation, we also treat idling as a TT sub-task where $p_1=p_2$.
    \item A \textbf{Manipulation Task (MT)} includes motions between a picking/placing/handoff pose and its corresponding pre-picking/pre-placing/pre-handoff pose.
    An MT sub-task can be represented as $(p_1, p_2)$, where $p_1,p_2\in SE(3)$, and the robot needs to move in a straight line from the current pose $p_1$ to the target pose $p_2$.
    \item A \textbf{Gripper task (GT)} includes opening/closing grippers when picking/placing an object or executing a handoff. 
    A GT task can be represented as $(p_1,k)$, where the robot arm stays at the current pose $p_1$ for $k$ timesteps.
\end{enumerate}

% Note that while the trajectories of TT tasks can be planned freely, trajectories of MT and GT tasks are pre-defined.
% Therefore, when planning 
% In our motion planner, we propose strategies to handle different situations.

\subsection{Motion Planning Phase}
\modap's dual-arm motion planning framework is described in \ref{alg:MotionPlanning}.
Each robot executes its sub-tasks sequentially; trajectory conflicts are resolved by leveraging cuRobo's motion generator.
%It consumes sub-task sequences.
In Line 1, we initialize $\mathcal T$, $\pi_1$, and $\pi_2$, where $\mathcal T$ is the dual-arm trajectory, $\pi_1$ and $\pi_2$ are current sub-tasks of $R_1$ and $R_2$. 
The rearrangement is finished if all of $\Pi_1, \Pi_2, \pi_1, \pi_2$ are empty (Line 2). 
In Line 3, $\pi_1$ and $\pi_2$ are updated.
Specifically, if the current sub-task is finished, the next sub-task in the sequence kicks in. 
%Otherwise, it remains the same. 
Given $\pi_1, \pi_2$, single-arm trajectories $\tau_1, \tau_2$ are computed (Line 4-5). 
$\tau_1$ and $\tau_2$ avoid self-collisions and collisions between the planning robot arm and the workspace, but they don't consider collisions between arms.
If there is no collision between $\tau_1$ and $\tau_2$, a dual-arm trajectory is generated by merging $\tau_1$ and $\tau_2$ until one of the trajectories reaches its ending timestep (Line 7).
Otherwise, there are two cases (Line 8 - 12): (1) both trajectories can be replanned (Line 8-10); (2) only one trajectory can be replanned (Line 11-12).
For (1), we sample a collision-free dual-arm goal configuration $q^\star$(Line 9). At $q^\star$, one robot finishes this sub-task while the other is on its way.
And then, we compute dual-arm trajectory via cuRobo (Line 10) from the current dual-configuration to $q^\star$.
For (2), we compute a dual-arm trajectory that prioritizes the sub-task of one arm, and the other arm yields as needed. (Line 11-12).
The details of our planner for the two cases are described in \ref{sec:dual_planning}.
A third case beyond (1) and (2) exists in theory where neither of the trajectories can be replanned. We avoid this case by providing additional clearance for collision checking.
A small clearance is enough since non-TT sub-task trajectories are short (a few centimeters).
The planner computes a dual-arm trajectory $\tau$, sped up by Toppra\cite{pham2018new}, a time-optimal path parametrization solver, and added to $\mathcal T$ (Line 13).
\begin{algorithm}
\begin{small}
    \SetKwInOut{Input}{Input}
    \SetKwInOut{Output}{Output}
    \SetKwComment{Comment}{\% }{}
    \caption{Fast Dual-Arm Motion Planning}
		\label{alg:MotionPlanning}
    \SetAlgoLined
		\vspace{0.5mm}
    \Input{$\{\Pi_1, \Pi_2\}$: sub-task sequences
    }
    \Output{$\mathcal T$: dual-arm motion}
\vspace{0.5mm}

$\mathcal T\leftarrow \emptyset$; 
$\pi_1, \pi_2\leftarrow \emptyset, \emptyset$;\\
\While{
not Finished($\Pi_1, \Pi_2, \pi_1, \pi_2$)
}{
$\pi_1, \pi_2\leftarrow $UpdateSubTasks($\Pi_1, \Pi_2, \pi_1, \pi_2$);\\
$\tau_1\leftarrow SingleArmPlanning(\pi_1)$;\\
$\tau_2\leftarrow SingleArmPlanning(\pi_2)$;\\
\If{not $InCollision(\tau_1,\tau_2)$}{
    $\tau\leftarrow$ MergeTraj($\tau_1,\tau_2$);\\
}
\ElseIf{$\pi_1$ is TT and $\pi_1$ is TT}{
$q^\star\leftarrow $ SampleDualGoal($\tau_1,\tau_2$);\\
$\tau \leftarrow$ DualArmPlanning($q^\star$);\\
}
\lElseIf{$\pi_1$ not TT}
{
$\tau\leftarrow$ PlanPartialPath($\tau_1$)
}
\lElseIf{$\pi_2$ not TT}
{
$\tau\leftarrow$ PartrialPlanning($\tau_2$)
}
$\tau\leftarrow$ ToppraImprovement($\tau$);
$\mathcal T \leftarrow \mathcal T + \tau$\\
}
\Return $\mathcal T$;\\
\end{small}
\end{algorithm}
\vspace{-6mm}

\subsection{Enhancing Inverse Kinematic (IK) Solutions}\label{sec:single_planning}
Given the goal end-effector pose $p_2$, cuRobo motion generator will sample IK seeds and compute multiple IK solutions of $p_2$.
However, it may lead to undesired IK solutions of $p_2$.
For example, in IK computation for pre-picking poses, cuRobo may select a solution where the elbow is positioned below the end effector, preventing the robot from executing the subsequent grasping.
We overcame issues caused by undesirable IK solutions by dictating IK seeds. The goal pose's seed is set to the current configuration for MT tasks because we want the target configuration to be close.
The current configuration may be too far from the target configuration for TT tasks, leading to the undesired configuration mentioned above.
Therefore, in TT tasks, the seed is set to a specific configuration with its corresponding end-effector pose $30$ cm above its reachable workspace region as shown in \ref{fig:simnreal-setup}.
%With this setting, the IK solutions to pre-picking/pre-placing poses would be close to this configuration.
% Given the IK solution, we compute single-arm trajectories via cuRobo trajectory planner without consideration of the other arm.
\subsection{Improving Path Planning to Avoid Detours}
cuRobo's geometric (path) planner samples graph nodes around the straight line between the start, goal, and home (retract) configurations. Based on these nodes, seed trajectories are planned via graph search. 
When the planning problem is challenging (e.g., the second arm blocks the straight line between the start and goal configurations for the first arm), the computed trajectories tend to make lengthy detours by detouring through the neighborhood of the home configuration, which wastes time and yields unnatural robot paths. This is especially problematic when solving many rearrangement sub-tasks. 
We improve seed trajectory quality for each arm by uniformly sampling 32 overhand poses above the workspace. 
The corresponding IK solutions of these manipulation-related poses form 1024 dual-arm configurations, among which the collision-free configurations are added to the geometric graph. This approach proves to be highly effective in avoiding unnecessary detours.

\subsection{Strategies to Resolve Dual-Arm Trajectory Conflicts}\label{sec:dual_planning}
Two arms must coordinate to resolve conflicts and progress toward their next sub-tasks when collisions exist between individual trajectories.
As mentioned in \ref{alg:MotionPlanning}, there are two different collision scenarios.
In the first, when both robots have TT sub-tasks, neither has a specific trajectory to follow.
In this scenario, we sample a collision-free dual-arm configuration (\ref{alg:sample_dual_goal}) and plan a trajectory from the current configuration to the sampled configuration.
In the second, when one of the robots has a specific trajectory to follow, the other arm needs to plan a trajectory to avoid collision.
Since cuRobo does not support planning among dynamic obstacles, we propose an algorithm, PlanPartialPath (\ref{alg:partial}), to resolve conflicts.
The details of the algorithms are elaborated on in the following sections.

\subsubsection{Dual-arm goal sampling}
\ref{alg:sample_dual_goal} takes single arm trajectories $\tau_1$ and $\tau_2$ as input and outputs the collision-free dual-arm configuration $q$.
Without loss of generality, we assume $\tau_1$ is the shorter trajectory in time dimension.
Usually, this means that the sub-task goal configuration of $R_1$ is closer to the current configuration.
Therefore, we sample a configuration as the goal configuration to the dual-arm motion planning problem, where $R_1$ reaches the sub-task goal configuration while $R_2$ progresses toward its goal configuration.
In Lines 1-2, we let $q_1^\star$ be the sub-task goal configuration of $R_1$ and $q_2^\star$ be the configuration in $\tau_2$ when $R_1$ reaches $q_1^\star$.
If the dual-arm configuration $(q_1^\star,q_2^\star)$ is collision-free, then we output the configuration (Line 3).
Otherwise, we sample the configuration of $R_2$ in the neighborhood of $q_2^\star$ (Line 5-11). 
The sampling starts with the $\gamma_0-$sphere centered at $q_2^\star$ and the range is constantly expanded to a $\gamma_n-$sphere, where $\gamma_0$ and $\gamma_n$ are constants.
When local sampling fails, 
we let $R_2$ yield to $R_1$ by sampling $R_2$ configurations on the straight line between its current configuration and its home configuration since the home configuration is guaranteed to be collision-free with the other arm.
Let $S_2$ be the discretization (with $k_2$ waypoints and $k_2$ is a constant) of the straight-line trajectory between the current configuration of $R_2$ and the home configuration of $R_2$ (Line 12).
We then find the collision-free waypoint of $R_2$ nearest the current configuration and return the corresponding dual-arm configuration (Line 13-15).
\vspace{-3mm}

\begin{algorithm}
\begin{small}
    \SetKwInOut{Input}{Input}
    \SetKwInOut{Output}{Output}
    \SetKwComment{Comment}{\% }{}
    \caption{SampleDualGoal}
		\label{alg:sample_dual_goal}
    \SetAlgoLined
		\vspace{0.5mm}
    \Input{$\tau_1:$ shorter single-arm trajectory, $\tau_2:$ longer single-arm trajectory, $k_1, k_2, \gamma_0, \gamma_n:$ hyper-parameters
    }
    \Output{$q$: collision-free dual-arm configuration}
\vspace{0.5mm}

$q^\star_1\leftarrow$ Ending configuration of $\tau_1$;\\
$q^\star_2\leftarrow$ configuration of $\tau_2$ when $\tau_1$ ends;\\
\lIf{CollisionFree($q^\star_1, q^\star_2$)}
{
\Return Merge($q^\star_1,q^\star_2$);
}
\Else{
$r = \gamma_0$;\\
\While{$r<\gamma_n$}{
    \For{$i=1$ to $k_1$}
    {
    $q_2\leftarrow$ Sample configuration in the $r$-sphere of $q^\star_2$;\\
    \If{CollisionFree($q^\star_1, q_2$)}
    {
    \Return Merge($q^\star_1,q_2$)
    }
    }
    $r\leftarrow 2r$;\\
}
$S_2\leftarrow k_2$-discretization of the straight line between the start of $\tau_2$ and robot home configuration;\\ 
\For{$q_2\in S_2$}
{
\If{CollisionFree($q^\star_1, q_2$)}
    {
    \Return Merge($q^\star_1,q_2$)
    }
}
}
\end{small}
\end{algorithm}

\vspace{-3mm}
\subsubsection{Priority based planning}
In \ref{alg:partial}, we compute a collision-free dual-arm trajectory when one of the robots has a pre-defined trajectory to follow.
Without loss of generality, we assume $R_1$ has a pre-defined trajectory $\tau_1$ to execute while $R_2$ needs to yield.
We first compute a collision-free trajectory $\tau_2$ for $R_2$ from the current configuration to its home configuration assuming $R_1$ is fixed at the current configuration (Line 2).
In the worst case, it allows $R_1$ to execute its pre-defined trajectory $\tau_1$ after $R_2$ moves back to its home configuration.
In Line 3-9, we progress on executing $\tau_1$ and $R_2$ yields along $\tau_2$ as needed.
$\tau_1[i_1]$ is the next configuration of $R_1$ on $\tau$.
In Line 5, we compute the first index $i_2'$ such that $i_2'\geq i_2$ and $\tau_2[i_2':]$ is collision-free with $\tau_1[i_1]$.
Therefore, $\tau_2[i_2], ..., \tau_2[i_2'-1]$ is collision-free with $\tau_1[i_1-1]$, and $\tau_2[i_2']$ is collision-free with $\tau_1[i_1]$.
We add the collision-free dual-arm configurations to $\tau$ (Line 6-7) and update indices $i_1$ and $i_2$ (Line 8).
To summarize, if $R_2$ blocks $\tau_1$, then $R_2$ yields along $\tau_2$ as needed; otherwise, $R_2$ stays.

\begin{algorithm}
\begin{small}
    \SetKwInOut{Input}{Input}
    \SetKwInOut{Output}{Output}
    \SetKwComment{Comment}{\% }{}
    \caption{PlanPartialPath}
		\label{alg:partial}
    \SetAlgoLined
		\vspace{0.5mm}
    \Input{$\tau_1:$ pre-defined trajectory
    }
    \Output{$\tau$: dual arm trajectory}
\vspace{0.5mm}
$\tau\leftarrow \emptyset;$\\ 
$\tau_2\leftarrow$ Motion planning between the current configuration and home configuration of $R_2$ while $R_1$ stays at the current configuration;\\
$i_1, i_2\leftarrow 0, 0$;\\
\While{$i_1<|\tau_1|$}
{
$i_2'\leftarrow$ CollisionFreeFirstIndx($\tau_1[i_1], \tau_2[i_2:]$);\\
$\tau\leftarrow$ $\tau$.Extend($\tau_1[i_1-1],[\tau_2[i_2],\dots, \tau_2[i_2'-1]]$);\\
$\tau\leftarrow$ $\tau$.Extend($\tau_1[i_1],[\tau_2[i_2']]$);\\
$i_1+=1$; $i_2=i_2';$
}
\Return $\tau$;
\end{small}
\end{algorithm}

\subsection{Real2Sim2Real}
Solving CDR mainly involves computing solutions, simulating them, and transferring them to real robot systems. To apply \modap to real robot systems, however, requires us first to build an accurate model (digital twin) of the actual system.
%due to systematic error introduced in the process of mounting a dual-arm setup in real, the relative location of one arm to the other is not guaranteed to be identical to the simulation. Thus it is important to reduce the gap between the simulation and the reality if we want to transfer the planned trajectory to the real world. 
To achieve this, we calibrate the robots relative to each other by placing them (at different times) at the same seven diverse positions in the workspace alternately. This allows us to compute an accurate transformation between the two robot arms, upon which a digital twin of the actual setup is obtained. 
%bration, we place the two robot arms at the same relative positions in simulation according to the calculated transformation of the dual-arm system. 
All the objects to be manipulated are labeled with ArUco markers\cite{garrido2014automatic}, and we can directly transfer their positions into the simulation with a calibrated camera.

\begin{figure}[h]
    \vspace{-1mm}
    \centering
\includegraphics[width=1\columnwidth]{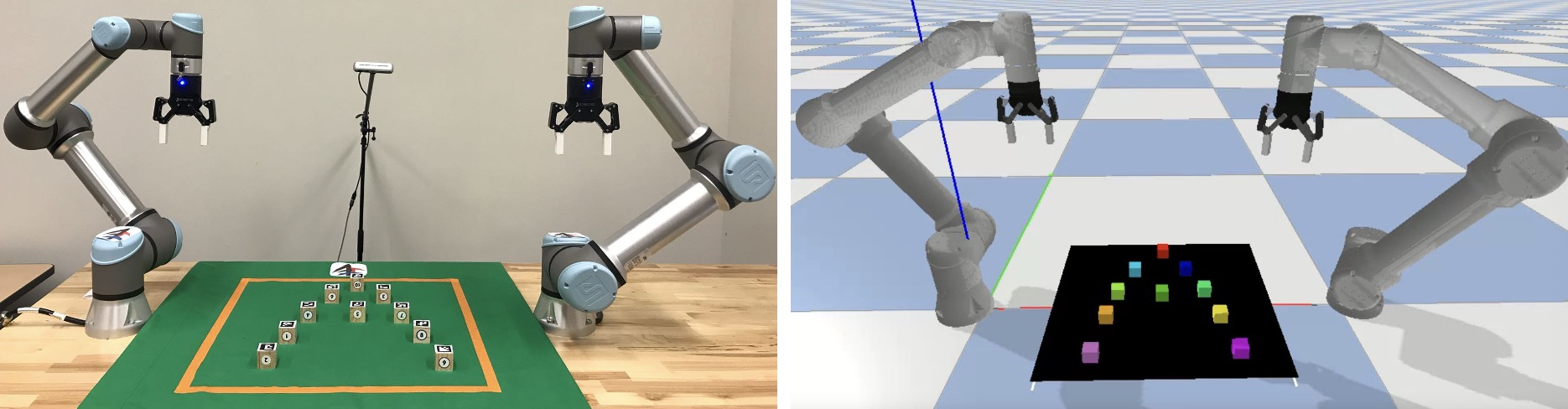}
    \caption{[left] A configuration of our dual-arm setup, also showing that the camera is mounted at the side to detect objects' poses. [right] The corresponding Dual-Arm setup in pyBullet simulation.}
    \label{fig:simnreal-setup}
    \vspace{-3mm}
\end{figure}

Since the two real robot arms are calibrated and their positions are transferred into the simulation, the trajectories planned in simulations can be seamlessly transferred to the real robot system and executed directly without further post-processing. \ref{fig:simnreal-setup} shows that our real and simulated setups mirror each other.

\subsection{Real Robot Control}
The open-source Python library ur\_rtde\cite{urrtde} is used to communicate with and control the robot, which allows us to receive robot pose/configuration and send control signals in real-time (up to 500Hz). Our controller sends joint positions for continuous control via the ServoJ() function, a PD-styled controller. The control signals are sent out alternatively at every cycle (e.g., with a 0.002-second interval), and synchronized between the two arms.

\section{Experimental Results}
\subsection{Task Planning Evaluation}
To evaluate the effectiveness of our dual-arm task planning algorithms, we 
integrate them with a simple priority-based motion planner. Essentially, 
in the presence of potential arm-arm conflict in the workspace, the motion 
planner will give the arm that is moving first priority while the second arm 
yields.
The yielding arm will either take a detour to the target pose or go back to the rest pose and wait for the execution of the other arm.
We make this choice, instead of using more sophisticated sampling-based
asymptotically optimal planners, to highlight the benefit of our task planner. 
%
%We integrate proposed task schedulers with a simple motion planner based on robot priority:
%Both arms move in the shortest paths to execute the plan and the arm with lower priority yield to the other robot if the shortest paths conflict.

The proposed algorithms are implemented in Python; we choose objects to be uniform cylinders.
In each instance, robot arms move objects from a randomly sampled start arrangement to an ordered goal arrangement (see, e.g., \ref{fig:organized}).
Besides different overlap ratios $\rho$, we test our algorithms in environments with different density levels $D$, defined as the proportion of the tabletop surface occupied by objects, 
i.e., $D:=(\Sigma_{o_i\in \mathcal O} S_{o_i})/S_{\mathcal W}$, 
where $S_{o_i}$ is the base area of $o_i$ and $S_{\mathcal W}$ is the area of $\mathcal W$.
The computed rearrangement plans are executed in PyBullet with UR-5e robot arms.

For the FC makespan, we propose an interval state space heuristic search algorithm (FCHS) in \ref{sec:CDR-KC}.
In experiments, we let the execution time of a pick $t_g$, a place $t_r$, and a handoff $t_h$ equal to the time spent for an arm traveling across the diameter of the workspace $t_d$, i.e. $t_g=t_r=t_h=t_d$.
This setting mimics general manipulation scenarios in the industry, where a successful pick-n-place relies on accurate pose estimation of the target object, a careful picking to firmly hold the object, 
and a careful placing to stabilize the object at the desired pose.
The experiments are executed on an Intel$^\circledR$ Xeon$^\circledR$ CPU at 3.00GHz. 
Each data point is the average of $20$ test cases except for unfinished trials, if any, given a time limit of $300$ seconds for each test case.

\begin{figure}[h]
    \centering
    \includegraphics[width=0.8\textwidth]{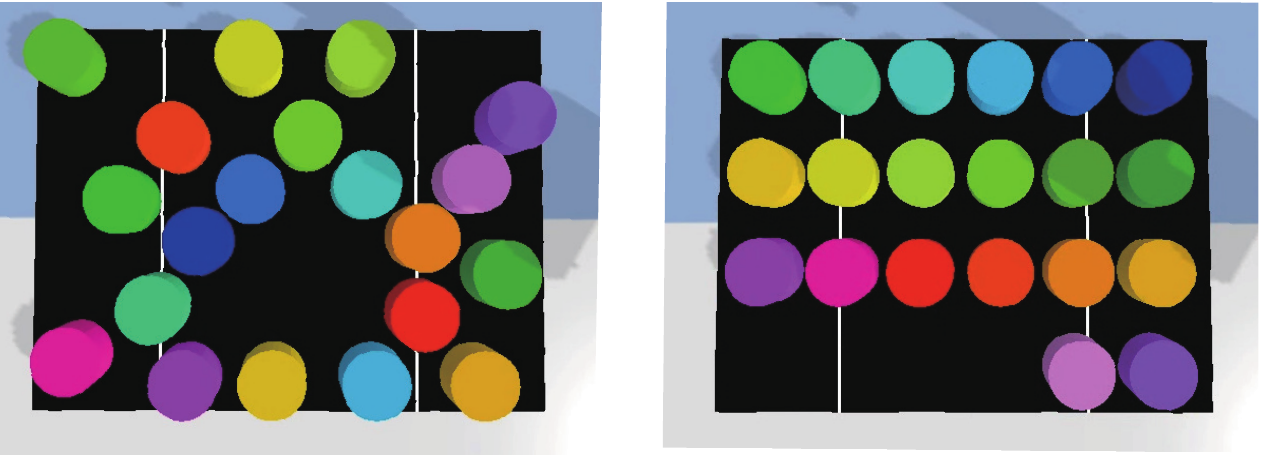}
    \caption{A 20-cylinder example of rearrangement instance with $D=0.4$ and $\rho=0.5$, moving object from a randomly sampled start arrangement [Left], to an organized goal arrangement [Right].}
    \label{fig:organized}
\end{figure}

% In full overlap cases ($\rho=1.0$), 
% we compare our proposed algorithms MCHS and FCHS with Single-MCHS, 
% where $r_1$ takes the responsibility of all rearrangement tasks.
% In partial overlap cases ($\rho<1.0$), 
\textbf{Manipulation versus Full Cost as Proxies.} In \ref{fig:MC_vs_FC}, we first compare optimal task schedules in both MC and FC makespans in instances with $\rho=0.5$, $D=0.3$.
While FCHS shows slight advantage in actual execution time, it incurs significantly more computation time and is more prone to planning failure. 
%Comparing with FCHS, MCHS is more scalable while its plans can be executed as fast as those of FCHS.
The results suggest that MC is more suitable as a proxy for optimizing CDR planning process.
%This is because travel cost is often small in tabletop manipulation tasks. Therefore, optimizing robot traveling distance is unnecessary. 
%Instead, minimizing the number of manipulations is more important.
We also compare the execution time of MC-optimal and FC-optimal plans in instances under different density levels and overlap ratios and observe similar outcome. 
Based on the observation, for later experiments, MCHS is used exclusively.

\begin{figure}[h]
    \centering
    \includegraphics[width=0.96\textwidth]{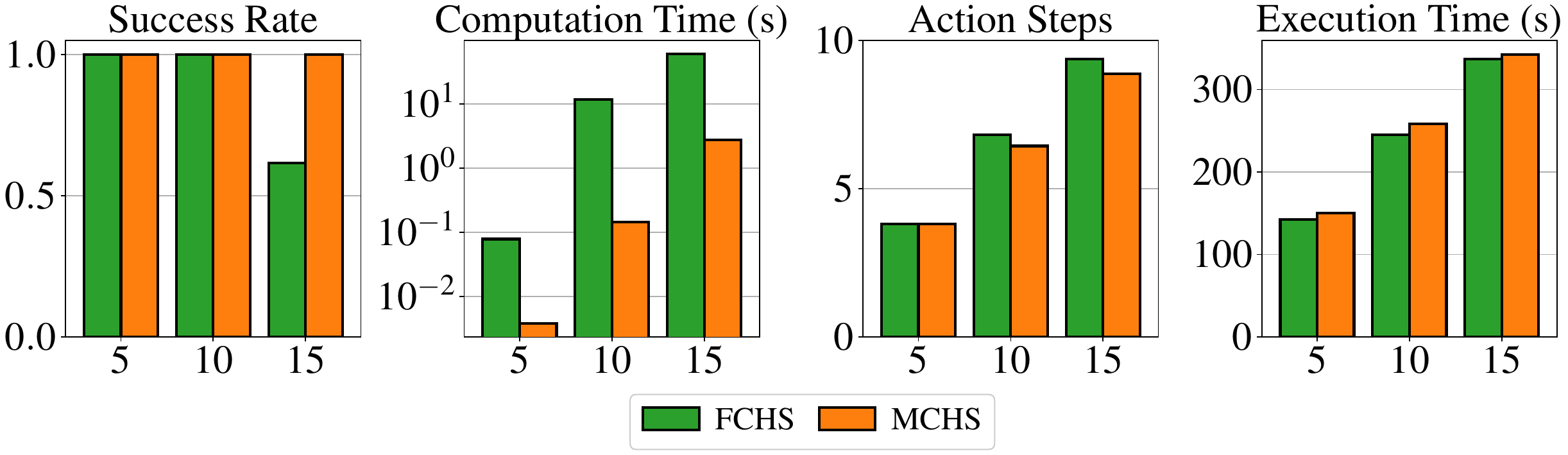}
    \caption{Performance of FCHS and MCHS in instances with $\rho=0.5$, $D=0.3$. The x-axis represents the number of objects.}
    \label{fig:MC_vs_FC}
\end{figure}

\textbf{Comparison with Baselines.} We compare the proposed algorithms with two baseline methods:
\begin{enumerate}
    \item \emph{Single-MCHS-Split:} A schedule of primitive actions for a single arm is first computed, minimizing the MC makespan, i.e. the total number of actions.
    Then, the tasks are assigned to the two arms as evenly as possible.
    A handoff is coordinated if the manipulating object is moving between $\mathcal S(r_1)$ and $\mathcal S(r_2)$.
    \item \emph{Greedy}: Each arm prioritizes the manipulations moving objects from buffers to goals.
    When there is no such manipulation available, the arm will choose the object closest to the current end-effector position.
\end{enumerate}

In \ref{fig:density_results}, we compare MCHS to the baseline algorithms in the environments with $\rho=0.5$ and different density levels $D$.
Comparing with Single-MCHS-Split, MCHS saves up to $10\%$ execution time, which is fairly significant for logistic applications.
We note that Single-MCHS-Split, guided by dependency graph, already performs quite well in dense environments.
The execution time gap comes from failures in coordinating Single-MCHS-Split plans, e.g. an arm idles for a few action steps waiting for a handoff or an arm holds an object for a few action steps waiting for obstacle clearance.
Comparing with the greedy method, MCHS saves more execution time as $D$ increases. Specifically, in 20-cylinder instances with $D=0.4$, the execution time of MCHS plans is $35\%$ shorter than greedy plans. 
Without a long-term plan in mind, greedy planning incurs more temporary object displacements.
%The comparison between dependency graph guided search algorithms (MCHS and Single-MCHS-Split) and Greedy indicates that dealing with object dependencies is significant in increasing efficiency in a CDR system.

\begin{figure}[h]
    \centering
    \includegraphics[width=0.96\textwidth]{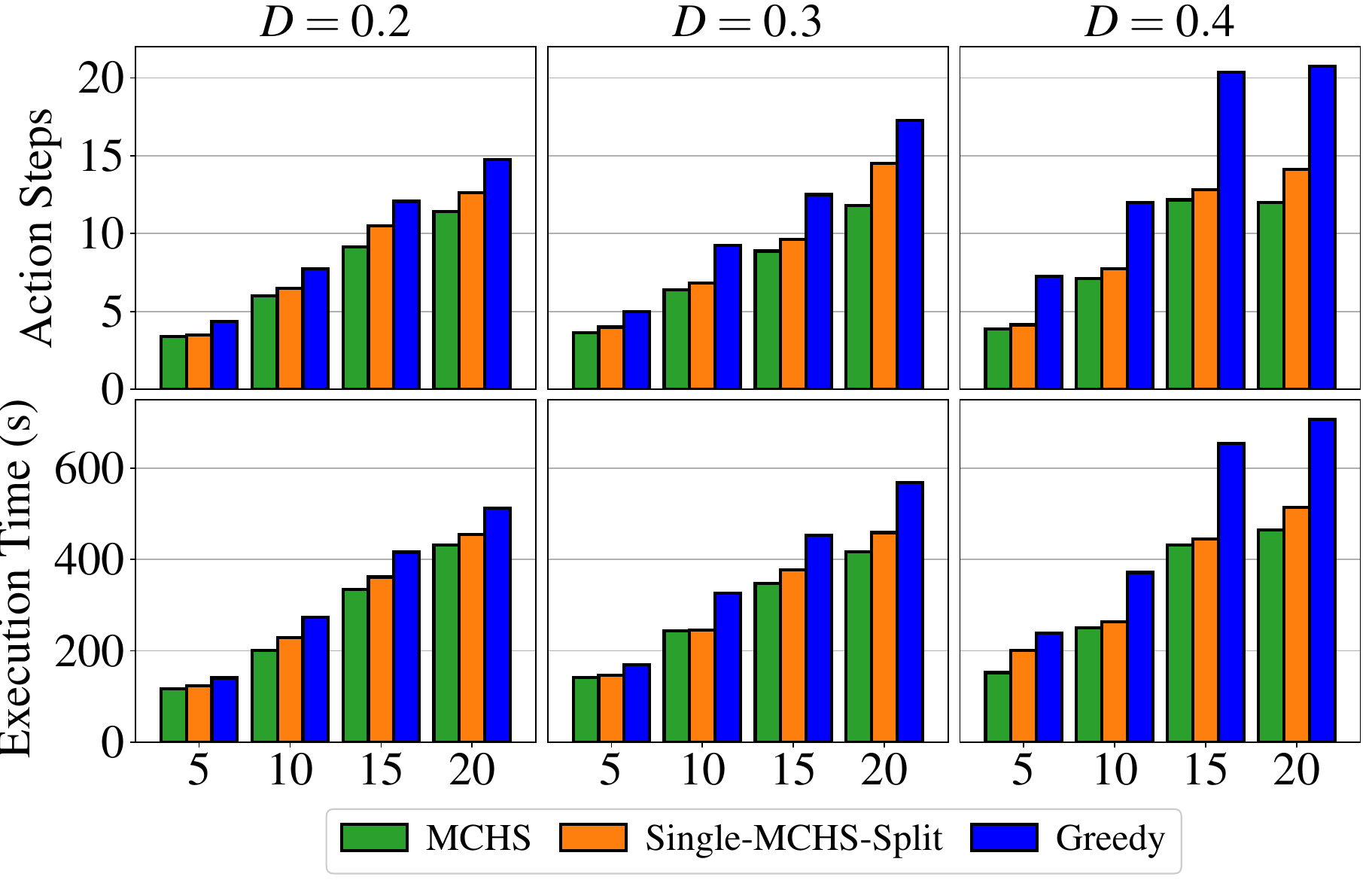}
    \caption{Algorithm comparison in environments with different density levels $D$. The x-axis represents the number of objects.}
    \label{fig:density_results}
\end{figure}

In instances with $D=0.3$, we also evaluate the tendency of the makespan and execution time as $\rho$ varies.
The conflict proportion shows the proportion of execution time for an arm to yield to the other arm during the execution.
The results are shown in \ref{fig:various_overlap}.
As $\rho$ increases, the shared area of robot arms expands. 
On one hand, the number of objects that need handoffs decreases. So does the number of action steps in the schedule.
On the other hand, we see an increase of path conflicts.
In instances with larger $\rho$, robot arms spend more time on yielding.
Based on the two factors, there is a shallow ``U'' shape in execution time as $\rho$ increases.

\begin{figure}[h]
    \centering
    \includegraphics[width=0.8\textwidth]{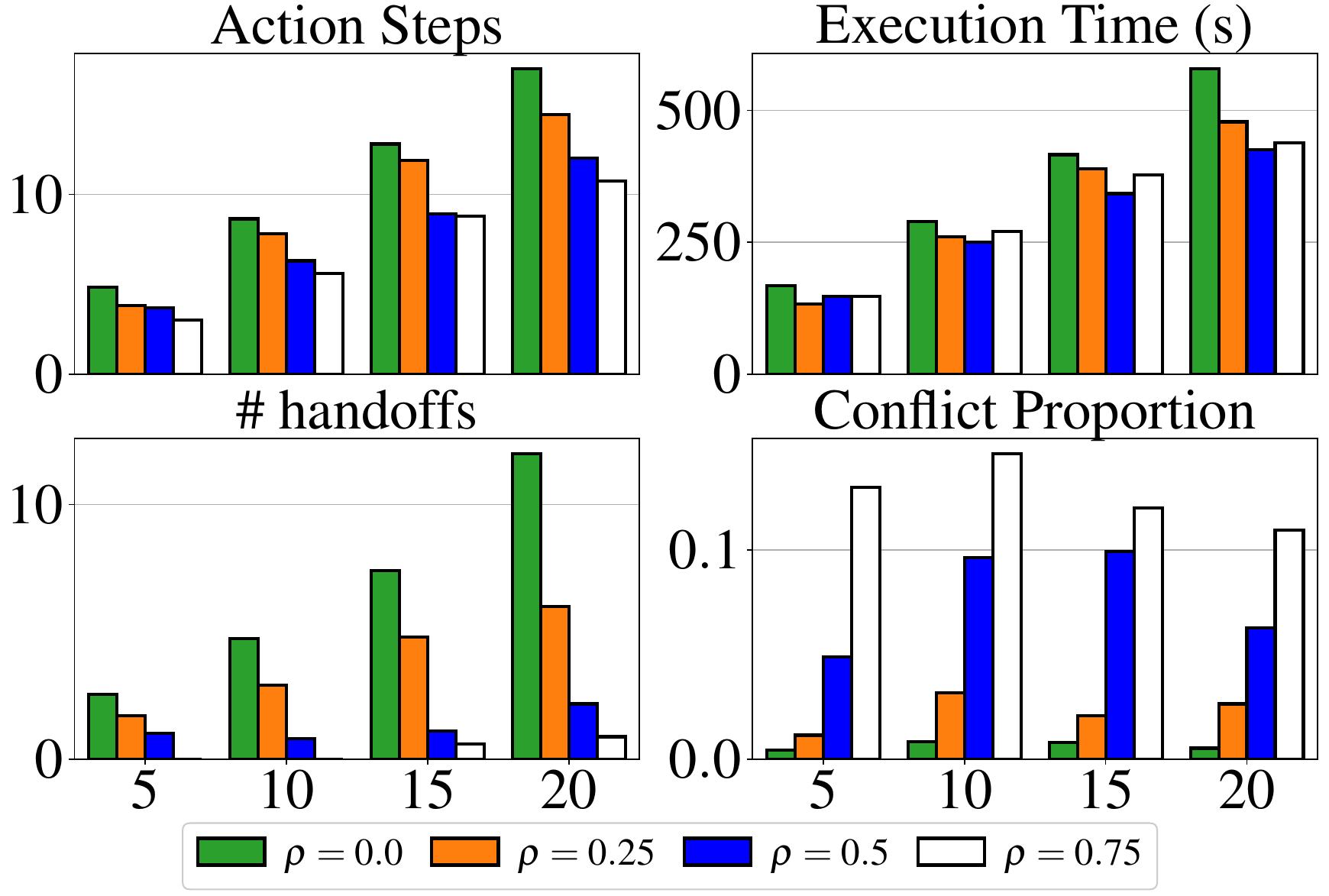}
    \caption{Evaluation of instances with different $\rho$. The x-axis represents the number of objects.}
    \label{fig:various_overlap}
\end{figure}

\textbf{Fully Overlapping Workspaces.} MCHS can also compute CDR plans with full overlap (CDRF), i.e., $\rho=1.0$.
As shown in \ref{fig:overlap_ratio}, due to the reachability limit of UR-5e, the workspace is a bounded square with sides of length $0.6m$.
In this case, each arm can reach every corner of the workspace but its manipulation is likely to be blocked by the other arm.
We compare dual-arm MCHS plans to the single-arm MCHS plans, 
where $r_1$ take responsibility of all rearrangement tasks and the number of total actions is minimized.
The results (\ref{fig:full_result}) indicate that each arm in the dual-arm system spends around $18\%$ of execution time yielding or making detours due to the blockage of the other arm.
Therefore, even though MCHS saves $50\%$ action steps, the execution of the plans is only around $10\%$ faster than that of the single-arm rearrangement plans.
However, the efficiency gain shown in action steps also suggests that the dual-arm system has the potential to save up to half of the execution time with specially designed arms for CDRF problems.

\begin{figure}[h]
    \centering
    \includegraphics[width=0.6\textwidth]{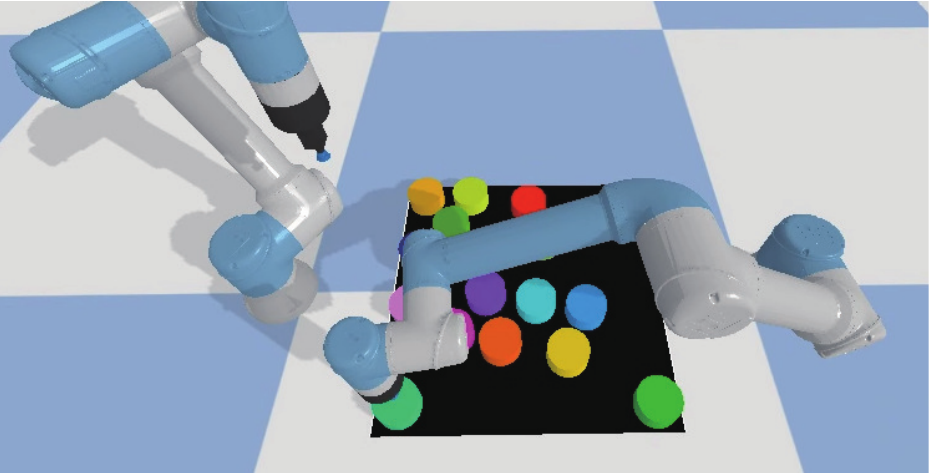}
    \caption{An instance of CDR with full overlap (CDRF), where each arm can reach every corner of the workspace but its manipulation is likely to be blocked by the other arm.}
    \label{fig:overlap_ratio}
\end{figure}

\begin{figure}[h]
    \vspace{-3mm}
    \centering
    \includegraphics[width=0.9\textwidth]{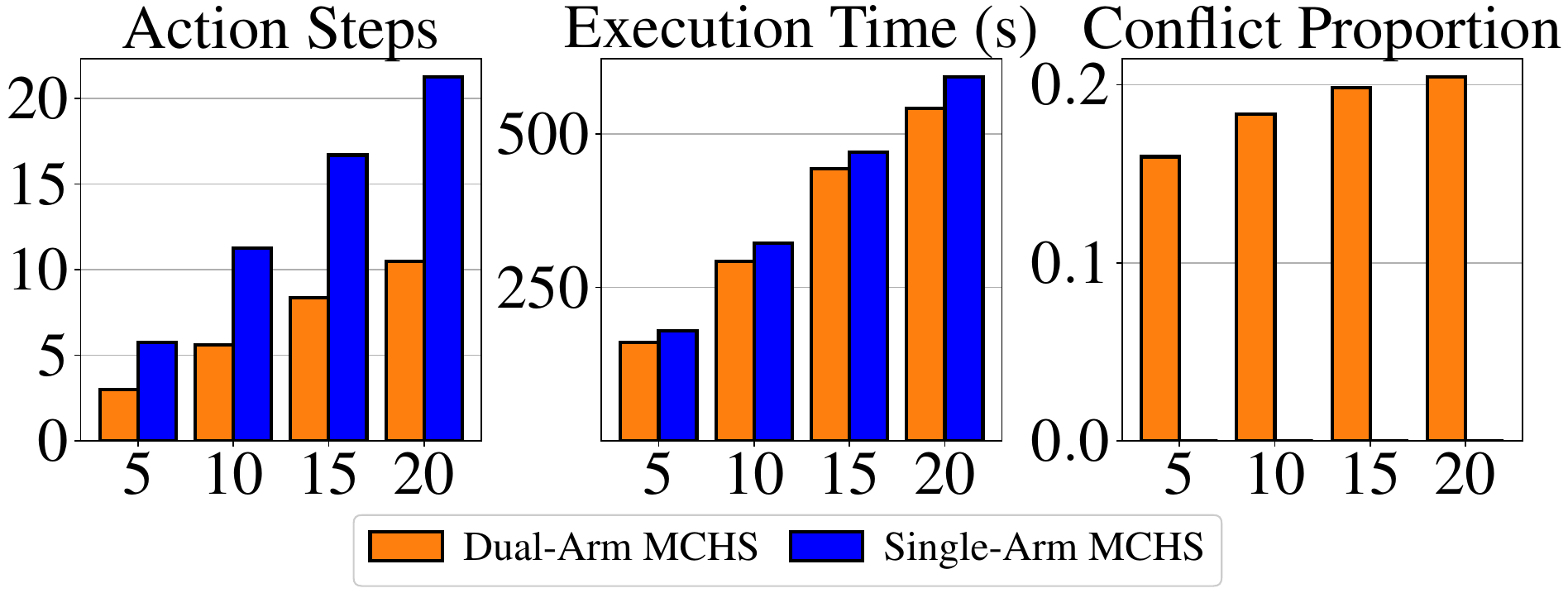}
    \caption{Algorithm performance in CDR instances with $\rho=1.0$. The x-axis represents the number of objects.}
    \label{fig:full_result}
    
\end{figure}
\vspace{-4mm}

\subsection{Motion Planning Evaluation}
In this section, we evaluate our \modap and compare it with two baseline planners both in simulation and on real robots.
The two baseline planners are:
\begin{enumerate}[leftmargin=5mm]
    \item \textbf{Baseline (BL)} Similar to \modap, BL computes collision-free dual-arm motion plans based on MCHS task plans.
    When the clearance between robots is larger than a certain threshold, both arms execute their tasks individually.
    Otherwise, one arm yields by moving toward its home state, which is guaranteed to be collision-free.
    After obtaining a collision-free initial trajectory, BL smoothens the yielding trajectories: instead of switching frequently between yielding and moving toward the next target, the robot attempts to move back to a collision-free state and waits until the conflict is resolved.
    \item \textbf{Baseline+Toppra (BL-TP)} BL-TP uses Toppra to speed up the trajectory planned by BL.
\end{enumerate}

\subsection{Sim2Real Gap: Trajectory Tracking Accuracy}
Whereas we must evaluate plans on real robots to ensure they work in the real world, doing so extensively prevents us from running a large number of experiments as it is time-consuming to set up and run real robots. 
To that end, we first evaluate the sim2real gap of our system. 
The Root-Mean-Square Deviation(RMSD) is introduced as a metric to evaluate the accuracy of the real robot control for one waypoint $q_i$ in the planned trajectory $\tau$
\vspace{-2mm}
\begin{equation*}
    RMSD(q_i) = \sqrt{\frac{\left( \hat{q_i} - q_i\right)^T \left( \hat{q_i} - q_i\right)}{\left| q_i\right|}}
\vspace{-2mm}
\end{equation*}
where $\hat{q_i}$ denotes the real configuration retrieved from the robot while executing the control signal $q_i$ and $\left| q_i \right|$ denotes the degree of freedom of the robot system. The control quality of a trajectory can be evaluated as
\vspace{-2mm}
\begin{equation*}
     Q(\tau) = \mathbb{E}[RMSD(q_i)]
\vspace{-2mm}
\end{equation*}
Trajectories are collected from 5 different scenes with four randomly placed cubes as the start and goal, planned by different methods, and then executed in simulation and real-world at different speed limits. 
In the meantime, RMSD and Q for the end-effector position can be recorded.
In \ref{fig:real_data}[top], we show the RMSD of the configuration and the end effector's position on each waypoint in a trajectory planned using \modap. 
We observe that we can see that the execution error is about $1^{\circ}$ for each joint and 5mm as the end-effector position is concerned, executing at a speed of 1.57rad/s at every joint.
\begin{figure}[h]
    \centering
        \includegraphics[width=0.95\columnwidth]{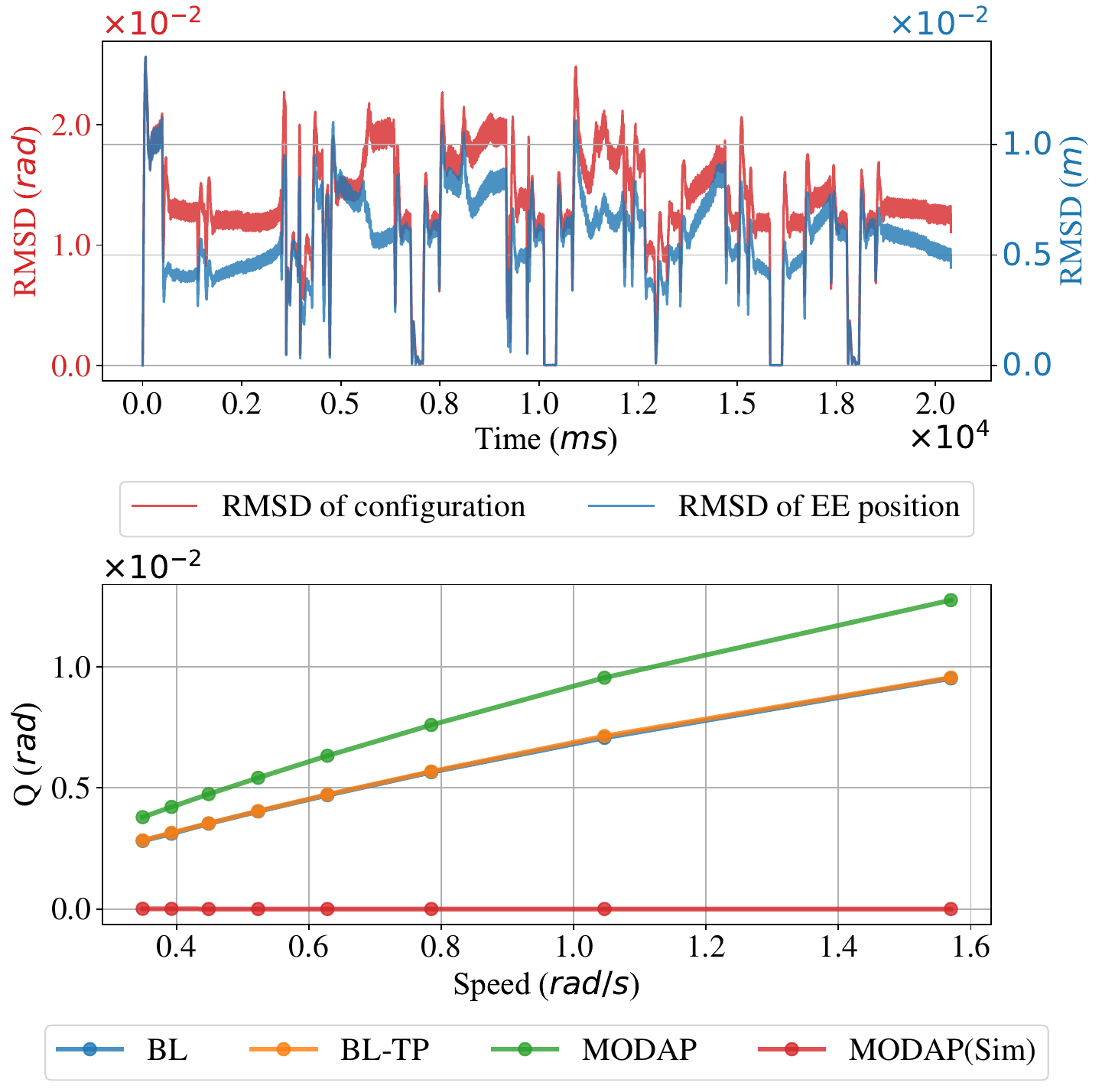}
    \caption{[top] RMSD of configuration and end-effector position on a sampled trajectory executed with a maximum speed of 1.57rad/s on real robots. [bottom] The average Q value of 5 sampled trajectories executed at different maximum speeds (sim denotes execution in simulation; otherwise, the execution is on a real robot). Note that the lines of BL and BL-TP overlap with each other.} 
    \label{fig:real_data}
    \vspace{-6mm}
\end{figure}
\begin{figure}[ht]
    \centering
%    \vspace{-3mm}
\includegraphics[width=0.98\columnwidth]{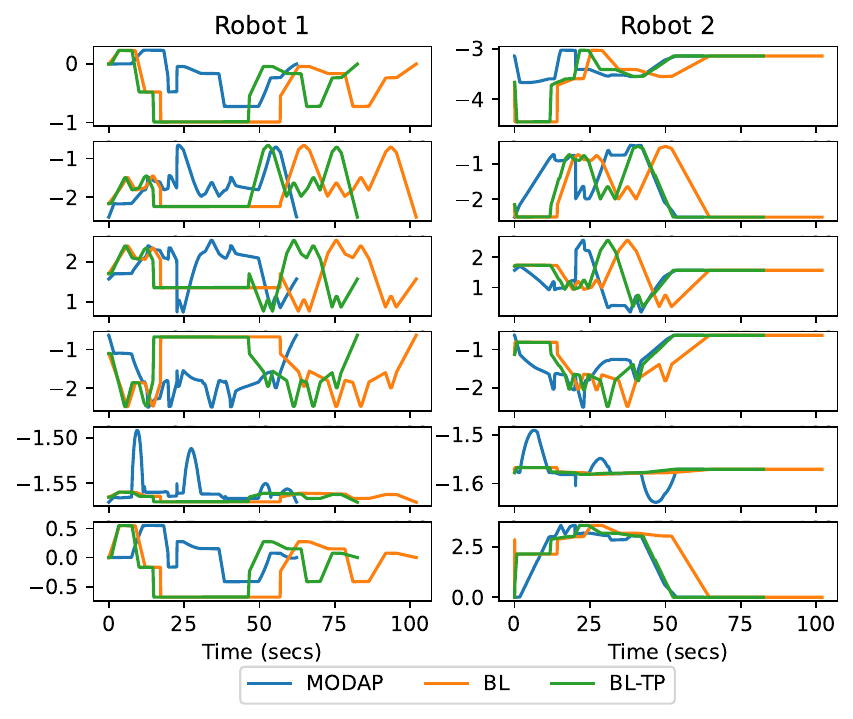}
    \vspace{-2mm}
    \caption{Trajectories of BL, BL-TP, and \modap for the a typical rearrangement task. Each row, from top to bottom, corresponds to a robot joint in radians from one to six. Each column shows the trajectory of the six joint angles of a robot.}
    \label{fig:sim_exec_joints}
\end{figure}
\ref{fig:real_data}[bottom] shows the average configuration Q value of 5 trajectories executed at different speed limits. 
First, we note that trajectory execution is perfect in simulation, as expected. 
We observed that \modap induces a slightly larger (though perfectly acceptable) execution error than BL and BL-TP. This is due to trajectories produced by \modap running at faster overall speeds. 

Overall, our evaluation of the trajectory tracking accuracy indicates that the sim2real gap is negligibly small across the methods, including \modap. This suggests we can confidently plan in simulation and expect the planned time-parameterized trajectory to execute as expected on the real system. It also means that we can compare the optimality of the methods using simulation, which we perform next. 

\subsection{Performance on Rearrangement Plans Execution} 
In evaluating the performance of \modap and comparing that with the baseline methods, we first examine the execution of \modap, BL, and BL-TP on a typical instance involving the rearrangement of five objects with $\rho=1.0$ (full workspace overlap). The time-parameterized trajectories of the six robot joints are plotted in \ref{fig:sim_exec_joints} for all three methods.
We can readily observe that \modap is much more time-efficient than BL and BL-TP. 
We can also readily observe the reason behind this: 
BL has one arm idling in most of the execution due to inefficient dual-arm coordination. 
After Toppra acceleration, BL-TP is still $32\%$ slower than \modap.

\ref{fig:sim_data} shows the result of a full-scale performance evaluation under three overlapping ratios ($\rho = 0.5, 0.75, 1$) and different numbers of objects to be rearranged ranging between $6$-$18$. 
The simulation experiments are executed on an Intel$^\circledR$ Core(TM) i7-9700K CPU at 3.60GHz and an NVIDIA GeForce RTX 2080 Ti GPU for cuRobo planner. For each method, the average execution time and total trajectory length are given; each data point (with both mean and standard deviation) is computed over $20$ randomly generated instances. Two example test cases are shown in  \ref{fig:sim_workspace}. We set the maximum per-instance computation time allowance to be 500 seconds, sufficient for all computations to complete. \modap, exploring many more trajectories than BL and BL-TP, takes, on average, 2-3 times more computation time. 
The total length is the sum of robot trajectory lengths computed under the L2 norm in the 6-DoF robot's configuration space.
\begin{figure}[ht]
\vspace{-1.5mm}
    \centering
\includegraphics[width=0.48\columnwidth]{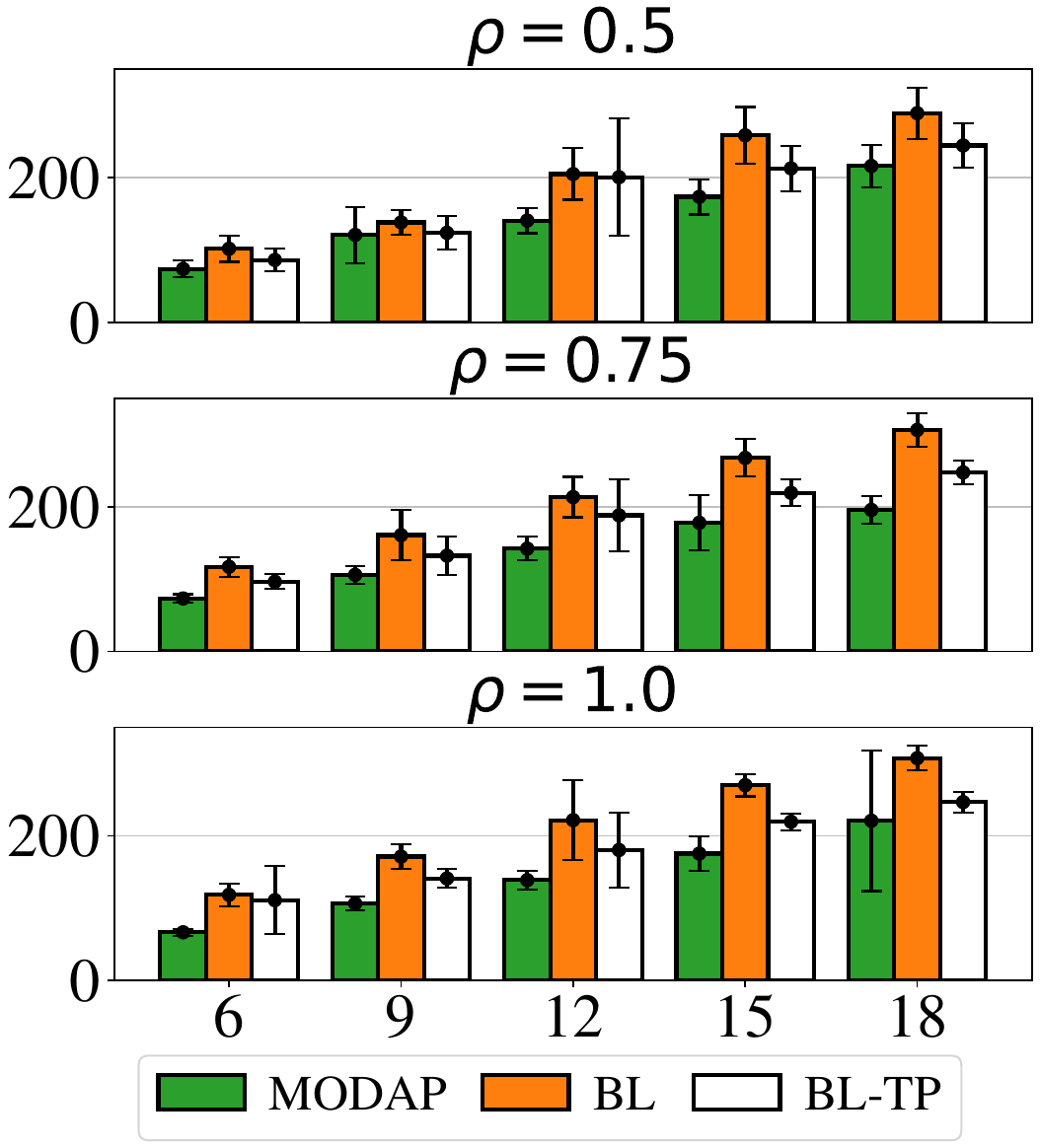}
%\vspace{2mm}
\includegraphics[width=0.48\columnwidth]{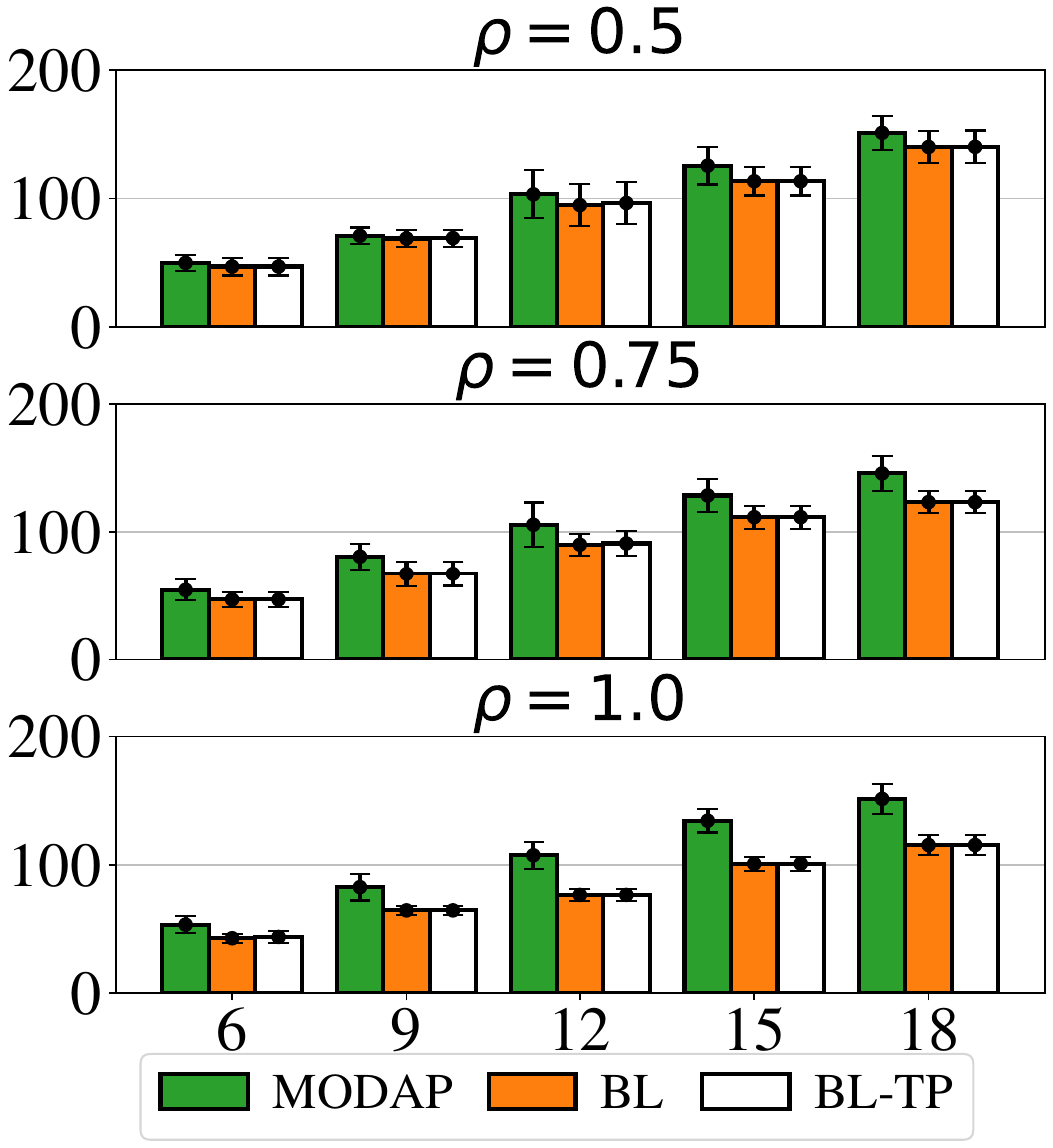}
    \caption{Average execution time (left) and trajectory total length (right) computed by compared methods under different overlap ratios $\rho$ and the number of workspace objects.}
    \label{fig:sim_data}
\vspace{-3mm}
\end{figure}
\begin{figure}[ht]
    \centering
\includegraphics[trim={48cm, 22cm, 48cm, 25cm}, clip, width=0.24\columnwidth]{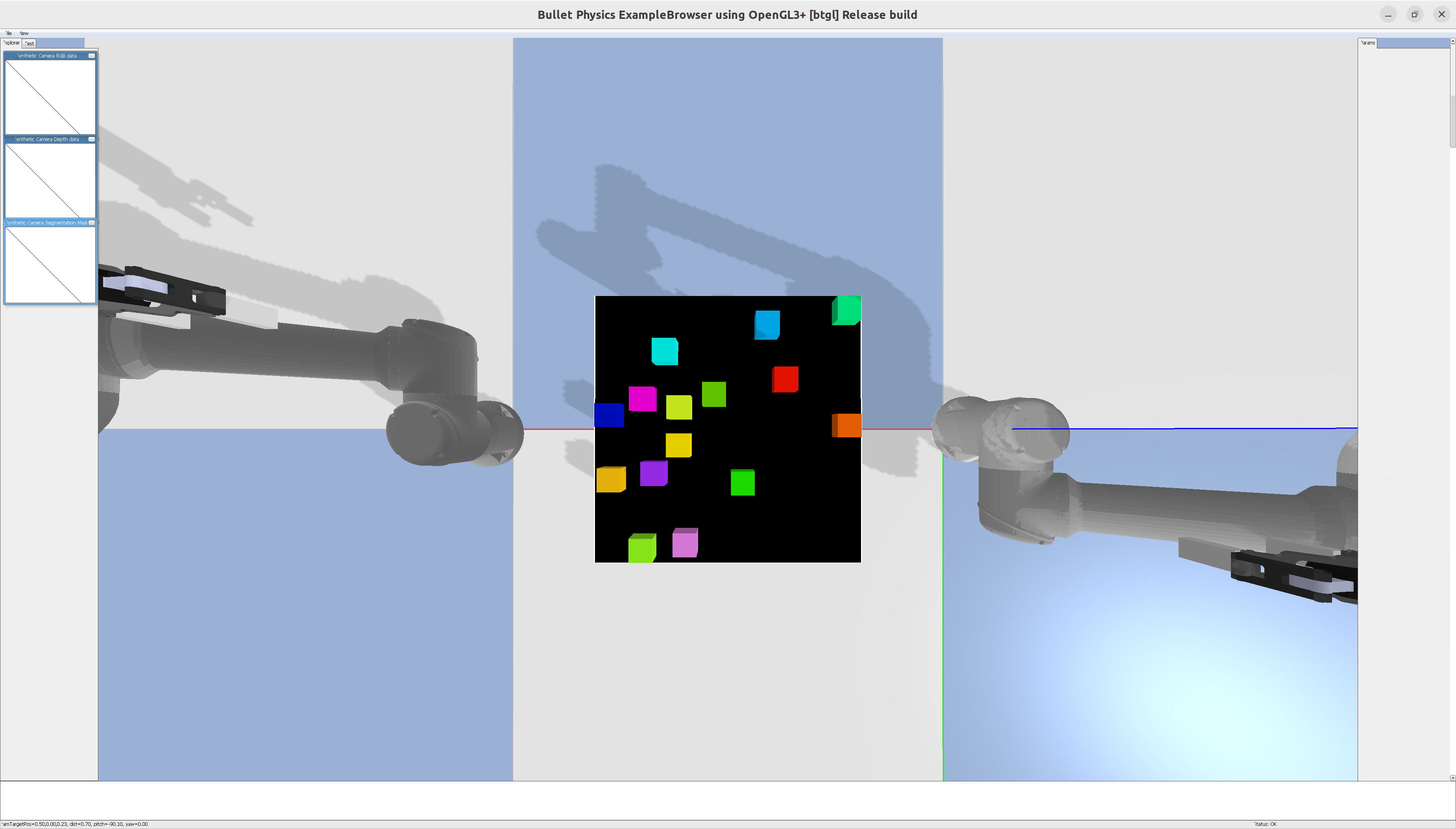}
\put(-58, -10){(a)}
%\vspace{2mm}
\includegraphics[trim={48cm, 22cm, 48cm, 25cm}, clip, width=0.24\columnwidth]{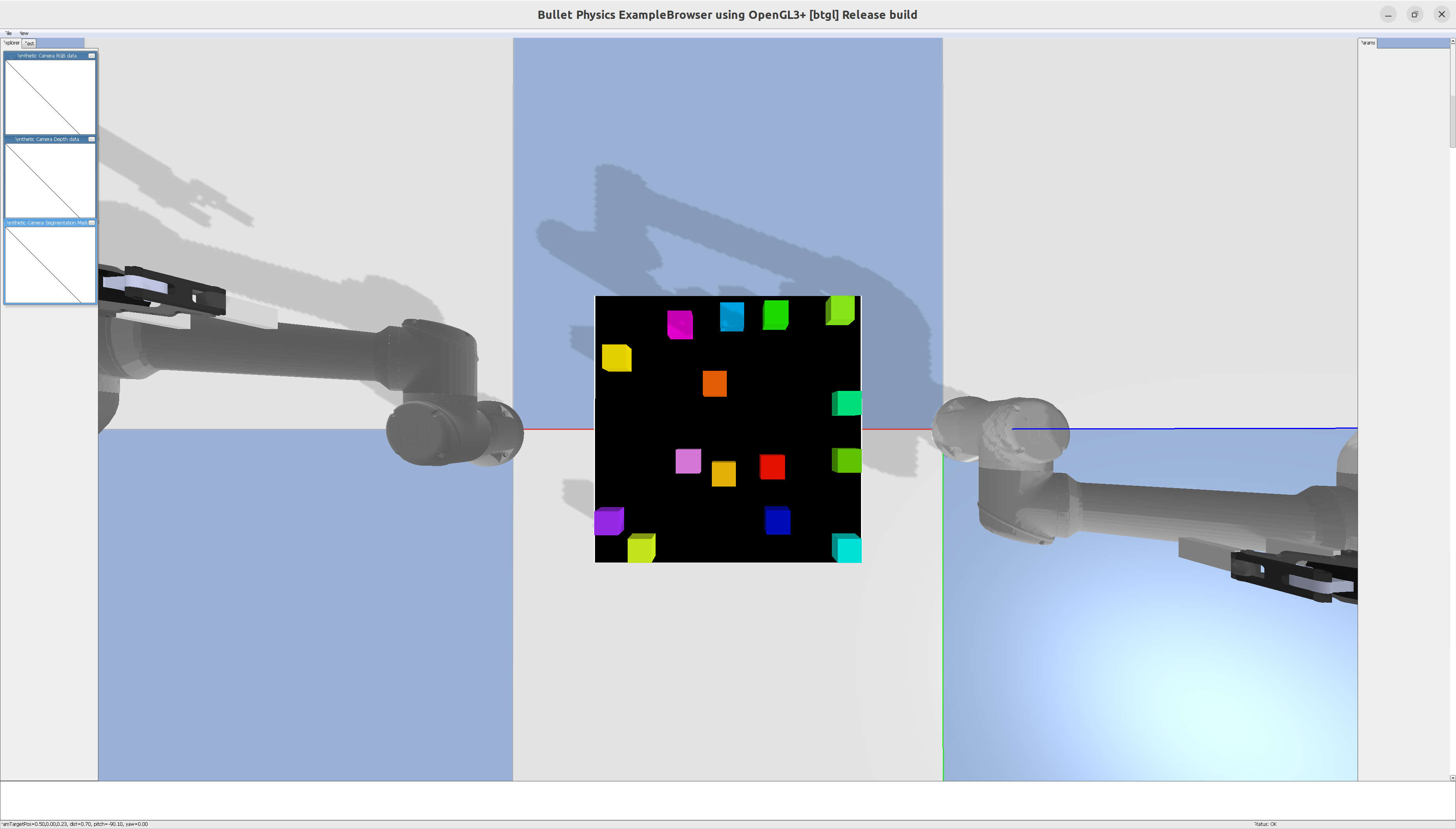}
\put(-58, -10){(b)}
\includegraphics[trim={48cm, 22cm, 48cm, 25cm}, clip, width=0.24\columnwidth]{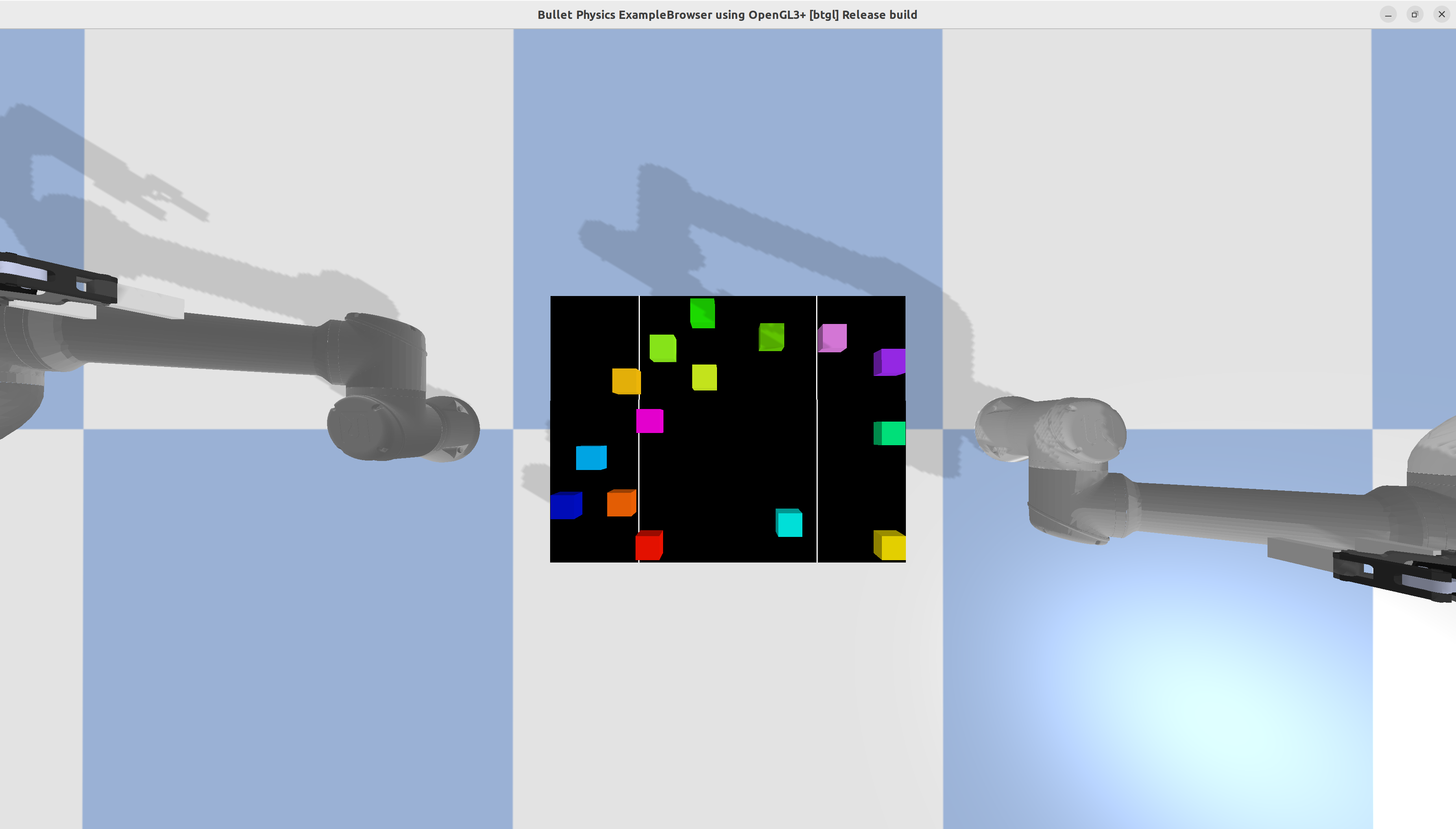}
\put(-58, -10){(c)}
%\vspace{2mm}
\includegraphics[trim={48cm, 22cm, 48cm, 25cm}, clip, width=0.24\columnwidth]{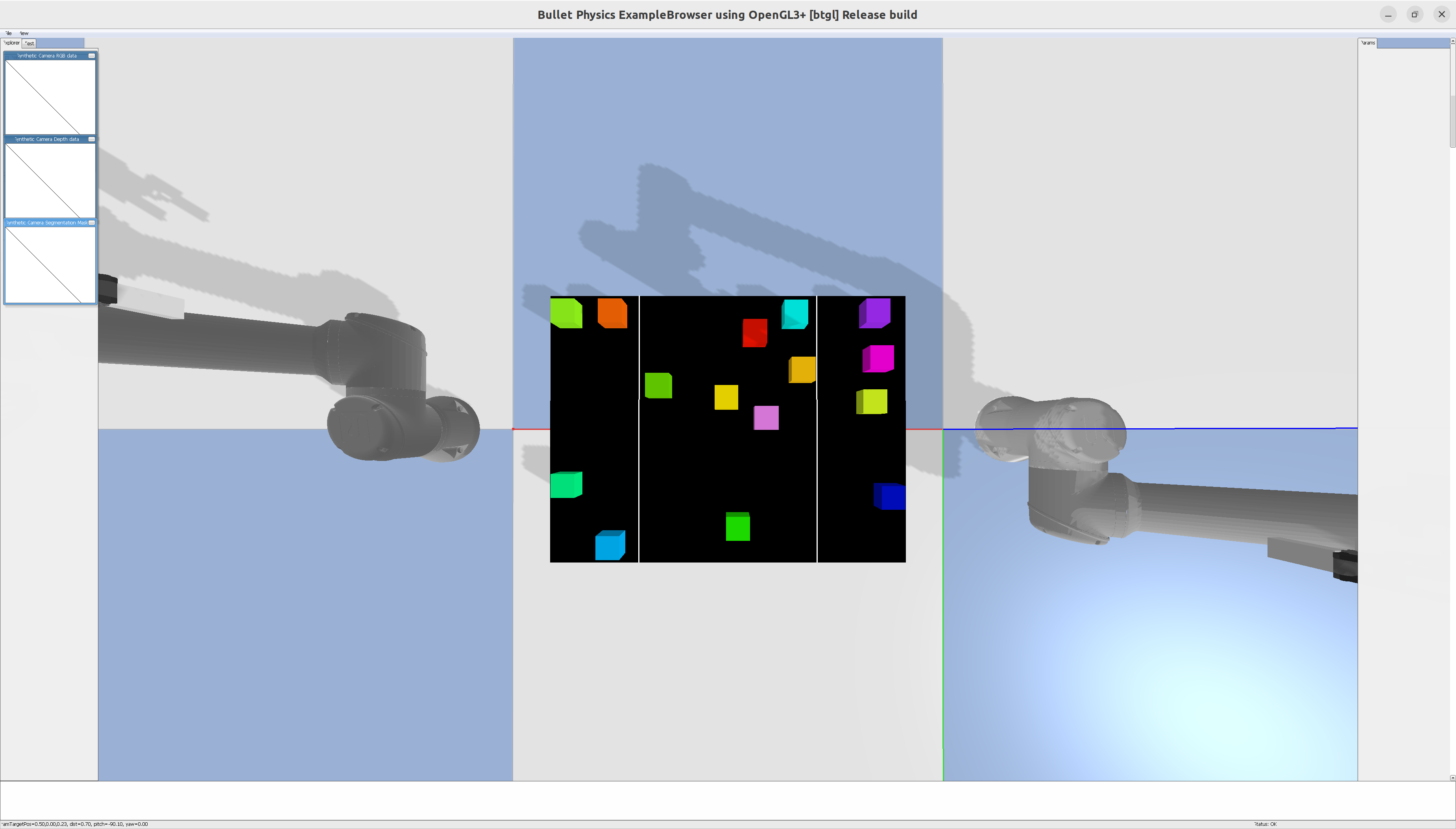}
\put(-58, -10){(d)}
    \caption{Two examples of start and goal configurations with 15 objects at overlap rates $\rho=1$ (a, b) and $\rho=0.5$ (c, d). (a) and (c) (resp., (b) and (d)) are starting (resp., goal) configurations.}
    \label{fig:sim_workspace}
\vspace{-3mm}
\end{figure}
We observe that \modap's plans can be executed faster than BL and BL-TP even though cuRobo generates longer trajectories, which aligns with the observation based on \ref{fig:sim_exec_joints}. 
When $\rho=0.5$, \modap saves $13\%-32\%$ execution time compared with BL, and $3\%-30\%$ compared with BL-TP.
When $\rho=1.0$, \modap saves $29\%-44\%$ execution time compared with BL, and $11\%-40\%$ compared with BL-TP.
\modap shows apparent efficiency gain as the overlap ratio $\rho$ increases, which shows that the proposed \modap effectively resolves conflicts between robot trajectories.

%% file: chapters/distance_optimal.tex
\chapter{Distance-Optimal TRLB: Object Rearrangement with Lazy A*}\label{chap:orla}
\thispagestyle{myheadings}

\section{Motivation}
Effectively performing object rearrangement is an essential skill for mobile manipulators, e.g., setting up a dinner table. A key challenge in such problems is deciding an appropriate ordering to effectively untangle object-object dependencies while considering the necessary motions for realizing the manipulations (e.g., pick and place).
To our knowledge, computing time-optimal multi-object rearrangement solutions for mobile manipulators remains a largely untapped research direction.
This study develops effective solutions for mobile manipulators over a larger workspace and jointly considers the arrangement's stability. 

To that end, this chapter proposes \orla: \emph{Object Rearrangement with Lazy A*} for solving mobile manipulator-based rearrangement tasks. 
This study carefully investigated factors impacting the optimality of a rearrangement plan for mobile manipulators and provided insightful structural understandings for the same. 
Among these, a particularly interesting one is that the mobile base travels on $S^1$ (i.e., a cycle), leading to intricate interactions with other factors in the optimization task. 
\orla designs a suitable cost function that integrates the multiple costs and employs the idea of lazy buffer allocation (buffers are temporary locations for objects that cannot be placed at their goals) into the \astar framework to search for rearrangement plans minimizing the cost function.
To accomplish this, the $f,g,h$ values in \astar are redefined when some states in the search tree are non-deterministic due to the delayed buffer computation.
With optimal buffer computation, \orla returns globally optimal solutions.
A thorough feasibility and optimality study backs our buffer allocation strategies. 
To estimate the feasibility of a buffer pose, especially when an object needs to be placed on top of others temporarily, a learning model, \model, is proposed to estimate the stability of the placing pose.

\section{Related Work}
{\bf Tabletop Rearrangement} In tabletop rearrangement tasks, the primary challenge lies in planning a long sequence of actions in a cluttered environment.
Such rearrangement planning can be broadly divided into three primary categories: prehensile\cite{zeng2021transporter,gao2023minimizing,han2018complexity,zhang2022visually,gao2022fast,labbe2020monte,ding2023task, xu2023optimal}, non-prehensile\cite{yuan2019end,vieira2022persistent,huang2021dipn,song2020multi,huang2021visual}, and a combination of the two\cite{tang2023selective}.
Compared with non-prehensile operations (e.g., pushing and poking), prehensile manipulations, while demanding more precise grasping poses\cite{zeng2021transporter} prior to picking, offer the advantage of placing objects with higher accuracy in desired positions and facilitate planning over longer horizons\cite{gao2022toward}.
In this domain, commonly employed cost functions encompass metrics like the total count of actions\cite{gao2022fast,gao2023minimizing,xu2023optimal}, execution duration\cite{gao2022toward}, and end-effector travel\cite{han2018complexity,song2020multi}, among others\cite{gao2023effectively}.
In this paper, we manipulate objects with overhand pick-n-places. 
Besides end-effector traveling costs, we also propose a novel cost function considering the traveling cost of a mobile robot.

{\bf Buffer Allocation} In the realm of rearrangement problems, there are instances where specific objects cannot be moved directly to their intended goal poses. 
These scenarios compel the temporary movement of objects to collision-free poses. 
To streamline the rearrangement planning, several rearrangement methodologies exploit external free spaces as buffer zones\cite{gao2023minimizing,xu2023optimal}. 
One notable concept is the \emph{running buffer size}\cite{gao2023minimizing}, which quantifies the requisite size of this external buffer zone. 
In situations devoid of external space for relocation, past research either pre-identifies potential buffer candidates\cite{wang2021uniform,cheong2020relocate} or segments the rearrangement tasks into sequential subproblems\cite{krontiris2016efficiently,wang2022lazy}.
TRLB\cite{gao2022fast}, aiming for an optimized buffer selection, prioritizes task sequencing and subsequently employs the task plan to dictate buffer allocation. 
However, TRLB doesn't factor in travel costs. 
Contrarily, in our study, we incorporate lazy buffer allocation within the \astar search and prioritize buffer poses based on various cost function optimizations.

% {\bf \astar Search} Due to the flexibility of buffer poses, our proposed Lazy \astar searches for a plan in a continuous state space rather than a graph-like discrete space. Additionally, to handle situations where some objects are in non-deterministic states (i.e., at unspecified buffer poses), we extend the definitions of $f,g,h$ values in \astar.

{\bf Manipulation Stability} 
Structural stability is pivotal in robot manipulation challenges. 
Wan et al.\cite{wan2018assembly} assess the stability of Tetris blocks by scrutinizing their supporting boundaries. 
For truss structures, finite element methods have been employed to assess stability of intermediate stages\cite{mcevoy2014assembly,garrett2020scalable}. 
Utilizing deep learning, Noseworth et. al.\cite{noseworthy2021active} introduce a Graph Neural Network model dedicated to evaluating the stability of stacks of cuboid objects. 
However, these methodologies often come with shape constraints, requiring objects to be in forms such as cuboid blocks or truss structures. 
For objects of more general shapes, recent research\cite{xu2023optimal,lee2023object} leverages stability checkers grounded in physics simulators, which are effective but tend to be computationally intensive. 
In contrast, our study introduces a deep-learning-based prediction model, \model, tailored for the rearrangement of objects with diverse shapes. 
\model offers a speed advantage over simulation-based checks and demonstrates robust generalization to previously unseen object categories.

\section{Mobile Robot Tabletop Rearrangement}
Consider a 2D tabletop workspace centered at the origin of the world coordinate with $z$ pointing up. 
A point $(x,y,z)$ is within the tabletop region $\mathcal W$ if $(x,y)$ are contained in the workspace and $z \geq 0$. 
The workspace has $n$ objects $\mathcal{O}$. 
The pose of a workspace object $o_i\in \mathcal O$ is represented as $(x, y, z, \theta)$.
An arrangement of $\mathcal O$ is feasible if all objects are contained in the tabletop region and are collision-free.
In this paper, we allow objects to be placed on top of others.
An object is \emph{graspable} at this arrangement if no other object is on top.
A goal pose is \emph{available} if both conditions below are true: (1). There is no obstacle blocking the pose; (2). If other objects have goal poses under the pose, these objects should have been at the goal poses.

On the table's edge, a mobile robot equipped with an arm moves objects from an initial arrangement $\mathcal A_I$ to a desired goal arrangement $\mathcal A_G$ with overhand pick and place actions.
Each action $a$ is represented by $(o_i, p_1, p_2)$, which moves object $o_i$ from current pose $p_1$ to a collision-free target pose $p_2$.
A rearrangement plan $\Pi=\{a_1, a_2, a_3, ...\}$ is a sequence of actions moving objects from $\mathcal A_I$ to $\mathcal A_G$.

We evaluate solution quality with a cost function $J(\Pi)$ (Eq.\ref{eq:cost}), which mimics the execution time of the plan. 
The first term is the total traveling cost, and the second is the manipulation cost associated with pick-n-places, which is linearly correlated with the number of actions.
\vspace{1mm}
\begin{align}\label{eq:cost}
    %\begin{split}    
    J(\Pi)&=dist(\Pi)+mani(\Pi),  \ 
    mani(\Pi)=C|\Pi|
    %\end{split}
\end{align}

Based on the description so far, we define the studied problem as follows.

\begin{problem}[Mobile Robot Tabletop Rearrangement (\motar)]
    Given a feasible initial arrangement $\mathcal A_I$ and feasible goal arrangement $\mathcal A_G$ of an object set $\mathcal O$, find the rearrangement plan $\Pi$ minimizing the cost $J(\Pi)$.
\end{problem}

We study \motar under two scenarios.
On the one hand, for a small workspace ( \ref{fig:scenarios}[Left]), where the mobile robot can reach all tabletop poses at a fixed base position.
We define $dist()$ as the Euclidean travel distance of the end effector (EE) in the x-y plane.
On the other hand, for a large table workspace( \ref{fig:scenarios}[Right]), where the mobile base needs to travel around to reach the picking/placing poses, we define $dist()$ as the Euclidean distance that the mobile base (MB) travels.
As shown in \ref{fig:workingExample}[Left], we assume the mobile base travels along the boundaries of the table. 
When the robot attempts to pick/place an object at pose $(x,y,z,\theta)$, it will move to the closest point on the track to $(x,y)$ before executing the pick/place.
In the remainder of this paper, we will refer to the scenarios as \emph{EE} and \emph{MB}, respectively.

\begin{figure}[h]
    \centering
    \includegraphics[width=\columnwidth]{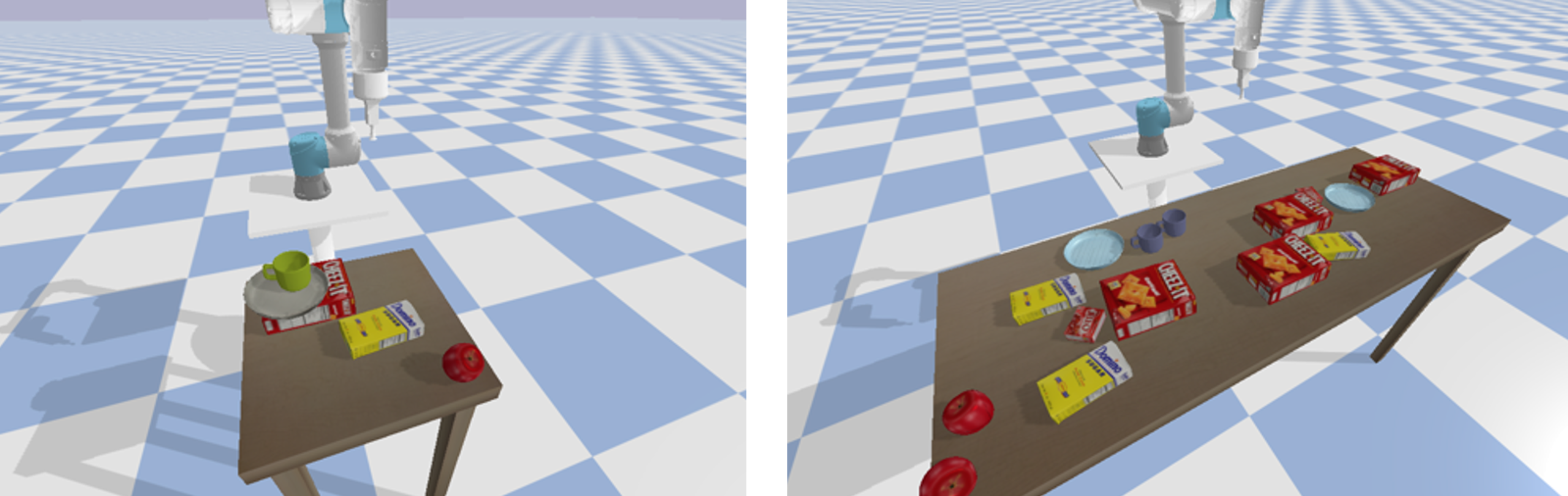}
    \caption{[Left] An example of the EE scenario, where the table is small and the robot can reach all poses from a fixed position. We count the traveling cost of the end-effector (EE) in the cost function. [Right] An example of the MB scenario, where the table is large and the robot can only reach a portion of tabletop poses from a fixed position. We count the traveling cost of the mobile base (MB) in the cost function.}
    \label{fig:scenarios}
\end{figure}

\section{\orla: A* with Lazy Buffer Allocation}\label{sec:planner}
We describe \orla, a \emph{lazy} A*-based rearrangement planner that delays buffer computation, specially designed for mobile manipulator-based object rearrangement problems.
As a variant of A*, \orla always explores the state $s$ that minimizes the estimated cost $f(s)=g(s)+h(s)$. 
An action from $s$ to its neighbor moves an object to the goal pose or a buffer.
Specifically, actions from $s$ follow the rules below:
\begin{enumerate}[leftmargin=5mm]
    \item \textbf{R1.} If $o_i$ is graspable and its goal pose is also available at $s$, move $o_i$ to its goal.
    \item \textbf{R2}. If $o_i$ is graspable, its goal is unavailable, and it causes another object to violate R1, then move $o_i$ to a buffer.
\end{enumerate}
When \orla decides to place an object at a buffer, it does not allocate the buffer pose immediately. 
Instead, the buffer pose is decided after the object leaves the buffer.
In this way, lazy buffer allocation effectively computes high-quality solutions with a low number of actions as shown in \ref{chap:trlb}.
Under the A* framework, \orla searches for buffer poses with the minimum additional traveling cost.

\subsection{Deterministic and Nondeterministic States}
To enable lazy buffer allocation, \orla categorizes states into deterministic states (DS) and non-deterministic states (NDS).
Like traditional \astar, a DS represents a feasible object arrangement in the workspace.
Each object has a deterministic pose at this state.
NDS is a state where some object is at a buffer pose to be allocated.

\begin{figure}[ht]
    \centering
    \includegraphics[width=0.32\textwidth]{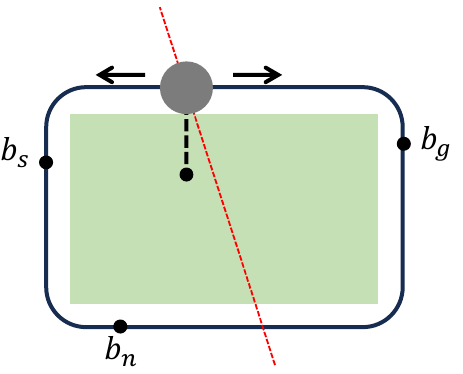}
    \hspace{6mm}
    \includegraphics[width=0.36\textwidth]{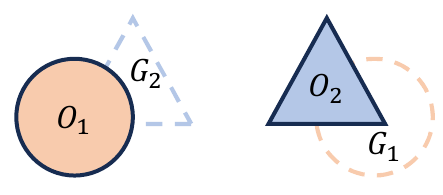}
    \caption{[Left] An example of MB scenario, where the robot (gray disc) travels along the green table following the black track along the table boundaries. 
    To pick/place an object on the table, the robot moves to the nearest position before the manipulation. 
    [Right] A working example where $o_1$ and $o_2$ block each other's goal pose. One must move to a buffer pose to finish the rearrangement.
    }
    \label{fig:workingExample}
\end{figure}

In the working example in \ref{fig:workingExample}[Right], $o_1$ and $o_2$ blocks each other from goal poses.
The path in the A* search tree from the initial state $(p_1^I,p_2^I)$ to the goal state $(p_1^G,p_2^G)$ may be:
\vspace{-1mm}
\begin{small}
$$
S_1:(p_1^I,p_2^I)\to S_2:(B_1,p_2^I)\to S_3:(B_1,p_2^G)\to S_4: (p_2^G,p_2^G)
$$
\end{small}

In this rearrangement plan, the robot moves $o_1$ to buffer and then moves $o_2$ and $o_1$ to goal poses, respectively.
In this example, the initial and goal states $S_1$ and $S_4$ are DS, and both intermediate states $S_2$ and $S_3$ are NDSs since $o_1$ is at a non-deterministic buffer $B_1$ in these states.

\subsection{Cost Estimation}\label{sec:cost}
In \astar, each state $s$ in the search space is evaluated by $f(s)=g(s)+h(s)$, where $g(s)$ represents the cost from the initial state to $s$ and $h(s)$ represents the estimated cost from $s$ to the goal state.
\orla defines $g(s)$ and $h(s)$ for DS and define $f(s)$ for NDS.
For a DS $s_D$, $g(s_D)$ is represented by the actual cost measured by $\mathcal J(\cdot)$ from the initial state to $s_D$, which can be computed as follows:
$$
g(s_D) = g(s_D') + \mathcal J(s_D',s_D)
$$
where $s_D'$ is the last DS in this path, and $\mathcal J(s_D',s_D)$ is the cost of the path between $s_D'$ and $s_D$.
When there are NDSs between $s_D'$ and $s_D$, we compute buffer poses minimizing the cost.
The details are discussed in \ref{sec:allocation}.
In $h(s_D)$ computation, we only count the transfer path and manipulation costs of one single pick-n-place for each object away from goal poses.
For example, in \ref{fig:workingExample}, $h(S_1)=dist(p_1^s,p_1^g)+dist(p_2^s,p_2^g)+C*|\{o_1, o_2\}|$.

For an NDS $s_{N}$, we have
\begin{align}
\begin{split}
f(s_{N}) &= g(s_D') + \mathcal J(s_D',s_{N}) + h(s_{N})
\end{split}
\end{align}
Since the $g(x)$ computation of DSs and NDSs only relies on $g(\cdot)$ of $S_D'$, rather than any NDS, we do not compute $g(s_{N})$ explicitly. 
Instead, we directly compute the lower bound of the actual cost as $f(s_{N})$.
The general idea of $f(s_{N})$ computation is presented in \ref{alg:fx}.
In Line 2, we add all manipulation costs in $\mathcal J(s_D',s_N)$ and $h(s_N)$, which are defined in the same spirit as those of $s_D$.
In Lines 3-4, we add the traveling cost along deterministic poses between $s_D'$ and $s_N$, 
which is a lower bound of the actual traveling cost as it assumes buffer poses do not induce additional costs.
In Lines 5-8, we add traveling cost to $h(s_N)$.
If the object is at a buffer, we add traveling distance based on \ref{alg:refine}, which we will mention more details later.
If the object is at a deterministic pose, we add the traveling cost of the transfer path between the current pose and its goal pose.
In our implementation, we store buffer information in each NDS to avoid repeated computations in \ref{alg:fx}.

\begin{algorithm}
\begin{small}
    \SetKwInOut{Input}{Input}
    \SetKwInOut{Output}{Output}
    \SetKwComment{Comment}{\% }{}
    \caption{ $f(s_{N})$ Computation}
		\label{alg:fx}
    \SetAlgoLined
		\vspace{0.5mm}
    \Input{$s_N$: an NDS state; $\mathcal A_G$: goal arrangement
    }
    \Output{$c$: $f(s_N)$}
		\vspace{0.5mm}
		$c\leftarrow g(s_D')$\\
        $c\leftarrow$ Add manipulation costs.\\
        $P\leftarrow$ All deterministic waypoints from $s_D'$ to $s_N$.\\
        $c\leftarrow c+dist(P)$\\
        \For{$o_i$ away from goal in $s_N$}
        {
        \If{$o_i$ in buffer}
        {$c\leftarrow c+$distanceRefinement($s_N$, $o_i$, $\mathcal A_G$)}
        \lElse
        {
        $c\leftarrow c+dist(s_N[o_i],\mathcal A_G[o_i])$
        }
        }
		\Return $c$\\
\end{small}
\end{algorithm}

\ref{alg:refine} computes the traveling cost in $f(s_N)$ related to a buffer pose $p_b$ in addition to the straight line path between its neighboring deterministic waypoints (\ref{alg:fx} Line 3).
For an object $o_i$ at a buffer, the related traveling cost involves three deterministic points$\{p_s, p_n, p_g\}$: $p_s$ and $p_n$ are the deterministic poses right before and after visiting the buffer.
$p_g$ is the goal pose of $o_i$.
$dist(p_b,p_s)+dist(p_b,p_n)$ is in $g(s_N)$ and $dist(p_b,p_g)$ is in $h(s_N)$.
In the EE scenario, if $\{p_s, p_n, p_g\}$ forms a triangle in x-y space, the distance sum is minimized when $p_b$ is the Fermat point of $\Delta p_sp_np_g$. If $\{p_s, p_n, p_g\}$ forms a line instead, the optimal $p_b$ is at the pose in the middle.
In the MB scenario, we minimize the total base travel. Denote the base positions of $\{p_s, p_n, p_g, p_b\}$ as $\{b_s, b_n, b_g, b_b\}$.
When $b_b$ is not at the three points or their opposite points, there are two points of $\{b_s, b_n, b_g\}$ on one half of the track and another on the other half of the track. 
In the example of \ref{fig:workingExample}[Left], if $b_b$ is at the current position, then $b_s$ and $b_n$ are on the left part of the track, and $b_g$ is on the other side.
Moving toward the two-point direction by $d$, $b_b$ can always reduce the total traveling cost by $d$.
Therefore, the extreme points of total distance are $\{b_s, b_n, b_g\}$ and their opposites on the track, and $\{b_s, b_n, b_g\}$ are the minima.
As a result, in MB scenario, the optimal $p_b$ minimizing the total distance to $\{p_s, p_n, p_g\}$ can be chosen among them.

\begin{algorithm}
\begin{small}
    \SetKwInOut{Input}{Input}
    \SetKwInOut{Output}{Output}
    \SetKwComment{Comment}{\% }{}
    \caption{ Distance Refinement}
		\label{alg:refine}
    \SetAlgoLined
		\vspace{0.5mm}
    \Input{$s_N$: an NDS state; $o_i$: object at buffer; $\mathcal A_G$: goal arrangement
    }
    \Output{$c$: additional cost}
		\vspace{0.5mm}
		$p_s\leftarrow$ the pose of $o_i$ before moved to buffer.\\
        $p_n\leftarrow$ the pose that the robot visit after placing $o_i$ to buffer.\\
        $p_g\leftarrow \mathcal A_G[o_i]$.\\ 
        \lIf{EE scenario}
        {
        $p_b\leftarrow$ Fermat point of $\Delta p_sp_np_g$
        }
        \ElseIf{MB scenario}
        {
        $p_b\leftarrow$ The one in $\{p_s,p_n,p_g\}$ with the shortest total distance to the other two.
        }
        \vspace{1mm}
		\Return $\sum_{p\in \{p_s,p_n,p_g\}}(dist(p_b,p))-dist(p_s,p_n)$\\
\end{small}
\end{algorithm}

\subsection{Buffer Allocation}\label{sec:allocation}
We now discuss \orla's buffer allocation process.
Buffer poses are allocated when a new DS is reached, and some object moves to a buffer pose since the last DS node.

\subsubsection{Feasibility of a Buffer Pose}
A buffer pose $p_b$ of $o_i$ is feasible if $o_i$ can be stably placed at $p_b$ without collapsing and $o_i$ should not block other object actions during $o_i$'s stay.
To predict the stability of placement of a general-shaped object, we propose a learning model \model based on Resnest\cite{zhang2022resnest}.
\model consumes two $200*200$ depth images of the surrounding workplace from the top and the placed object from the bottom, respectively.
It outputs the possibility of a successful placement.
The data and labels are generated by PyBullet simulation.
The depth images are synthesized in the planning process based on object poses and the stored point clouds to save time.
In addition to buffer pose stability, we must avoid blocking actions during the object's stay at the buffer pose.
For an action moving an object $o_j$ from $p_j^1$ to $p_j^2$. $o_i$ at the buffer pose needs to avoid $o_j$ at both $p_j^1$ and $p_j^2$.
We add these constraints when we allocate buffer poses.

\begin{figure}[h]
% \vspace{-3mm}
    \centering
    \includegraphics[width=0.8\columnwidth]{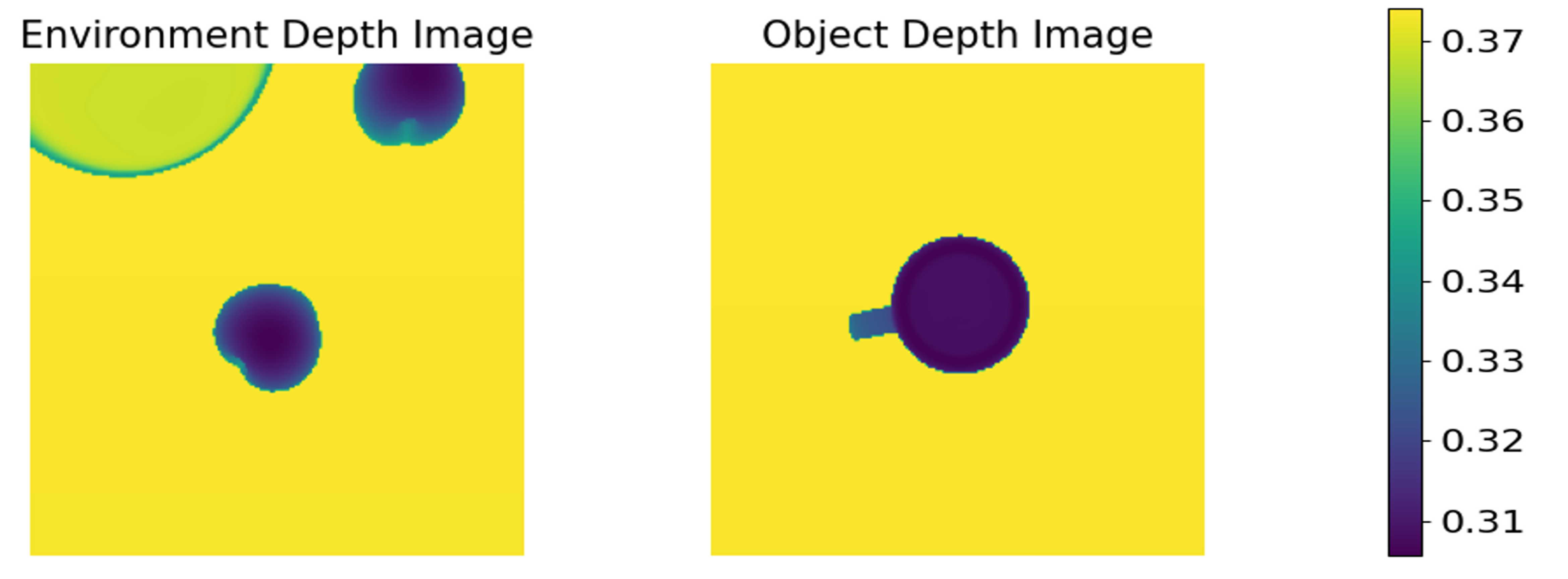}
\vspace{-3mm}
    \caption{An example input of \model when attempting to place a cup right on top of an apple in the environment. The ground truth label given by the simulation is a failure.}
    \label{fig:modelExample}
\end{figure}

\subsubsection{Optimality of a Buffer Pose}
Given the task plan, a set of buffer poses is optimal if the traveling cost is minimized.
For each object needing a buffer pose, we compute $P_b^*$, the set of buffer poses minimizing the traveling cost.
In the traveling cost, the trajectories from the buffer pose $p_b$ to four deterministic poses are involved:
the last deterministic pose the robot visits before placing and picking at $p_b$, 
and the first pose the robot visits after placing and picking at $p_b$.

In the EE scenario, if the four points form a quadrilateral, the total distance is minimized when $p_b$ is placed at the intersection of the diagonal lines. Otherwise, if two or more points overlap, the optimal $p_b$ is the point among four points minimizing the traveling cost.
Therefore, $P_b^*$ in EE scenario is a set of poses with above-computed $x,y$.

In the MB scenario, let $b_b$ be the mobile base position when visiting $p_b$.
And denote the mobile base positions of the four involved poses as $P=\{b_1,b_2,b_3,b_4\}$.
The four points and their opposites $\widehat{P}=\{\widehat{b_1},\widehat{b_2},\widehat{b_3},\widehat{b_4}\}$ partition the track into up to eight segments.
Similar to the case in \ref{alg:refine},
$\widehat{P}\bigcup P$ are extreme points of the distance sum.
% The total distance from $p_b$ to $P$ changes monotonely when $p_b$ moves in the same segment.
% Based on this observation, to compute the optimal $p_b$, we first compute the total distance from $\widehat{P}\bigcup P$ to $P$.
And the points with the minimum value are optimal $b_b$ solutions.
Moreover, for any of the eight segments, if both of its endpoints are with the minimum value, the whole segments are optimal $b_b$ solutions.
Therefore, in MB scenario, $P_b^*$ is a set of poses whose corresponding mobile base positions are at the points or segments computed above.

\subsubsection{Buffer Sampling}
\ref{alg:sampling} handles buffer sampling.
When multiple objects need buffers since the last DS, we sample buffers for them one after another (Line 1-3).
The sampling order is sorted by the time the object is placed in the buffer.
In Lines 4-5, we collect information for pose feasibility checks.
$E$ is used to predict the stability of a buffer pose.
$A$ contains object footprints that $o_i$ must avoid during the buffer stay.
We first sample buffer poses in $P_b^*$ (Line 8-12), and gradually expand the sampling region when we fail to find a feasible buffer (Line 14).
If we cannot find a feasible buffer when the sampling region covers the tabletop area, we return the latest $P$.
In the case of a failure, we remove $S_D$ from \astar search tree and create another new DS node $S_D''$ at the failing step.
$S_D''$ represents the state before the failing $o_i$ is moved to buffer.
At $S_D''$, all objects at buffers have found feasible buffer poses in the returned $P$.
Therefore, all objects in $S_D''$ are at deterministic poses.
Note that in our implementation, $P_b^*$ in MB scenario is represented by a segment of the mobile base track.
As a result, for Line 9 in MB scenario, we first sample a mobile base position on the track and then sample a pose based on that.

\begin{algorithm}[h]
\begin{small}
    \SetKwInOut{Input}{Input}
    \SetKwInOut{Output}{Output}
    \SetKwComment{Comment}{\% }{}
    \caption{Buffer Sampling}
		\label{alg:sampling}
    \SetAlgoLined
		\vspace{0.5mm}
    \Input{$s_D$: an DS state,\\ 
    }
    \Output{$P$: buffer poses}
		\vspace{0.5mm}
        $B \gets$ Objects go to buffers since the last DS.\\
        $P \gets \{o_i:\emptyset \ \forall o_i \in B\}$\\
        \For{$o_i \in B$}
        {
        $E\leftarrow$ The environment when the buffer is placed.\\
        $A\leftarrow$ A list of poses to avoid based on the task plan.\\
        $P_b^*\leftarrow$ The placing region minimizing traveling cost.\\ 
        \While{$o_i$'s Buffer Not Found}
        {
        \For{$i \gets 0$ \textbf{to} $k-1$}
        {$p_b\leftarrow$ Sample a pose in $P_b^*$.\\
        \If{poseFeasible($p_b,o_i,E,A$)}
        {
        Add $p_b$ to $P$;\\
        \textbf{go to} Line 3;
        }
        }
        \lIf{$P_b^*$ is the whole tabletop region}
        {\Return $P$}
        \lElse{
        $P_b^* \gets$ expand($P_b^*$)}
        }
        }
        \Return $P$
\end{small}
\end{algorithm}

\subsection{Optimality of \orla}
Our \orla consists of two main components: high-level \astar search and low-level buffer allocation process.
In terms of the high-level lazy \astar,
the $h(s)$ of DS is consistent, and $g(s)$ of DS does not rely on $g(s)$ of NDSs.
Therefore, by only considering DSs, lazy \astar is a standard \astar with global optimality.
Additionally, $f(s)$ values for NDSs underestimate the actual cost.
Therefore, high-level lazy \astar is globally optimal in its domain.
Note that the formulated buffer allocation problem is an established \emph{Constraints Optimization Problem}, which can be solved optimally if we assume \model makes correct stability predictions.
By replacing the buffer sampling with an optimal solver, \orla is globally optimal.

\section{Evaluation}\label{sec:experiments}
We present simulation evaluations on \orla and compare them with state-of-the-art rearrangement planners implemented in Python.
For each experiment, costs and computation time are averages of test cases finished by all methods.
Experiments are executed on an Intel$^\circledR$ Xeon$^\circledR$ CPU at 3.00GHz.
To measure instance difficulty, we define the density level of workspace $\rho$ as $\dfrac{\sum_{o\in \mathcal O} S(o)}{S_T}$, where $S(o)$ is the footprint size of the object $o$ and $S_T$ is the size of the tabletop region.
We set $C=10$ in cost function $\mathcal J(\Pi)$.

We first compare \orla-Full, \orla-Action, Greedy-Sampling
, MCTS\cite{labbe2020monte}, and TRLB from \ref{chap:trlb} 
in disc instances (\ref{fig:discInstance}) without \model.
\orla-Full is our \orla minimizing $\mathcal J(\Pi)$.
\orla-Action only considers manipulation costs, i.e. minimizes the number of pick-n-places.
Greedy-Sampling maintains an A$^*$ search tree with $\mathcal J(\Pi)$. Instead of lazy buffer allocation, it samples buffer poses immediately as close to goal poses as possible when objects are moved to buffers.
MCTS and TRLB compute rearrangement plans using Monte-Carlo tree search with random buffer sampling and bi-directional tree search with lazy buffer allocation respectively.
These two planners only consider buffer poses in the free space.
% In these experiments, we show the performance difference between whether \orla considers the traveling cost or not and that between lazy buffer allocation and immediate buffer sampling.
In disc instances, we assume all poses are stable to be placed on but the robot cannot manipulate an object if another object is on top.
The table sizes of EE and MB instances are $1m\times 1m$ and $3m\times 1m$, respectively.
In disc instances with different numbers of objects, we adjust the disc radius to keep $\rho$ constant.
We also test \orla-Full and \orla-Action in general-shaped object instances (\ref{fig:ycbInstance}) with \model for buffer pose sampling.

\begin{figure}[h]
    \centering
    \includegraphics[width=0.7\textwidth]{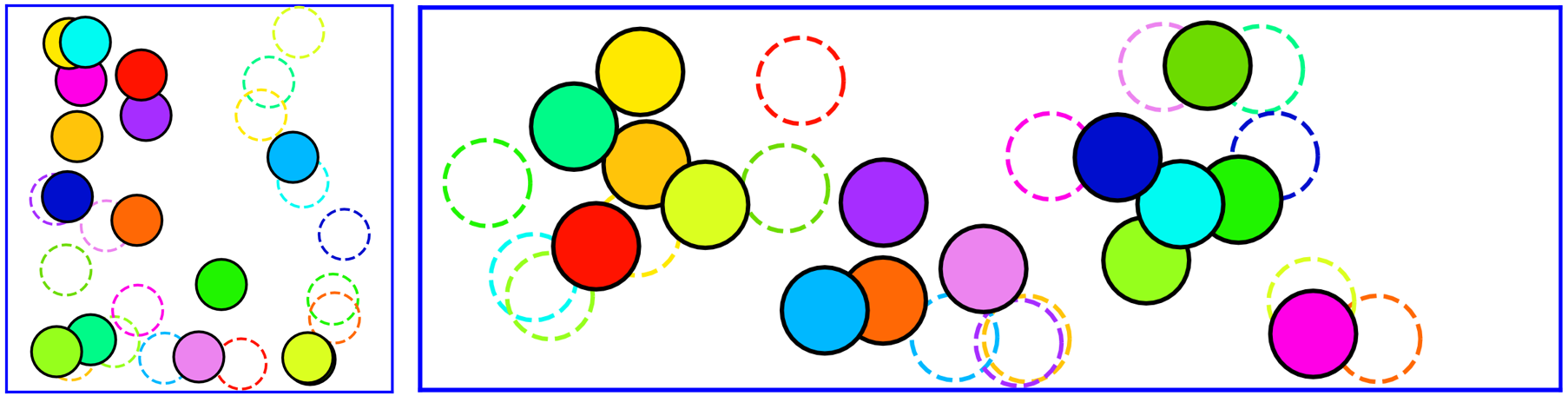}
    \includegraphics[width=0.18\textwidth]{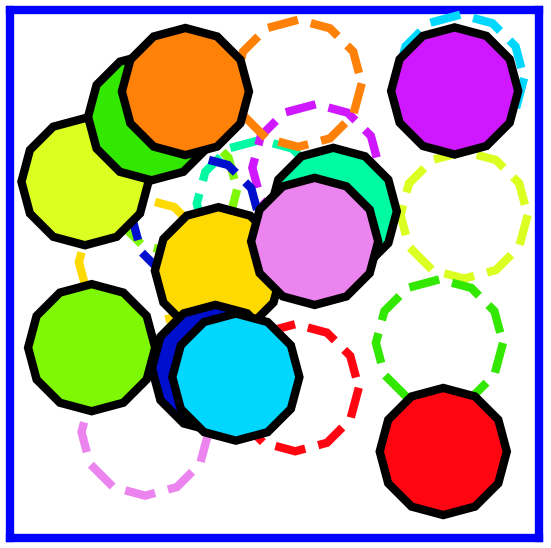}
    \caption{Examples of disc instances. [Left] EE scenario with $\rho=0.2$; [Middle] MB scenario with $\rho=0.2$; [Right] EE scenario with $\rho=0.5$. Colored and transparent discs represent the initial and goal arrangements respectively. }
    \label{fig:discInstance}
\end{figure}

\begin{figure}[h]
\vspace{-3mm}
    \centering
    \includegraphics[width=0.96\textwidth]{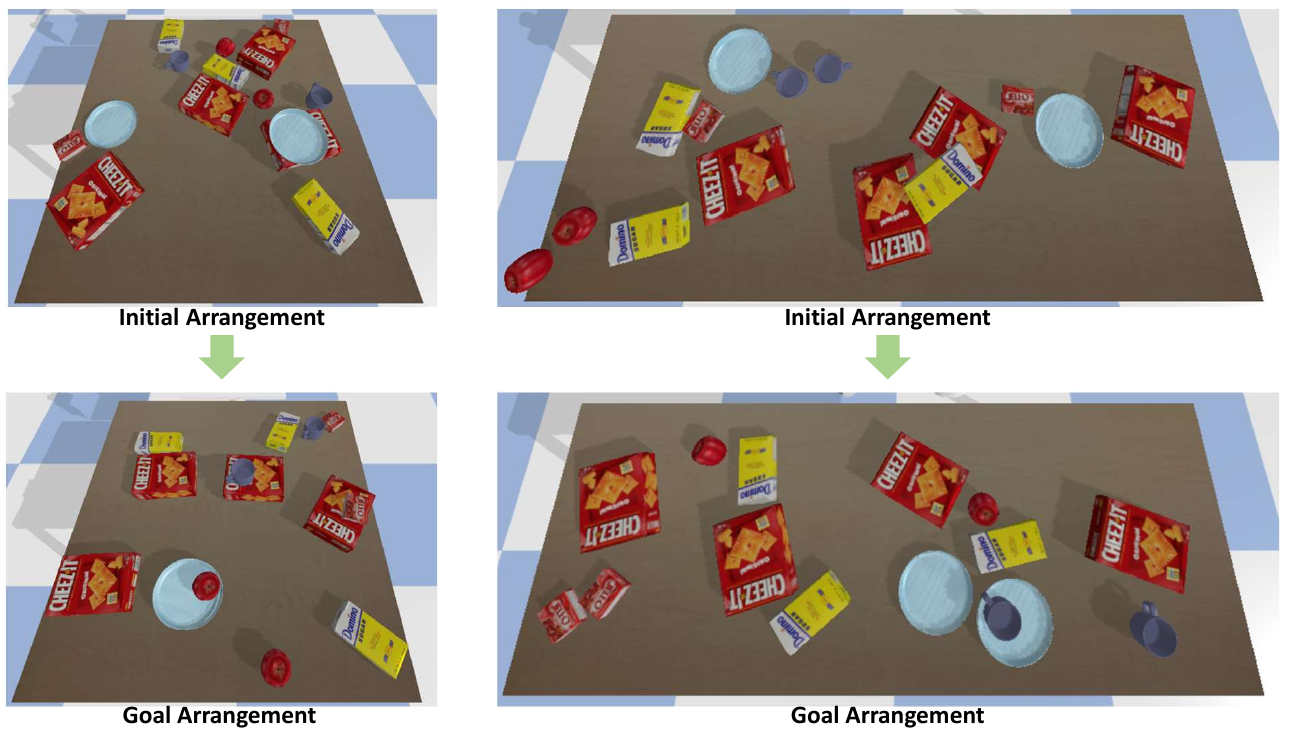}
    \vspace{-4mm}
    \caption{Examples of instances with general-shaped objects in [Left] EE and [Right] MB scenarios.}
    \label{fig:ycbInstance}
\end{figure}

\subsection{Disc Rearrangement in Simulation}
% We first examine the performance of \orla without the prediction model in disc instances (\ref{fig:discInstance}). 
\ref{fig:discs11} shows algorithm performance in EE scenario with $\rho=0.2$. 
Comparing \orla-Full and \orla-Action, \orla-Full saves around $15\%$ path length without an increase in the manipulation cost.
Without the lazy buffer strategy, Greedy-Sampling spends more time on planning but yields much worse plans in general, e.g., additional $23\%$ actions in 5-discs instances.
That is because \orla-Full allocates buffer poses avoiding future actions while immediately sampled buffer poses in Greedy-Sampling may block some of these actions. 
Due to the repeated buffer sampling, Greedy-Sampling can only solve $40\%$ instances in $7-$object instances.
However, Greedy-Sampling has reasonably good performance in the total path length despite the large number of actions, so allocating buffers near the goal pose is a good strategy for saving traveling costs.
We also note that the number of actions as multiplies of $|\mathcal O|$ reduces as $|\mathcal O|$ increases, which indicates a reducing difficulty in rearrangement.
That is because given a fixed density level, as $|\mathcal O|$ increases, the relative size of each object to the workspace is smaller, which makes it easier to find valid buffer poses.

\begin{figure}[ht]
    \centering
    \includegraphics[width=0.96\textwidth]{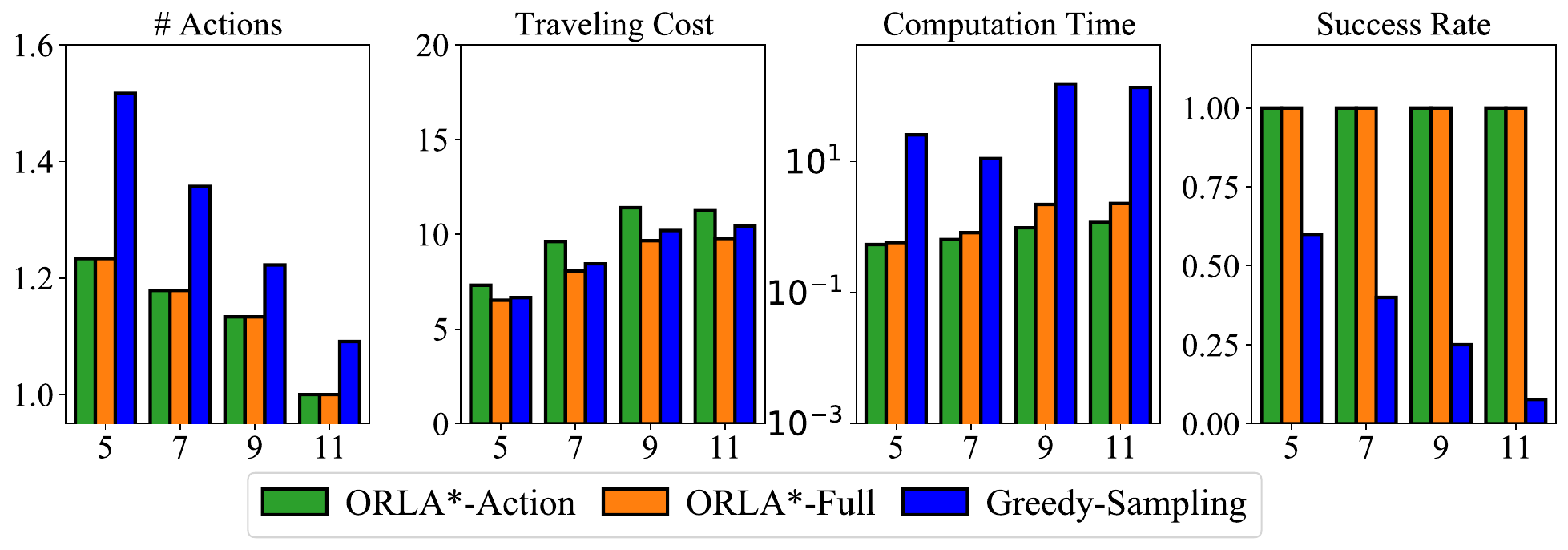}
    \caption{Algorithm performance in EE disc instances with $\rho=0.2$ and 5-11 objects. (a) $\#$ pick-n-places in solutions as multiplies of $|\mathcal O|$. (b) Traveling cost (m). (c) Computation time (secs). (d) Success rate. }
    \label{fig:discs11}
\end{figure}

We also compare \orla variants with TRLB and MCTS in dense disc instances ($\rho=0.5$) of EE scenario. 
An example of the dense instance is shown in \ref{fig:discInstance}[Right]. 
While MCTS fails in all test cases, the results of other methods are shown in \ref{fig:disc_11_dense}.
Comparing \orla-Action and TRLB, the results suggest that considering buffer poses on top of other objects not only effectively increases the success rate, but also provides shorter paths with lower traveling cost and manipulation cost in dense test cases.
Comparing \orla variants, \orla-Full saves $4.2\%-11.7\%$ traveling cost, but \orla-Action has a much higher success rate in 5-object instances, which is the hardest for buffer sampling as mentioned previously.

\begin{figure}[ht]
    \centering
    \includegraphics[width=\textwidth]{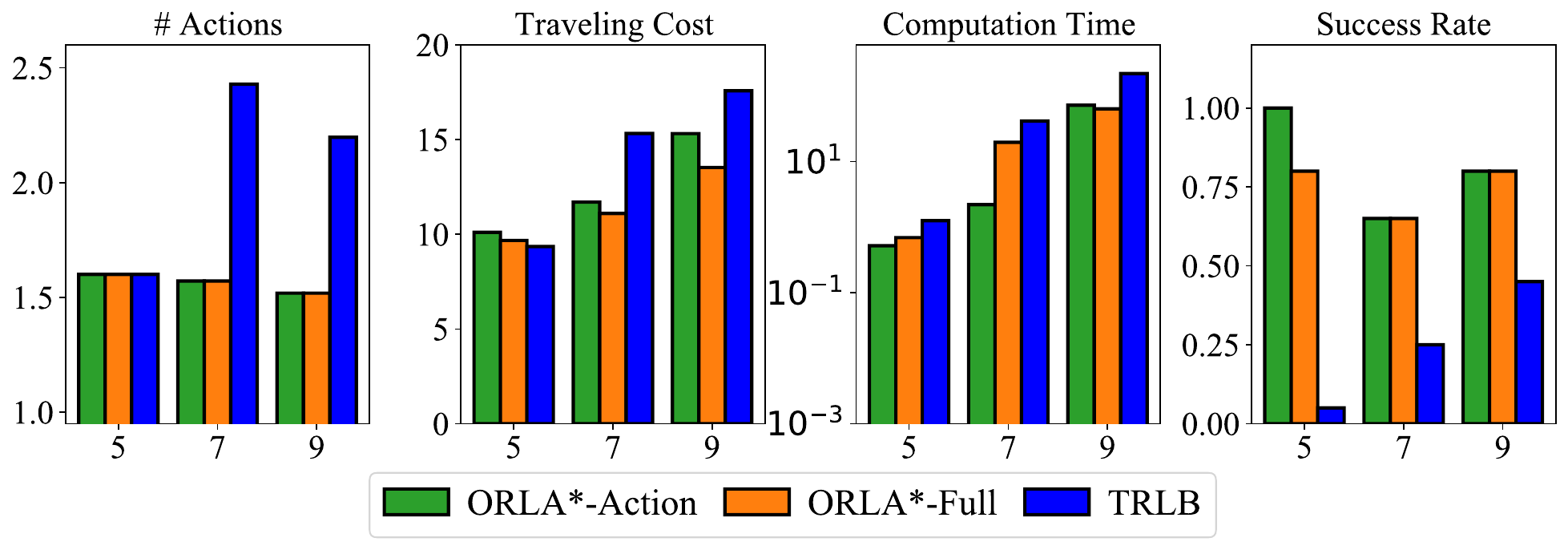}
    \caption{Algorithm performance in EE disc instances with $\rho=0.5$ and 5-11 objects. (a) $\#$ pick-n-places in solutions as multiplies of $|\mathcal O|$. (b) Traveling cost (m). (c) Computation time (secs). (d) Success rate. }
    \label{fig:disc_11_dense}
\end{figure}

\ref{fig:discs31} shows algorithm performance in MB scenario with $\rho=0.2$.
In this scenario, taking mobile base traveling costs into consideration significantly reduces traveling costs in solutions.
\orla-Full saves over $20\%$ traveling costs on average in 11-disc instances but spends more time in computation.

\begin{figure}[ht]
    \centering
    \includegraphics[width=0.96\textwidth]{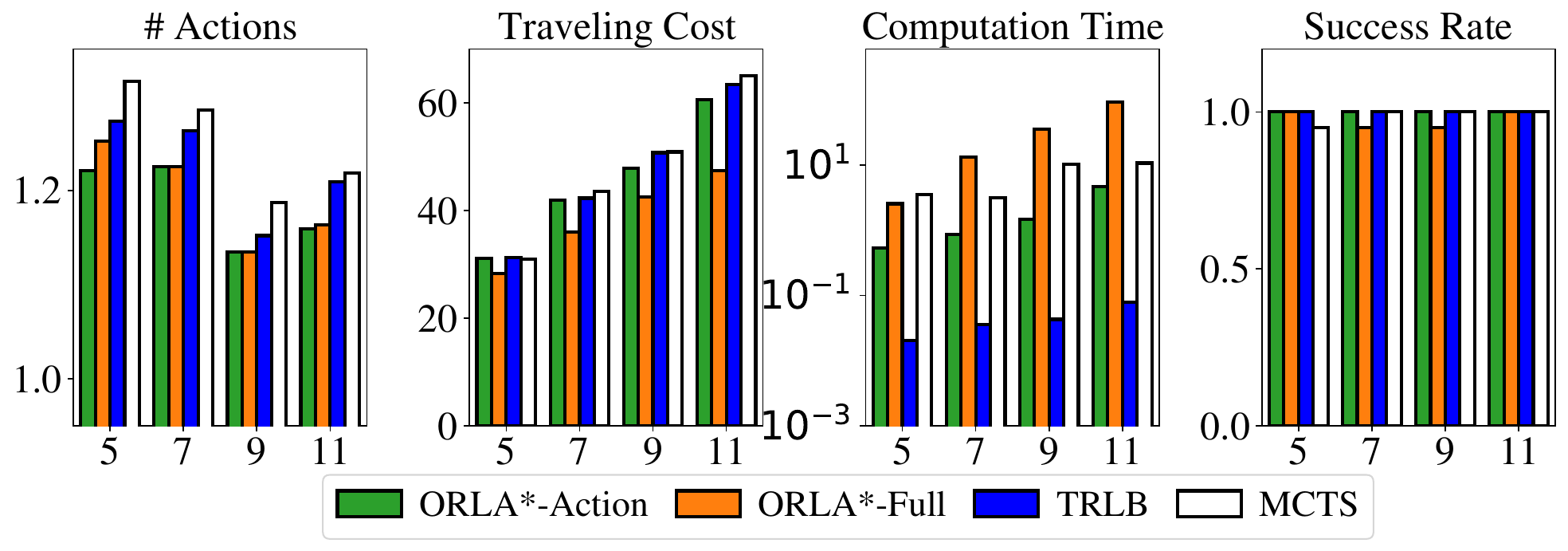}
    \caption{Algorithm performance in MB disc instances with $\rho=0.2$ and 5-11 objects. (a) $\#$ pick-n-places in solutions as multiplies of $|\mathcal O|$. (b) Traveling cost (m). (c) Computation time (secs). (d) Success rate. }
    \label{fig:discs31}
\end{figure}

To summarize, in disc rearrangement experiments, we have the following conclusions:
First, intelligent temporary object placement on top of other objects with \orla computes low-cost plans in dense test cases.
Second, considering traveling costs in the \orla framework effectively reduces traveling costs without an increase in manipulation costs but induces additional computation time.
Finally, the results suggest that lazy buffer allocation improves solution quality and saves computation time.

\subsection{Qualitative Analysis of \model}
Regarding the placement stability prediction model \model, we train it with four types of objects in Pybullet simulator: an apple, a pear, a plate, and a cup.
\ref{fig:prediction}(a) shows the prediction distribution of a scene with the plate and the apple when placing the cup.
We sample $8$ poses of the placed object with different orientations for each point in the distribution map. 
The distribution shows the average of the $8$ output probabilities.
When the cup pose is sampled far from both objects (the top-right corner and the bottom-left corner) or on the plate, \model supports placements with outputs around $0.999$ and $0.90$, respectively. 
When the cup pose is sampled at the rim of the plate and on top of the apple, \model also rejects placements with outputs around $0.02$ and $0.005$.
Due to the existence of the handle, \model is conservative when the cup is placed close to the plate rim. 
\vspace{-3mm}
\begin{figure}[ht]
    \centering
    \includegraphics[width=0.96\textwidth]{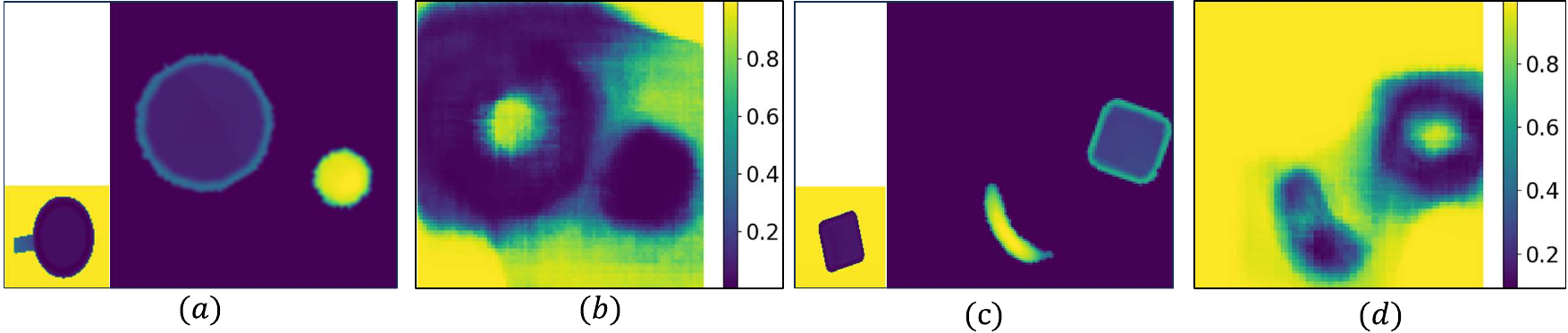}
    \vspace{-5mm}
    \caption{\textbf{(a)} A synthesized environment height map and the depth image of the placed object from the bottom. All objects are from the training set. \textbf{(b)} The corresponding stability prediction distribution of (a). \textbf{(c)} A synthesized environment height map and the depth image of the placed object from the bottom. All objects are outside the training set. \textbf{(d)} The corresponding stability prediction distribution of (c). }
    \label{fig:prediction}
\end{figure}

To show the model's generalization ability, we present the prediction distribution of a scene with untrained objects in \ref{fig:prediction}(c). In this environment, we place a square plate and a banana.
The plate's size and shape differ from those in the training set.
And not to say the banana.
The placed object is a small tea box while we do not have any cuboid object in the training set.
As shown in \ref{fig:prediction}(b), \model clearly judges the stability when placing the novel box into the workspace.

% \begin{figure}[ht]
%     \centering
%     \includegraphics[width=0.39\textwidth]{figures/orla_prediction_distribution_untrained.png}
%     \caption{[Left] Stability prediction distribution with objects outside the training set. [Right] A synthesized height map of the environment and the depth image of the placed object from the bottom.}
%     \label{fig:predictionUntrained}
% \end{figure}

\subsection{General-Shaped Object Rearrangement in Simulation}
In general-shaped instances, we equip \orla methods with \model.
% $\rho$ in these instances are kept at $0.25$ by adjusting the table region size. 
Given a set of objects in the workspace, we compute the total area of object footprints and adjust the size of the tabletop region to keep $\rho=0.25$.
\ref{fig:ycb11} presents algorithm performance in EE and MB scenarios, respectively.
In this scenario, we use \model to decide whether an object is safe to be temporarily placed on top of another.
Due to the inference time for \model in buffer sampling, the computation time of \orla methods is longer than that in disc experiments.
In $11-$object instances, \orla-Full saves $16.7\%$ and $13\%$ traveling cost than \orla-Action in EE and MB scenarios respectively.
However, \orla-Action maintains $100\%$ success rate while \orla-Full fails in $13.5\%$ MB test cases.

\begin{figure}[ht]
    \centering
    \includegraphics[width=0.96\textwidth]{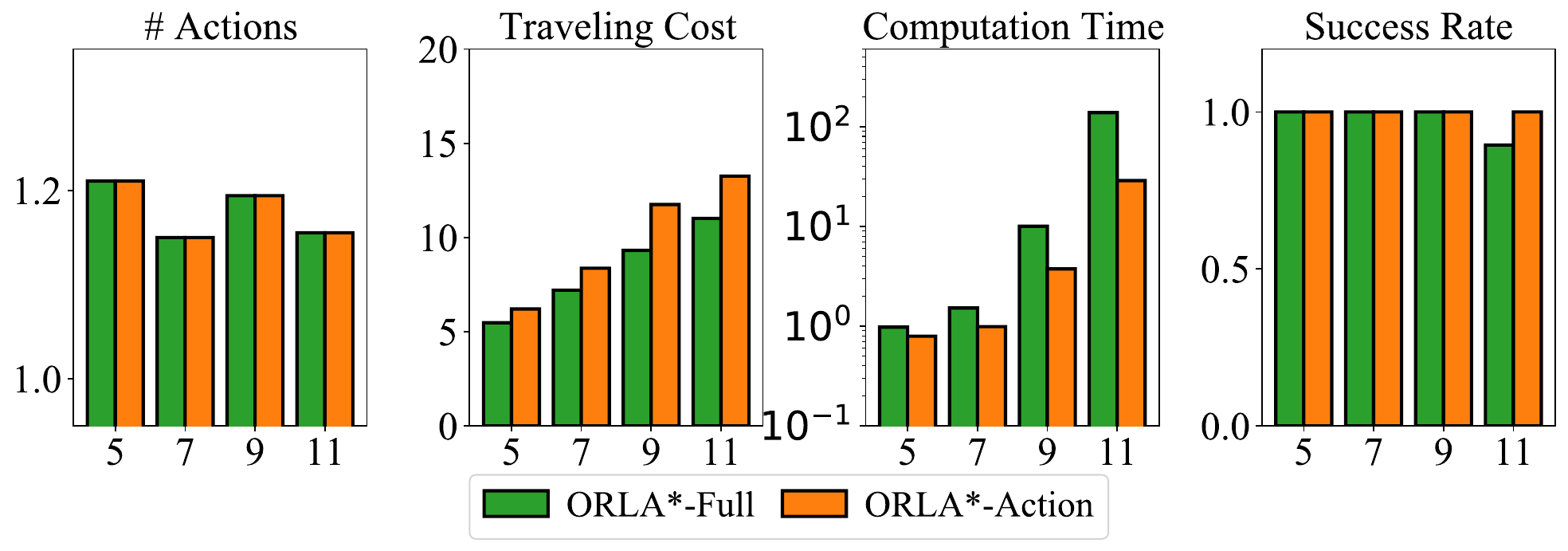}
    \includegraphics[width=0.96\textwidth]{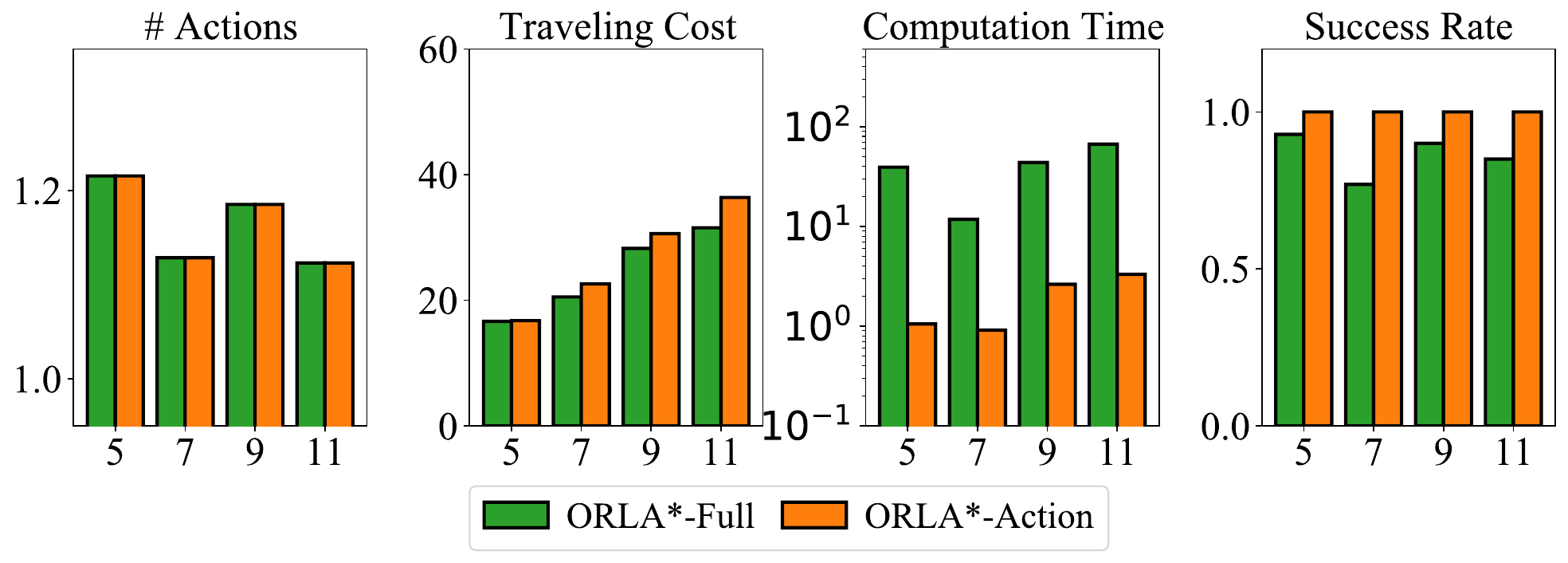}
\vspace{-5mm}
    \caption{Algorithm performance in [Top] EE scenario and [Bottom] MB scenario with general-shaped objects. $\rho=0.25$ and 5-11 objects. (a) $\#$ pick-n-places in solutions as multiplies of $|\mathcal O|$. (b) Traveling cost (m). (c) Computation time (secs). (d) Success rate. }
    \label{fig:ycb11}
\end{figure}

\section{Summary}
In this chapter, we discuss the cooperative dual-arm rearrangement (CDR) problem in a partially shared workspace.

For task planning, we employ a lazy strategy for buffer allocation in the dual-arm rearrangement system.
To coordinate manipulations of the arms, we develop dependency graph guided heuristic search algorithms computing optimal primitive task schedules under two makespan evaluation metrics.
The effectiveness of the proposed procedure is demonstrated with extensive simulation studies in the PyBullet environment with two UR-5e robot arms.
Specifically, plans computed by our algorithms are up to $35\%$ shorter than greedy ones in execution time in dense rearrangement instances.
Moreover, we observe a tradeoff between the number of needed handoffs and the occurrence of path conflicts in CDR problems as the overlap ratio of the workspace varies.

For motion planning, \modap outputs the planned trajectory while conforming to specified robot dynamics constraints, e.g., velocity, acceleration, and jerk limits. Compared to the previous state of the art, \modap achieves up to $40\%$ execution time improvement even though the planned trajectories are longer. In particular, \modap shines in cases where extensive joint dual-arm planning must be carried out to compact the overall task and motion plan.

%% file: chapters/conclusion.tex
\chapter{Conclusions}
\thispagestyle{myheadings}

This thesis summarizes my efforts in the problem of \toro with external and internal buffers.

For \toro with External buffers (\toroe), we study the maximum running buffer size (MRB), which is an important quantity to understand as it determines the problem's feasibility.
Despite the provably high computational complexity that is involved, we provide 
effective dynamic programming-based algorithms capable of quickly computing \mrb 
for large and dense labeled/unlabeled \toro instances.

Fast optimal algorithms for \toroe led to our development of \trlb in the internal buffer setup.
The \trlb framework employs the dependency graph representation and a lazy buffer allocation approach for efficiently solving the problem of rearranging many tightly packed objects on a tabletop.
Extensive simulation studies show that \trlb computes rearrangement plans of comparable or better quality as the state-of-the-art methods, and does so up to $2$ magnitudes faster. 

We extend the idea of lazy buffer verification to handle object heterogeneity (\hete), dual-arm coordination (CDR), and time-optimal mobile manipulation (\motar).
Specifically, in \hete, we consider object property difference in \trlb with weighting heuristics;
in CDR and \motar, we integrate \trlb into \astar framework for  \toroi under different objective functions.

%% file: chapters/appendix.tex
\begin{theappendices}
    \chapter{Structural Properties and Proofs for \toroe}\label{sec:app-toroe}
    \section*{Properties of Unlabeled Dependency Graphs for Special \toro Settings}
\begin{proof}[Proof of \ref{p:udg-polygon}]
 The maximum degree of the dependency graph is equal to the maximum number of disjoint objects that one object can overlap with. 
In this proof, we evaluate the upper bound of the maximum overlaps with a disc packing problem.
We first discuss the range of distance between two overlapping regular polygons when they are overlapping and disjoint. After that, we relate the problem to a disc packing problem and derive the upper bound. Without loss of generality, 
we assume the radius of the regular polygons to be 1.
When two $d$-sided regular polygons overlap, the distance between the polygon centers is upper bounded by two times the radius of their circumscribed circles, i.e. 2.
When two $d$-sided regular polygons are disjoint, the distance between the polygon centers is lower bounded by two times the radius of their inscribed circles, i.e.  $2\cos(\dfrac{\pi}{d})$.
 Due to the mentioned upper bound, the number of disjoint d-sided regular polygons that overlap with one d-sided regular polygon is upper bounded by that whose centers are inside a 2-circle (\ref{fig:prop_2}(a)). Additionally, due to the mentioned distance lower bound, this number is upper bounded by the number of disjoint $\cos(\dfrac{\pi}{d})$-circles whose centers are inside a 2-circle (\ref{fig:prop_2}(b)), which is equal to the number of $\cos(\dfrac{\pi}{d})$-circles packed inside a $(2+\cos(\dfrac{\pi}{d}))$-circle (\ref{fig:prop_2}(c)).
Since the radius ratio of circles $\dfrac{2+\cos(\dfrac{\pi}{d})}{\cos(\dfrac{\pi}{d})}$ is upper bounded by 5 when $d\geq 3$, there are at most 19 small circles packed in the large circle \cite{fodor1999densest}.
Therefore, the maximum degree of the unlabeled dependency graph is upper bounded by 19.

\end{proof}

\begin{proof}[Proof of \ref{p:udg}]
Constructing the dependency graph based on the start and goal positions in the workspace, the planarity can be proven if no two edges cross each other.
Assuming the contrary, there are two edges crossing each other, e.g. $AD$ and $BC$ in \ref{fig:planarity}[Left].
Since each edge comes from a pair of intersecting start and goal objects, without loss of generality,
we can assume $C$, $D$ are start positions and $A$, $B$ are goal positions.
Therefore, we have $|AB|,|CD|\geq 2r$,
where $r$ is the base radius of the cylinders.
Since $AD$ and $BC$ are both connected in the dependency graph, 
we have $|AD|,|BC| < 2r$, which leads to the conclusion that $|AD|+|BC|< 4r \leq |AB|+|CD|$, and contradicts the triangle inequality.
Therefore, the planarity is proven with contradiction.

Since uniform cylinders have uniform disc bases, one cylinder may only 
touch six non-overlapping cylinders and non-trivially intersect at most five 
non-overlapping cylinders. 
Therefore, the maximum degree of the unlabeled dependency graph is upper bounded by 5.

\end{proof}

\begin{figure}[ht!]
    \centering
    \includegraphics[width=0.49\textwidth]{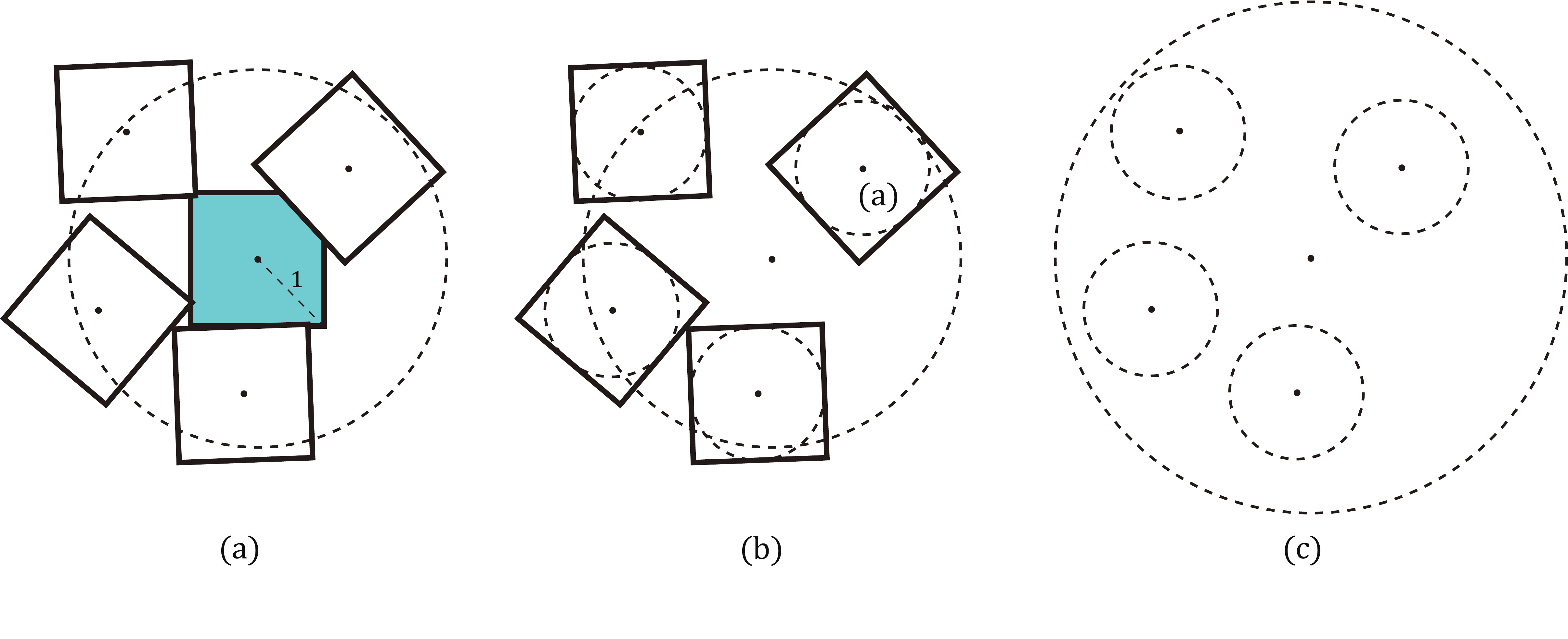}
    \caption{(a) The number of disjoint d-sided regular polygons that overlap with one d-sided regular polygon is upper bounded by that whose centers are inside a 2-circle. (b) The number of disjoint d-sided regular polygons whose centers are inside a 2-circle is upper bounded by the number of disjoint $\cos(\dfrac{\pi}{d})$-circles whose centers are inside a 2-circle. (c) The number of disjoint $\cos(\dfrac{\pi}{d})$-circles whose centers are inside a 2-circle is equal to the number of $\cos(\dfrac{\pi}{d})$-circles packed inside a $(2+\cos(\dfrac{\pi}{d}))$-circle.}
    \label{fig:prop_2}
\end{figure}

\begin{figure}[ht!]
    \centering
    \includegraphics[width=0.4\textwidth]{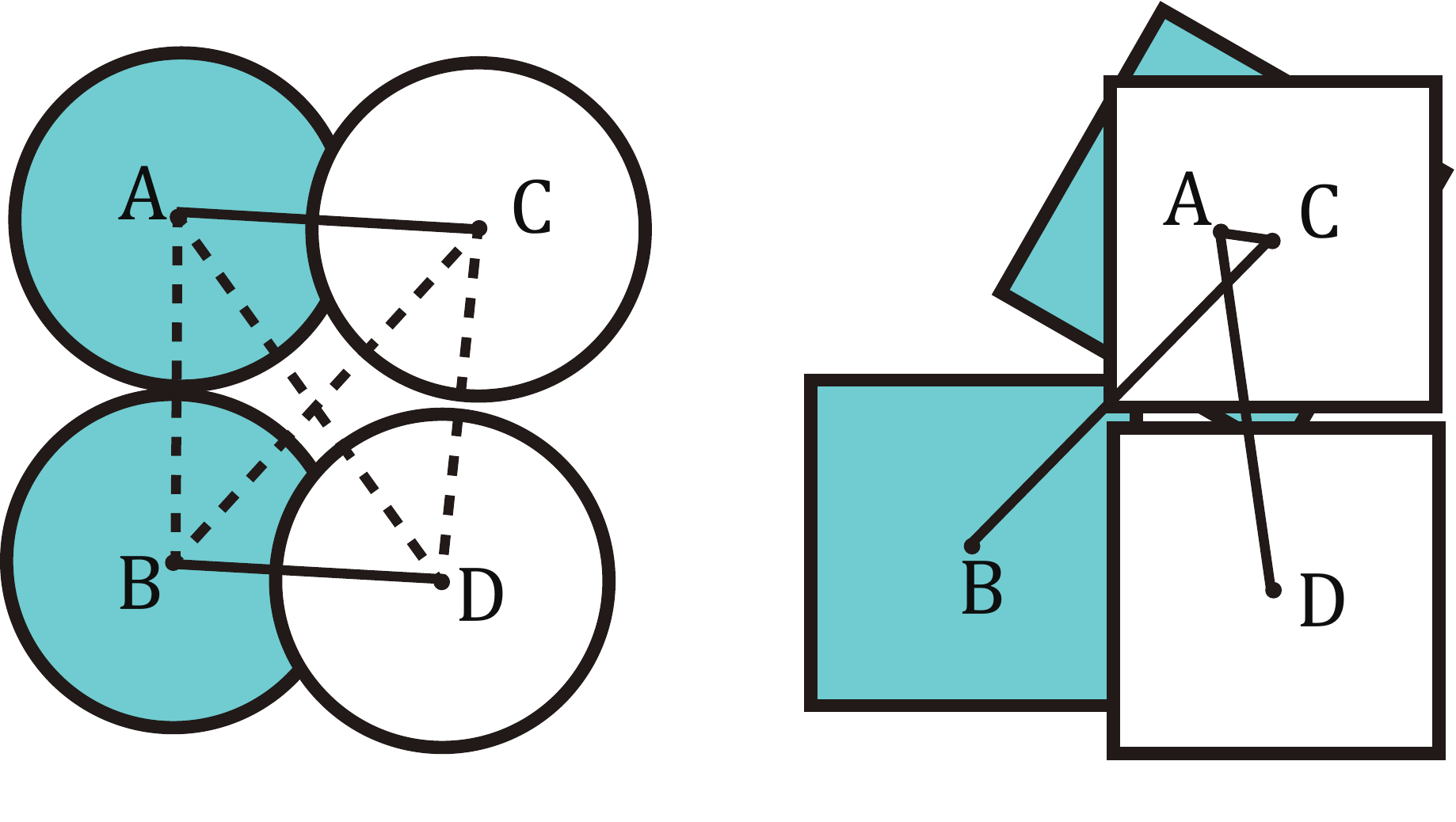}
    \caption{[Left] A case discussed in Prop.\ref{p:udg} [Right] A counterexample of the proof of Prop.\ref{p:udg} when the base of workspace objects is square.}
    \label{fig:planarity}
\end{figure}

\begin{remark}
As shown in \ref{fig:planarity}[Right],
the proof of the planarity in \ref{p:udg} does not hold for objects with regular polygon bases.
\end{remark}

\section*{$\Omega(\sqrt{n})$ Lower Bound for \urbm}
Let $(x, y)$ 
be the coordinate of a vertex $v_{x,y}$ on $\D(w, h)$ with the top left being $(1,
1)$. The parity of $x + y$ determines the partite set of the vertex (recall that 
unlabeled dependency graph for uniform cylinders is always a planar bipartite 
graph, by \ref{p:udg}), which may correspond to a start pose or 
a goal pose. With this in mind, we simply call vertices of $\D(w, h)$ start and 
goal vertices; without loss of generality, let $v_{1,1}$ be a start vertex. 
%
%Assuming the cylinders' footprints have unit radius, the distance between adjacent 
%vertices on $\D(w, h)$ is $\sqrt{2}$.
%

%For the unlabeled setting, we show that the \mrb of the worst cases is $\Omega(\sqrt{n})$ where $n$ is the number of objects in the instance. This property is proved with a instance whose dependency graph is a \emph{dependency grid} whose definition is as follows.

%\begin{definition}[Dependency Grid]
%A $W\times H$ dependency grid $\mathcal{D}$ is a square grid graph whose vertices correspond to points in the plane with integer coordinates, $x$-coordinates being in the range 1, ..., $W$, $y$-coordinates being in the range 1, ..., $H$ ( \ref{fig:DependencyGrid}). Each vertex of $\mathcal{D}$ is colored blue or black but the neighboring vertices cannot be with the same color. Denote the vertex at the $i$th column and $j$th row in $\mathcal{D}$ as $v_{i,j}$. Without loss of generality, we can assume that $v_{1,1}$ is black.
%\end{definition}

%When the poses in $\mathcal{A}_1$ and $\mathcal{A}_2$ are placed at the coordinates of black and blue vertices in $\mathcal{D}$ respectively and the unit length of the grid is $\sqrt{2}$ units, $\mathcal{D}$ is the dependency graph of the rearrangement instance.

We use $\D(m, 2m)$ for establishing the lower bound on \mrb. We use a \emph{vertex 
pair} $p_{i, j}$ to refer to two adjacent vertices $v_{i, 2j-1}$ and $v_{i, 2j}$ in 
$\D(m, 2m)$. It is clear that a vertex pair contains a start and a goal vertex. 
We say that a goal vertex is \emph{filled} if an object is placed at the corresponding goal pose. 
We say that a start vertex (which belongs to a vertex pair) is \emph{cleared} if the 
corresponding object at the vertex is picked (either put at a goal or at a buffer)
but the corresponding goal in the vertex pair is unfilled. 
At any moment when the robot is not holding an object, the number of objects in the 
buffer is the same as the number of cleared vertices. For each column $i, 1\leq i 
\leq m$, let $f_i$ (resp., $c_i$) be the number of goals (resp., start) vertices in 
the column that are filled (resp., cleared). Notice that a goal cannot be filled 
until the object at the corresponding start vertex is removed.

%In the rest of the subsection, we prove that for a rearrangement instance whose dependency graph is a $\sqrt{n}\times (2\sqrt{n})$ dependency grid $\mathcal{D}$, the minimum \mrb is $\Omega (\sqrt{n})$. The general idea is to evaluate the minimum needed buffer space when $\lfloor n/3\rfloor$ objects are at the goal poses.
%
%Without loss of generality, we can assume that $n$ is a square number. Note that in each column, there are $m$ start vertices and goal vertices respectively. A \emph{vertex pair} $P(i,j)$ is defined to be a pair of adjacent vertices in the same column ($v_{i,2j-1}$, $v_{i,2j}) (1\leq i,j \leq m)$. One of them is a start vertex and the other one is a goal vertex. A goal vertex is \emph{filled\ up} if there is an object at the corresponding pose. A \emph{cleared vertex} is a start vertex whose corresponding object has been moved away but the goal vertex in the same vertex pair is not filled up. Therefore, when $\lfloor n/3\rfloor$ goal vertices are filled up, the number of objects in the buffer is the same as the number of cleared vertices at the moment. For each column $i(1\leq i \leq m)$, denote the number of filled-up goal vertices and the number of cleared vertices as $g_i$ and $p_i$ respectively.

\begin{lemma}\label{lemma:adjacent}
On a dependency grid $\D(m, 2m)$, for two adjacent columns $i$ and $i+1$, $1 \le i < m$, if
$f_i+f_{i+1}\neq 0$ or $2m$, then $c_i+c_{i+1}\geq 1$. In other words, there is at least one 
cleared vertex in the two adjacent columns unless $f_i=f_{i+1}=0$ or $f_i=f_{i+1}=m$.
\end{lemma}
%\begin{proof}See \ref{sec:proofs}.\end{proof}
\begin{proof}
If there is a $j, 1\leq j \leq m$, such that only one of the goal vertices in vertex pairs 
$p_{i,j}$ and $p_{i+1, j}$ is filled ( \ref{fig:LemmaProof}(a)), then the start 
vertex in the other vertex pair must be cleared. Therefore, $c_i+c_{i+1}\geq 1$.

On the other hand, if, for each $j, 1\leq j \leq m$, both or neither of the goal vertices 
in $p_{i,j}$ and $p_{i+1,j}$ is filled, then there is a $j, 1 \leq j \leq m-1$, such that both 
goal vertices in $p_{i,j}$ and $p_{i+1,j}$ are filled but neither of those in $p_{i,j+1}$ and $p_{i+1,j+1}$ is filled (\ref{fig:LemmaProof}(b)) or the opposite 
(\ref{fig:LemmaProof}(c)). Then, for the vertex pairs whose goal vertices are not filled, 
say $p_{i,j+1}$ and $p_{i+1, j+1}$, one of their start vertices is a neighbor of the filled 
goal in $p_{i,j}$ and $p_{i,j+1}$. Therefore, at least one of the start vertices in $p_{i, j+1}$ 
and $p_{i+1,j+1}$ is a cleared vertex. And thus, $c_i+c_{i+1}\geq 1$.

\begin{figure}[h!]
    \centering
    \vspace{2mm}
    \begin{overpic}
    [width = 0.47\textwidth]{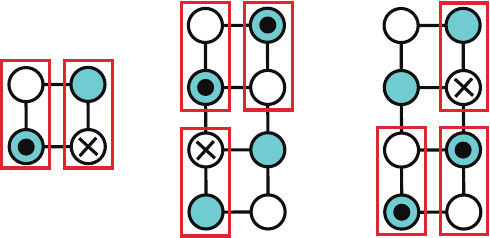}
    \put(0, 38){{\small $p_{i,j}$}}
    \put(11, 38){{\small $p_{i+1,j}$}}
    
    \put(37, 51){{\small $p_{i,j}$}}
    \put(49, 51){{\small $p_{i+1,j}$}}
    \put(34, -3){{\small $p_{i,j+1}$}}
    
    \put(89, 51){{\small $p_{i+1,j}$}}
    \put(70, -3){{\small $p_{i,j+1}$}}
    \put(89, -3){{\small $p_{i+1,j+1}$}}
    
    \put(8, -8.5){{\small $(a)$}}
    \put(44, -8.5){{\small $(b)$}}
    \put(85, -8.5){{\small $(c)$}}
    \end{overpic}
    \vspace{6mm}
    \caption{Some cases discussed in the \ref{lemma:adjacent}. The unshaded and shaded nodes represent the start and goal vertices in the dependency graph. Specifically, the shaded nodes with a dot inside represent the filled vertices and the unshaded nodes with a cross inside represent the cleared vertices. (a) When only one goal vertex in $p_{i,j}$ and $p_{i+1, j}$ is filled up, the start vertex in the other vertex pair is a cleared vertex. (b) When both goal vertices in $p_{i,j}$ and $p_{i+1,j}$ are filled but neither of those in $p_{i,j+1}$ and $p_{i+1,j+1}$ is filled, one of the start vertices $p_{i,j+1}$ and $p_{i+1,j+1}$ is a cleared vertex. (c) The opposite case of (b).}
    \label{fig:LemmaProof}
\end{figure}
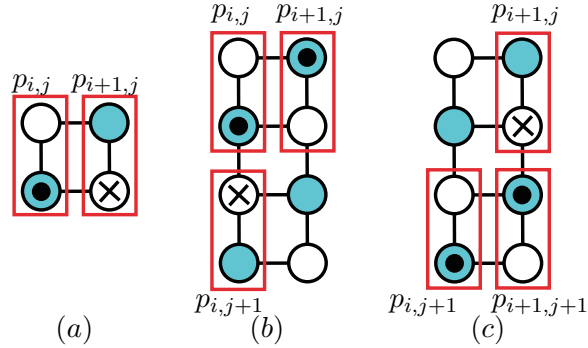

\end{proof}

%\begin{proof}See \ref{sec:proofs}.\end{proof}
\begin{proof}[Proof of \ref{l:urbm-lower}]
We show that there are $\Omega(m)$ cleared vertices when $\lfloor n/3 \rfloor$ goal vertices 
are filled. Suppose there are $q$ columns in $\mathcal{D}$ with $1\leq f_i\leq m-1$. According 
to the definition of $f_i$, for each of these $q$ columns, there is at least one goal vertex 
that is filled and at least one goal vertex that is not.

If $q< \dfrac{\lfloor n/3 \rfloor}{3(m-1)}$, then there are two columns $i$ and $j$, such that $f_i=m$ and $f_j=0$. That is 
because $\sum_{1\leq i \leq m} f_i= \lfloor n/3 \rfloor$ and $0\leq f_i \leq m$ for all $1\leq i \leq m$. 
Therefore, for the vertex pairs in each row $j$, at least one goal vertex is filled 
but at least one is not. And thus, for each $j$, there are two adjacent columns $i, i+1, 1\leq 
i < m$, such that in vertex pairs $p_{i,j}$ and $p_{i+1,j}$, one goal vertex is filled while the
start vertex in the other vertex pair is cleared
(\ref{fig:LemmaProof}(a)). 
Therefore, there are at least $m$ cleared vertices in this case.

If $q\geq \dfrac{\lfloor n/3 \rfloor}{3(m-1)}$, then we partition all the columns in $\mathcal{D}$ into $\lfloor m/2\rfloor$ disjoint pairs: 
(1,2), (3,4), ... The $q$ columns belong to at least $\lfloor q/2\rfloor$ pairs of 
adjacent columns. Therefore, according to Lemma \ref{lemma:adjacent}, we have $\Theta(m)$ 
cleared vertices.

In conclusion, there are $\Omega(m)$ cleared vertices when there are $\lfloor n/3 \rfloor$ 
filled goal. Therefore, the minimum \mrb of this instance is $\Omega(m)$.

\end{proof}

\section*{$O(\sqrt{n})$ Algorithmic Upper Bound for \urbm}
To investigate the running buffer size of the plan computed by \spp,
we construct a binary tree $T$ based on \spp(\ref{fig:separator}(c)).
Each node represents a recursive call consuming a subgraph $\udg(V,E)$ and the left and right children of the node are induced from subgraphs $\udg(A',E(A'))$ and $\udg(B',E(B'))$ of $\udg(V,E)$.
\spp computes the plan by visiting the binary tree in a depth-first manner.
For each node on $T$, given the input dependency graph $\udg(V,E)$, 
denote the sequence before we deal with $\udg(V,E)$ as $\pi_0$. 
Denote the vertices pruned in RemovalTrivialGoals( $\udg(V,E)$) as $P$.
Without loss of generality,
assume that \spp recurses into $A'$ before $B'$.
Let $\pi_{P}$, $\pi_{C'}$, $\pi_{A'}$, and $\pi_{B'}$ be the goal removal sequence after we remove vertices in $P$, $C'$, $A'$ and $B'$ from $\udg(V,E)$ respectively.
We define function $RB(\pi)$ to represent ``generalized'' current running buffer size of a removal sequence $\pi$:
When vertices in $\pi$ are removed from the dependency graph, 
$RB(\pi)$ equals either the number of running buffers; 
or the negation of the number of empty goal poses.
In the depth-first recursion, each node on the binary tree $T$ is visited at most three times:

\begin{enumerate}
    \item Before exploring child nodes. The peak may be reached during the removal of $C'$. Let $\pi_{C^*}$ be the sequence when the running buffer size reaches the peak. 
    $RB(\pi_{C^*}) \leq RB(\pi_0) + 10\sqrt{2} \sqrt{|V|}$.
    \item After we deal with $A'$ and before we deal with $B'$.
    $RB(\pi_{A'}) = RB(\pi_{0})$$-\delta(C')-\delta(A')$. 
    \item After we deal with $B'$.
    $RB(\pi_{B'}) = RB(\pi_0)$$-\delta(V)$.
\end{enumerate}

\begin{lemma}\label{lemma:average}
$RB(\pi_{A'})\leq (RB(\pi_{C^*})+RB(\pi_{B'}))/2$.
\end{lemma}

\begin{proof}
\begin{equation*}
    \begin{split}
        & RB(\pi_{A'}) \\
        & = RB(\pi_{0})-\delta{P}-\delta{C'}-\delta{A'}\\
        %&\ \ \ \ \hcancel[black]{+ (s(P)-g(P)) + (s(C')-g(C')) + (s(A')-g(A'))}\\ 
%     \end{split}
% \end{equation*}
% \begin{equation*}
%     \begin{split}
        &= RB(\pi_0)  -\delta(P)-\delta(C')-\max[\delta(A'),\delta(B')]\\
        %&\ \ \ \ \hcancel[black]{+ (s(P)-g(P))+ (s(C')-g(C'))}\\
        %&\ \ \ \ \hcancel[black]{+ \min[s(A')-g(A'), s(B')-g(B')]}\\
%     \end{split}
% \end{equation*}
% \begin{equation*}
%     \begin{split}
        &\leq RB(\pi_0) -\delta(P)-\delta(C')+\dfrac{1}{2}[-\delta(V)+\delta(P)+\delta(C')]\\
       % &\ \ \ \  \hcancel[black]{+ (s(P)-g(P)) + (s(C')-g(C'))}\\ 
       % &\ \ \ \  \hcancel[black]{+\dfrac{1}{2}[(s(V)-g(V)) - (s(P)-g(P))}\\
        %&\ \ \ \ \hcancel[black]{- (s(C')-g(C'))]}\\
%     \end{split}
% \end{equation*}
% \begin{equation*}
%     \begin{split}
        &= \dfrac{1}{2} \{[RB(\pi_0)-\delta(V)]+[RB(\pi_0)-\delta(P)-\delta(C')]\}\\ 
        %&\ \ \ \ \hcancel[black]{+ s(V)-g(V)]}\\
        %&\ \ \ \ \hcancel[black]{+[RB(\pi_0) + (s(P)-g(P)) + (s(C')-g(C'))]\}}\\
%     \end{split}
% \end{equation*}
% \begin{equation*}
%     \begin{split}
        &= \dfrac{1}{2} [RB(\pi_{B'})+RB(\pi_{C'})]\\
        &\leq \dfrac{1}{2} [RB(\pi_{B'})+RB(\pi_{C^*})]
    \end{split}
\end{equation*}

\end{proof}

\ref{lemma:average} establishes the relationship among $\pi_{C^*}$, $\pi_{A'}$, and $\pi_{B'}$ of a node on $T$.
With this lemma, we obtain an upper bound for each node:

% \begin{proposition}
% A positive \mrb of the plan must be $RB(\pi_{C^*})$ of a node on $T$. 
% \end{proposition}
%\jy{This proposition statement is confusing to me.}

\begin{lemma}\label{lemma:induction}
Given a node $N$ with depth $d$ in the binary tree $T$, let $\pi_{C^*}(N)$, $\pi_{A'}(N)$, and $\pi_{B'}(N)$ be $\pi_{C^*}$, $\pi_{A'}$, and $\pi_{B'}$ of $N$ respectively. 
$RB(\pi_{C^*}(N))$, $RB(\pi_{A'}(N))$, and $RB(\pi_{B'}(N))$ are all upper bounded by 
\begin{equation}
    \dfrac{[1-(\sqrt{\dfrac{2}{3}})^{d+1}]}{1-\sqrt{\dfrac{2}{3}}}20\sqrt{n}
\end{equation}
where $n$ is the number of objects in the instance. 
\end{lemma}

\begin{proof}
The conclusion can be proven by induction.\\

When $d=0$, the dependency graph at the root node $r$ has $n$ start vertices and $n$ goal vertices. 
We have $RB(\pi_{C^*})\leq 20\sqrt{n}, RB(\pi_{B'})=0$. 
According to Lemma \ref{lemma:average}, 
$RB(\pi_{A'})\leq 10 \sqrt{n}$.
The conclusion holds.

Assume that the conclusion holds for all the nodes with depth less than or equal to $k$.
Given an arbitrary node $N$ in the depth $k$, 
let the left and right children of $N$ be $L$ and $R$, which are the nodes with depth $k+1$. 
The corresponding dependency graphs have at most $(2/3)^{k+1}\cdot 2n$ vertices respectively. 
$$RB(\pi_{C^*}(L))\leq RB(\pi_{C^*}(N)) + \sqrt{(2/3)^{k+1}\cdot 2n}\cdot 10\sqrt{2}$$
$$RB(\pi_{C^*}(R))\leq RB(\pi_{A'}(N)) + \sqrt{(2/3)^{k+1}\cdot 2n}\cdot 10\sqrt{2}$$
Since $\pi_{B'}(L)=\pi_{A'}(N)$ and $\pi_{B'}(R)=\pi_{B'}(N)$, upper bound for nodes with depth $k$ holds for $\pi_{B'}(L)$ and $\pi_{B'}(R)$. Therefore, the running buffer size for the depth $k+1$ nodes has an upper bound
\begin{equation}
    \begin{split}
        & \dfrac{[1-(\sqrt{2/3})^{k+1}]}{1-\sqrt{2/3}}20\sqrt{n} + \sqrt{(2/3)^{k+1}\cdot 2n}\cdot 10\sqrt{2}\\ 
        = &\dfrac{[1-(\sqrt{2/3})^{k+2}]}{1-\sqrt{2/3}}20\sqrt{n}        
    \end{split}
\end{equation}

Therefore, the upper bound holds for all the nodes of depth $k+1$. 
With induction, the lemma holds.

\end{proof}

With \ref{lemma:induction}, it is straightforward to establish that \mrb is bounded by $\dfrac{20}{1-\sqrt{2/3}}\sqrt{n}$, yielding \ref{t:urbm-upper}.

    \chapter{Interval State Space Heuristic Search for CDR in FC Makespan (FCHS)}\label{sec:CDR-KC}
    FC metric estimates the total execution time of rearrangement.
%Both $t_g$ and $t_r$ starts from the moment when the arm starts moving downward for the manipulation and ends at the moment when the arm stops moving upward after the manipulation.
%$t_h$ starts at the moment when both arms are ready for the handoff and ends at the moment when the handoff is finished and both arms are leaving for the next manipulations.
For each pick-n-place action, the execution time is counted from the moment when the robot leaves for the object to the moment when the robot finishes the placement of the object.
Therefore, the cost of a pick-n-place can be expressed by $t_{pp}(p_c, p_g, p_r)$, 
where $p_c$ is the current end-effector pose. 
%projection of the current end-effector pose.
%to $\mathcal W$, 
$p_g$, $p_r$ are the pick and place poses of the object.
Similarly, for each handoff operation, the execution time of the delivery (resp. receiving) arm is counted from the moment when the robot leaves for the pick pose (resp. handoff pose), 
to the moment when the handoff (resp. the placement) is finished.
With the same notations of locations, the cost of the delivery arm and the receiving arm in a handoff operation can be expressed as $t_{hd}(p_c, p_g)$ and $t_{hr}(p_c, p_r)$ respectively.
% the cost contains four parts:(1) transit cost, (2) grasping cost as in pick-n-places, (3) transfer cost moving the object to the handoff pose with speed $s$, (4) handoff cost $t_h$.
% Similar to the other actions, the cost of the receiving arm in a handoff operation contains four parts:(1) transit cost, (2)handoff cost $t_h$, (3)transfer cost, and (4) release cost.
Besides object manipulations, robot arms move back to the rest pose when the rearrangement task is complete.

When the FC metric is considered, objects at intermediate states cannot be guaranteed to be at stable poses inside the workspace.
Based on the observation, we introduce a state space based on time intervals.
Each state can be expressed by two mappings $\mathcal L: \mathcal O \rightarrow \{S,G,\mathcal B(r_1), \mathcal B(r_2),T\}$ and $\mathcal T: \mathcal R\rightarrow \mathcal (O\times\{G,\mathcal B(r_1), \mathcal B(r_2)\}) \bigcup \{E\}$.
$\mathcal L$ indicates the current object status.
$S,G,\mathcal B(r_1), \mathcal B(r_2),T$ are five kinds of possible status of objects in the state: staying at start pose, goal pose, buffer in $\mathcal S(r_1)$, buffer in $\mathcal S(r_2)$, or being involved in the current task of a robot arm.
Specifically, $\mathcal L(o_i)=T$ when either $o_i$ is being manipulated by a robot, or a robot is on the way to pick it.
$\mathcal T$ indicates the current executing task of robot arms.
When $\mathcal T(r_i)=(o_j, G)$ (resp. $(o_j, \mathcal B(r_i))$), $r_i$ is executing a task moving $o_j$ from the current pose to its goal pose (resp. buffers).
When $\mathcal T(r_i)=E$, $r_i$ idles or just finishes an execution.
% In each interval state, at most one object is being transferred while the other $n-1$ objects are located inside the workspace. 
Each state except the start/goal state mimics a moment when one arm just placed an object while the other arm is executing a manipulation task or idling.

A rearrangement plan of the instance in \ref{fig:CDR-Exp}[Left], as a path in the interval state space, is presented in \ref{fig:CDR-KC-Plan}.
The plan starts with the interval state $s_1$ representing the start arrangement.
In $s_1$, $r_1$ and $r_2$ choose to pick $o_2$ and $o_1$ respectively.
The second state in the path is the moment when $r_1$ completes the placement of $o_2$ while $r_2$ is still executing the manipulation.
When $r_2$ places $o_1$ at its goal pose, $r_1$ is attempting to move $o_3$ to the handoff pose, which yields to state $s_3$ in the path.
After the handoff, $r_2$ places $o_3$ at its goal pose and the plan ends when both arms return to their rest poses.

\begin{figure}[ht]
    \centering
    \includegraphics[width=0.96\textwidth]{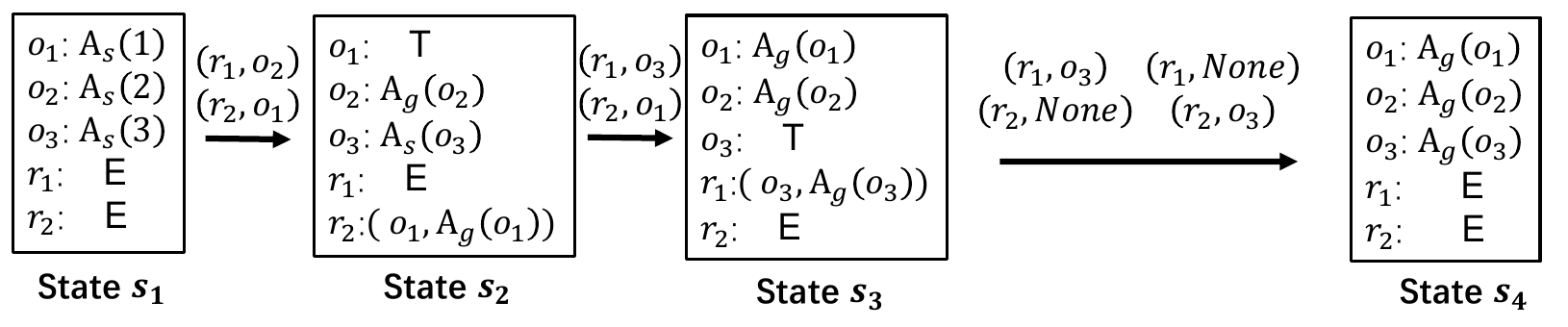}
    \includegraphics[width=0.4\textwidth]{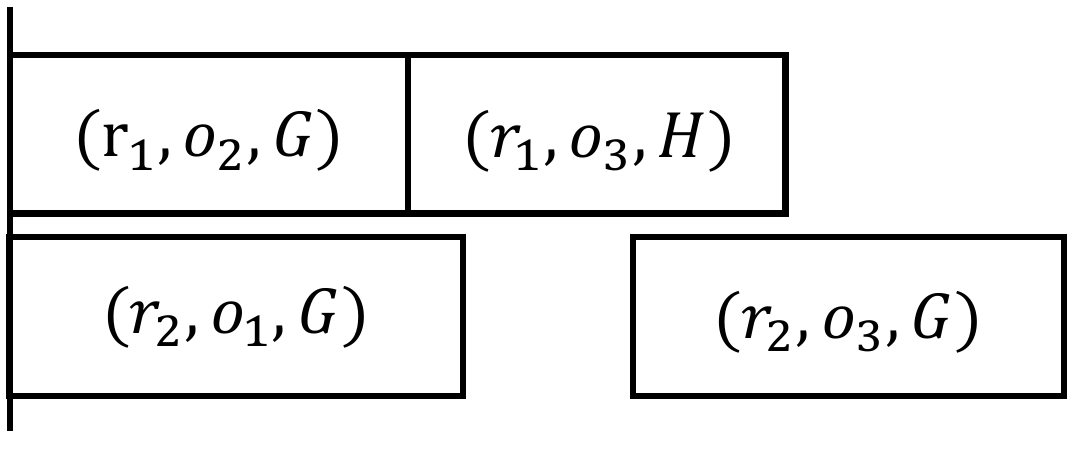}
    \caption{[Top] A path on the interval state space for the example instance in \ref{fig:CDR-Exp}[Left]. [Bottom] The corresponding task schedule for the instance.}
    \label{fig:CDR-KC-Plan}
\end{figure}

In either single arm rearrangement or dual-arm rearrangement with MC metric, object poses in arrangement states are determined, and the transition cost between states is the cost of specific manipulations. 
However, an interval state is more flexible and the transition cost between interval states is non-deterministic either. For example, in $s_3$ in \ref{fig:CDR-KC-Plan}, without mentioning the path from the start state, $r_1$ can be on the way to $\mathcal A_s(o_3)$ or waiting at the handoff pose (since $\mathcal A_g(o_3)$ is out of reach). 
Moreover, there is no non-trivial lower bound on the transition cost between interval states. For example, from state $s_2$ to state $s_3$, the transition cost may be close to 0 if $r_2$ finishes the placement right after $r_1$ does. It may also be close to the cost of the pick-n-place if $r_2$ just starts to leave for $A_s(o_1)$ when $r_1$ finishes the placement.

% The transition cost between two states on a path is defined as the difference of the makespan from the start state to the compared states. 
Given a path on the search tree and an interval state $s$, the current cost $g(s)$ on the path is defined as the current makespan.
In the path presented in \ref{fig:CDR-KC-Plan}[Top], 
$g(s_2)=t_{pp}(\gamma_1, \mathcal A_s(o_2), \mathcal A_g(o_2))$,
where $\gamma_1$ is the projection of the rest pose of $r_1$'s end-effector to the 2d space embedding $\mathcal W$. It is the pick-n-place cost of $r_1$ to manipulate $o_2$. 
At this moment, $r_1$ completes the placement of $o_2$.
Similarly, $g(s_3)=t_{pp}(\gamma_2, \mathcal A_s(o_1), \mathcal A_g(o_1))$. 
At this moment, $r_2$ completes the placement of $o_1$.
Note that the buffer locations are treated as variables in this task scheduling phase. 
When a buffer is involved in a primitive action, e.g. moving an object from/to buffer, 
the travel distance of the arm is set as the diagonal distance of its reachable region.
This setting further penalizes temporary object displacements and ensures the computed schedule is executable no matter where the buffers are allocated in the workspace.
Based on the path from state $s_1$ to $s_4$, we obtain a corresponding task schedule.
% following three rules: (1) For the same arm, a primitive action cannot start earlier than the ending time of the previous action, (2) An release operation will not execute until the 

Similar to the arrangement space heuristic search, for each interval state $s$, the heuristic function $h(s)$ is defined based on the cost of picking, transferring, placing, and handing off that are needed to rearrange objects which are still away from the goal poses. 
While the general idea is the same as manipulation cost based heuristic (\ref{alg:MC-heuristic}), we make some modifications for interval states. 
The details are presented in Algo.\ref{alg:KC-heuristic}.
In an interval state, the ongoing task can be arbitrarily close to completion.
Therefore, besides the objects at the goal poses, we also ignore objects moving to the goal poses (Line 3),
When the object needs a handoff to the goal pose (Line 8), the delivering arm bears the delivery cost (Line 9-10), and the receiving arm bears the receiving cost (Line 11-12).

\begin{algorithm}
\begin{small}
    \SetKwInOut{Input}{Input}
    \SetKwInOut{Output}{Output}
    \SetKwComment{Comment}{\% }{}
    \caption{Full Cost Search Heuristic}
		\label{alg:KC-heuristic}
    \SetAlgoLined
		\vspace{0.5mm}
    \Input{$s$: current interval state}
    \Output{h: heuristic value of $s$}
		\vspace{0.5mm}
		cost[$r_1$], cost[$r_2$], sharedCost $\leftarrow 0,0,0$\\
		\For{$o_i\in \mathcal O$}{
		\lIf{$o_i$ is at goal or moving to goal}{continue}
		armSet $\leftarrow$ ReachableArms($s$, $o_i$, $\{\mathcal L(o_i), G\}$)\\
		\lIf{armSet is $\{r_1, r_2\}$}{sharedCost $\leftarrow$ sharedCost + transfer($o_i, \mathcal L(o_i), G$)+$t_g$+$t_r$}
		\lElseIf{armSet is $\{r_1\}$}{cost[$r_1$]$\leftarrow$ cost[$r_1$] + transfer($o_i, \mathcal L(o_i), G$)+$t_g$+$t_r$}
		\lElseIf{armSet is $\{r_2\}$}{cost[$r_2$]$\leftarrow$ cost[$r_2$] + transfer($o_i, \mathcal L(o_i), G$)+$t_g$+$t_r$}
		\Else{
		deliverArm $\leftarrow$ ReachableArms($s$, $o_i$, $\{\mathcal L(o_i)\}$)\\
		cost[deliverArm] $\leftarrow$ cost[deliverArm] + transfer($o_i, \mathcal L(o_i), H$)+$t_g$+$t_h$\\
		receiveArm $\leftarrow$ ReachableArms($s$, $o_i$, $\{\mathcal G\}$)\\
		cost[receiveArm] $\leftarrow$ cost[receiveArm] + transfer($o_i, H, G$)+$t_h$+$t_r$
		}
		}
		\vspace{1mm}
		\lIf{$\|$cost[$r_1$]-cost[$r_2$] $\|\leq$ sharedCost}{
		\\ \quad \Return (cost[$r_1$]+cost[$r_2$]+sharedCost)/2
		}
		\lElse{
		\Return max(cost[$r_1$], cost[$r_2$])
		}
\end{small}
\end{algorithm}

% \begin{figure}[ht]
%     \centering
%     \includegraphics[width=0.48\textwidth]{figures/interval_plan_makespan.png}
%     \caption{Costs on the interval state space for the example instance in \ref{fig:CDR-Exp}[Left]}
%     \label{fig:CDR-KC-Plan-makespan}
% \end{figure}

\end{theappendices}

%% file: publications.tex
\begin{publications}
\thispagestyle{myheadings}

%% Please use \labelcref instead of \ref when refer to your publications!

\begin{enumerate}[leftmargin=*,start=1,label=\textbf{P\arabic*}, ref={P\arabic*}]
    \item \label{pubs:P0} {\bf Kai Gao*}, Zihe Ye*, Duo Zhang*, Baichuan Huang, and Jingjin Yu. "Toward Holistic Planning and Control Optimization for Dual-Arm Rearrangement." arXiv preprint arXiv:2404.06758 (2024).
    \item \label{pubs:P00} {\bf Kai Gao*}, Zhaxizhuoma*, Yan Ding, Shiqi Zhang, and Jingjin Yu. "ORLA*: Mobile Manipulator-Based Object Rearrangement with Lazy A*." arXiv preprint arXiv:2309.13707 (2023).
    % \item \label{pubs:P1} Vieiral, Ewerton R., {\bf Kai Gao}, Daniel Nakhimovich, and Kostas E. Bekris. "Effective and Robust Non-prehensile Manipulation via Persistent Homology." In Experimental Robotics: The 18th International Symposium, vol. 30, p. 192. Springer Nature, 2024.
    % \item \label{pubs:P2} Xu, Andy*, {\bf Kai Gao}*, Si Wei Feng*, and Jingjin Yu. "Optimal and Stable Multi-Layer Object Rearrangement on a Tabletop." In 2023 IEEE/RSJ International Conference on Intelligent Robots and Systems (IROS), pp. 2078-2085. IEEE, 2023.
    \item \label{pubs:P3} {\bf Kai Gao}, Justin Yu, Tanay Sandeep Punjabi, and Jingjin Yu. "Effectively rearranging heterogeneous objects on cluttered tabletops." In 2023 IEEE/RSJ International Conference on Intelligent Robots and Systems (IROS), pp. 2057-2064. IEEE, 2023.
    % \item \label{pubs:P4} Chang, Haonan, {\bf Kai Gao}, Kowndinya Boyalakuntla, Alex Lee, Baichuan Huang, Harish Udhaya Kumar, Jinjin Yu, and Abdeslam Boularias. "Lgmcts: Language-guided monte-carlo tree search for executable semantic object rearrangement." arXiv preprint arXiv:2309.15821 (2023).
    \item \label{pubs:P5} {\bf Kai Gao}, Si Wei Feng, Baichuan Huang, and Jingjin Yu. "Minimizing running buffers for tabletop object rearrangement: Complexity, fast algorithms, and applications." The International Journal of Robotics Research 42, no. 10 (2023): 755-776.
    % \item \label{pubs:P6} {\bf Kai Gao}, and Jingjin Yu. "On the utility of buffers in pick-n-swap based lattice rearrangement." In 2023 IEEE International Conference on Robotics and Automation (ICRA), pp. 5786-5792. IEEE, 2023.
    \item \label{pubs:P7} {\bf Kai Gao}, and Jingjin Yu. "Toward efficient task planning for dual-arm tabletop object rearrangement." In 2022 IEEE/RSJ International Conference on Intelligent Robots and Systems (IROS), pp. 10425-10431. IEEE, 2022.
    % \item \label{pubs:P8} Wang, Rui, {\bf Kai Gao}, Jingjin Yu, and Kostas Bekris. "Lazy rearrangement planning in confined spaces." In Proceedings of the International Conference on Automated Planning and Scheduling, vol. 32, pp. 385-393. 2022.
    % \item \label{pubs:P9} Vieira, Ewerton R., Daniel Nakhimovich, {\bf Kai Gao}, Rui Wang, Jingjin Yu, and Kostas E. Bekris. "Persistent homology for effective non-prehensile manipulation." In 2022 International Conference on Robotics and Automation (ICRA), pp. 1918-1924. IEEE, 2022.
    \item \label{pubs:P10} {\bf Kai Gao}, Darren Lau, Baichuan Huang, Kostas E. Bekris, and Jingjin Yu. "Fast high-quality tabletop rearrangement in bounded workspace." In 2022 International Conference on Robotics and Automation (ICRA), pp. 1961-1967. IEEE, 2022.
    % \item \label{pubs:P11} {\bf Kai Gao}, and Jingjin Yu. "Capacitated vehicle routing with target geometric constraints." In 2021 IEEE/RSJ International Conference on Intelligent Robots and Systems (IROS), pp. 7925-7930. IEEE, 2021.
    \item \label{pubs:P12} {\bf Kai Gao}, Si Wei Feng, and Jingjin Yu. "On running buffer minimization for tabletop rearrangement." In 17th Robotics: Science and Systems, RSS 2021. 
    % \item \label{pubs:P13} Feng, Si Wei, {\bf Kai Gao}, Jie Gong, and Jingjin Yu. "Sensor placement for globally optimal coverage of 3d-embedded surfaces." In 2021 IEEE International Conference on Robotics and Automation (ICRA), pp. 8600-8606. IEEE, 2021.
    % \item \label{pubs:P14} Wang, Rui, {\bf Kai Gao}, Daniel Nakhimovich, Jingjin Yu, and Kostas E. Bekris. "Uniform object rearrangement: From complete monotone primitives to efficient non-monotone informed search." In 2021 IEEE International Conference on Robotics and Automation (ICRA), pp. 6621-6627. IEEE, 2021.
    % \item \label{pubs:P15} Feng, Si Wei, Shuai D. Han, {\bf Kai Gao}, and Jingjin Yu. "Efficient Algorithms for Optimal Perimeter Guarding." In Robotics: Science and Systems 2020 (RSS 2020).
\end{enumerate}

\end{publications}